%% file: arxiv.tex
\documentclass[11pt]{article}

\usepackage{fullpage} 
\usepackage[utf8]{inputenc} 
\usepackage{amsmath,amssymb,amsthm,bbm,color}
\usepackage[colorlinks=true,citecolor=blue,urlcolor=blue,linkcolor=blue]{hyperref}
\usepackage[round]{natbib} 
\usepackage[showonlyrefs=true]{mathtools}
\usepackage[noabbrev]{cleveref}
\usepackage{xurl}
\usepackage{authblk}

\input{AOS/_preamble}

\renewcommand{\change}[1]{#1}
\newenvironment{longlist}{
  \begin{itemize}
}{
  \end{itemize}
}

\theoremstyle{plain}
\newtheorem{theorem}{Theorem}[section]
\newtheorem{lemma}[theorem]{Lemma}

\newtheorem{proposition}[theorem]{Proposition}

\newtheorem{remark}[theorem]{Remark}

\title{On the Computational Efficiency of Bayesian Additive Regression Trees: An Asymptotic Analysis}

\newcommand*\samethanks[1][\value{footnote}]{\footnotemark[#1]}

\author[1]{Yan Shuo Tan}
\author[2]{Omer Ronen\thanks{Equal contribution, alphabetical ordering}}
\author[2]{Theo Saarinen\samethanks}
\author[2,3,4]{Bin Yu}

\affil[1]{Department of Statistics and Data Science, National University of Singapore}
\affil[2]{Department of Statistics, UC Berkeley}
\affil[3]{Department of Electrical Engineering and Computer Sciences, UC Berkeley}
\affil[4]{Center for Computational Biology, UC Berkeley}

\date{\today}

\begin{document}

\maketitle

\begin{abstract}

\input{AOS/0_abstract}
\end{abstract}



\input{AOS/01_intro}

\input{AOS/02_prelims}
\input{AOS/03_main_results}
\input{AOS/04_proof_sketch}
\change{
    \input{AOS/04_upper_bounds}

\input{AOS/06_discussion}

\input{AOS/05_sims}
}

\section*{Acknowledgments}
We thank Antonio Linero, Matthew Pratola, Jungguem Kim, Sameer Desphande, and Veronika Rockova for insightful discussions.
Special thanks go to Ruizhe Deng, Tianyi Zhuang, Yankai Pei, and Hengrui Luo for their contributions to the \texttt{bart-playground} package.
We also acknowledge the computing support from the Statistical Computing Facility (SCF) at UC Berkeley's Department of Statistics, and we are particularly grateful to Jacob Steinhardt for providing access to his group's large-memory nodes on the SCF Linux cluster.

\paragraph{Funding:}
We gratefully acknowledge partial support from NSF TRIPODS Grant 1740855, DMS-1613002, 1953191, 2015341, 2209975, 20241842, IIS 1741340, ONR grant N00014-17-1-2176, NSF grant 2023505 on Collaborative Research: Foundations of Data Science Institute (FODSI), NSF grant MC2378 to the Institute for Artificial CyberThreat Intelligence and OperatioN (ACTION), the NSF and the Simons Foundation for the Collaboration on the Theoretical Foundations of Deep Learning through awards DMS-2031883 and 814639, and a Weill Neurohub grant.
YT gratefully acknowledges support from NUS Start-up Grant A-8000448-00-00 and MOE AcRF Tier 1 Grant A-8002498-00-00.

\bibliographystyle{plainnat} 
\bibliography{arxiv}

\newpage
\appendix

\input{AOS/a0_posterior_conc_literature}
\input{AOS/a1_bic_conc}

\input{AOS/a2_mc}
\input{AOS/a3_hitting_times}
\input{AOS/a4_upper_bounds}

\input{AOS/a4_pem_dim}
\input{AOS/a6_extra_sims}

\end{document}

%% file: AOS/_preamble.tex
\usepackage{hyperref}       
\usepackage{url}            
\usepackage{xcolor}         

\usepackage{booktabs}

\usepackage{algorithm}  
\usepackage{algorithmicx}
\usepackage[noend]{algpseudocode}
\definecolor{mygray}{gray}{0.5}
\definecolor{cblue}{RGB}{8, 85, 153}
\definecolor{darkblue}{RGB}{31, 64, 96}

\definecolor{cgreen}{RGB}{8, 153, 83}
\definecolor{green}{RGB}{8, 200, 50}
\definecolor{cmaroon}{RGB}{128, 0, 0}
\algdef{SE}[DOWHILE]{Do}{doWhile}{\algorithmicdo}[1]{\algorithmicwhile\ #1}%

\renewcommand\algorithmicdo{}

\newcommand{\be}{\mathbf{e}}
\renewcommand{\bf}{\mathbf{f}}

\newcommand{\bl}{\mathbf{l}}
\newcommand{\bv}{\mathbf{v}}

\newcommand{\bx}{\mathbf{x}}
\newcommand{\by}{\mathbf{y}}
\newcommand{\bz}{\mathbf{z}}
\newcommand{\br}{\mathbf{r}}

\newcommand{\bG}{\mathbf{G}}

\newcommand{\bI}{\mathbf{I}}

\newcommand{\bM}{\mathbf{M}}
\newcommand{\bP}{\mathbf{P}}

\newcommand{\bX}{\mathbf{X}}

\newcommand{\balpha}{\boldsymbol{\alpha}}
\newcommand{\bbeta}{\boldsymbol{\beta}}

\newcommand{\beps}{\boldsymbol{\epsilon}}

\newcommand{\bmu}{\boldsymbol{\mu}}

\newcommand{\bpi}{\boldsymbol{\pi}}

\newcommand{\bPhi}{\boldsymbol{\Phi}}

\newcommand{\bPsi}{\boldsymbol{\Psi}}

\newcommand{\bSigma}{\boldsymbol{\Sigma}}

\renewcommand{\cal}[1]{\mathcal{#1}}

\DeclarePairedDelimiter{\braces}{\lbrace}{\rbrace}
\DeclarePairedDelimiter{\paren}{(}{)}

\DeclarePairedDelimiter{\abs}{|}{|}

\DeclarePairedDelimiter{\norm}{\|}{\|}
\DeclarePairedDelimiter{\inprod}{\langle}{\rangle}

\newcommand{\R}{\mathbb{R}}
\renewcommand{\P}{\mathbb{P}}
\newcommand{\E}{\mathbb{E}}
\newcommand{\indicator}{\mathbf{1}}
\newcommand{\Var}{\operatorname{Var}}
\newcommand{\Cov}{\operatorname{Cov}}

\newcommand{\trace}{\operatorname{tr}}

\newcommand{\rank}{\operatorname{rank}}
\DeclareMathOperator*{\argmin}{\arg\!\min}
\newcommand{\nn}{\nonumber \\}

\newcommand{\tree}{\mathfrak{T}}

\newcommand{\tse}{\mathfrak{E}}
\newcommand{\cell}{\mathcal{C}}
\newcommand{\node}{\mathfrak{t}}
\newcommand{\leaf}{\mathcal{L}}
\newcommand{\nleaves}{b}
\newcommand{\ntrees}{m}
\newcommand{\tsespace}[1][m]{\Omega_{\text{TSE},#1}}

\newcommand{\leafparam}{\mu}
\newcommand{\leafparamvec}{\bmu}

\newcommand{\bic}{\operatorname{BIC}}
\newcommand{\df}{\operatorname{df}}

\newcommand{\partition}{\mathcal{P}}
\newcommand{\pem}{\mathbb{V}}

\newcommand{\mapone}{\mathcal{F}}

\newcommand{\bias}{\operatorname{Bias}^2}
\newcommand{\opt}[1][m]{\operatorname{OPT}_{#1}}
\newcommand{\setone}{\mathcal{A}}
\newcommand{\settwo}{\mathcal{B}}
\newcommand{\setthree}{\mathcal{C}}
\newcommand{\bart}{\textbf{BART}}
\newcommand{\indep}{\perp\!\!\!\perp}
\newcommand{\tsebad}{\tse_{\text{bad}}}
\newcommand{\tsegood}{\tse_{\text{good}}}
\newcommand{\treebad}{\tree_{\text{bad}}}
\newcommand{\statespace}{\Omega}
\newcommand{\hpp}{\operatorname{HPP}}
\newcommand{\emptytree}{\tree_\emptyset}
\newcommand{\emptytse}{\tse_\emptyset}
\newcommand{\linspan}{\operatorname{span}}
\newcommand{\coverage}{\operatorname{coverage}}

\newcommand{\xspace}{\mathcal{X}}
\newcommand{\xmeasure}{\nu}
\newcommand{\nfeats}{d}
\newcommand{\nsamples}{n}
\newcommand{\ncats}{B}

\newcommand{\data}[1][n]{\mathcal{D}_{#1}}

\newcommand{\maxf}{D_f}
\newcommand{\noisesg}{D_\epsilon}

\newcommand{\var}[1]{\textit{#1}}

\newcommand{\change}[1]{{\color{blue}{#1}}}

%% file: AOS/0_abstract.tex
Bayesian Additive Regression Trees (BART) is a popular Bayesian non-parametric regression model that is commonly used in causal inference and beyond. 
Its strong predictive performance is supported by \change{well-developed estimation theory}, comprising guarantees that its posterior distribution concentrates around the true regression function at optimal rates under various data generative settings and for appropriate prior choices. 
\change{However, the computational properties of the widely-used BART sampler proposed by \citet{chipman2010bart} are yet to be well-understood.
In this paper, we perform an asymptotic analysis of a slightly modified version of the default BART sampler when fitted to data-generating processes with discrete covariates.
We show that the sampler's time to convergence, evaluated in terms of the hitting time of a high posterior density set, increases with the number of training samples, due to the multi-modal nature of the target posterior. 
On the other hand, we show that this trend can be dampened by simple changes, such as increasing the number of trees in the ensemble or raising the temperature of the sampler.
These results provide a nuanced picture on the computational efficiency of the BART sampler in the presence of large amounts of training data while suggesting strategies to improve the sampler.
We complement our theoretical analysis with a simulation study focusing on the default BART sampler. 
We observe that the increasing trend of convergence time against number training samples holds for the default BART sampler and is robust to changes in sampler initialization, number of burn-in iterations, feature selection prior, and discretization strategy.
On the other hand, increasing the number of trees or raising the temperature sharply dampens this trend, as indicated by our theory.}

%% file: AOS/01_intro.tex
\section{Introduction}
\label{sec:intro}

\subsection{The rise of BART}

Decision tree models such as CART \citep{breiman1984classification} are piecewise constant regression models obtained by recursively partitioning the covariate space along coordinate axes.
They and their ensembles such as Random Forests (RFs) \citep{breiman2001random} and Gradient Boosted Trees (GBTs) \citep{friedman2001greedy,chen2016xgboost} have proved to be enormously successful because of their strong predictive performance \citep{caruana2006empirical,caruana2008empirical, fernandez2014we}.
Indeed, RFs and GBTs regularly outperform even deep learning on medium-sized tabular datasets \citep{grinsztajn2022tree}.
Nonetheless, these tree-based methods still suffer from several notable problems:
They are defined via algorithms rather than via statistical models, so it is often difficult to quantify the uncertainty of their predictions; they use greedy splitting criteria, so there is no guarantee for the optimality of the fitted model; RFs in particular grow their trees independently of each other, therefore making them statistically inefficient
when fitted to data with additive structure \citep{tan2021cautionary}.

To address these issues, \citet{chipman1998bayesian} proposed a Bayesian adaptation of CART (BCART) and later an ensemble of Bayesian CART trees, which they called Bayesian Additive Regression Trees (BART).
These are Bayesian non-parametric regression models, which put a prior on the space of regression functions, assume a likelihood for the observed data, and combine these to obtain a posterior.
In the case of Bayesian CART and BART, priors and posteriors are supported on the subspace of functions that can be realized by decision trees (or their ensembles).
Similar to Gaussian process (GP) regression, the posterior distribution can be used to provide posterior predictive credible intervals.
On the other hand, unlike GP regression, there is no closed form formula for the BART posterior and one has to sample from it approximately via a Markov chain Monte Carlo (MCMC) algorithm.

BART has been shown empirically to enjoy strong predictive performance that is sometimes even superior to that of RFs and GBTs, especially after hyperparameter optimization \citep{hill2020bayesian}.
Naturally, it has become increasingly popular in diverse fields ranging from the social sciences \citep{green2010modeling,yeager2019national} to biostatistics \citep{wendling2018comparing,starling2020bart} and has been particularly enthused causal inference researchers \citep{hill2011bayesian,green2012modeling,kern2016assessing,dorie2019automated,hahn2019atlantic}.
Extending and improving BART methodology remains a highly active area of research, with many variants of the algorithm proposed over the last few years (see for instance \cite{linero2018bayesianJASA,pratola2016efficient,pratola2020heteroscedastic,luo2023sharded}, as well as the survey \cite{hill2020bayesian} and the references therein.) 

The strong predictive performance of BART is supported by a burgeoning body of theoretical evidence regarding the BART posterior.
Most significantly, researchers have shown that the BART posterior concentrates around the true regression function used to generate the response data as $n$, the number of training samples, increases, with this concentration occuring at optimal rates under various assumptions on the smoothness and sparsity of the regression function and for appropriate prior choices \citep{rovckova2019theory,rovckova2020posterior,linero2018bayesian,jeong2020art,rockova2021ideal,castillo_uncertainty_2021}.
Achieving these optimal rates does not require any oracle knowledge or hyperparameter tuning---instead BART automatically adapts to the level of smoothness and sparsity, with the former even happening at a local level \citep{rockova2021ideal}.

\subsection{Observed poor mixing of BART and its significance}

While there is evidence that the BART \emph{posterior} enjoys favorable properties under a variety of settings, the fact that we can only sample approximately from the posterior via MCMC means that understanding the performance of the BART \emph{algorithm} also entails understanding its sampler's dynamics.
Specifically, if the sampler chain does not converge efficiently to the posterior distribution, that is, if it does not mix well, the output of BART algorithm may not enjoy the same desirable inferential properties as the BART posterior.
Most popular BART implementations use remarkably similar samplers based on the original design of \cite{chipman2010bart}, which uses a Bayesian backfitting approach to update one tree at a time via proposed local changes to the tree nodes coupled with a Metropolis-Hastings filter.
Unfortunately, as described by \cite{hill2020bayesian}, ``while this algorithm is often effective, it does not always mix well.''
Indeed, poor mixing for this sampler has been empirically documented by multiple sources \citep{chipman1998bayesian,Carnegie2019bartmixing}.

The literature contains various suggestions on how to improve the computational efficiency of the BART sampler.
These include parallelization \citep{pratola_parallel_2014}, modifying the MCMC proposal moves \citep{wu2007bayesian,pratola2016efficient,kim2023mixing}, warm starts from greedily constructed tree ensembles \citep{he2021stochastic}, running multiple chains \citep{Carnegie2019bartmixing}, or using a particle Gibbs variant \citep{lakshminarayanan2015particle}.
Despite this interest, there has been minimal theoretical work done to quantify the mixing time and to understand why and under what settings slow mixing occurs.



\subsection{BART MCMC and prior theoretical work} \label{subsec:prior_work}

A regression tree is parameterized via its \emph{tree structure} $\tree$ (which features are split on and at which thresholds) and its leaf parameters $\bmu$ (the function value on each leaf).
BART combines these parameters over multiple trees to parameterize a tree ensemble (see Section \ref{subsec:bayesian_model}.).
The BART sampler is a Metropolis-within-Gibbs sampler that, at each iteration, cycles through all trees in the ensemble.
For each tree $(\tree,\bmu)$, it conditions on the parameters of all other trees in the ensemble and proposes a local change to $\tree$, whose acceptance is controlled by a Metropolis-Hastings filter.
$\bmu$ can then be sampled in closed form due to conditional conjugacy (see Section \ref{subsec:inpractice_sampler}).

The non-standard state space makes this Markov chain difficult to study using standard techniques in Markov chain theory (see Section \ref{subsec:analyzed_sampler}).
As such, prior theoretical work \citep{kim2023mixing,ronen2022mixing} has only studied the sampler for Bayesian CART, which is the restriction to BART to a single tree, and which case, the sampler reduces to a Markov chain on a discrete space of tree structures.

Both \cite{kim2023mixing} and our prior work \cite{ronen2022mixing} made the surprising discovery that the mixing time of this chain can grow exponentially in the training sample size.\footnote{For this result, both works used a restricted move set which only allows growing new leaves and pruning existing ones.}
\cite{kim2023mixing} studied this in a one-dimensional setting and further showed that, with a more aggressive move set, Bayesian CART constrained to dyadic splits has a mixing time upper bound that is linear in the sample size.
\cite{ronen2022mixing}'s result applied to a multi-dimensional setting.
Their proof strategy was to show that this Markov chain has a bottleneck state---the trivial tree comprising a single node.
This bottleneck arises because when Bayesian CART makes a wrong first split followed by other informative splits, the only way to reverse the wrong first split involves pruning the informative splits, which becomes increasingly difficult as the training sample size increases.

When we submitted our prior work to a peer review, some reviewers rightly pointed out that our mixing time lower bounds may be misleading to practitioners.
This is because different tree structures could realize the same partition of the covariate space and hence implement the same regression function.
In other words, the Bayesian CART model is not identifiable at the level of tree structures, which means that failure to mix in this space of tree structures may be benign \citep{redner1984mixture}.
It does not reflect a failure to mix at the level of regression functions, let alone any degradation in the inferential properties of the algorithm's output. 
\begin{figure}[b]
    \centering
    \includegraphics[scale=0.55]{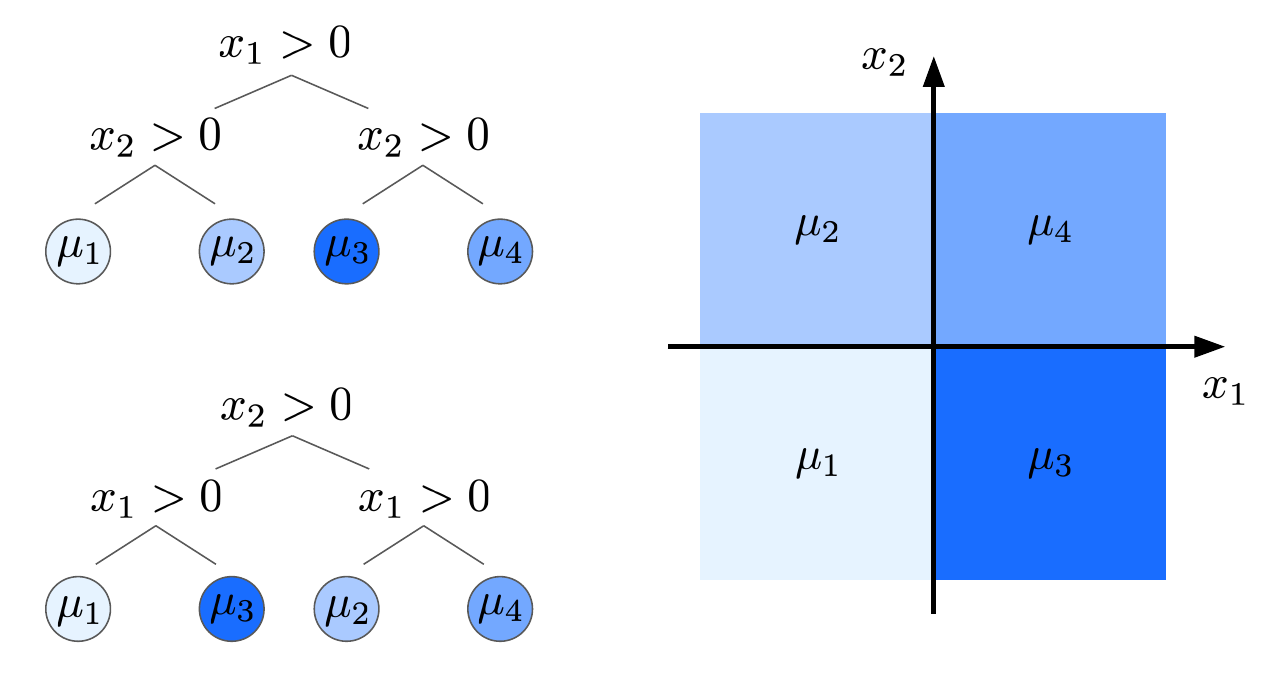}
    \caption{The Bayesian CART model is not identifiable at the level of tree structures. 
    The two tree structures shown on the left both realize the same partition of the covariate space and even the same regression function, which is shown on the right.}
    \label{fig:non_idn}
\end{figure}
A further limitation of our prior analysis is that mixing time is defined as the worst case time to convergence as we vary over all possible initial states.
Since Bayesian CART is always initialized at the trivial tree, a lower bound on this worst case quantity may be unreasonably pessimistic.

\subsection{\change{Main contributions}}

\change{The main contributions of our paper are as follows.
\begin{itemize}
    \item \emph{Developing a theoretical framework to study BART MCMC.} 
    We modify the BART sampler slightly by replacing the conditional (on other trees) likelihood formula in the Metropolis-Hastings filter with a full marginal likelihood formula (see Section \ref{subsec:analyzed_sampler}).
    This marginalizes out the leaf parameters, which allows us to interpret the sampler as a Markov chain on a finite state space comprising tree structure ensembles, thereby enabling tractable analysis.
    We perform an asymptotic analysis of this modified sampler, in the vein of prior work on posterior contraction for BART and mixing rates for Bayesian CART.
    This means that we assume a frequentist generative model\footnote{This model can and will be different from the BART parameterization.} for the training data used for fitting BART and analyze the scaling of the computational performance of BART as $n$, the amount of training data, increases toward infinity.
    To ensure that the state space remains invariant to $n$, we assume discrete features.
    Under this framing, we show that the BART posterior concentrates on a set, $\text{OPT}$, comprising all optimal tree structure ensembles (i.e. those with the smallest bias and complexity) and which also forms a highest posterior density region (HPDR).
    As such, the hitting time of this set from the BART initialization captures the computational efficiency of the sampler in a meaningful sense, avoiding the issues with non-identifiability of tree parameters discussed in Section \ref{subsec:prior_work}.
    \item \emph{Hitting time lower bounds.} When the true generative model is additive and its number of additive components is larger than or equal to the number of trees in the fitted BART model, we prove that the hitting time of $\text{OPT}$ grows at at least the rate $n^{1/2}$.
    Although this lower bound is relatively mild, it is still of practical interest because the current literature does not give guidance on tuning the number of MCMC iterations according to $n$.
    If default values are used for the number of iterations, as is current practice, our result suggests that the posterior approximation of the BART sampler may deviate from the true posterior at large $n$, despite posterior contraction, potentially resulting in undesirable properties such as poor calibration of credible intervals.
    We also prove other supplementary results comprising hitting time lower bounds for BART under other data generative and algorithmic settings, all of which grow with $n$.
    These lower bounds can be explained by the fact that under these scenarios, the posterior distribution on the sampling space is ``multimodal'' and that the difficulty of traversing between modes increases with $n$.
    \item \emph{Mixing time upper bounds.} In the case of an additive generative model, we formulate three possible modifications to the BART model or sampler to dampen its multimodality, thereby addressing the problem we identified earlier.
    These are (i) increasing the number of trees in the ensemble, (ii) making more global changes to the tree structures before applying the Metropolis-Hastings filter, (iii) using a tempered posterior distribution, whereby the marginal likelihood on tree structures is raised to a $1/T$ power.
    We call $T \geq 1$ the temperature of the sampler.
    Under each modification, we prove an upper bound for the mixing time of the resulting Markov chain that is either constant in $n$ (in the case of (i) and (ii)) or grows more slowly than any power of $n$ (in the case of (iii)).
    Meanwhile, these modifications do not affect the HDPR set, which means that their target posterior has essentially the same inferential properties as the original BART model.
\end{itemize}}

\change{In summary, our theoretical results provide a nuanced picture on the computational efficiency of the BART sampler in the presence of large amounts of training data.
On the one hand, the default sampler when used naively can sometimes suffer from poor mixing.
On the other hand, there are simple modifications to the sampler that partially address this issue.
To reiterate, these results focus on a particular aspect of the sampler's computational efficiency---its dependence on $n$---while ignoring dependence on other data-generating and algorithmic parameters.
In addition to the limited scope of our theoretical analysis, its connection to how BART as used in practice is mediated by several factors, including our modifications to the sampler as well as restrictions of the data generative model to the somewhat unrealistic settings of additive regression functions and discrete features.
We will discuss these limitations further in Section \ref{sec:limitations}.
Consequently, our theoretical results should not be seen as a definitive assessment of the sampler's quality, but rather as a starting point for theoretical analysis and a guide for its potential improvement.
Nonetheless, to bridge some of these limitations and thereby strengthen the practical significance of our results, we further contribute:}

\change{\begin{itemize}
    \item \emph{A simulation study using the original BART sampler.}
    Using the Gelman-Rubin statistic $\hat R$,\footnote{Since $\hat R$ is defined on scalar Markov chains, we use root mean squared error (RMSE) on a held-out test set to summarize the BART ensemble at each iteration of the sampler.} we evaluate the mixing performance of the in-practice BART sampler with default hyperparameter choices when applied to six real-world and synthetic datasets, for a fixed number of sampler iterations.
    We find that $\hat R$ has an increasing trend with $n$, indicating that mixing performance degrades as $n$ increases.
    We also find that this trend can be dampened by either increasing the temperature or by increasing the number of trees, with both strategies also often yielding improvements to the coverage of credible intervals and sometimes even root mean squared error (RMSE) on a held-out test set.
    Both these findings support the narrative suggested by our theoretical lower and upper bounds.
    On the other hand, the increasing trend for $\hat R$ seems to be either robust to or even exacerbated by changing the sampler initialization, number of burn-in iterations, feature selection prior, or discretization strategy away from the default. 
    Finally, we find that the modifications we make to the sampler for the sake of theoretical analysis---limiting the proposal moves or using the full marginal posterior formula in the Metropolis-Hastings filter---do not substantially affect either $\hat R$ or the predictive properties of the MCMC approximate posterior.
\end{itemize}
}

%% file: AOS/02_prelims.tex
\section{Data generation models for BART and for frequentist analysis}
\label{sec:methods}

\subsection{Generative model for frequentist analysis}
\label{subsec:generative_model}

We assume a $d$-dimensional discrete covariate space $\xspace = \braces{1,2,\ldots,\ncats}^d$ and a regression function $f^*\colon\xspace \to \R$.
Suppose that we observe a training data set $\data$ comprising $n$ independent and identically distributed tuples $(\bx_i,y_i)$ with $\bx_i \sim \nu$ and $y_i = f^*(\bx_i) + \epsilon_i$ for $i=1,2\ldots,n$.
Here, $\nu$ is a measure on $\xspace$ with full support, while $\epsilon$ is a sub-Gaussian random variable.
For notational convenience, we will use $\bX$ to denote the $n \times d$ matrix formed by stacking the covariate vectors as rows.
We will also use $\by$, $\bf^*$, and $\beps$ to denote the $n$-dimensional vectors formed by stacking the responses $y_i$, the function values $f^*(\bx_i)$ and the noise components $\epsilon_i$ respectively.
Similarly, we will denote all vectors and matrices with boldface notation, with subscripts referencing the index of the vector.
Vector coordinates will be denoted using regular font.
We will denote probabilities and expectations with respect to $\data$ using $\P_n$ and $\E_n$ respectively.

\subsection{Bayesian model specification for BART}
\label{subsec:bayesian_model}

We first describe the version of BART that we analyze in our paper, before discussing its differences with the version described by \cite{chipman2010bart} and which is still most commonly used in practice.

\subsubsection{Regression trees}
A binary axis-aligned regression tree is parameterized by a tuple $(\tree,\leafparamvec)$.
Here, $\tree$ refers to the \emph{tree structure}, which specifies the topology of the tree as a rooted binary tree planar graph and, given an ordering of the graph's vertices (e.g. via breadth-first search), specifies the splitting rule for each internal node $j$.
Note that the splitting rule comprises a feature $v_j$ and a threshold $t_j$.
The leaves $\leaf_1,\leaf_2,\ldots,\leaf_\nleaves$~ of $\tree$ thus correspond to rectangular regions of the covariate space that together form a partition of the space.
We let $\leafparamvec \in \R^\nleaves$ be a vector of \emph{leaf parameters}, one for each leaf of $\tree$.
Together, $(\tree,\leafparamvec)$ specify a piecewise constant function $g$ that outputs
\begin{equation} \nonumber
    g(\bx;\tree,\leafparamvec) = \leafparam_{l(\bx)},
\end{equation}
where $l(\bx)$ is the index of the leaf containing $\bx$.

\subsubsection{Sum-of-trees model}
Given observed data $\data$, the BART model posits $y_i = f(\bx_i) + e_i$ for $i=1,2\ldots,n$, where $e_1,e_2,\ldots,e_n \sim_{i.i.d.} \mathcal{N}(0,\sigma^2)$, and $f$ is a sum of the outputs of $\ntrees$ trees:
\begin{equation} \nonumber
    f(\bx) = g(\bx;\tree_1,\leafparamvec_1) + g(\bx;\tree_2,\leafparamvec_2) + \cdots + g(\bx;\tree_\ntrees,\leafparamvec_\ntrees).
\end{equation}
We denote the ordered tuple $(\tree_1,\tree_2,\ldots,\tree_\ntrees)$ by $\tse$ and call it a \emph{tree structure ensemble} (TSE).
We shall abuse notation and use $\leafparamvec$ to refer to the concatenation of $\leafparamvec_1,\leafparamvec_2,\ldots,\leafparamvec_\ntrees$.
Note that, when conditioned on $\tse$ and $\bX$, this is just a Bayesian linear regression model.
To see this, let $\bPsi$ denote the $n \times b$ matrix whose columns are the indicator vectors over the training set of each leaf in $\tse$.
We then have
\begin{equation} \label{eq:Bayesian_lin_reg}
    \by = \bPsi\leafparamvec + \mathbf{e}.
\end{equation}

\subsubsection{Priors}
We assume a fixed prior distribution $p$ on $\tsespace$, the space of TSEs with $m$ trees.\footnote{Because the emphasis in our theory is on the dependence of hitting times on training sample size in large samples, the specific form of the prior holds no bearing on our results}
Conditioned on a TSE $\tse$, the conditional prior distribution on the leaf parameters is an isotropic Gaussian, i.e. $p(\leafparamvec|\tse) \sim \mathcal{N}(0,(\sigma^2/\lambda) \bI_b)$, where $b$ is the total number of leaves in all trees in $\tse$.
Both $\sigma^2$ and $\lambda$ are assumed to be fixed hyperparameters, with $\sigma^2$ taking the same value as that used in the variance of the additive noise $e_i$, while $\lambda$ is a modulation parameter that should be set to approximately the reciprocal of the signal-to-noise ratio.

\subsubsection{Differences with in-practice BART}
\label{subsec:differences_bart}
\cite{chipman1998bayesian} proposed a prior on tree structures defined in terms of a stochastic process.
Starting from a single root node, the process recursively splits each node at depth $d$ with probability $\alpha(1+d)^{-\beta}$, where $\alpha$ and $\beta$ are hyperparameters with default values $\alpha=0.95$ and $\beta=2$ respectively.
Features and thresholds for splits are selected uniformly at random.
\cite{chipman2010bart} extended this to a prior on TSEs by independence, i.e. $p(\tse) = \prod_{j=1}^m p(\tree_j)$.
After rescaling the response variable to lie between $-0.5$ and $0.5$, the leaf parameter standard deviation is set to be $\sigma_\mu = 0.5/k\sqrt{m}$, where $k$ is a further hyperparameter with default value $k=2$.
Finally, an inverse-$\chi^2$ hyperprior is placed on the noise variance $\sigma^2$ and is calibrated to the observed data.
This last assumption is the only way in which the BART model we study in this paper departs from that in \cite{chipman2010bart}. 
We make this change for analytical tractability and believe it to be minor, since simulations show that the posterior on $\sigma^2$ quickly converges to a fixed value and our theoretical guarantees hold for any fixed choice of $\sigma^2$.

\section{Sampling from BART via MCMC}

\subsection{The in-practice BART sampler}
\label{subsec:inpractice_sampler}

The sampler proposed by \cite{chipman2010bart} can be described as a ``Metropolis-within-Gibbs MCMC sampler'' \citep{hill2020bayesian}.
More precisely, for each outer loop of the algorithm, it iterates over the tree indices $j=1,2,\ldots,m$ and updates the $j$-th pair $(\tree_j,\leafparamvec_j)$ using an approximate draw from the conditional distribution $p(\tree_j,\leafparamvec_j|\tse_{-j},\leafparamvec_{-j},\by,\sigma^2)$, where $\tse_{-j}$ and $\leafparamvec_{-j}$ refer to the concantenation of current tree structures and leaf parameter vectors respectively, each with the $j$-th index omitted.\footnote{Note that the conditional distribution is of course also conditional on the observed covariate data $\bX$.
However, since this is always conditioned upon, we omit it from our notation to avoid clutter}
As a final step in the loop, it updates $\sigma^2$ using a draw from its full conditional distribution $p(\sigma^2|\tse,\leafparamvec,\by)$.

To describe how to sample (approximately) from $p(\tree_j,\leafparamvec_j|\tse_{-j},\leafparamvec_{-j},\by,\sigma^2)$, we first describe \cite{chipman1998bayesian}'s algorithm for Bayesian CART, i.e. when the ensemble comprises a single tree.
In this case, we first factorize the posterior into a conditional posterior on leaf parameters and a marginal posterior on tree structures:\footnote{\cite{chipman1998bayesian}'s formulation of the Bayesian CART sampler marginalizes out $\sigma^2$ instead of conditioning on it. As this is no longer done for BART, we omit discussing it to avoid confusing readers}
\begin{equation} \label{eq:bart_practice_factorization}
    p(\tree,\leafparamvec|\by,\sigma^2) = p(\leafparamvec|\tree,\by,\sigma^2)p(\tree|\by,\sigma^2).
\end{equation}
The first multiplicand on the right,
$p(\bmu|\tree,\by,\sigma^2)$, is a multivariate Gaussian (with diagonal covariance) and can be sampled from directly.
The second multiplicand is proportional to $p(\by|\tree,\sigma^2)p(\tree)$, which is the product between a marginal likelihood of a Bayesian linear regression model and the prior.
The marginal likelihood can be computed using standard techniques, but the posterior cannot be sampled directly.
As such, a Metropolis-Hastings sampler is used with the following types of proposed moves:
\begin{longlist}
    \item Pick a leaf in the tree and split it (grow);
    \item Pick two adjacent leaves and collapse them back into a single leaf (prune);
    \item Pick an interior node and change the splitting rule (change);
    \item Pick a pair of parent-child nodes that are both internal and swap their splitting rules, unless both children of the parent have the same splitting rules, in which case, swap the splitting rule of the parent with that of both children (swap).
\end{longlist}
Note that all selections in these proposed moves (of nodes, splitting rules, etc.) are made uniformly at random from all available choices.\footnote{Splits that result in empty leaves are not allowed.}
The proposed move types are chosen with probabilities $\pi_g$, $\pi_p$, $\pi_c$, and $\pi_s$ respectively.
Let $Q(-,-)$ denote the transition kernel of the proposal, i.e. $Q(\tree,\tree^*)$ is the probability of tree structure $\tree^*$ being proposed given current tree structure $\tree$.
With $\tree$ and $\tree^*$ thus defined, the Metropolis-Hastings algorithm accepts the proposal with probability
\begin{equation} \nonumber
    \alpha(\tree,\tree^*) \coloneqq \min\braces*{\frac{Q(\tree^*,\tree)p(\tree^*|\by,\sigma^2)}{Q(\tree,\tree^*)p(\tree|\by,\sigma^2)}, 1}.
\end{equation}

A simple reparameterization trick is used to adapt this sampler to the case when the ensemble has multiple trees.
Because of the independence of the priors on different trees and the Gaussian likelihood, the conditional posterior $p(\tree_j,\leafparamvec_j|\tse_{-j},\leafparamvec_{-j},\sigma^2,\by)$ can be rewritten in terms of the residual vector
\begin{equation} \nonumber
    \br_{-j} \coloneqq \by - \sum_{k\neq j} g(\bX;\tree_k,\leafparamvec_k).
\end{equation}
Specifically, we have
\begin{equation} \nonumber
    p(\tree_j,\leafparamvec_j|\tse_{-j},\leafparamvec_{-j},\sigma^2,\by) = p(\tree_j,\leafparamvec_j|\br_{-j},\sigma^2),
\end{equation}
where the right-hand side is the single-tree posterior.
A single Metropolis-Hastings update step as described above is performed to draw an approximate sample $(\tree_j,\leafparamvec_j)$ from the conditional posterior.

\subsection{The analyzed BART sampler}
\label{subsec:analyzed_sampler}
The BART sampler described above is difficult to analyze because the deterministic Gibbs outer loop makes it a time-varying Markov chain.
More significantly, it is convenient in analyzing Bayesian CART to collapse the Markov chain state space by marginalizing out the leaf parameters.
The collapsed state space is simply the space of tree structures, which is discrete and finite.
However, we are unable to do this for BART in general because of the conditioning on the residuals from other trees in the inner loop.
Both of these difficulties make it impossible to apply standard techniques in Markov chain theory.

To overcome this, we propose an adaptation of the sampler that brings it closer to Bayesian CART.
We imitate \eqref{eq:bart_practice_factorization} and factorize the posterior into a conditional posterior on leaf parameters and a marginal posterior on tree \emph{ensemble} structures:
\begin{equation} \nonumber
    p(\tse,\leafparamvec|\by) = p(\leafparamvec|\tse,\by)p(\tse|\by).
\end{equation}
The conditional posterior on leaf parameters is still a multivariate Gaussian and can be sampled from directly, while the marginal posterior on tree ensemble structures remains proportional to the product of the marginal likelihood of a Bayesian linear regression model and the prior: $p(\tse|\by) \propto p(\by|\tse)p(\tse)$ (see \eqref{eq:Bayesian_lin_reg}.)
To sample from this marginal posterior, we run Metropolis-Hastings MCMC similarly to before, except for two differences:
First, instead of cycling deterministically through the trees in an inner loop as before, we pick a tree index uniformly at random.
Second, while we propose an updated tree using the same transition kernel $Q(-,-)$, we write the acceptance probability in terms of the full marginal posterior instead of conditioning on the residuals from other trees.
For further clarity, the algorithm is summarized in pseudocode as Algorithm \ref{alg:method}.

\change{We believe that neither of the two modifications substantially changes the sampler's dynamics.
First, recent work has shown that deterministic and randomized versions of the same Gibbs sampler are essentially equivalent \citep{chlebicka2025solidarity}.
Second, our simulations provide some numerical evidence that altering the acceptance probability as described above does not worsen the convergence time, nor does it seem to systematically affect the prediction performance of the approximate posterior obtained from MCMC samples.
We further note that the relationship between the conditional and marginal posterior versions of the sampler is analogous to that between orthogonal matching pursuit and forward stepwise selection, which are iterative methods for sparse linear regression---at each iteration, the former selects the variable that best explains the variance in the residual, while the latter selects the variable that, together with the current active set, best explains the variance in the response---and neither of these is known to dominate the other in terms of prediction performance.
On the other hand, because the marginal posterior is more expensive to compute than the conditional posterior, the per-iteration cost of the analyzed sampler is much higher.}


In the rest of this paper, we denote the transition kernel of the analyzed sampler using $P(-,-)$.
Furthermore, to avoid confusion with randomness arising from sampling the training set, we will denote all probabilities and expectations with respect to the algorithmic randomness using $P$ and $E$ respectively.

\begin{algorithm}[h]
  \caption{BART sampler}
  \label{alg:method}
  \small
\begin{algorithmic}[1]
  \State {\bfseries BART}(\var{$\data$}: data, \var{$m$}: no. of trees, \var{$\sigma^2$}: guess for noise variance, \var{$\lambda$}: guess for reciprocal SNR, \var{$\bpi$}: proposal probabilities, \var{$p_{TSE}$}: TSE prior, \var{$t_{max}$}: no. of sampler iterations)

  \State Initialize $\tree_1,\tree_2,\ldots,\tree_m$ as trivial trees.
  \For{$t=1,2,\ldots,t_{max}$}
    \State Sample $k \sim \text{Unif}(\braces{1,2,\ldots,m})$.
    \State Propose $\tree^* \sim Q(\tree_k,\tree^*)$.
    \State Set $\alpha(\tree_k,\tree^*) = \min\braces*{\frac{Q(\tree^*,\tree_k)p(\tree_1,\ldots,\tree_{k-1},\tree^*,\tree_{k+1},\ldots,\tree_m|\by)}{Q(\tree_k,\tree^*)p(\tree_1,\ldots,\tree_{k-1},\tree_k,\tree_{k+1},\ldots,\tree_m|\by)},1}$.
    \State Set $\tree_k = \tree^*$ with probability $\alpha(\tree,\tree^*)$.
  \EndFor
\end{algorithmic}
\end{algorithm}

%% file: AOS/03_main_results.tex
\section{BIC for BART and posterior concentration on optimal TSEs}
\label{sec:posterior_conc}

The goal of this section is to first show how to quantify the bias and complexity of a TSE separately and then jointly via BIC.
We will then show that, as a function of TSEs, the 
posterior probability $p(\tse|\by)$ concentrates on the set of TSEs with zero bias and the lowest possible complexity, and are therefore minimizers of BIC.
As such, as argued in the introduction, the highest posterior density region contains all of the most desirable TSEs.
This implies that lower bounds on the hitting times of this region reflect computational drawbacks of practical consequence.

\subsection{Measuring bias and complexity for TSEs}

We first discuss how to quantify the bias and complexity of a TSE.

\subsubsection{Partitions}
A \emph{cell} $\cell$ is a rectangular region of $\xspace$, i.e.
\begin{equation} \nonumber
    \cell = \braces*{\bx \in \xspace ~\colon~ l_i < x_i \leq u_i ~\text{for}~i=1,\ldots,\nfeats},
\end{equation}
with lower and upper limits $l_i$ and $u_i$ respectively in coordinate $i$ for $i=1,2,\ldots,d$.
A \emph{partition} is a collection of disjoint cells $\cell_1,\ldots,\cell_\nleaves$ whose union is the whole space $\xspace$.
Every tree structure $\tree$ induces a partition $\partition$ via its leaves.
Not only is $\partition$ a sufficient statistic for $\tree$, it also completely characterizes the bias and complexity of the resulting data model conditioned on $\tree$.
Indeed, this data model is just Bayesian linear regression on the indicator functions on the leaves of $\tree$.
Since the functions are orthogonal, the degrees of freedom of the regression is equal to the size of the partition.

\subsubsection{Partition ensemble models (PEMs)} \label{subsec:pem}
A TSE $\tse$ induces an ensemble of partitions $\partition_1,\partition_2,\ldots,\partition_\ntrees$.
Indeed the data model conditioned on $\tse$ is still a Bayesian linear regression on the indicator functions of all leaves in $\tse$. 
However, these indicators are no longer orthogonal, which means that different ensembles, with possibly different numbers of leaves, can give rise to the same subspace of regressors, making $\tse$ not identifiable from data.
To avoid this issue, we define
$\pem \subset L^2(\xspace^\nfeats,\xmeasure)$ to be the subspace spanned by indicators of the cells in $\partition_j$ for $j=1,2,\ldots,\ntrees$.
We call this the \emph{partition ensemble model} (PEM) associated to the TSE $\tse$ and indicate this association via the mapping $\pem = \mapone(\tse)$.

\subsubsection{Measuring bias}
Let $\Pi_\tse$ denote orthogonal projection onto $\mapone(\tse)$ in $L^2(\xmeasure)$.
We define the squared (mangitude of the) bias of $\tse$ with respect to a regression function $f$ as
\begin{equation} \label{eq:bias}
    \bias(\tse;f) \coloneqq \int (f-\Pi_\tse[f])^2 d\nu.
\end{equation}
This is precisely the squared bias of Bayesian linear regression on $\tse$ if we ignore the regularization effect from the leaf parameter priors, which is inconsequential in large sample sizes.

\subsubsection{Measuring complexity}
When conditioned on a TSE $\tse$ and ignoring regularization from leaf parameter priors, the degrees of freedom of the resulting Bayesian linear regression model is just the dimension of $\mapone(\tse)$.
We denote this by $\df(\tse)$ and use it as a measure of complexity of $\tse$.
Note that this definition does not depend on the covariate distribution $\nu$ (see Lemma \ref{lem:invariance_of_dim} in the appendix.)

\subsubsection{Function dimension and optimal sets}
Excessive complexity leads to overfitting and is hence undesirable.
To quantify the excess, we first define the \emph{$m$-ensemble dimension} of a regression function $f$ as
\begin{equation} \label{eq:dimension_function}
    \dim_m(f) \coloneqq \min\braces*{\df(\tse) \colon f \in \mapone(\tse) ~\text{and}~ \tse \in \tsespace}.
\end{equation}
In large sample sizes $n$, which is the setting we are concerned with, the TSEs that result in the smallest MSE must be bias-free.
We hence define the set of optimal TSEs in $\tsespace$ to be the minimizers of \eqref{eq:dimension_function}.
More generally, we define a series of nested sets with increasing levels of suboptimality tolerance via:
\begin{equation} \nonumber
    \opt(f,k) \coloneqq \braces*{\tse \in \tsespace \colon f \in \mapone(\tse) ~\text{and}~ \df(\tse) \leq \dim_m(f) + k}.
\end{equation}



\subsection{BIC and BART posterior concentration}
\label{subsec:bic_concentration}

The Bayesian information criterion (BIC) \citep{schwarz1978estimating} of a TSE $\tse$ is given by
\begin{equation} \nonumber
    \bic(\tse) = \frac{\by^T\paren*{\bI - \bP_{\tse}}\by}{\sigma^2} + \df(\tse)\log n + \log\paren*{2\pi\sigma^2}n.
\end{equation}
Here, $\bP_{\tse}$ refers to projection onto $\mapone(\tse)$ with respect to the empirical norm $\norm{\cdot}_n$ (realized as a matrix.)\footnote{This can be explicitly defined as $\bP_{\tse} \coloneqq \bPhi\bPhi^\dagger$, where $\bPhi$ denotes the concantenation of the leaf indicator vectors with respect to $\data$.}
Ignoring the effect of the noise vector $\mathbb{e}$ for now, we see that the first term, divided by the sample size $n$, is an estimate for the squared bias.
Meanwhile, the second term directly measures the model complexity.
Hence, BIC quantifies the quality of a TSE by accounting for both bias and complexity.
Indeed, under our data generative model (Section \ref{subsec:generative_model}) we have the following concentration lemma:

\begin{proposition}[Concentration of BIC differences]
    \label{prop:concentration_bic_diff}
    Consider two TSEs $\tse$ and $\tse'$ and denote the difference in their BIC values as $\Delta \bic(\tse,\tse') \coloneqq \bic(\tse) - \bic(\tse')$.
    We have
    \begin{equation} 
        \Delta \bic(\tse,\tse') = \frac{n}{\sigma^2}\paren*{\bias(\tse;f^*) - \bias(\tse';f^*)} + O_\P\paren*{n^{1/2}}
    \end{equation}
    If furthermore, both TSEs have the same bias, i.e. $\Pi_{\tse}[f^*] = \Pi_{\tse'}[f^*]$, then we have
    \begin{equation} 
        \Delta \bic(\tse,\tse') = \paren*{\df(\tse) - \df(\tse')}\log n + O_\P\paren*{1}.\footnote{Note that here and elsewhere in the paper, we use asymptotic notation to compare pairs of sequences of random variables $(a_n)$ and $(b_n)$ that are each indexed by $n$. 
        We say that $a_n = O_{\mathbb{P}}(b_n)$ (equivalently $b_n = \Omega_{\mathbb{P}}(a_n)$) if, for every $0 < \delta < 1$, there exists $N, C > 0$ such that $\sup_{n > N}\P\braces*{\abs*{a_n/b_n} > C} < \delta$.
        With the exception of $n$, $C$ and $N$ will be allowed to depend on any other parameters of both the BART algorithm as well as the data to which it is fitted. }
    \end{equation}
\end{proposition}
From this proposition, we also see that $\opt(f^*,k)$ for $k=0,1,2,\ldots$ are just sublevel sets of BIC when $n$ is large enough.
We next show that BIC is closely connected to the log marginal likelihood for TSEs as follows:

\begin{proposition}[Log marginal likelihood and BIC]
    \label{prop:lml_bic}
    Consider a TSE $\tse$.
    The log marginal likelihood satisfies
    \begin{equation} \nonumber
        \log p(\by|\tse) = -\frac{\bic(\tse)}{2} +  O_\P\paren*{1}.
    \end{equation}
    Consequently, the log marginal posterior also satisfies
    \begin{equation} \nonumber
        \log p(\tse|\by) = -\frac{\bic(\tse)}{2} - \log p(\by) + O_\P\paren*{1}.
    \end{equation}
\end{proposition}

This almost linear relationship implies that $\opt(f,k)$ for $k=0,1,2,\ldots$ are also superlevel sets of the marginal posterior.
In other words, they form highest posterior density regions (HPDR). 
As advertised, these will be the target sets of our hitting time analysis.
Finally, combining the previous two propositions gives the following result on posterior concentration.

\begin{proposition}[BART posterior concentration] \label{prop:posterior_conc_treespace}
    \label{prop:posterior_conc}
    The marginal posterior distribution on $\tsespace$ satisfies
    \begin{equation} \nonumber
        1 - p(\opt(f^*,0))~|~ \by) = O_\P\paren*{n^{-1/2}}.
    \end{equation}
\end{proposition}

\section{Hitting time lower bounds for BART MCMC}
\label{sec:hitting_times}


Let $(X_t)$ be a discrete time Markov chain on a finite state space $\Omega$.
Let $\setone \subset \Omega$ be a subset.
The \emph{hitting time} of $\setone$ is defined as:
\begin{equation} \nonumber
    \tau_{\setone} \coloneqq \min\braces*{t\geq 0 \colon X_t \in \setone}.
\end{equation}
Note that this is a random variable and that it, in principle, depends on the initial state $X_0$.
In our analysis, the initial state is always chosen to be an ensemble of trivial trees and so will not be referenced in the notation to avoid clutter.
As stated in the introduction, the first set of main results in our paper concern hitting time lower bounds for the HDPR identified in Section \ref{sec:posterior_conc}.

\subsection{Lower bounds for additive models}
\label{subsec:lower_bound_additive}

Our first two lower bounds are for the setting where the data-generating process (DGP) is an additive model.

\begin{theorem}[Lower bounds for additive model] 
\label{thm:additive}
    Let $f^*$ be an additive function, i.e.
    \begin{equation} \label{eq:additive}
        f^*(\bx) = f_1(x_1) + f_2(x_2) + \cdots + f_{\ntrees'}(x_{\ntrees'})
    \end{equation}
    with $2 \leq \ntrees' \leq \nfeats$ and $f_1,f_2,\ldots,f_{m'}$ non-constant.
    Suppose $x_1,x_2,\ldots,x_{\ntrees'}$ are independent.
    Suppose $\ntrees \leq \ntrees'$, and we make arbitrary choices for the other BART hyperparameters $\sigma^2$, $\lambda$, $\bpi$, $p_{TSE}$.
    The Markov chain induced by $\bart(\data,m,\sigma^2,\lambda,\bpi,p_{\tse},-)$ satisfies
    \begin{equation} \nonumber
        E\braces*{\tau_{\opt(f^*,(q_{\text{max}}-2)(q_{\text{min}} -2))}} = \Omega_{\P}\paren*{n^{1/2}},
    \end{equation}
    where $q_{\text{max}} = \max_{1\leq i\leq m'} \dim_1(f_i)$,  $q_{\text{min}} = \min_{1\leq i\leq m'} \dim_1(f_i)$.
    If furthermore, $m < m'$, $q_{\text{min}} < q_{\text{max}}$,
    and we disallow ``change'' and ``swap'' moves, i.e. $\pi_c = \pi_s = 0$, then we have
    \begin{equation} \nonumber
        E\braces*{\tau_{\opt(f^*,q_{\text{max}}- q_{\text{min}}-1)}} = \Omega_{\P}\paren*{n^{q_{\text{min}}/2-1}}.
    \end{equation}
\end{theorem}

Additive models are natural generalizations of linear models that have been widely studied in statistics and machine learning \citep{hastie1986generalized,hastie2009elements}.
When fitted to real world datasets, they often enjoy good prediction accuracy.
Hence, we view an additive generative model to be a natural class of functions for our study.

Furthermore, several works deriving consistency guarantees for frequentist greedy decision trees and random forests have used additive models as a generative function class, as the assumption of additivity helps to circumvent some of the practical and theoretical difficulties arising from greedy splitting \citep{scornet2015consistency,klusowski2024large}.
On the other hand, \cite{tan2021cautionary} showed generalization lower bounds for decision trees in this setting, with the recommendation that models with multiple trees fit additive models better than those comprising a single tree (see also \cite{tan2022fast} and Theorem 6.1 in \cite{rovckova2020posterior}.)
Indeed, given the assumptions of Theorem \ref{thm:additive}, we have $\dim_l(f^*) > \dim_{m'}(f^*)$ for any $l < m'$, while $\dim_l(f^*) = \dim_{m'}(f^*)$ for any $l \geq m'$ (see Proposition \ref{prop:additive_dim}.)
In other words, a minimum BIC value is achievable if and only if the number of trees in the BART model is larger than or equal to the number of components in the additive model.

However, at the critical threshold $m' = m$, even though the BART model is correctly specified and contains an efficient representation of $f^*$, Theorem \ref{thm:additive} tells us that the BART sampler may fail to converge efficiently if the number of training samples is large.
On the other hand, as we will subsequently show in Theorem \ref{thm:mixing_upper_bound_ntrees}, if the number of trees is much larger, the mixing time of the sampler does not grow in $n$.
Together, these results add another perspective to recent arguments that there may be \emph{computational} value in overparameterization (see for instance \cite{bartlett2020benign}.)


\subsection{Lower bound for pure interactions}
\label{subsec:lower_bound_interaction}

Our next lower bound is for the setting when the DGP has a pure interaction, which we define as follows.
Let $x_i$ and $x_j$ be two features, i.e. components of $\bx \sim \nu$.
We say that they form a \emph{pure interaction} with respect to a regression function $f^*$ if they are jointly dependent with the response $y$, but are separately conditionally independent of $y$ for any conditioning set of indices $I$ unless it includes the other's index.
Mathematically, we can write this as follows:
\begin{longlist}
    \item $(x_i,x_j) \not\indep y$;
    \item $x_i \indep y ~|~ x_I$ and $x_j \indep y ~|~ x_I$ for any $I \subset \braces*{1,2,\ldots,\nfeats}$ such that $i,j \notin I$.
\end{longlist}

A canonical example of a pure interaction is the exclusive-or (XOR) function over binary features with a uniform distribution, i.e. $f(\bx) = x_1x_2$ for $\xspace = \braces*{-1,1}^\nfeats$.
This function is well-known to be difficult to learn using either CART \citep{tan2024statistical} or neural networks \citep{abbe2022merged}.
As such, it is perhaps unsurprising that the BART sampler also experiences difficulties in this setting.

\begin{theorem}[Lower bound for pure interaction] 
\label{thm:pure_interaction}
    Let $f^*$ contain a pure interaction.
    Suppose we disallow ``change'' moves, i.e. $\pi_c = 0$, and we make arbitrary choices for all other BART hyperparameters $m$, $\sigma^2$, $\lambda$, $\pi_g$, $\pi_p$, $\pi_s$, $p_{TSE}$.
    The Markov chain induced by $\bart(\data,m,\sigma^2,\lambda,\bpi,p_{\tse},-)$ satisfies
    \begin{equation} \nonumber
        E\braces*{\tau_{\opt(f^*,\infty)}} = \Omega_\P\paren*{n^{1/2}},
    \end{equation}
    where $\opt(f^*,\infty) \coloneqq \cup_{k=0}^\infty \opt(f^*,k)$.
\end{theorem}

Note that the suboptimality gap for Theorem \ref{thm:pure_interaction} is much wider than that for Theorem \ref{thm:additive}.
Indeed, it implies that the only TSEs that are reachable within $o\paren*{n^{1/2}}$ iterations of the sampler have nonzero bias.
As such, the MSE of the BART sampler output does not even converge to zero with the training sample size, unless we allow for $\Omega(n^{1/2})$ iterations of the sampler.

\subsection{Exponential lower bound for Bayesian CART}
\label{subsec:lower_bound_bcart}

Our final hitting time lower bound shows that the HPDR hitting time for Bayesian CART can be exponential in the training sample size.
This complements and improves our results in \cite{ronen2022mixing}, which provided an exponential lower bound for mixing time for Bayesian CART.
While the previous result relied on extremely weak assumptions, the improved version requires a new assumption on the asymmetry of the regression function $f^*$ in terms of its dependence on different features.
Specifically, we say that a regression function $f^*$ has \emph{root dependence} if there exists a feature $x_i$ and threshold $t$ such that:
\begin{longlist}
    \item $\text{Corr}^2\paren*{y, \indicator\braces*{x_i \leq t}} > 0$;
    \item $(i,t)$ does not occur as a root split on any tree structure $\tree \in \opt[1](f,0)$.
\end{longlist}
An example of such a function is the ``staircase'' function $f^*(\bx) = \indicator\braces*{x_1 > 1} + \indicator\braces*{x_1, x_2 > 1}$.
The feature $x_2$ and the threshold $t=1$ satisfies properties (i) and (ii).\footnote{This example can be generalized to the family of functions $f^*(\bx) = \sum_{j=1}^s \prod_{k=1}^j \indicator\braces*{x_k > 1}$, for $s=2,3,\ldots$.}
On the other hand, additive functions on independent features do not satisfy these properties.

\begin{theorem}[Lower bound for Bayesian CART with root dependence]
\label{thm:bcart_mix}
    Suppose $f^*$ has root dependence.
    Suppose $m=1$ and that we disallow ``change'' and ``swap'' moves, i.e. $\pi_c = \pi_s = 0$.
    Suppose we make arbitrary choices for all other BART hyperparameters $\sigma^2$, $\lambda$, $\pi_g$, $\pi_p$, $p_{TSE}$.
    Then 
    the Markov chain induced by $\bart(\data,1,\sigma^2,\lambda,\bpi,p_{\tse},-)$ satisfies
    \begin{equation} \nonumber
        \liminf_{n \to \infty}\E_n\braces*{\frac{\log E\braces*{\tau_{\opt[1](f^*,0)}}}{n}} \geq \frac{1}{2\sigma^2}\paren*{\int (f^*)^2 d\nu -  \paren*{\int f^* d\nu}^2}.
    \end{equation}
\end{theorem}


\begin{remark}
    For the sake of narrative clarity, we have not tried to optimize the suboptimality gaps (i.e. $k$ in $\opt(f^*,k)$) in our lower bounds.
    Note also that we have not attempted to investigate the dependence of our lower bounds on other data or algorithmic hyperparameters.
    These of course influence the minimum sample size $N$ as well as the hidden constant factor in the Big-Omega notation of our lower bounds.
\end{remark}

\begin{remark}[Relationship between hitting times and mixing times]
    The expected hitting time of a sizable set naturally serves as a lower bound for a Markov chain's mixing time, as mixing requires the chain to visit all parts of the state space. A more profound result is that, under standard conditions, this inequality can be reversed to establish an equivalence. Specifically, the mixing time is equivalent (up to constant factors) to the maximum hitting time over all sets with a stationary probability of at least $1/2$ \citep{griffiths2014tight}.
\end{remark}

%% file: AOS/04_proof_sketch.tex
\section{Hitting time lower bounds via barrier sets}
\label{sec:proof_sketch}

In this section, we briefly outline the proof strategy we use to derive our hitting time lower bounds.
Our first move is to make use of the standard interpretation of a symmetric Markov chain as a random walk on a network, whose vertices comprise the states of the Markov chain and whose edges comprise pairs of states with positive transition probability.
We next notice that hitting times are closely related to escape probabilities, which can be interpreted as voltages on the network.
Voltages can then be calculated using standard network simplification techniques.
Putting these ingredients together creates the following recipe for deriving hitting time lower bounds:

\begin{proposition}[Recipe for hitting time lower bounds]
\label{prop:recipe}
    Let $\tse_0,\tse_1,\tse_2,\ldots$ denote the Markov chain induced by a run of $\bart(\data,m,\sigma^2,\lambda,\bpi,p_{\tse},-)$ for any fixed dataset $\data$ and any choice of hyperparameters.
    Let $\tsebad \in \tsespace$ be a TSE such that,
    \begin{equation} \nonumber
        P\braces*{\tau_{\tsebad} < \tau_{\opt(f^*,k)}} = \Omega_\P(1).
    \end{equation}
    Let $\settwo \subset \tsespace$ be a subset such that every path from $\tsebad$ to $\opt(f^*,k)$ intersects $\settwo$.
    The hitting time of $\opt(f^*,k)$ satisfies
    \begin{equation} \nonumber
        E\braces*{\tau_{\opt(f^*,k)}} = \Omega_\P\paren*{\exp\paren*{\frac{1}{2}\min_{\tse \in \settwo}\Delta\bic(\tse,\tsebad)}}.
    \end{equation}
\end{proposition}


\begin{figure}[t]
    \centering
    \includegraphics[scale=0.7]{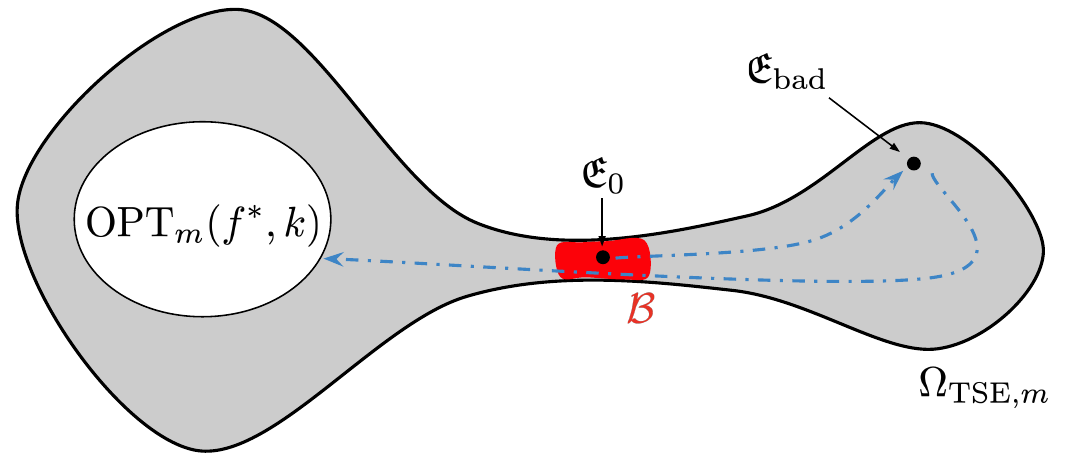}
    \caption{
    Visual illustration of 
    Proposition \ref{prop:recipe}. 
    The chain is initialized at $\tse_0$ and with positive probability hits a suboptimal TSE $\tsebad$ before $\opt(f^*,k)$. 
    This causes to chain to get stuck, as it can only reach $\opt(f^*,k)$ by passing through an ``impassable'' barrier set $\mathcal{B}$.
    }
    \label{fig:proof_technique}
\end{figure}

In other words, our hitting time lower bounds follow immediately once we are able to identify (i) a suboptimal TSE $\tsebad$ that is reachable by the sampler before it hits $\opt(f^*,k)$ and (ii) a barrier $\settwo$ that separates $\tsebad$ from $\opt(f^*,k)$, i.e. is a vertex cutset, and which has higher BIC, therefore acting as an impassable ``barrier''.
We now briefly describe our choices for each of the three different lower bound settings we have considered in this paper.

\subsection{Additive models} \label{subsec:proof_sketch_additive}
For simplicity, we only discuss the case of $m'= m$ (when the number of components is equal to the number of trees.)
For $i=1,2,\ldots,m$, denote $q_i = \dim_1(f_i)$, and let $0=\xi_{i,0} < \xi_{i,1} < \cdots < \xi_{i,q_{i}} = b$ denote the \emph{knots} of $f_i$, i.e. the values for which $f_i(\xi_{i,j}) \neq f_i(\xi_{i,j}+1)$, together with the endpoints.\footnote{ $\dim_1(f_i)$ is simply the number of constant pieces of $f_i$ or alternatively, one larger than the number of knots of $f_i$}
Without loss of generality, assume that $f_1,\ldots,f_{m}$ are ordered in descending order of their 1-ensemble dimension, i.e. $q_1 \geq q_2 \geq \cdots \geq q_{m}$.

We now define $\tsebad$ and a ``bad set'' $\setone$.
To this end, we define a collection of partition models $\pem_1,\pem_2,\ldots,\pem_m$ (spans of indicators in a single partition) as follows.
First, for $i=3,4,\ldots,m$, $j=1,2,\ldots,q_i$, define the cells $\leaf_{i,j} \coloneqq \braces{\bx \colon \xi_{i,{j-1}} < x_i \leq \xi_{i,j}}$,
and set $\pem_i = \linspan\paren*{\braces*{\indicator_{\leaf_{i,j}} \colon j=1,2,\ldots,q_i}}$, i.e. for each $i$, $\pem_i$ contains splits only on feature $i$ and only at the knots of $f_i$. 
To introduce inefficiency, we define each of $\pem_1$ and $\pem_2$ to have splits on both features 1 and 2.
This construction is demonstrated in Figure \ref{fig:additive_construction}.
The formal details are fairly involved and will be deferred to the appendix. 
We define $\setone$ via
\begin{equation} \nonumber
    \setone \coloneqq \braces*{(\tree_1,\tree_2,\ldots,\tree_m) \colon \mapone(\tree_i) = \pem_i ~\text{for}~i=1,2,\ldots,m},
\end{equation}
and set $\settwo$ to be the \emph{outer boundary} of $\setone$, i.e.
\begin{equation} \nonumber
    \settwo = \braces*{\tse \in \tsespace \colon \tse~\text{has an edge to}~\setone~\text{and}~\tse \notin \setone}.
\end{equation}
Finally, we pick $\tsebad$ to be a particular element of $\setone$ whose precise construction will be detailed in the Section \ref{sec:proof_of_additive}.
Therein, we will also show that 
\begin{longlist}
    \item $\tse$ has zero bias for all $\tse \in \setone$;
    \item $\setone \cap \opt(f^*,(q_{\text{max}}-2)(q_{\text{min}}-2)-2) = \emptyset$;
    \item $\min_{\tse \in \settwo}\Delta\bic(\tse,\tsebad) \geq \log n - O(1)$.
\end{longlist}

\begin{figure}[t]
    \centering
    \includegraphics[scale=0.55]{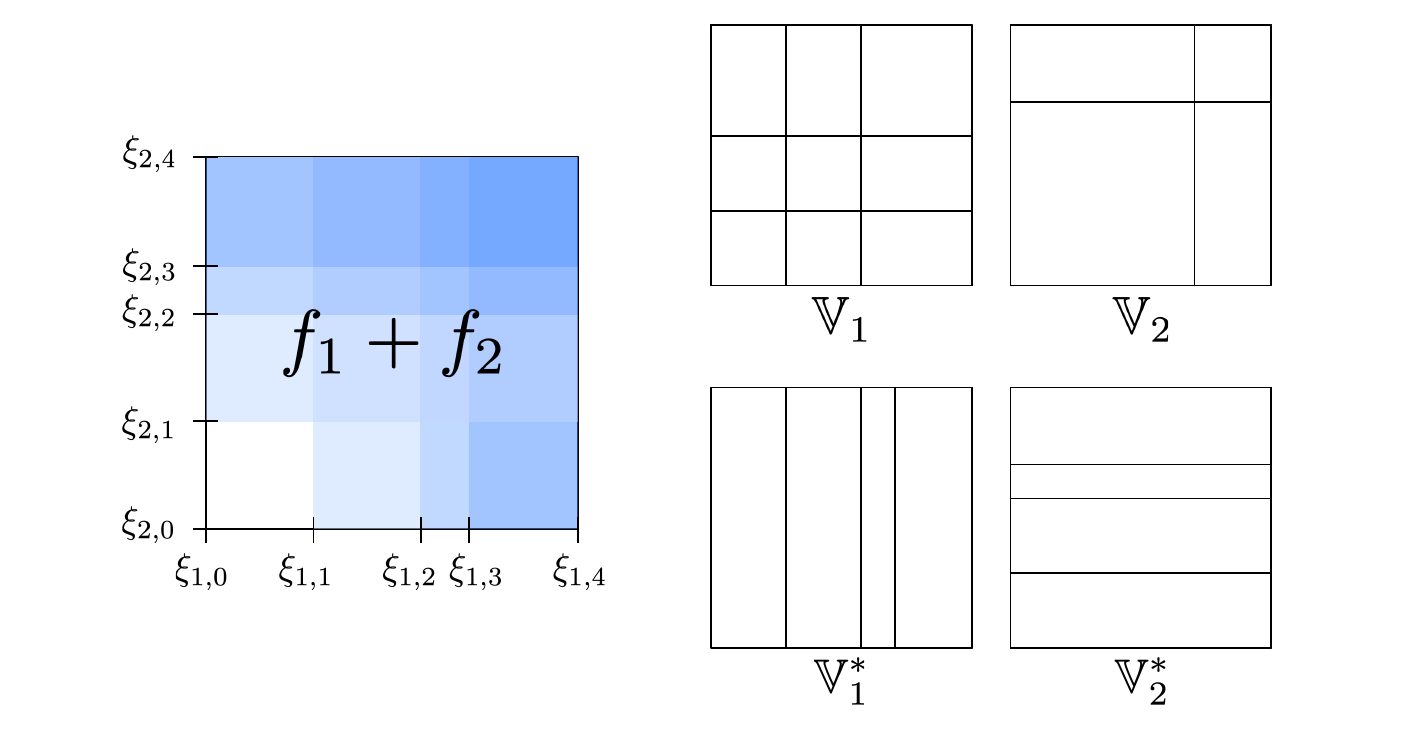}
    \vspace{-0.2in}
    \caption{
    Visual illustration of the construction for $\tsebad$ used in the proof of Theorem \ref{thm:additive}.
    The left panel displays the function $f_1(x_1) + f_2(x_2)$ together with all knots of $f_1$ and $f_2$.
    We define $\tsebad$ so that its first two trees $\tree_1$ and $\tree_2$ induce the partitions $\pem_1$ and $\pem_2$ respectively.
    These combine for a total of 13 leaves.
    On the other hand, an optimal TSE will instead make use of $\pem_1^*$ and $\pem_2^*$, which combine for a total of only 8 leaves.
    Nonetheless, we still have $f_1 + f_2 \in \pem_1^* + \pem_2^*$.
    }
    \label{fig:additive_construction}
\end{figure}

\subsection{Pure interactions}
We first define the notion of \emph{reachability} as follows:
\change{Given $\tse, \tse' \in \tsespace$, we say that 
$\tse \succsim \tse'$ if $\tse$ and $\tse'$ are connected by an edge and if one of the following holds: (a) $\bias(\tse;f^*) > \bias(\tse';f^*)$; or (b) $\bias(\tse;f^*) = \bias(\tse';f^*)$ and $\df(\tse) \geq \df(\tse')$.}
Note that, because we allow only ``grow'' and ``prune'' moves, for adjacent $\tse$ and $\tse'$, $\mapone(\tse)$ and $\mapone(\tse')$ are nested subspaces, so that $\bias(\tse;f^*) = \bias(\tse';f^*)$ if and only if $\Pi_{\tse}[f^*] = \Pi_{\tse'}[f^*]$.
We say that $\tse'$ is reachable from $\tse$, denoted $\tse \succeq \tse'$, if there is a sequence of TSEs $\tse = \tse^0, \tse^1,\ldots,\tse^k = \tse'$ such that $\tse^i \succsim \tse^{i+1}$ for $i=0,1,\ldots,k-1$.

Without loss of generality, let $(x_1,x_2)$ be a pure interaction for $f^*$.
Let $\tsebad$ be any TSE such that
\begin{longlist}
    \item $\tsebad$ is reachable from $\emptytse$;
    \item There does not exist $\tse \in \tsespace$ such that $\tse$ is reachable from $\tsebad$ but $\tsebad$ is not reachable from $\tse$.
\end{longlist}
Note that such a TSE exists because $\tsespace$ is finite and $\succeq$ is a partial ordering on this space.
We set $\setone$ to be the equivalence class of $\tsebad$ under $\succeq$ and set $\settwo$ to be the outer boundary of $\setone$.
We will show in the appendix that no TSE in $\setone$ makes a split on either $x_1$ or $x_2$, which implies that $\setone \cap \opt(f^*,\infty) = \emptyset$.
We will also show that  $\min_{\tse \in \settwo}\Delta\bic(\tse,\tsebad) \geq \log n - O(1)$.

\subsection{Bayesian CART}
Without loss of generality, let $x_1$ be the feature that gives $f^*$ root dependence.
By assumption, there is a threshold $t$ such that splitting the trivial tree on $x_1$ at $t$ gives a decrease in squared bias.
We set
\begin{equation} \nonumber
    \setone = \braces*{\tree \in \tsespace[1] \colon \tree~\text{has root split on}~x_1~\text{at}~t},
\end{equation}
and 
\begin{equation} \nonumber
    \treebad = \arg\min \braces*{\bic(\tse) \colon \tse \in \setone}.
\end{equation}
Note that the outer boundary of $\setone$ is a singleton set comprising the trivial tree $\emptytree$.
By assumption, we have $\setone \cap \opt[1](f^*,0) = \emptyset$.
We will show in the appendix that
\begin{equation} \nonumber
    \liminf_{n \to \infty} \frac{\sigma^2}{n}\Delta\bic(\emptytree,\treebad) \geq \int (f^*)^2 d\nu -  \paren*{\int f^* d\nu}^2.
\end{equation}

From these constructions, we also see that the reason why hitting times grow with training sample size is because the barrier sets become increasingly difficult to pass through.
Heuristically, we can say that this is because the intrinsic ``temperature'' of the BART sampler is inversely proportional to the training sample size.
Our analysis strategy is quite general and can probably be applied to other sampling problems in Bayesian nonparametrics with ``multimodal'' posteriors, thereby revealing similar mixing issues.

%% file: AOS/04_upper_bounds.tex
\section{Mixing time upper bounds for BART MCMC}
\label{sec:upper_bounds}

In this section, we describe how the undesirable scaling of hitting times with training sample size can be ameliorated using any one of three strategies: (i) increasing the number of trees in the ensemble, (ii) making more global changes to the tree structures before applying the Metropolis-Hastings filter, (iii) using a tempered posterior distribution, whereby the marginal likelihood on tree structures is raised to a $1/T$ power.
We will derive corresponding mixing time upper bounds, proving these results in the context of additive models.
Note that mixing time is easily shown to be an upper bound for hitting time of the HDPR set, which means that these upper bounds stand as a counterpoint to the lower bounds in Theorem \ref{thm:additive}.

Following \cite{kim2023mixing}, we prove our upper bounds for the \emph{lazification} of the BART sampler analyzed in previous sections of this paper.
Specifically, recalling that $P$ denotes the transition matrix of the analyzed sampler, we will focus in this section on the Markov chain with transition kernel $\widetilde P \coloneqq (P+I)/2$.
Note that $\widetilde{P}$ and $P$ both have the same stationary distribution.
Furthermore, since $P$ already does not display periodic behavior, there is little reason to believe that lazification, which is meant to combat periodicity, substantially affects mixing times.
On the other hand, all eigenvalues of $\widetilde{P}$ are nonnegative, which makes its analysis easier.
Indeed, this ensures that its mixing time is controlled by its spectral gap, which we can bound using a canonical path argument \citep{sinclair1992improved}.

We now state our first result, which makes the same assumptions as Theorem \ref{thm:additive}, except for replacing the upper bound on the number of trees with a lower bound.
As explained in the discussion following the statement of Theorem \ref{thm:additive}, that result together with Theorem \ref{thm:mixing_upper_bound_ntrees} give evidence for the computational value in overparameterization.

\begin{theorem}[Upper bound from increasing number of trees]
\label{thm:mixing_upper_bound_ntrees}
    Assume $f^*$ is additive (see \eqref{eq:additive}).
    Suppose $\pi_g,\pi_p > 0$, and $p_{\text{TSE}}$ has full support, and let $\sigma^2$, $\lambda$ be arbitrary.
    If the number of trees satisfies $\ntrees \geq 2\abs{\tsespace[1]}$, 
    then for any tolerance $\varepsilon > 0$, 
    the lazification of the Markov chain induced by $\bart(\data,m,\sigma^2,\lambda,\bpi,p_{\text{TSE}},-)$ has $\varepsilon$-mixing time satisfying
    \begin{equation}
    \label{eq:more_trees_mixing}
        t_{\operatorname{mix}}(\varepsilon) = O_\P(1).
    \end{equation}
\end{theorem}

To understand this result, it is helpful to revisit the proof sketch for Theorem \ref{thm:additive} presented in Section \ref{subsec:proof_sketch_additive}.
There, we argued that if the number of trees is less than or equal to the number of additive components, one may construct a suboptimal partition ensemble such that any sampler move modifying these partitions would lead to either larger bias (by removing a relevant split) or larger degrees of freedom (by adding an irrelevant split).
As such, there exists a high BIC barrier between $\tsebad$, a TSE whose trees implement the suboptimal partition ensemble, and the set of optimal TSEs.
On the other hand, if we were allowed to expand $\tsebad$ by appending additional trivial trees, it would be possible to grow these trees sequentially until they implement the optimal partition ensemble, before pruning the original trees in $\tsebad$ towards their root nodes.
This sequence of moves can be shown to be non-increasing in BIC, and therefore bypasses the high BIC barrier.
Indeed, one can check that the corresponding partition ensemble \emph{model} (see Section \ref{subsec:pem}) is invariant under the ``growing'' phase, while the ``pruning'' phase only removes irrelevant splits.\footnote{This mechanism for avoiding a bottleneck is reminiscent of the popular puzzle ``The Tower of Hanoi''.}

Under the assumption $m \geq 2\abs{\tsespace[1]}$, one can generalize this construction to define appropriate paths between any pair $\tse,\tse' \in \tsespace$.
This defines a \emph{canonical path ensemble}, which can be used to bound the spectral gap $\gamma$ of the Markov chain \citep{sinclair1992improved,LevinPeresWilmer2006}.
Note that $\gamma$ is defined to be the difference between the first and second eigenvalues of the Markov chain's transition matrix and it is a standard result that the mixing time of a lazy irreducible Markov chain is bounded as $t_{\operatorname{mix}}(\varepsilon) \leq \log(\pi_{\min}/\varepsilon)/\gamma$, where $\pi_{\min}$ is the minimum probability of a state under its stationary distribution.
The full details of the proof can be found in Section \ref{sec:upper_bounds_proofs}.


While our first strategy for ``fixing'' the BART sampler was indirect and proceeded by changing the BART parameterization, the next two strategies directly modify the sampler logic and mirror well-known approaches for improving the mixing of sampling algorithms over continuous parameter spaces, such as random walk Metropolis-Hastings (RWMH).
One perspective on the bottleneck described in Section \ref{subsec:proof_sketch_additive} is that the sampler is too myopic---it is unwilling to make a ``bad'' move even if it may subsequently lead towards a TSE with higher posterior value.
In RWMH, a natural solution is to increase the random walk step size.
Although there is no notion of step size for the BART sampler, we may still adjust the proposal by chaining together a sequences of $r$ moves, each drawn according to the original proposal kernel $Q$.
The resulting proposal kernel is denoted as $Q^r$ and can be computed according to the formula
\begin{equation} \label{eq:multi_step_proposal}
    Q^r(\tse,\tse') = \sum_{(\tse_1,\tse_2,\ldots,\tse_{r-1}) \in \tsespace^{\otimes r-1}} Q(\tse,\tse_1) \paren*{\prod_{i=2}^{r-1} Q(\tse_{i-1},\tse_{i})} Q(\tse_{r-1},\tse').
\end{equation}

Yet another natural solution for dealing with a myopic sampler is to perform simulated tempering \citep{marinari1992simulated}, which allows the sampler to switch between different temperatures.\footnote{Simulated tempering for Bayesian CART was proposed by \citet{angelopoulos2005tempering}, but has yet to be investigated for BART.}
One may adjust the temperature of RWMH by raising the MH acceptance probability to a $1/T$ power, where $T > 0$ is the temperature parameter.
For a fixed temperature, the stationary distribution of this tempered sampler is now proportional to $\pi^{1/T}$, where $\pi$ is the original target density.
By allowing temperature switches, one is finally able to extract a sequence of samples approximating $\pi$.
Since our target distribution is a posterior, it is arguably more meaningful to raise the temperature of the likelihood factor in the posterior, rather than that of the entire distribution.
More precisely, noting that the posterior ratio between any pair of TSEs decomposes as $p(\tse'|\by)/p(\tse|\by) = p(\tse')p(\by|\tse')/p(\tse)p(\by|\tse)$, we define the acceptance probability at temperature $T$ (and for a $r$-step proposal) as
\begin{equation} \label{eq:temperature_T_MH_ratio}
    \widetilde{\alpha}_{r,T}(\tse,\tse') \coloneqq \min\braces*{\frac{Q^r(\tse',\tse)p(\tse')}{Q^r(\tse,\tse')p(\tse)}\paren*{\frac{p(\by|\tse')}{p(\by|\tse)}}^{1/T}, 1}.
\end{equation}

Using the above two ideas, we may define an enhanced algorithm $\textbf{BART+}$ such that, for any choice of inputs $\data$ (data), $m$ (number of trees), $\sigma^2,\lambda$ (likelihood parameters), $\bpi$ (move weights), $p_{\text{TSE}}$ (tree structure prior), $r$ (number of proposal steps), $T$ (temperature), and $t_{\max}$ (number of iterations), $\textbf{BART+}(\data,m,\sigma^2,\lambda,\bpi,p_{\text{TSE}},r,T,t_{\max})$ follows the the logic of $\textbf{BART}(\data,m,\sigma^2,\lambda,\bpi,p_{\text{TSE}},t_{\max})$ described in Algorithm \ref{alg:method}, except that $Q$ on line 5 is replaced by $Q^r$ and $\alpha$ on lines 6 and 7 is replaced by $\widetilde{\alpha}_{r,T}$.
This sampler can be shown to satisfy the following mixing time upper bounds:

\begin{theorem}[Upper bound from multi-step proposals]
\label{thm:mixing_upper_bound_multistep}
    Assume $f^*$ is additive (see \eqref{eq:additive}).
    Suppose $\pi_g,\pi_p > 0$, and suppose the proposal contains a trivial move (returning the current state) that is selected with some probability $\pi_0 > 0$.
    Also suppose that $p_{TSE}$ has full support, and let $m, \sigma^2$, $\lambda$ be arbitrary.
    If $r \geq \operatorname{diam}(\tsespace)$, then for any tolerance $\varepsilon > 0$, 
    the lazification of the Markov chain induced by $\textbf{BART+}(\data,m,\sigma^2,\lambda,\bpi,p_{\text{TSE}},r,1,-)$ has $\varepsilon$-mixing time satisfying
    \begin{equation}
    \label{eq:multistep_mixing}
        t_{\operatorname{mix}}(\varepsilon) = O_\P(1).
    \end{equation}
\end{theorem}

\begin{theorem}[Upper bound from increasing temperature]
\label{thm:mixing_upper_bound_temperature}
    Assume $f^*$ is additive.
    Suppose $m \geq m'$, $\pi_g,\pi_p,\pi_s > 0$, and $p_{TSE}$ has full support, and let $ \sigma^2$, $\lambda$ be arbitrary.
    Suppose $n, T \to \infty$ such that $T \geq \log^{\beta} n$ for some $\beta > 0$, then for any tolerance $\varepsilon > 0$, 
    the lazification of the Markov chain induced by $\textbf{BART+}(\data,m,\sigma^2,\lambda,\bpi,p_{\text{TSE}},1,T,-)$ has $\varepsilon$-mixing time satisfying
    \begin{equation}
    \label{eq:tempered_mixing}
    \log t_{\operatorname{mix}}(\varepsilon) = O_\P(\log^{1-\beta}\nsamples).
    \end{equation}
\end{theorem}

The proofs of these two theorems follow the same approach as that of Theorem \ref{thm:mixing_upper_bound_ntrees} and are deferred to Section \ref{sec:upper_bounds_proofs}.
Note that the upper bound in \eqref{eq:tempered_mixing} grows more slowly than any polynomial in $n$ and is thus an improvement over the $\Omega(n^{1/2})$ lower bound in Theorem \ref{thm:additive}.
On the other hand, the fact that we use a fixed temperature instead of the full simulated tempering framework means that the stationary distribution for the sampler analyzed in Theorem \ref{thm:mixing_upper_bound_temperature} is not the posterior and instead has the formula
\begin{equation} \label{eq:tempered_distribution}
    p(\tse;\by,T) \coloneqq \frac{p(\tse)p(\by|\tse)^{1/T}}{\sum_{\tse' \in \tsespace}p(\tse')p(\by|\tse')^{1/T} }.
\end{equation}
We call this the \emph{tempered posterior distribution} for BART.
Despite its difference from the original posterior, one can still show that it concentrates on the HDPR set (albeit at a slower rate) so long as $T$ does not grow too quickly.

\begin{proposition}[Tempered BART posterior concentration] \label{prop:tempered_posterior_conc_treespace}
    For any $0 < \beta < 1$, if $T \leq \log^{\beta}\nsamples$ as $n,T \to \infty$,
    the tempered posterior distribution on $\tsespace$ satisfies
    \begin{equation} \nonumber
        1 - p(\opt(f^*,0);\by,T) = \exp\paren*{- \frac{1}{2}\log^{1-\beta} n + O_\P(1)}.
    \end{equation}
\end{proposition}


\begin{remark}
    Similar to our lower bounds, we have not tried to optimize the assumptions of our upper bounds (such as the lower bound on $m$ in Theorem \ref{thm:mixing_upper_bound_ntrees}), nor have we attempted to investigate the dependence of the upper bounds on other data or algorithmic hyperparameters.
    On the other hand, constant dependence on $n$ is the best possible, which means that our upper bounds are tight in a certain sense.
\end{remark}

%% file: AOS/06_discussion.tex
\section{Theoretical limitations and future work}
\label{sec:limitations}

In this section, we detail the limitations of our theoretical results, which naturally suggest directions for future work.
We first note that all of our results focus purely on how the number of iterations required for sampler convergence scales with the number of training samples $n$. 
In other words, our paper does not address:
\begin{itemize}
\item \emph{Dependence on other DGP parameters.}
We did not analyze the dependence of our results on $\nfeats$, the dimension of the feature space, $\ncats$, the number of categories for each discrete feature, $s$, the sparsity of the true regression function, and $\xmeasure$, the covariate distribution.
\item \emph{Dependence on other algorithmic parameters.}
We did not analyze the dependence of our results on $p_{\text{TSE}}$, the prior on $\tsespace$, the move probabilities $\pi_g, \pi_p,\pi_c, \pi_s$, and the variance parameters $\sigma^2$ and $\lambda$.
\end{itemize}

As such, it is possible that the modifications to the default sampler proposed in Section \ref{sec:upper_bounds} could result in worse dependence on these other parameters.
Furthermore, note that \emph{our results are asymptotic.} Both the hitting time lower bounds as well as the mixing time upper bounds hold only when the training sample size is larger than a minimum number $N$.
This number depends on the DGP and algorithmic parameters, and if tracked, can be exponentially large in some of them.
Together, these limitations imply that for a \emph{fixed} $n$, the proposed modifications to the sampler could potentially make its hitting/mixing time worse. 
In other words, they may not improve the overall effectiveness of the sampler.

In addition to the limited scope of our theoretical analysis, its connection to how BART as used in practice is mediated by several factors:

\begin{itemize}
    \item \emph{Differences between the analyzed and in-practice BART samplers.} The differences are that:
\begin{enumerate}
    \item[(i)] We assume a fixed noise parameter $\sigma^2$ instead of putting a prior on it;
    \item[(ii)] At each iteration, we pick a random tree to update instead of cycling deterministically through the trees;
    \item[(iii)] We change the acceptance probability in the Metropolis filter to be in terms of the fully marginalized posterior on TSEs instead of being conditioned on the parameters of all other trees apart from the one being updated;
    \item[(iv)] Some theoretical results require a restriction of the move set to just ``grow'' and ``prune'' moves.
\end{enumerate}
\item \emph{Restriction to discrete covariates.}
We assumed that our covariate distribution was discrete, i.e. $\xspace = \braces*{1,2,\ldots,\ncats}^d$.
Many real datasets, of course, contain continuous features.
In this case, in-practice BART either restricts splits to observed feature values or computes a grid of quantile values for each continuous feature, and selects a split threshold only from among these values.
Although this effectively makes the covariate space discrete, it also means that the space varies with the training sample size $n$.
Our proofs rely heavily on the uniform concentration of all node and split-based quantities across a finite set and so does not automatically generalize to this setting.
\item \emph{Restriction to additive generative models.}
In most of our theoretical results, we assumed that our training data was generated according to an additive model.
While this is a widely-used class of statistical models, it is a stretch to say that it accurately describes most real-world datasets.
\item \emph{Failure (of our lower bounds) to address MSE and coverage directly.}
Let $\tse_0,\tse_1,\tse_2,\ldots$ denote the Markov chain induced by a run of BART.
For $j=0,1,2,\ldots$, let $h_j$ be a draw from the conditional posterior on regression functions, $p(f|\tse_j,\by)$.
The BART algorithm returns the collection $\mathcal{S}_{\operatorname{post}} \coloneqq \braces*{h_{t_\text{burn-in}+1}, h_{t_\text{burn-in}+2},\ldots h_{t_{\text{max}}}}$,
 where $t_{\text{burn-in}}$ is the number of burn-in iterations to be discarded.
The final fitted function is the mean of this collection, while approximate credible prediction intervals are derived from quantiles.
While our hitting time lower bounds imply that the Markov chain on TSEs fails to converge and hence that the $p(f|\tse_j,\by)$'s are separately suboptimal, this does not preclude the \emph{mean} of $\mathcal{S}_{\operatorname{post}}$ having good MSE performance.
It also does not address coverage of the approximate credible intervals.
As such, it is possible that the BART output nonetheless performs well on these metrics prior to reaching the HDPR hitting time.

\end{itemize}

Due to these limitations, our theoretical results should not be seen as a definitive assessment of the sampler's quality, but rather as a starting point for theoretical analysis and a guide for its potential improvement.
In addition to our proposed modifications, we believe that the proposal distribution should favor more promising split directions, instead of being uniform at random.
There is now a vast literature on how using gradient and even Hessian information can help to accelerate MCMC in continuous state spaces (see for instance \cite{neal2011mcmc}), and there is recent work in extending this to discrete spaces \citep{zanella2020informed}.
Third, we believe instead of constraining the proposal distribution to ``local'' moves, it could benefit from incorporating moves that alter tree structures more drastically.
This has been explored somewhat by \cite{kim2023mixing}.

%% file: AOS/05_sims.tex
\section{Simulations}
\label{sec:sims}

\subsection{Simulations summary}

In this section, we describe the results of a simulation study of the BART sampler designed to bridge some of the theoretical limitations raised in the previous section.
Specifically, we want to investigate whether the trend identified theoretically in Section \ref{sec:upper_bounds}---that convergence time increases with $n$---occurs for real datasets and for the in-practice BART sampler with default hyperparameter choices.
We further investigate the sensitivity of this trend to various changes in the algorithm hyperparameters and implementation choices, including (1) temperature, (2) number of trees, (3) feature selection prior, (4) number of burn-in iterations, (5) initialization, (6) discretization strategy, (7) choices of proposal moves, (8) using the fully marginalized vs. conditional posterior in the Metropolis-Hastings acceptance probability.

Instead of measuring time to convergence, we found it more convenient to measure the degree of convergence of the sampler at a fixed number of iterations.
Following longstanding practice in the Monte Carlo literature, we quantify degree of convergence using
the Gelman-Rubin statistic $\hat R$ \citep{gelman1992inference}, which compares the within-chain variance against between-chain variance for multiple independently run chains.
Since $\hat R$ is defined for scalar Markov chains, we first compute a scalar summary of each sampler iterate via its RMSE evaluated on a held-out test set and then compute $\hat R$ for the RMSE sequences thus obtained.
Smaller $\hat R$ values indicate better mixing, with $\hat R = 1.1$ being the conventional threshold below which a Markov chain is considered to have ``mixed well'' \citep{gelman1992inference}.
To explore the consequences of potential poor mixing on the predictive performance of the model, we used the held-out test set to evaluate the empirical coverage values for level 0.95 posterior prediction intervals, in addition to RMSE.\footnote{For simulated datasets, we instead calculate RMSE and empirical coverage with respect to the true function values.}
We note, however, that RMSE and coverage are also affected by how the hyperparameters modify the target posterior in addition to the sampler's dynamics.
More details on these metrics can be found in Appendix \ref{sec:additional_sims}.

Some of our key findings are as follows:
\begin{longlist}
    \item $\hat R$ indeed increases with $n$ for the in-practice BART sampler and across all datasets tested. The strength of the trend as well as the magnitude of the $\hat R$ values varies substantially across datasets.
    \item Increasing the temperature of the BART MCMC sampler either by a constant factor, or with a time varying schedule dampens the increasing trend for $\hat R$, improves coverage, and sometimes even RMSE.
    \item Increasing the number of trees consistently dampens the trend in $\hat R$. Its effect on coverage and RMSE is ambiguous.
    \item Initializing the chains from a fitted XGBoost ensemble (rather than the default trivial ensemble) or simply increasing the number of burn-in iterations can increase $\hat R$ values and exaggerate their increasing trend.
    We believe this suggests that the chains become stuck around different local maxima for the target posterior.
    \item The increasing trend for $\hat R$ is relatively robust to the factors (6), (7), and (8) mentioned above.
\end{longlist}

In the rest of this section, we describe and present the results of experiments that lead to our findings (i), (ii), and (iii), with the description of all other experiments deferred to Section \ref{sec:additional_sims} in the appendix.
We also note that the experiments leading to finding (v) also provide evidence that while we made certain modifications to the sampler---limiting the proposal moves or using the full marginal likelihood formula in the Metropolis-Hastings filter---for the sake of theoretical tractability, these do not substantially affect either $\hat R$ or the predictive properties of the MCMC approximate posterior.

\subsubsection{Code availability}
\label{sec:sim_code}

All the code necessary to reproduce the experiments in this section is publicly available at \href{https://github.com/yanshuotan/bart-comp-efficiency/}{\url{https://github.com/yanshuotan/bart-comp-efficiency/}}.
The computing infrastructure used was a Linux cluster managed by the Department of Statistics at UC Berkeley. The runs used both 24-core nodes with 128 GB of RAM, and large-memory nodes with 792 GB RAM and 96 cores.

\subsubsection{Algorithm settings and hyperparameters}
We use a custom implementation of BART (included with our experiments code) and, unless stated otherwise, adopt the hyperparameter configurations suggested by \citet{chipman2010bart}.
To be precise, we set the number of trees, $m=200$ and use the other prior hyperparameters detailed in Section \ref{subsec:differences_bart}.
We use the sorted unique values of each feature in the training data as candidate thresholds for splits.
The responses are scaled to lie between $-.5$ and $.5$.
We set the proposal move type probabilities to be $\pi_g = \pi_p = 0.25$, $\pi_c = 0.4$, $\pi_s = 0.1$.
For each run of each experiment, we run 8 independent chains in parallel, discard the first 1000 burn-in iterations and use the subsequent 10,000 iterations to calculate convergence diagnostics. 
In each experiment, we typically vary a single hyperparameter, keeping all other hyperparameter choices constant at these specified values.
The results are averaged over 25 replicates.


\subsubsection{Data}
We conducted experiments on six datasets: four large-scale real-world datasets from the PMLB benchmark \citep{romano2020pmlb} and two synthetic datasets generated from known data-generating processes (DGPs).
The characteristics of the real-world datasets are summarized in Table~\ref{tab:primary_datasets}.
For these datasets, we held out a fixed 10\% test set. Training sets of size \( n = 200, 500, 1000 \), and \( 10{,}000 \) were generated by subsampling without replacement from the remaining 90\%. The \textit{Satellite Image} dataset was an exception, where the maximum sample size was limited to \( n = 5{,}000 \) due to its smaller total size. For the synthetic datasets, we generated a single, independent test set of 10,000 observations. Training sets of the target sizes were then drawn independently from the same DGP.
The two DGPs are borrowed from recent papers on heterogeneous treatment effect modeling---a key application area for BART---and are detailed as follows:

\begin{longlist}
    \item \textit{Low-Dimensional Smooth} \citep{lei2021conformal}
Let $\bx_i \in \R^{10}$, $\bx_i \sim \mathcal{N}(0,\bSigma)$ with $\Sigma_{ii} = 1$, $\Sigma_{ij} = 0.01$ for $i \neq j$.
The response is generated as $y_i = g(x_{i,0}) \cdot g(x_{i,1}) + \epsilon_i$, where $g(x) = 2/(1 + \exp(-12(x - 0.5)))$.
The noise variance is calibrated to achieve a signal-to-noise ratio of 3.
    \item \textit{Piecewise Linear} \citep{kunzel2019metalearners}
We generate $\bx_i \in \R^{20}$ as a Gaussian copula; that is, let $\tilde\bx_i \in \R^{20}$, $\tilde\bx_i \sim \mathcal{N}(0,\bSigma)$ with $\Sigma_{ii} = 1$, $\Sigma_{ij} = 0.01$ for $i \neq j$ and set $x_{ij} = \Phi(x_{ij})$, where $\Phi$ is the standard normal CDF.
The response is generated as a piecewise linear function based on the last covariate:
\begin{align}
    y_i = \begin{cases}
\bx_i^\top \bP_{1:5}\bbeta_l + \epsilon_i & \text{if } {x}_{i,20} < -0.4 \\
\bx_i^\top \bP_{6:10}\bbeta_m + \epsilon_i & \text{if } -0.4 \leq {x}_{i,20} < 0.4 \\
\bx_i^\top \bP_{11:15}\bbeta_u + \epsilon_i & \text{if } {x}_{i,20} \geq 0.4.
\end{cases}
\end{align}
Here, $\bP_{j:k}$ is a diagonal projection matrix that zeroes out all entries apart from those in positions $j$ to $k$.
The coefficient vector $\bbeta$ has components sampled independently from $\text{Uniform}[-15, 15]$, with the same random seed used across all Monte Carlo replications.
The noise variance is set equal to 1.
\end{longlist}

\begin{table}[h]
\centering
\caption{Real-world datasets used in the primary experiments. \\ $n$ denotes the number of observations and $p$ the number of features.}
\label{tab:primary_datasets}
\begin{tabular}{lccc}
\hline
Dataset & $n$ & $p$  \\
\hline
Breast Tumor & 116,640 & 9  \\
Echo Months & 17,496 & 9  \\
Satellite Image & 6,435 & 36 \\
California Housing & 20,640 & 8 \\
\hline
\end{tabular}
\end{table}

\subsection{Experiment 1: Varying temperature}

Recall the definition of the acceptance probability with temperature $T$ in \eqref{eq:temperature_T_MH_ratio}.
We investigate the effect of temperature in the BART sampler, comparing the default (\(T=1.0\)) to constant temperatures of \(T=2.0\) and \(T=3.0\). We also evaluate a variable temperature schedule that decays linearly from \(T_{\max}=3.0\) to \(T_{\min}=1.0\) over the MCMC iterations. This schedule encourages high exploration early during the burn-in phase before converging to the standard sampler.


\begin{figure}[H]
    \centering
    \begin{tabular}{c}
    \begin{minipage}{0.9\textwidth}
        \centering
        \includegraphics[width=0.32\linewidth]{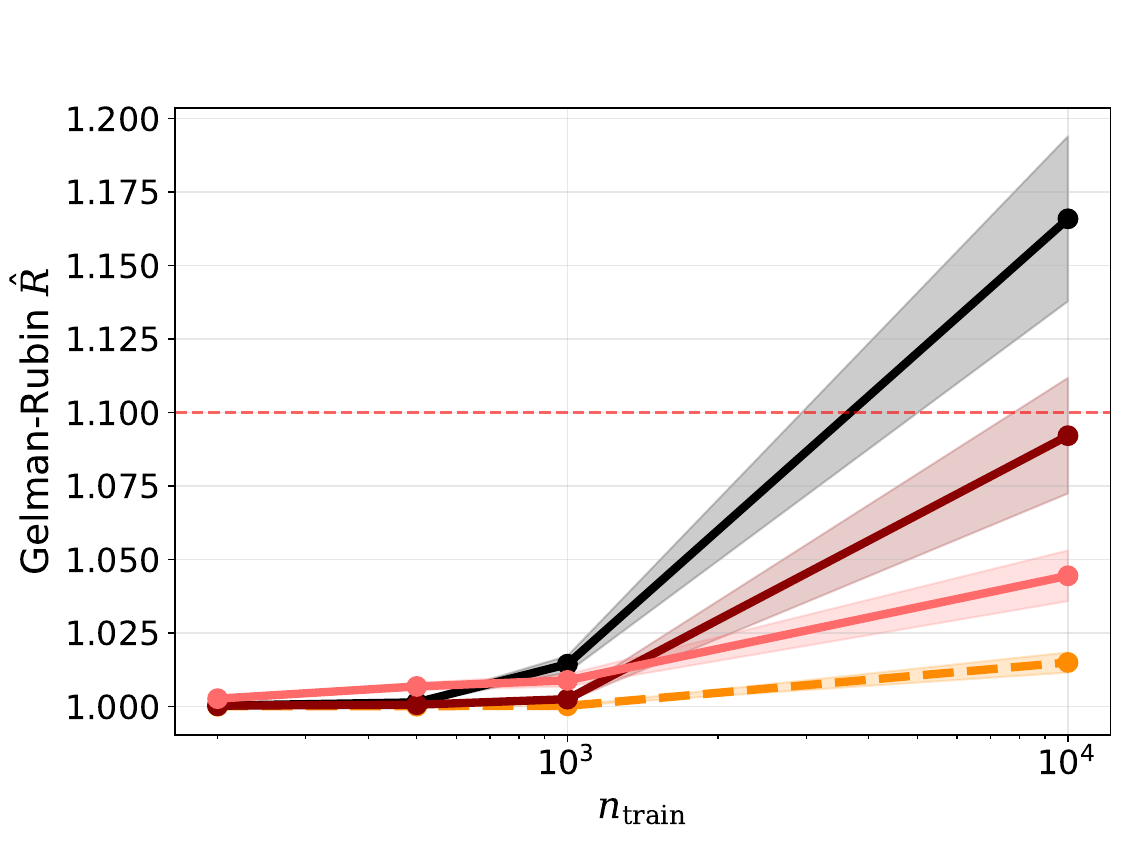}%
        \hfill
        \includegraphics[width=0.32\linewidth]{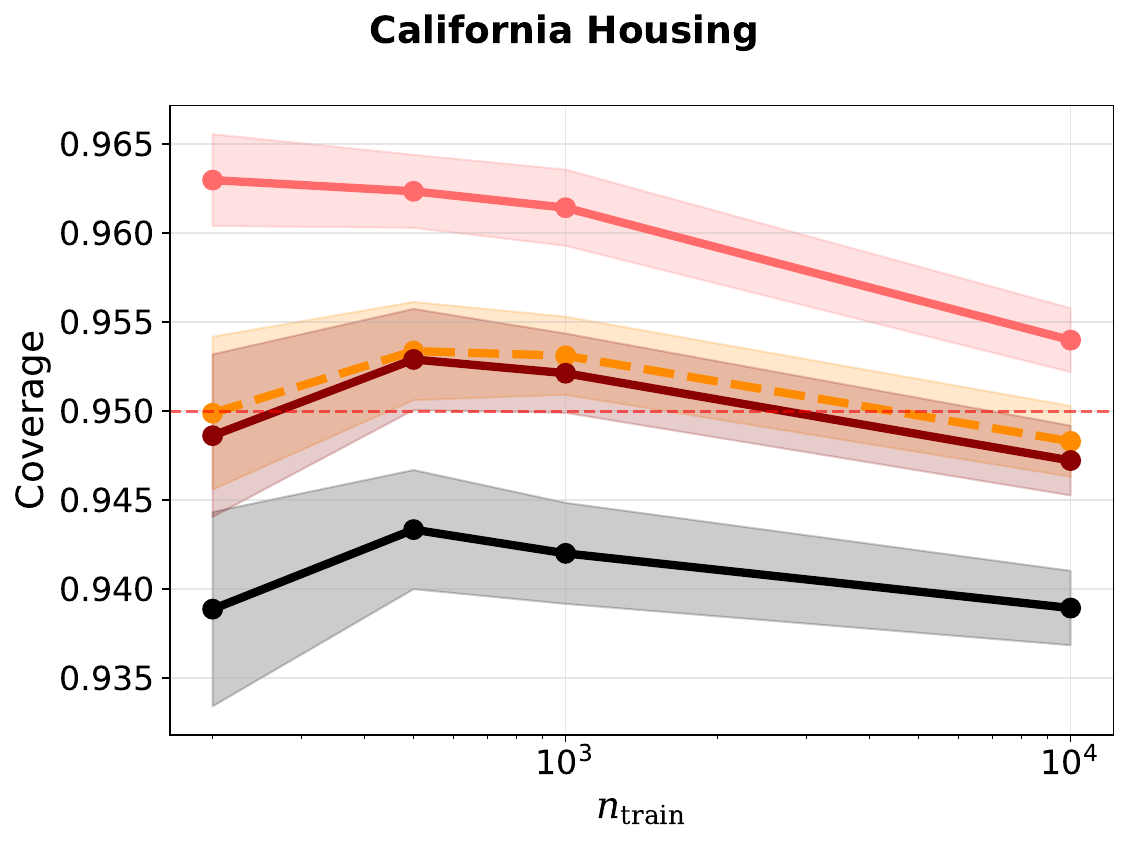}
        \hfill
        \includegraphics[width=0.32\linewidth]{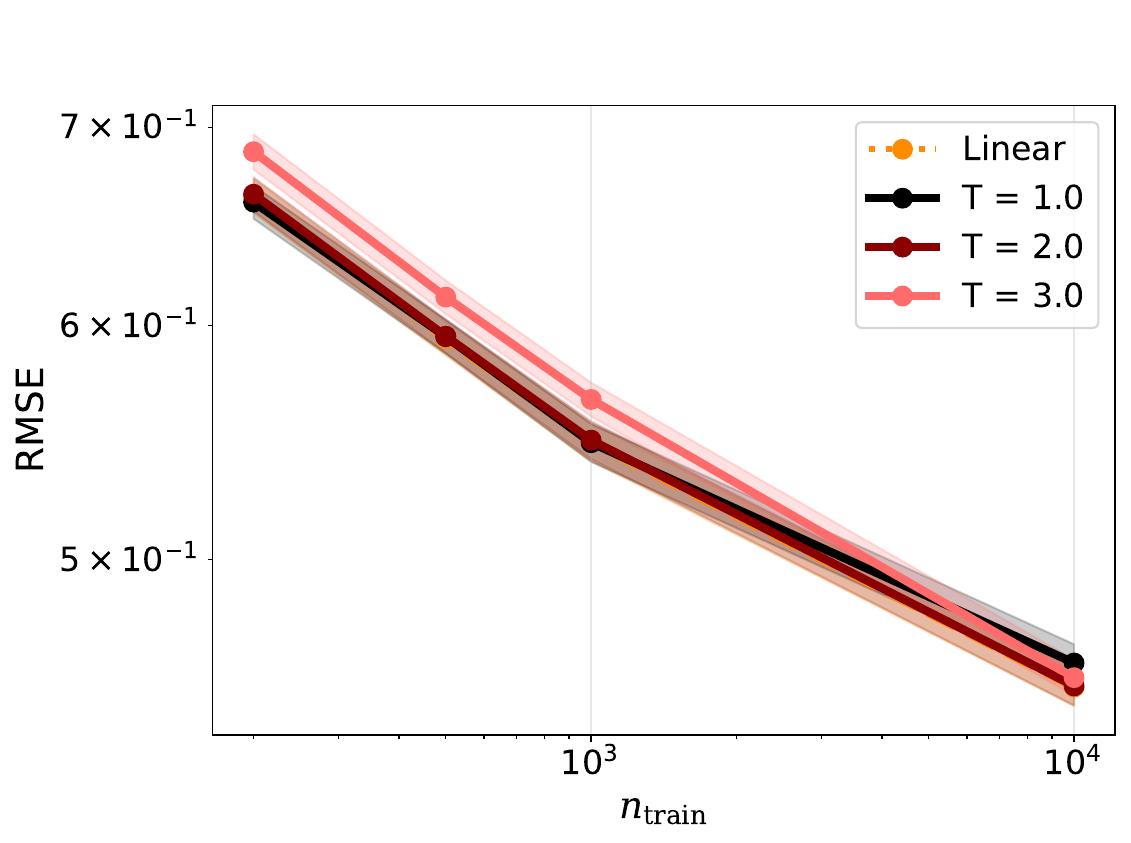}%
    \end{minipage} \\
    \begin{minipage}{0.9\textwidth}
        \centering
        \includegraphics[width=0.32\linewidth]{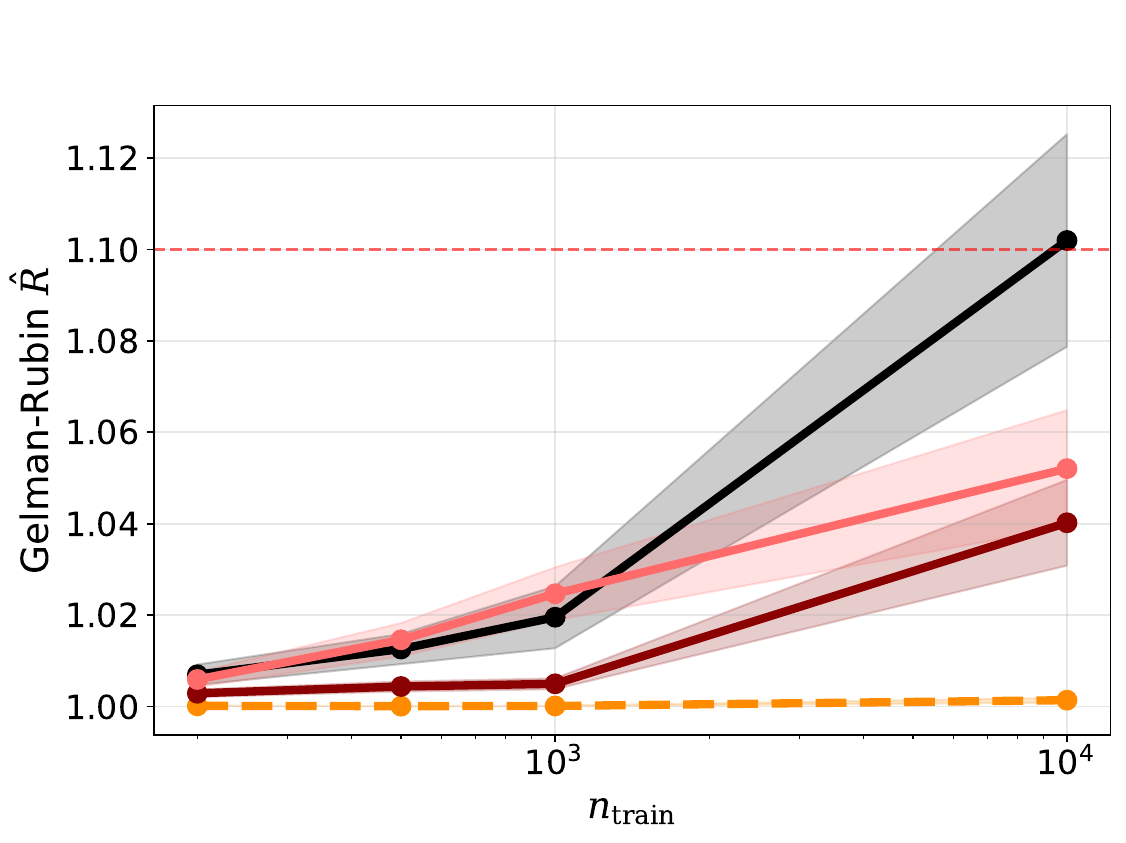}%
        \hfill
        \includegraphics[width=0.32\linewidth]{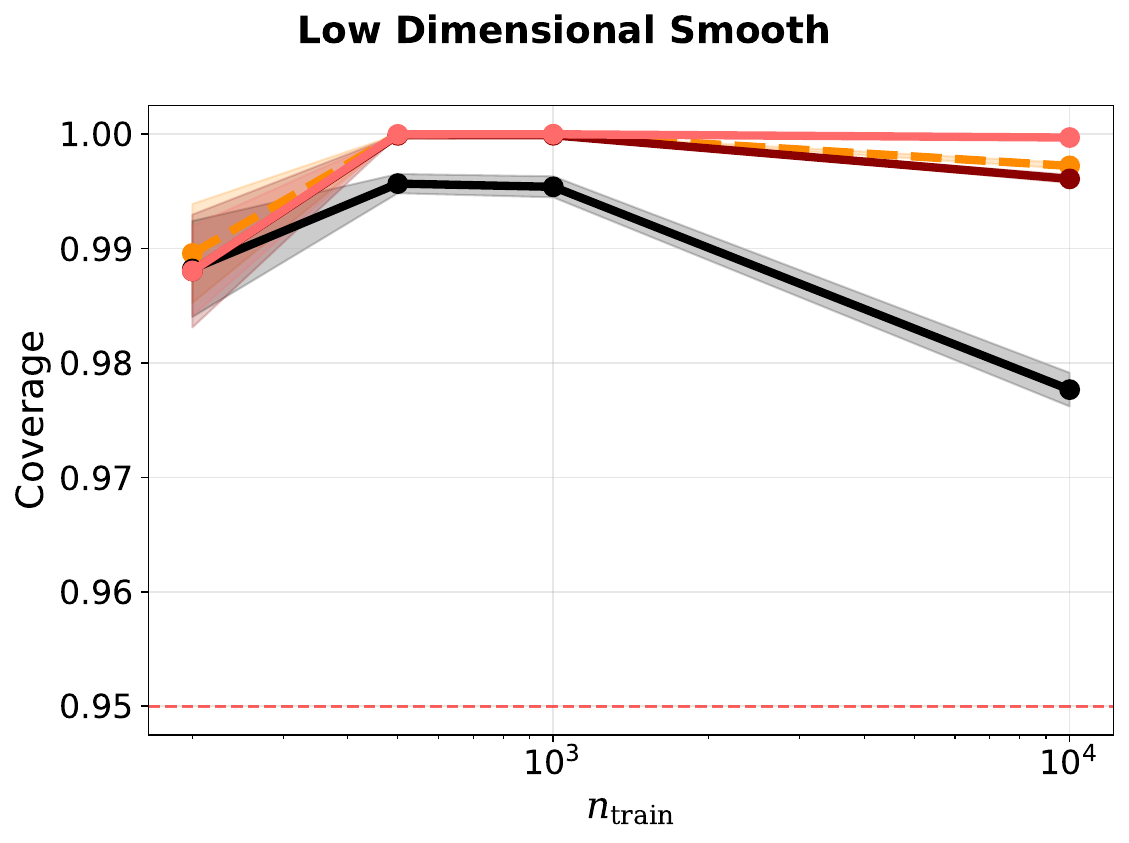}
        \hfill
        \includegraphics[width=0.32\linewidth]{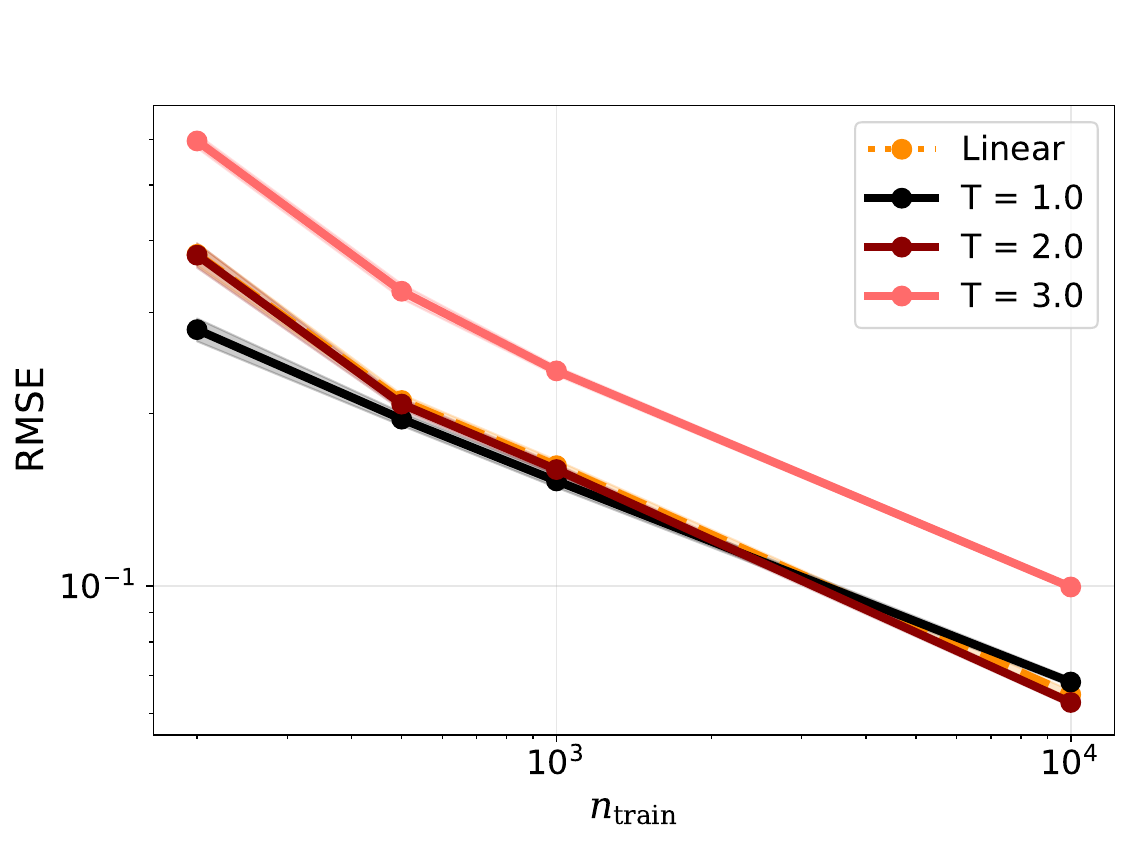}%
    \end{minipage} \\
    \begin{minipage}{0.9\textwidth}
        \centering
        \includegraphics[width=0.32\linewidth]{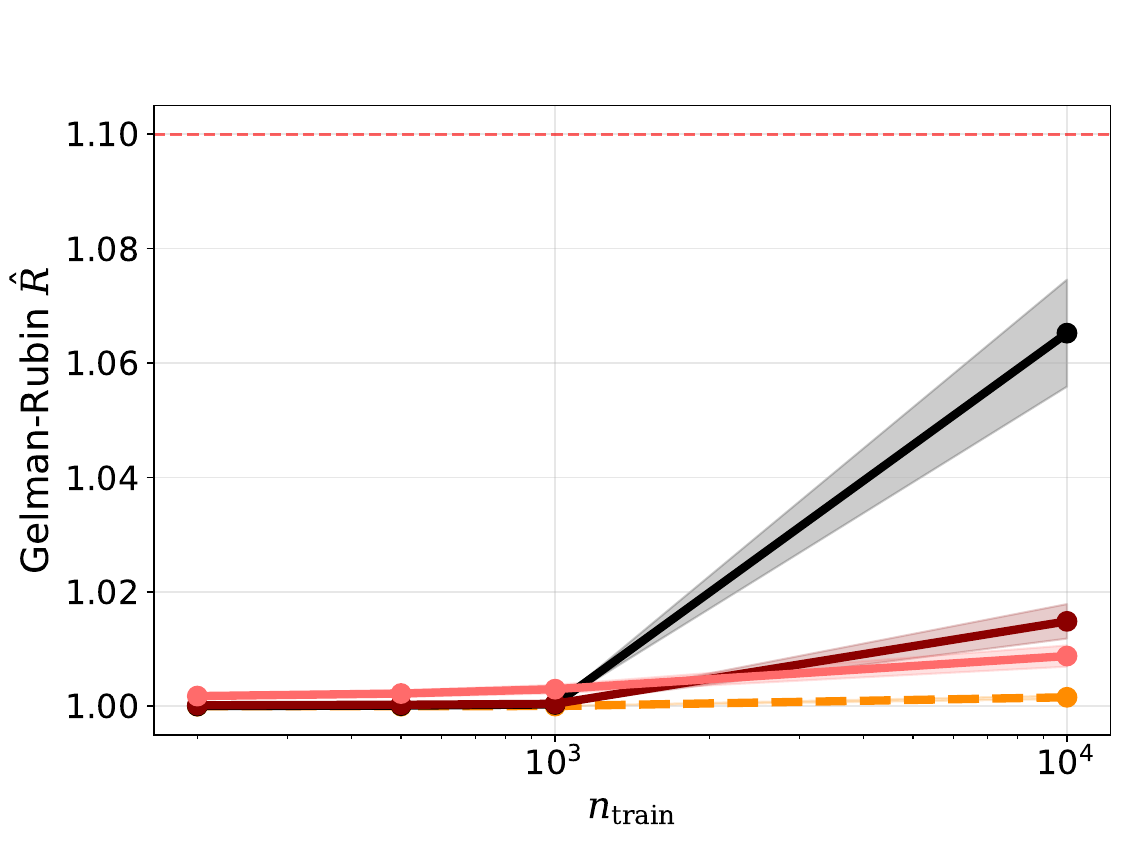}%
        \hfill
        \includegraphics[width=0.32\linewidth]{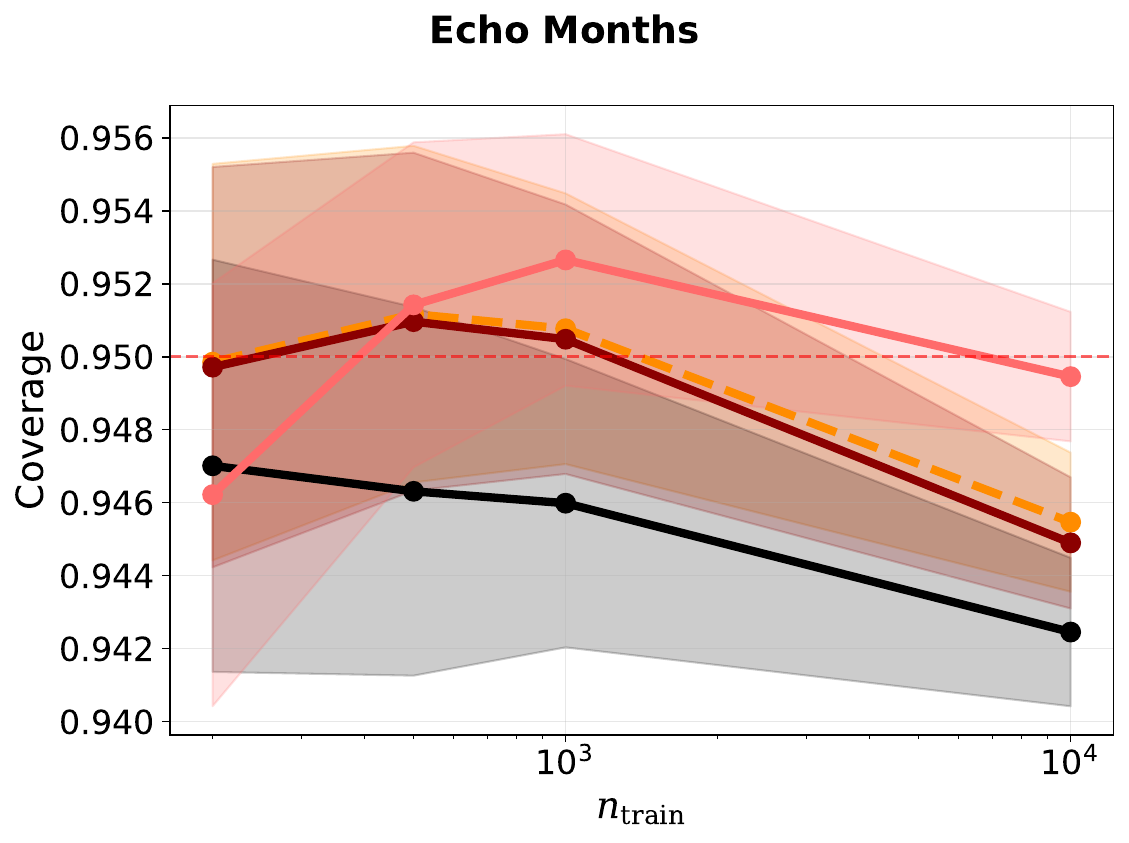}
        \hfill
        \includegraphics[width=0.32\linewidth]{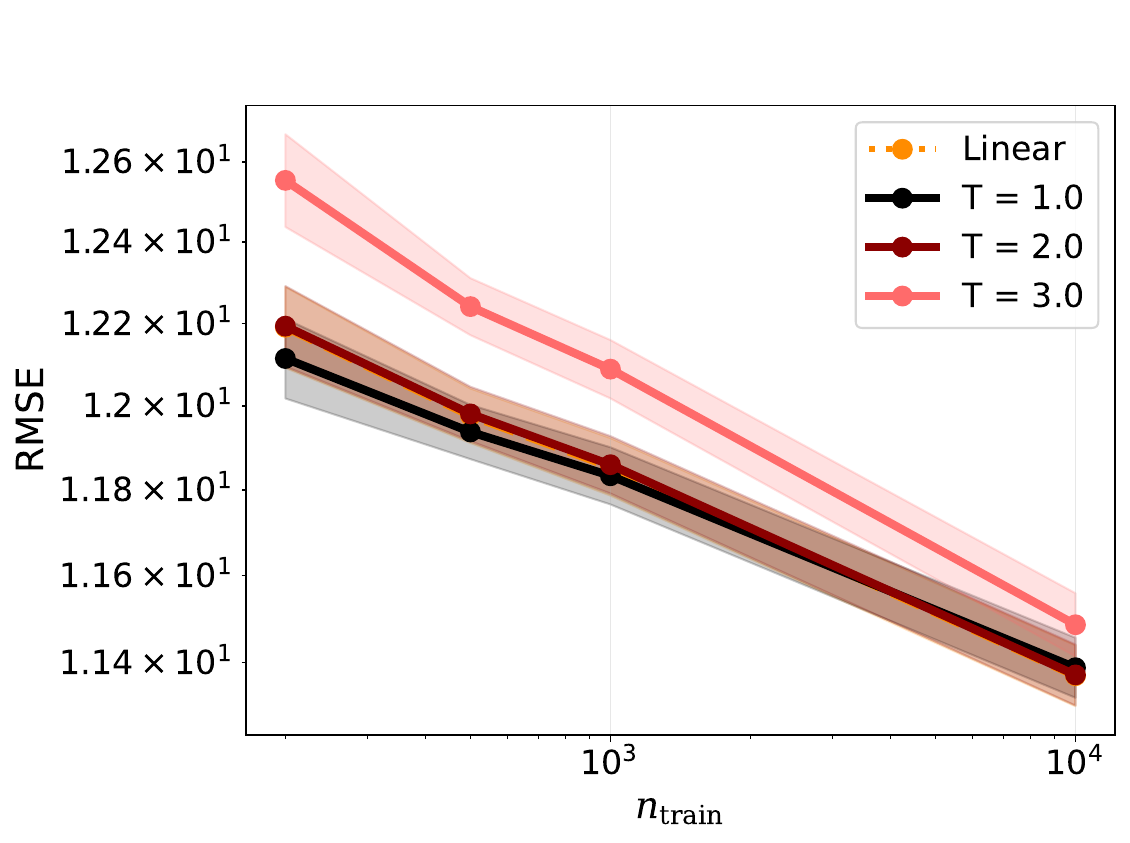}%
    \end{minipage} \\
    \end{tabular}
    \caption{
    Values for Gelman-Rubin $\hat R$ (left), coverage (center), and RMSE (right) for the BART sampler under different fixed temperatures ($T \in \braces{1,2,3}$) as well as a linear temperature schedule.
    Results are plotted for the California Housing (top), Low-Dimensional Smooth (middle), and Echo Months (bottom) datasets.
    Error bars represent $\pm 1.96$ standard errors from 25 replicates.
    }
    \label{fig:experiment1}
\end{figure}

The results of the experiment for three of the six datasets are shown in Figure \ref{fig:experiment1}, with the rest of the results deferred to Figure \ref{fig:experiment1_additional} in the appendix.
We see that under default settings ($T=1$), there is an increasing trend for $\hat R$, although the strength of the trend as well as the magnitude of the $\hat R$ values varies substantially across datasets.
Increasing the temperature has a clear dampening effect on this trend, with the best $\hat R$ results achieved by the linear temperature schedule. 
Coverage also improves with higher temperatures, approaching the nominal 95\% level more closely, suggesting that better exploration of the posterior leads to more reliable uncertainty quantification.\footnote{Note that temperature does only affects the acceptance probability for sampling tree structures. The sampling of leaf parameters as well as the noise variance parameter is unaffected by temperature.}
Finally, while the default temperature often achieves the best RMSE performance, this is not always the case, especially for the largest values of $n$.

To supplement these results, we replace the RMSE values used to calculate $\hat R$ with other summary statistics, namely the 0.05, 0.25, 0.50, 0.75, and 0.95 quantiles for the fitted responses on the held-out test set.
In Figure \ref{fig:experiment1_bar}, we plot these $\hat R$ values calculated using 10000 posterior samples.
The results for the rest of the datasets can be found in Figure \ref{fig:experiment1_bar_additional} in the appendix.
Observe that there is often significant heterogeneity in the $\hat R$ values for different quantiles, but increasing the temperature uniformly decreases the $\hat R$ values across all quantiles.
\begin{figure}[H]
    \centering
    \begin{tabular}{c}
    \begin{minipage}{0.9\textwidth}
        \centering
        \includegraphics[width=0.32\linewidth]{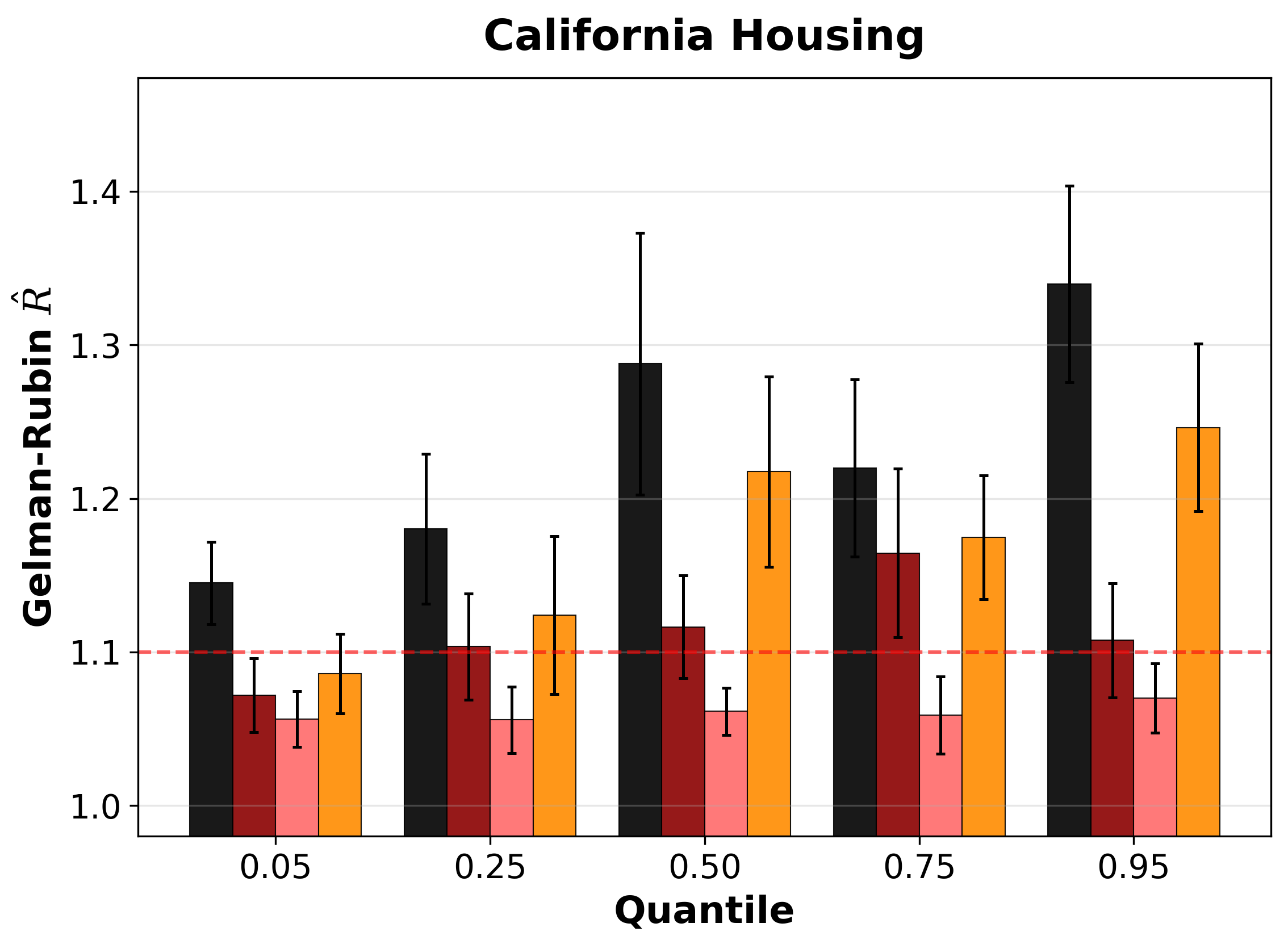}%
        \hfill
        \includegraphics[width=0.32\linewidth]{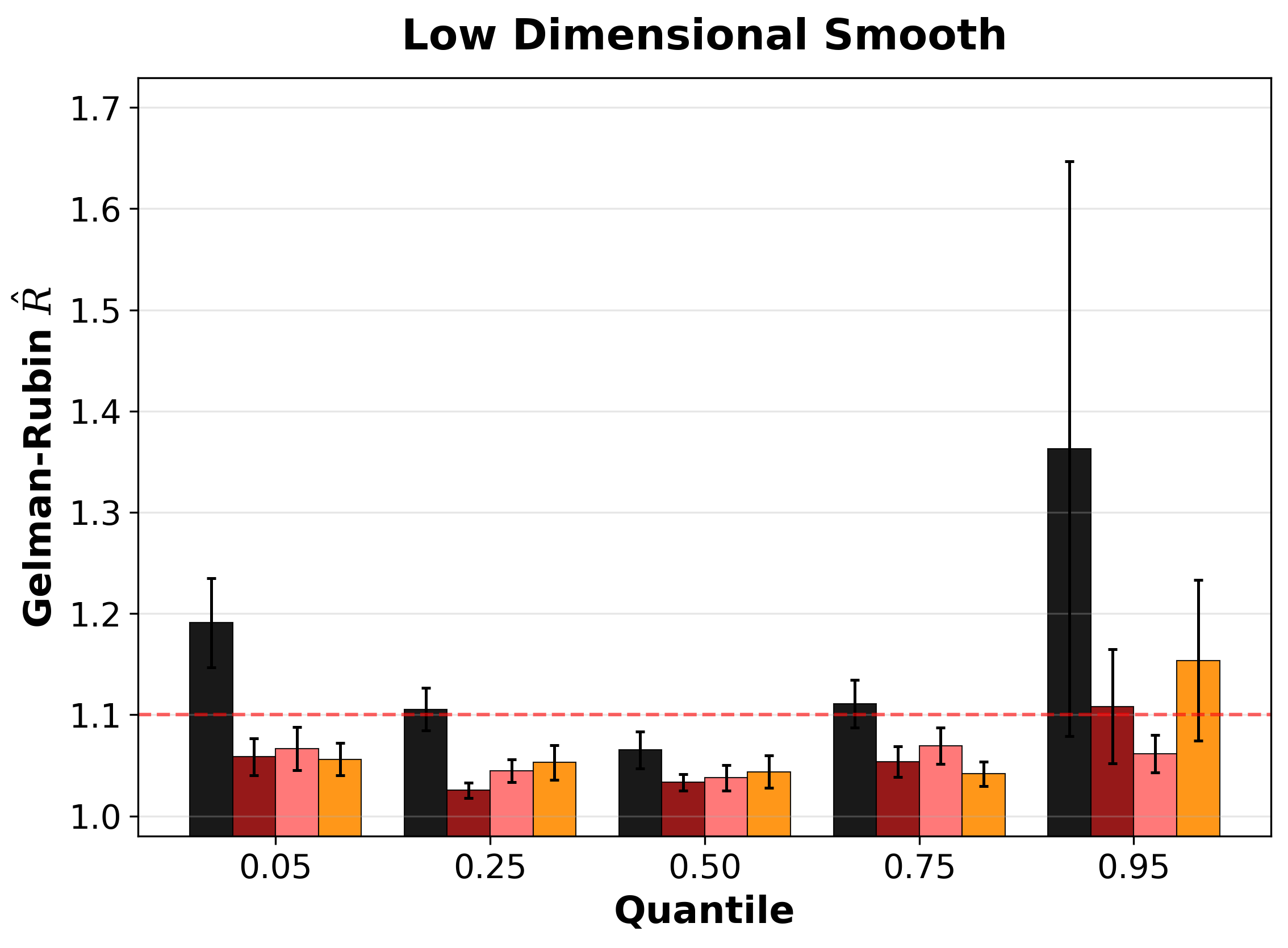}
        \hfill
        \includegraphics[width=0.32\linewidth]{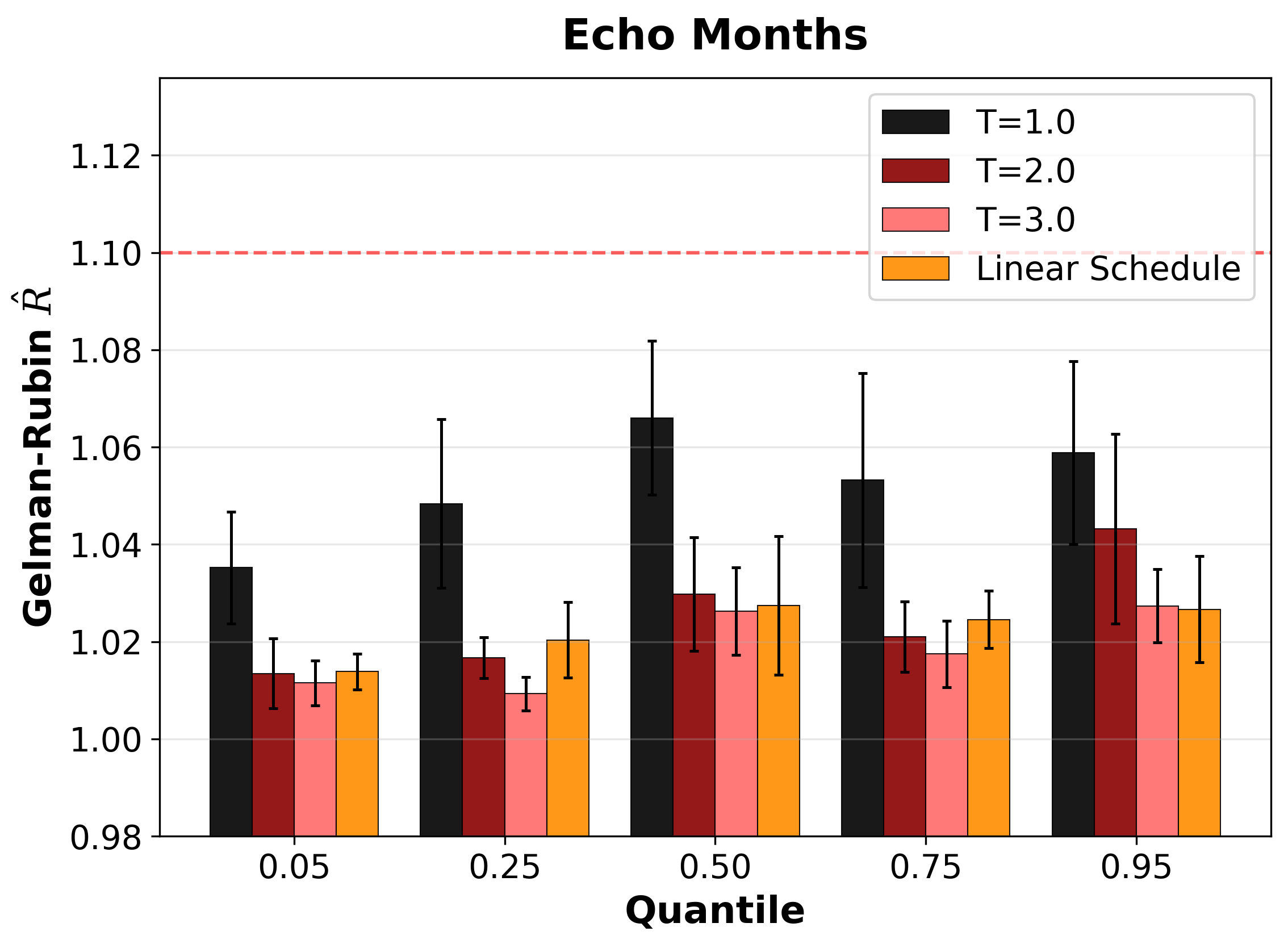}%
    \end{minipage} \\

    \end{tabular}
    \caption{
    Values for Gelman-Rubin $\hat R$ for the BART sampler under different fixed temperatures ($T \in \braces{1,2,3}$) when $\hat R$ is computed with respect to the 0.05, 0.25, 0.50, 0.75, and 0.95 quantiles for the fitted responses on the held-out test set.
    Results are plotted for the California Housing (left), Low Dimensional Smooth (center), and Echo Months (right) datasets.
        }
    \label{fig:experiment1_bar}
\end{figure}

\subsection{Experiment 2: Varying number of trees}


We vary the number of trees over \( m \in \{10, 20, 50, 100, 200, 500\} \), ranging from small ensembles to sizes exceeding typical practice.


\begin{figure}[H]
    \centering
    \begin{tabular}{c}
    \begin{minipage}{0.9\textwidth}
        \centering
        \includegraphics[width=0.32\linewidth]{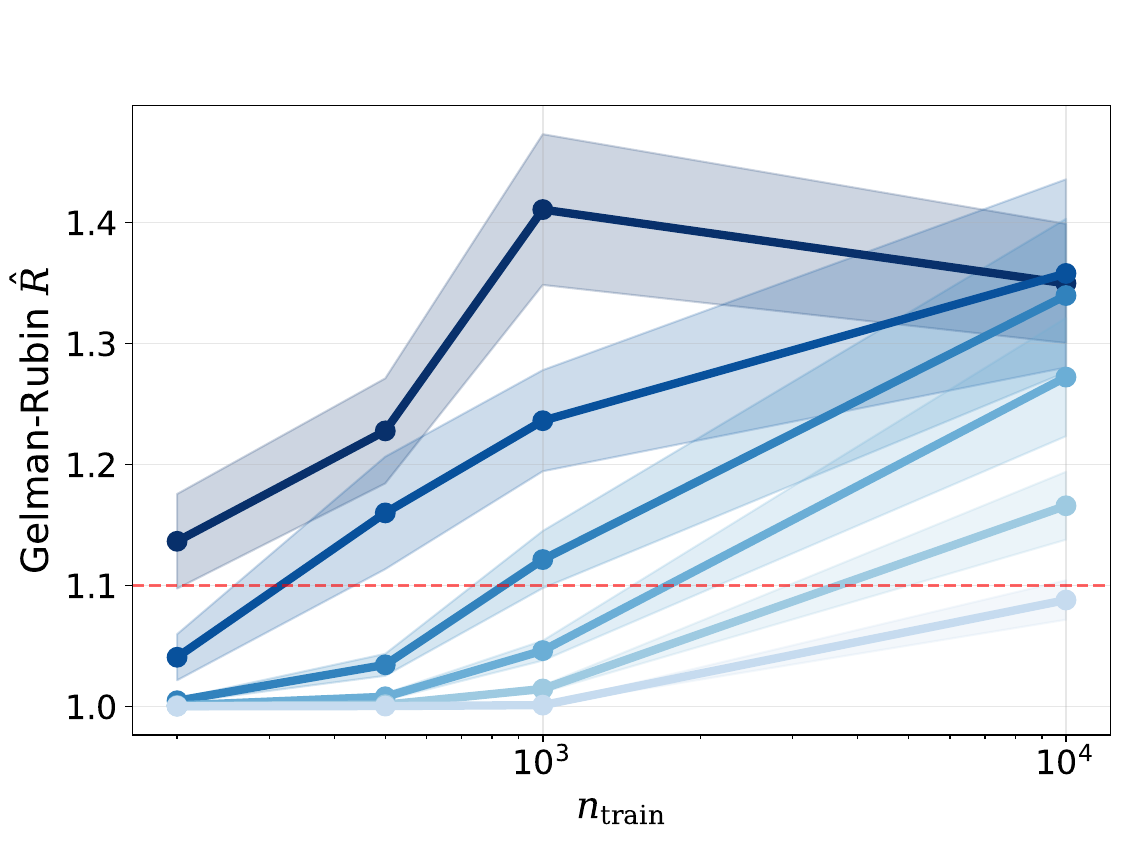}%
        \hfill
        \includegraphics[width=0.32\linewidth]{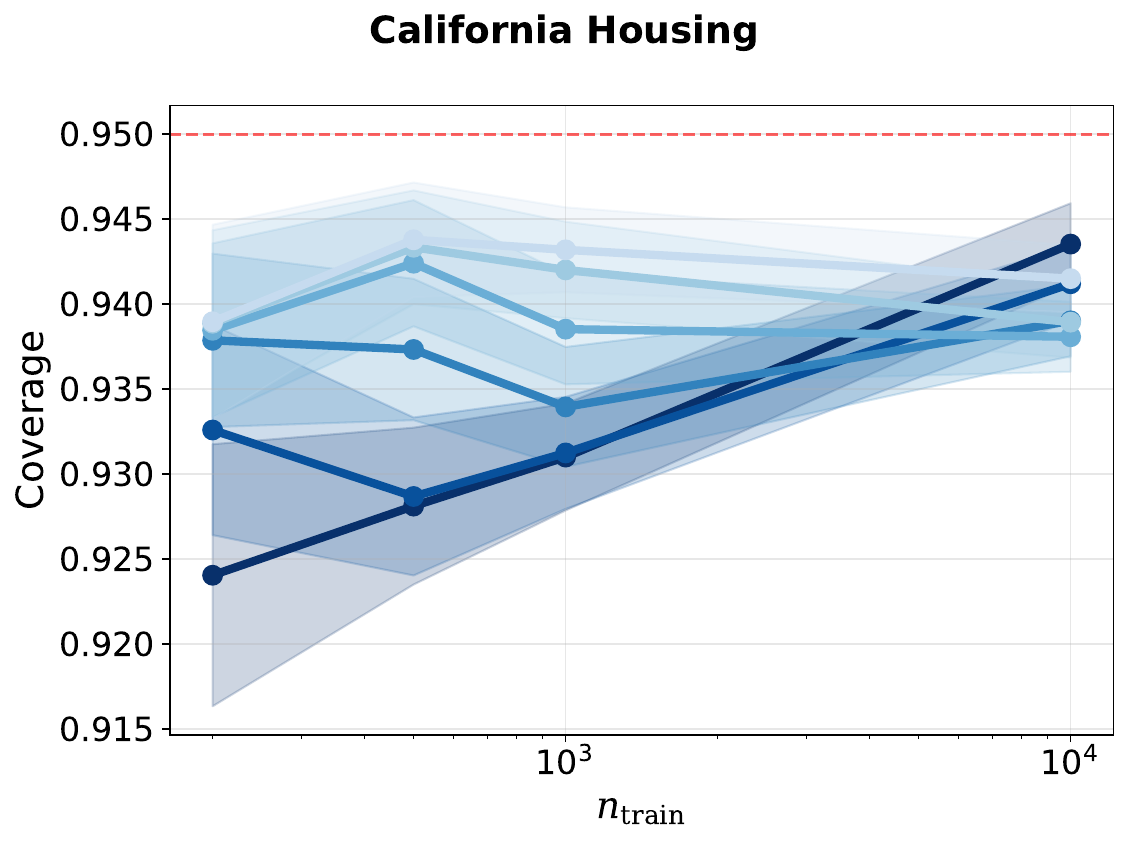}%
        \hfill
        \includegraphics[width=0.32\linewidth]{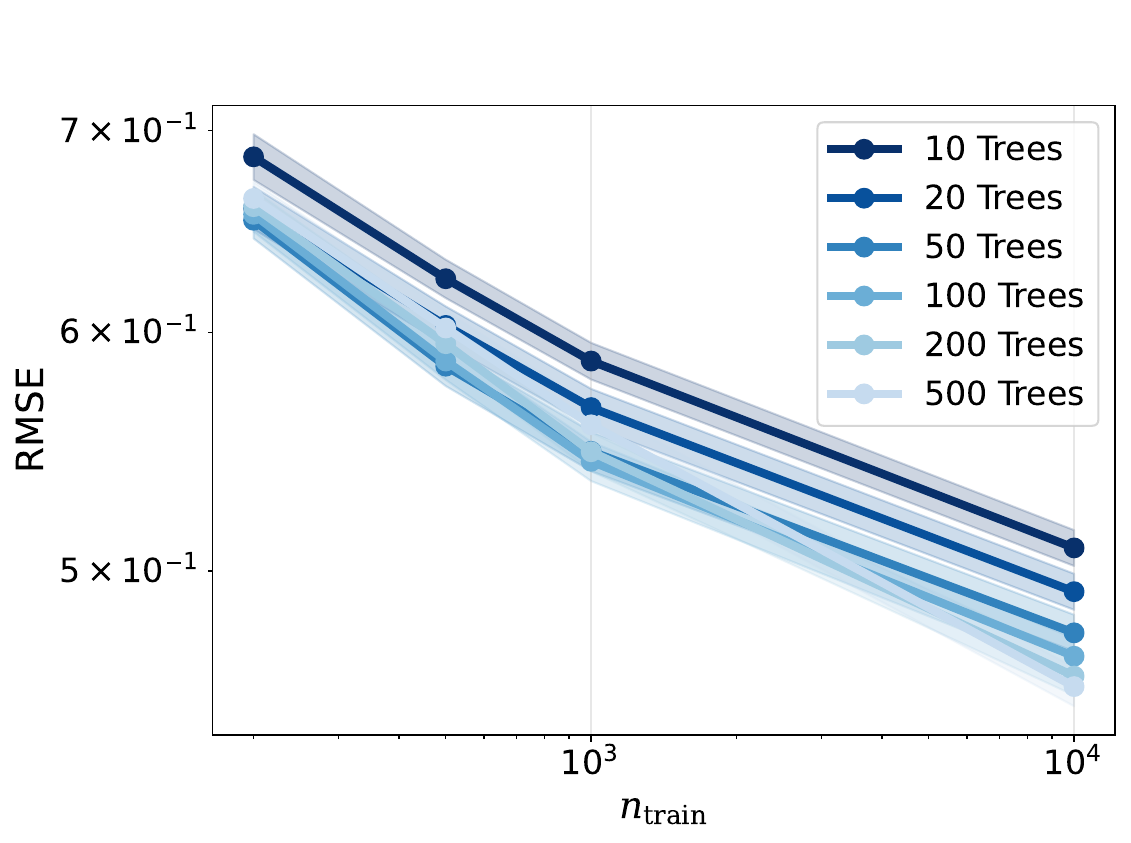}
    \end{minipage} \\
    \begin{minipage}{0.9\textwidth}
        \centering
        \includegraphics[width=0.32\linewidth]{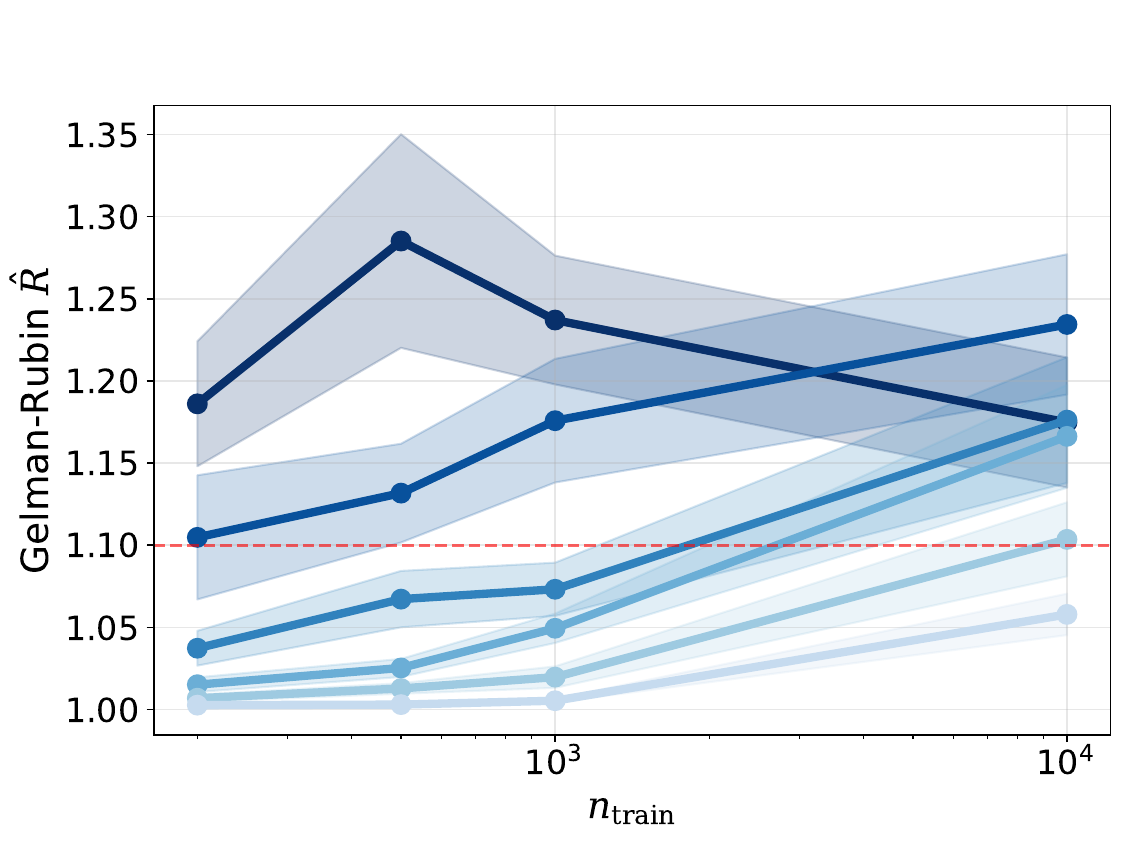}%
        \hfill
        \includegraphics[width=0.32\linewidth]{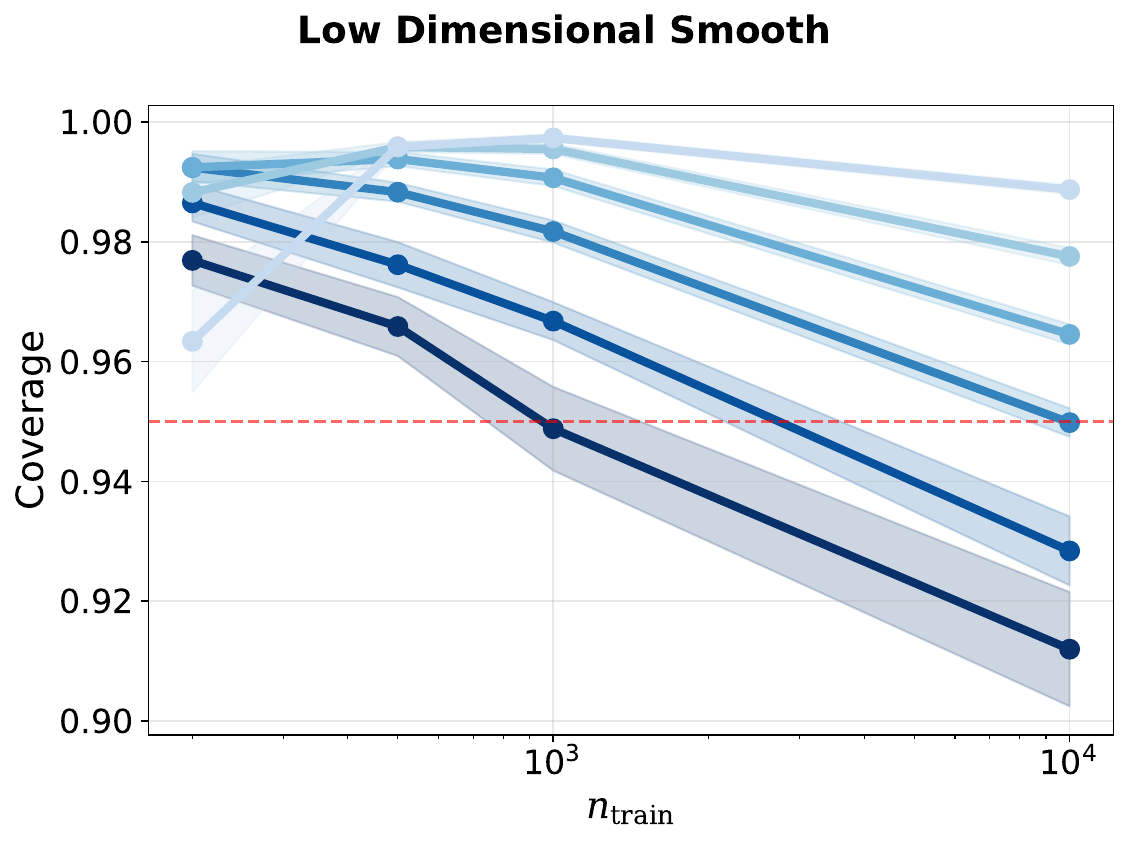}%
        \hfill
        \includegraphics[width=0.32\linewidth]{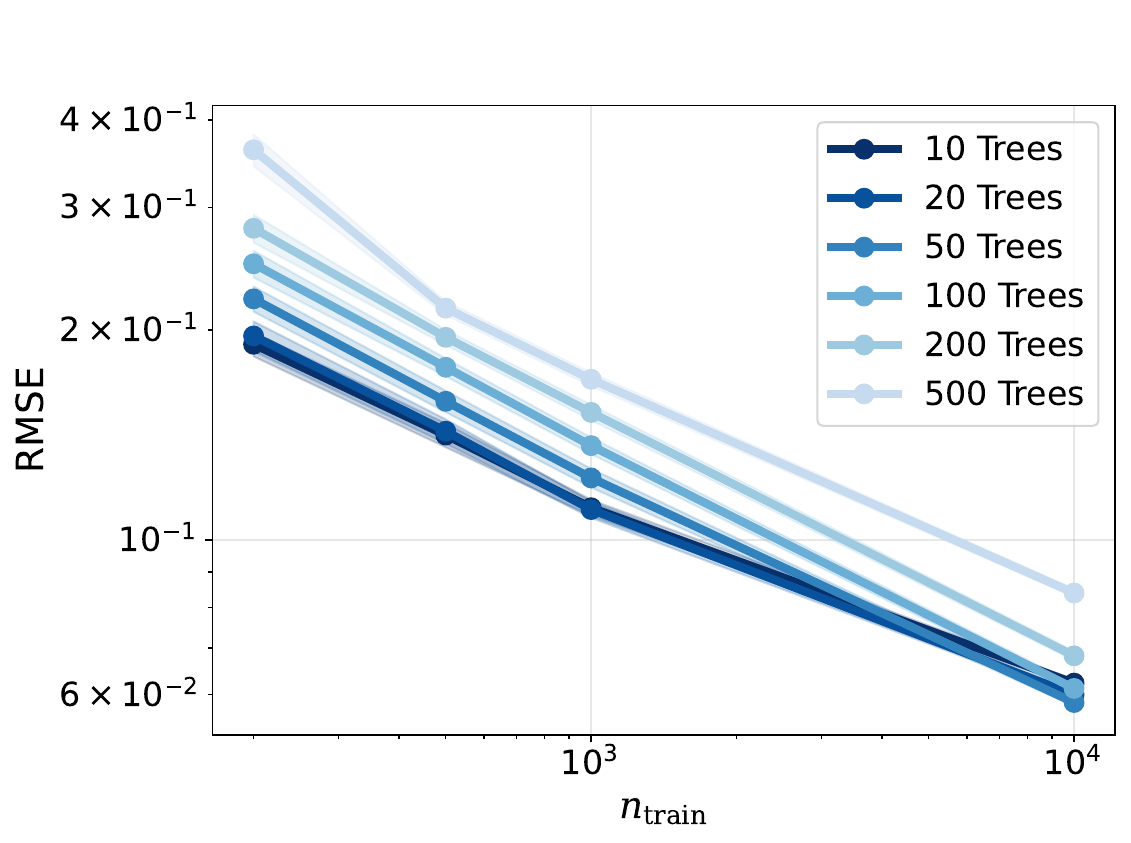}
    \end{minipage} \\
    \begin{minipage}{0.9\textwidth}
        \centering
        \includegraphics[width=0.32\linewidth]{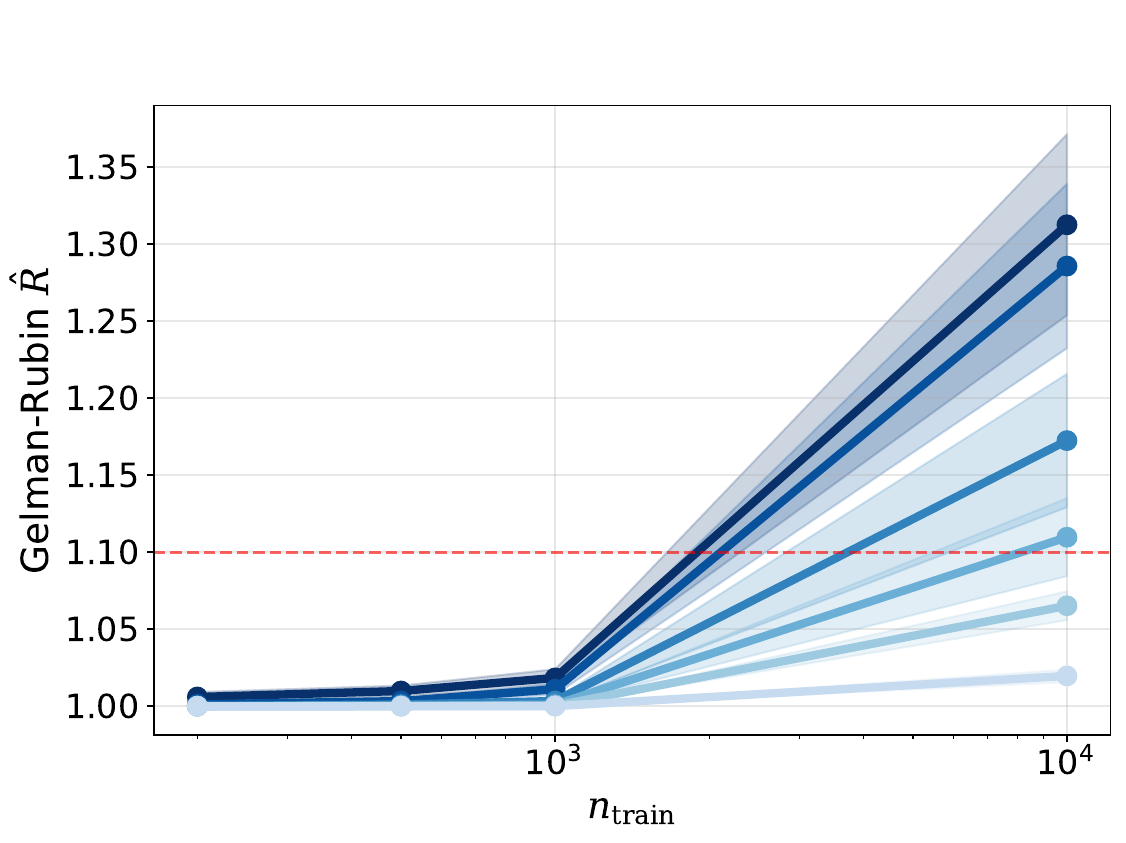}%
        \hfill
        \includegraphics[width=0.32\linewidth]{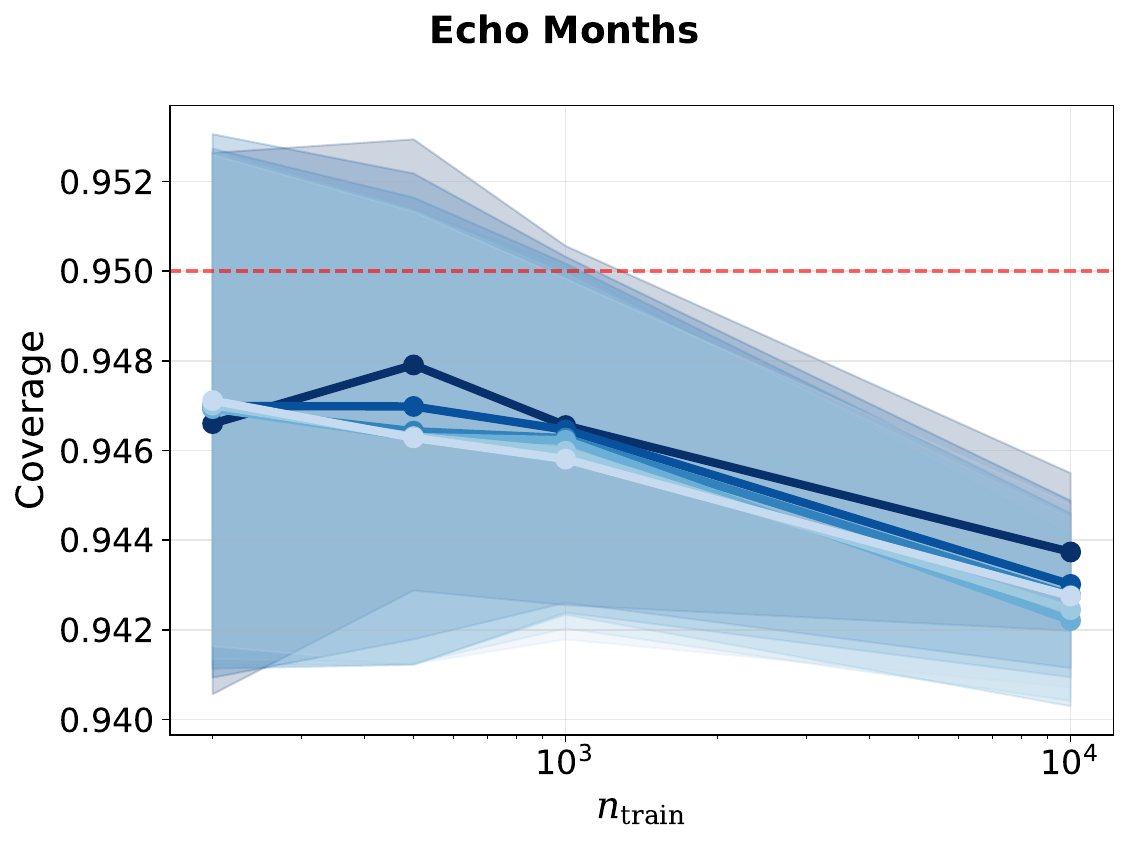}%
        \hfill
        \includegraphics[width=0.32\linewidth]{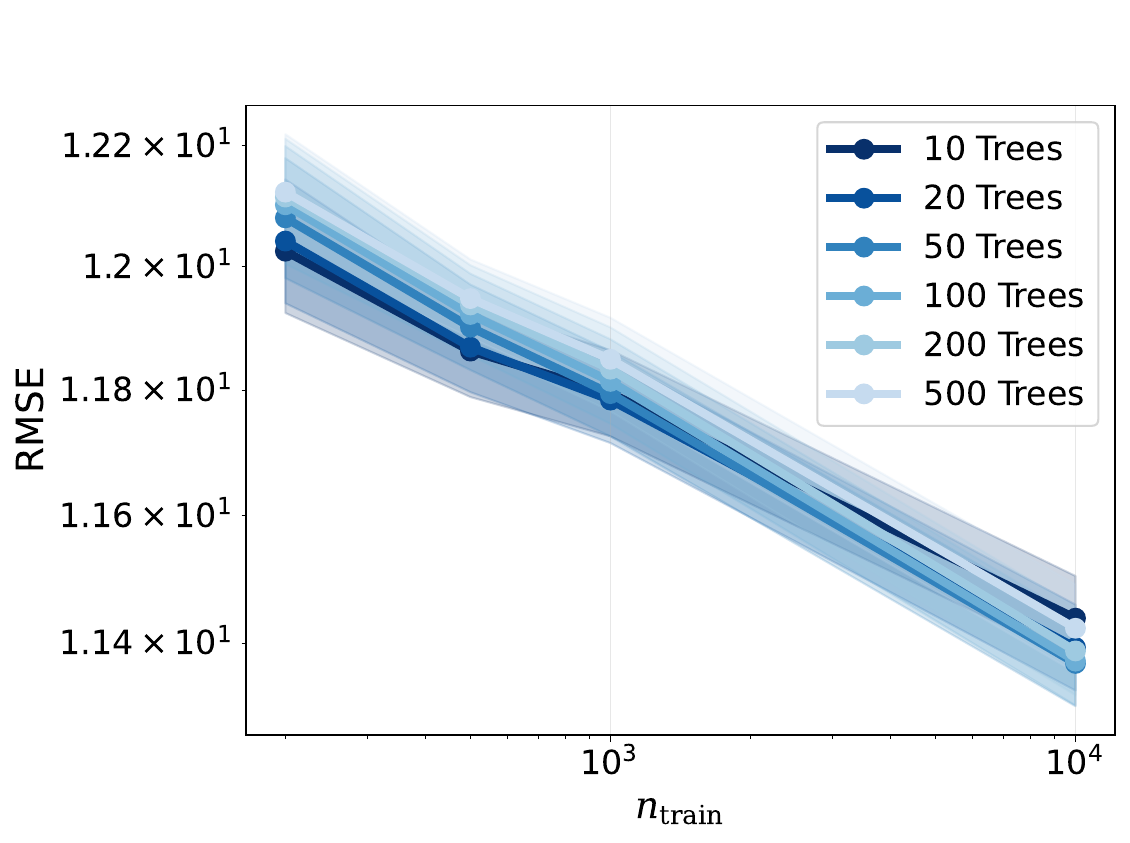}
    \end{minipage} \\
    \end{tabular}
    \caption{
    Values for Gelman-Rubin $\hat R$ (left), coverage (center), and RMSE (right) for the BART sampler under different numbers of trees in the BART model.
    Results are plotted for the California Housing (top), Low Dimensional Smooth (middle), and Echo Months (bottom) datasets.
    Error bars represent $\pm 1.96$ standard errors from 25 replicates.
    }
    \label{fig:experiment2}
\end{figure}

The results of the experiment for three of the six datasets are shown in Figure \ref{fig:experiment2}, with the rest of the results deferred to Figure \ref{fig:experiment2_additional} in the appendix.
Increasing the number of trees consistently dampens the increasing trend for $\hat R$, which provides empirical evidence to support Theorem \ref{thm:additive} and Theorem \ref{thm:mixing_upper_bound_ntrees}.
The results for coverage are highly ambiguous.
Meanwhile, the number of trees achieving the optimal RMSE seems to depend on the dataset and is sometimes achieved at intermediate values of $m$.

%% file: AOS/a0_posterior_conc_literature.tex
\section{Literature summary for BART posterior concentration results}

We list here all results on posterior concentration that we are aware of:

\begin{longlist}
    \item Posterior concentration at the optimal $n^{-\alpha/(2\alpha + p)}$ rate for BCART and BART when the regression function is H{\"o}lder $\alpha$-smooth, $0 < \alpha \leq 1$ \citep{rovckova2019theory}.
    \item Posterior concentration at the optimal $n^{-\alpha/(2\alpha + s)}$ rate for BCART and BART when the regression function is H{\"o}lder $\alpha$-smooth, $0 < \alpha \leq 1$, and $s$-sparse \citep{rovckova2020posterior}.
    \item Posterior concentration at the almost optimal $n^{-\alpha/(2\alpha + 1)}\log n$ rate for BART when the regression function is an additive sum of univariate components, each of which is H{\"o}lder $\alpha$-smooth, $0 < \alpha \leq 1$ \citep{rovckova2020posterior}.
    \item Extensions of the above results to $\alpha > 1$, provided BART is modified to allow for ``soft'' splits \citep{linero2018bayesian}.
    \item Posterior concentration at the optimal rate (up to log factors) for BART when the regression function is piecewise heterogeneous anisotropic H{\"o}lder smooth \citep{jeong2020art,rockova2021ideal}.
\end{longlist}

%% file: AOS/a1_bic_conc.tex
\section{Proofs for Section \ref{subsec:bic_concentration}}

We first introduce some additional notation that will be used throughout the rest of the appendix.
We denote $\maxf = \norm{f^*}_\infty$ and let $\noisesg$ denote the sub-Gaussian parameter of the noise random variable $\epsilon$ \citep{wainwright2019high}.
To augment the asymptotic notation ($O_\P$ and $\Omega_\P$) comparing sequences of random variables described in Section 4, we introduce asymptotic notation comparing sequences of fixed numbers, $(a_n)$ and $(b_n)$.
We say that $a_n = O(b_n)$ (equivalently $b_n = \Omega(a_n)$ if there exists some $C > 0$ such that $\sup_n \abs*{a_n/b_n} \leq C$.
With the exception of $n$ and $\delta$, $C$ will be allowed to to depend on any other parameters of the both the BART algorithm as well as the data to which it is fitted.
In our calculations, we will also use $C$ to denote constants satisfying this rule, with the warning that they may change from line to line.
We further introduce two slightly more refined versions of Proposition \ref{prop:concentration_bic_diff} and Proposition \ref{prop:lml_bic}, which we will use to establish our subsequent results.

\begin{proposition}[Concentration of BIC differences]
    \label{prop:concentration_bic_diff_appendix}
    Consider two TSEs $\tse$ and $\tse'$ and denote the difference in their BIC values as $\Delta \bic(\tse,\tse') \coloneqq \bic(\tse) - \bic(\tse')$.
    Then for any $0 < \delta < 1/2$, with probability at least $1-\delta$ with respect to $\P_n$, we have
    \begin{equation} \label{eq:bic_diff_eq1}
        \Delta \bic(\tse,\tse') = \frac{n}{\sigma^2}\paren*{\bias(\tse;f^*) - \bias(\tse';f^*)} + 
        O\paren*{\sqrt{n\log(1/\delta)} + \log(1/\delta)}.
    \end{equation}
    If furthermore, both TSEs have the same bias, i.e. $\Pi_{\tse}[f^*] = \Pi_{\tse'}[f^*]$, then we have
    \begin{equation} \label{eq:bic_diff_eq2}
        \Delta \bic(\tse,\tse') = \paren*{\df(\tse) - \df(\tse')}\log n + O\paren*{\log(1/\delta)}.
    \end{equation}
\end{proposition}

\begin{proposition}[Log marginal likelihood and BIC]
    \label{prop:lml_bic_appendix}
    Consider a TSE $\tse$.
    Then for any $0 < \delta < 1$, there is a minimal sample size $N$ so that for all $n \geq N$, with probability at least $1-\delta$ with respect to $\P_n$, the log marginal likelihood satisfies
    \begin{equation} \nonumber
        \log p(\by|\tse) = -\frac{\bic(\tse)}{2} +  O\paren*{1}.
    \end{equation}
    Consequently, the log marginal posterior also satisfies
    \begin{equation} \nonumber
        \log p(\tse|\by) = -\frac{\bic(\tse)}{2} - \log p(\by) + O\paren*{1}.
    \end{equation}
\end{proposition}

\begin{remark} \label{rem:lml_bic}
    The proof of Proposition \ref{prop:lml_bic_appendix} will show that, in fact, the $O(1)$ terms have an absolute bound that is independent of $\delta$.
\end{remark}

\subsection{Main proofs}

\begin{lemma}[Concentration of empirical risk difference]\label{lem:con_emp_rsk_diff}
    Let $\tse$ and $\tse'$ be two partition ensemble models.
    Then for any $0 < \delta < 1$, with probability at least $1-\delta$, the risk difference between them satisfies
    \begin{equation} \label{eq:con_emp_rsk_diff1}
        \abs*{\by^{\top}\paren*{\bP_\tse - \bP_{\tse'}}\by - n\int \paren*{\Pi_\tse f^*}^2 - \paren*{\Pi_{\tse'} f^*}^2 d\xmeasure} = O\paren*{\sqrt{n\log(1/\delta)}}.
    \end{equation}
    If furthermore, $\Pi_\tse f^* = \Pi_{\tse'}f^*$, then the above bound can be improved to
    \begin{equation} \label{eq:con_emp_rsk_diff2}
        \abs*{\by^{\top}\paren*{\bP_\tse - \bP_{\tse'}}\by } = O\paren*{\log(1/\delta)}.
    \end{equation}
\end{lemma}

\begin{proof}[Proof of Proposition \ref{prop:concentration_bic_diff_appendix}]
    Since
    \begin{equation} \nonumber
        \Delta \bic(\tse,\tse') = \frac{\by^{\top}\paren*{\bP_{\mathcal{\tse}} - \bP_{\tse'}}\by}{\sigma^2} + (\df(\tse)-\df(\tse'))\log n,
    \end{equation}
    the desired concentration follows immediately from Lemma \ref{lem:con_emp_rsk_diff}.
\end{proof}

\begin{lemma}[Log marginal likelihood formula] \label{lem:log_marg_lklhd_formula}
    Let $\tse$ be a tree ensemble structure (TSE), and let $\bPsi$ be an $\nsamples \times \nleaves$ matrix whose columns comprise the indicators of the leaves in $\tse$.
    The log marginal likelihood satisfies
    \begin{align} \label{eq:log_marg_lklhd_formula}
        -2\log p(\by|\tse) & = n\log\paren*{2\pi\sigma^2} + \log\det\paren*{\lambda^{-1}\bPsi^{\top}\bPsi + \bI} \nn
        & \quad\quad + \frac{1}{\sigma^2}\paren*{\norm*{\bPsi\hat\leafparamvec_{LS} - \by}_2^2 +  \hat\leafparamvec_{LS}^{\top}\bPsi^{\top}\paren*{\bI- \bPsi\paren*{\bPsi^{\top}\bPsi+\lambda\bI}^{-1}\bPsi^{\top}}\bPsi\hat\leafparamvec_{LS}},
    \end{align}
    where
    \begin{equation}\nonumber
        \hat\leafparamvec_{LS} \coloneqq \argmin_{\leafparamvec} \norm{\bPsi\leafparamvec - \by}_2^2
    \end{equation}
    is the solution to the least squares problem.
\end{lemma}

\begin{lemma}[Concentration of log-determinent]\label{lem:logdet_concentration}
    With the notation of Lemma \ref{lem:log_marg_lklhd_formula},
    denote $\hat\bSigma \coloneqq \frac{1}{n}\bPsi^{\top}\bPsi$ and $\bSigma \coloneqq \E\braces*{\hat\bSigma}$.
    Then for any $0 < \delta < 1$, for
    \begin{equation} \nonumber
        n \geq \max\braces*{64m^{3/2}s_{\min}^{-2}\log(2\df(\tse)/\delta),4\lambda s_{\min}^{-1}},
    \end{equation}
     with probability at least $1-\delta$, we have
    \begin{equation} \label{eq:logdet_concentration}
        \abs*{\log\det\left(\lambda^{-1}\bPsi^{\top} \bPsi + \bI\right) - \df(\tse)\log(n) - \sum_{i=1}^{\df(\tse)} \log(s_i/\lambda)} = O\paren*{\sqrt{\frac{\log(1/\delta)}{n}}},
    \end{equation}
    where $s_1,s_2,\ldots,s_{\df(\tse)}$ denote the values of the nonzero eigenvalues of $\bSigma$.
\end{lemma}

\begin{lemma}[Concentration of error term] \label{lem:err_term_conc}
    With the notation of Lemma \ref{lem:logdet_concentration}, for any $0 < \delta < 1$, when $n \geq 16\ntrees^{3/2}s_{\min}^{-2}\log(2\df(\tse)/\delta)$, where $s_{\min}$ is the minimum nonzero eigenvalue of $\bSigma$, with probability at least $1-\delta$, we have
    \begin{equation} \label{eq:err_term_conc}
        \hat\leafparamvec_{LS}^{\top}\bPsi^{\top}\paren*{\bI- \bPsi\paren*{\bPsi^{\top}\bPsi+\lambda\bI}^{-1}\bPsi^{\top}}\bPsi\hat\leafparamvec_{LS} = O\paren*{1}.
    \end{equation}
\end{lemma}

\begin{proof}[Proof of Proposition \ref{prop:lml_bic_appendix}.]
Starting with equation \eqref{eq:log_marg_lklhd_formula} from Lemma \ref{lem:log_marg_lklhd_formula}, plug in equations \eqref{eq:logdet_concentration} and \eqref{eq:err_term_conc} from Lemma \ref{lem:logdet_concentration} and Lemma \ref{lem:err_term_conc} respectively.
Notice that $\bPsi\hat\bmu_{LS} = \bP_\tse\by$, so that
\begin{equation}\nonumber
    \norm*{\bPsi\hat\leafparamvec_{LS} - \by}_2^2 = \norm*{\paren*{\bI-\bP_\tse}\by}_2^2 = \by^{\top}\paren*{\bI-\bP_\tse}\by.
\end{equation}
This completes the proof.
\end{proof}

\begin{proof}[Proof of Proposition \ref{prop:posterior_conc_treespace}.]
Let $\tse^*\in\opt(f^*, 0)$ and $\tse \notin \opt(f^*, 0)$.
Using Proposition \ref{prop:lml_bic} and Proposition \ref{prop:concentration_bic_diff}, we get 
\begin{align}
    & \log p(\tse^*|\by) - \log p(\tse|\by) \nn
    & \quad = \begin{cases}
        \frac{n}{2\sigma^2}\paren*{\bias(\tse;f^*)} + O_\P\paren*{n^{-1/2}} & \text{if}~ \bias(\tse;f^*) \neq 0 \\
        \frac{\log n}{2}\paren*{\df(\tse) - \df(\tse^*)} + O_\P\paren*{1} & \text{otherwise}.\nonumber
    \end{cases}
\end{align}
In either case, we exponentiate to get
\begin{equation}
    p(\tse|\by)\leq \frac{p(\tse|\by)}{p(\tse^*|\by)} = O_\P(n^{-1/2}).
\end{equation}
Since there are only finitely many $\tse \notin \opt(f^*, 0)$, taking an intersection of the ``error'' events and adding up the inequalities over $\tse \in \opt(f^*, 0)^c$ gives
\begin{equation}
    p(\opt(f^*, 0)^c~|~\by) = O_\P(n^{-1/2}),
\end{equation}
which is our desired result.
\end{proof}

\subsection{Further details}

\begin{proof}[Proof of Lemma \ref{lem:con_emp_rsk_diff}.]
    Recall that $\Pi_\tse$ refers to orthogonal projection onto $\mapone(\tse)$ in $L^2(\xmeasure)$, while $\bP_{\tse}$ refers to orthogonal projection onto $\mapone(\tse)$ with respect to the empirical norm $\norm{\cdot}_n$.
    With this in mind, decompose $y = \epsilon + (f^*(x)-\Pi_{\tse}f^*(x)) + \Pi_{\tse}f^*(x)$ and write this in vector form as $\by = \beps + \paren*{\bf^* - \bf^*_\tse} + \bf^*_\tse$.
    Since $\Pi_{\tse}f^* \in \mapone(\tse)$, we have $\bP_\tse \bf^*_\tse = \bf^*_\tse$.
    We can then therefore expand the quadratic form $\by^{\top}\bP_\tse\by$ as follows:
    \begin{align} \label{eq:impurity_quadratic_form_expansion}
    \begin{split}
        \by^{\top}\bP_\tse\by 
        & = \beps^{\top}\bP_\tse\beps + (\bf^* - \bf^*_{\tse})^{\top}\bP_\tse(\bf^* - \bf^*_{\tse}) + 2\beps^{\top}\bP_{\tse}(\bf^* - \bf^*_{\tse}) + (\bf^*_{\tse})^{\top}\bP_\tse\bf^*_{\tse} \\
        & \quad\quad + 2\beps^{\top}\bP_{\tse}\bf^*_{\tse} + 2(\bf^* - \bf^*_{\tse})^{\top}\bP_{\tse}\bf^*_{\tse} \\
        & = \beps^{\top}\bP_\tse\beps + (\bf^* - \bf^*_{\tse})^{\top}\bP_\tse(\bf^* - \bf^*_{\tse}) + 2\beps^{\top}\bP_{\tse}(\bf^* - \bf^*_{\tse}) + (\bf^*_{\tse})^{\top}\bf^*_{\tse} \\
        & \quad\quad + 2\beps^{\top}\bf^*_{\tse} + 2(\bf^* - \bf^*_{\tse})^{\top}\bf^*_{\tse}.
    \end{split}
    \end{align}
    Note that $\epsilon$, $(f^*(x)-\Pi_{\tse}f^*(x))$, and $\Pi_{\tse}f^*(x)$ are uncorrelated random variables, with $\epsilon$ being also independent of the other two variables.
    This implies that the third, fifth and sixth terms in \eqref{eq:impurity_quadratic_form_expansion} have zero mean.
    On the other hand, because of finite sample fluctuations, $(\bf^* - \bf^*_{\tse})$ and $\bf^*_{\tse}$, are not necessarily orthogonal as vectors.

    To bound the expectation of \eqref{eq:impurity_quadratic_form_expansion}, first observe that $f^* - \Pi_\tse f^*$ and $\Pi_\tse f^*$ are bounded random variables and thus have both standard deviation and sub-Gaussian norm bounded by $\maxf$ \citep{wainwright2019high}.
    We then compute
    \begin{equation} \label{eq:impurity_quadratic_form_expectation}
    \begin{split}
        & \E\braces*{\beps^{\top}\bP_\tse\beps + (\bf^* - \bf^*_{\tse})^{\top}\bP_\tse(\bf^* - \bf^*_{\tse}) + 2\beps^{\top}\bP_{\tse}(\bf^* - \bf^*_{\tse}) + (\bf^*_{\tse})^{\top}\bf^*_{\tse} + 2\beps^{\top}\bf^*_{\tse} + 2(\bf^* - \bf^*_{\tse})^{\top}\bf^*_{\tse}} \\
        =~ & \E\braces*{\beps^{\top}\bP_\tse\beps} + \E\braces*{(\bf^* - \bf^*_{\tse})^{\top}\bP_\tse(\bf^* - \bf^*_{\tse})} + \E\braces*{(\bf^*_{\tse})^{\top}\bf^*_{\tse}}  \\
        =~ & \trace\braces*{\bP_\tse}\paren*{\Var\braces*{\epsilon} + \Var\braces*{f^*-f^*_\tse}} + n\int \paren*{\Pi_\tse f^*}^2d\nu \\
        \leq~ & C\paren*{\noisesg + \maxf}\df(\tse) + n\int \paren*{\Pi_\tse f^*}^2d\nu. 
        \end{split}
    \end{equation}
    
    Next, we bound the fluctuations of each term in \eqref{eq:impurity_quadratic_form_expansion} separately.
    Using the Hanson-Wright inequality \citep{wainwright2019high,vershynin2018high}, we get $1-\delta$ probability events over which
    \begin{equation} \label{eq:conc_emp_risk_eq1}
        \abs*{\beps^{\top}\bP_\tse\beps - \E\braces*{\beps^{\top}\bP_\tse\beps}} \leq C\noisesg\max\braces*{\log(1/\delta), \sqrt{\df(\tse)\log(1/\delta)}},
    \end{equation}
    and 
    \begin{equation} \label{eq:conc_emp_risk_eq2}
        \abs*{(\bf^* - \bf^*_{\tse})^{\top}\bP_\tse(\bf^* - \bf^*_{\tse}) - \E\braces*{(\bf^* - \bf^*_{\tse})^{\top}\bP_\tse(\bf^* - \bf^*_{\tse})}} \leq C\maxf\max\braces*{\log(1/\delta), \sqrt{\df(\tse)\log(1/\delta)}}.
    \end{equation}
    Using Hoeffding's inequality \citep{wainwright2019high}, we have further $1-\delta$ probability events over which
    \begin{equation} \label{eq:conc_emp_risk_eq3}
        \abs*{(\bf^*_{\tse})^{\top}\bf^*_{\tse} - n\int \paren*{\Pi_\tse f^*}^2d\nu} \leq C\maxf^2\sqrt{n\log(1/\delta)},
    \end{equation}
    \begin{equation} \label{eq:conc_emp_risk_eq4}
        \abs*{\beps^{\top}\bf_\tse^*} \leq C\maxf \noisesg\sqrt{n\log(1/\delta)},
    \end{equation}
    \begin{equation} \label{eq:conc_emp_risk_eq5}
        \abs*{\paren*{\bf^*-\bf^*_\tse}^{\top}\bf_\tse^*} \leq C\maxf^2\sqrt{n\log(1/\delta)}.
    \end{equation}
    For the third term in \eqref{eq:impurity_quadratic_form_expansion}, we use Cauchy-Schwarz followed by Young's inequality to get
    \begin{align} \label{eq:conc_emp_risk_eq6}
        2\abs*{\beps^{\top}\bP_{\tse}(\bf^* - \bf^*_{\tse})} & \leq 2\paren*{\beps^{\top}\bP_{\tse}\beps}^{1/2}\paren*{(\bf^* - \bf^*_{\tse})^{\top}\bP_\tse(\bf^* - \bf^*_{\tse})}^{1/2} \nn
        & \leq \beps^{\top}\bP_{\tse}\beps + (\bf^* - \bf^*_{\tse})^{\top}\bP_\tse(\bf^* - \bf^*_{\tse}).
    \end{align}
    
    Conditioning on all the events guaranteeing \eqref{eq:conc_emp_risk_eq1} to \eqref{eq:conc_emp_risk_eq5} and plugging in these bounds together with \eqref{eq:impurity_quadratic_form_expectation} into \eqref{eq:impurity_quadratic_form_expansion}, we get
    \begin{align}
        & \abs*{\by^{\top}\bP_\tse\by - n\int \paren*{\Pi_\tse f^*}^2d\nu} \nn
        \leq~ & C\paren*{\paren*{\noisesg + \maxf}\df(\tse) + \paren*{\maxf + \noisesg}\log(1/\delta) + \maxf(\maxf + \noisesg)\sqrt{n\log(1/\delta)}} \nn
        & \quad\quad + C\paren*{(\maxf + \noisesg)\sqrt{\df(\tse)\log(1/\delta)}} \nn
        =~ &  O\paren*{\sqrt{n\log(1/\delta)} + \log(1/\delta)} \nonumber.
    \end{align}
    Repeating the same argument for $\tse'$ and adjusting $\delta$ so that the intersection of all events conditioned on has probability at least $1-\delta$ completes the proof of \eqref{eq:con_emp_rsk_diff1}.

    To prove \eqref{eq:con_emp_rsk_diff2}, observe that under the additional assumption, we can cancel terms in \eqref{eq:impurity_quadratic_form_expansion} to get
    \begin{equation} \nonumber
        \by^{\top}\paren*{\bP_\tse - \bP_{\tse'}}\by = \beps^{\top}\paren*{\bP_\tse - \bP_{\tse'}}\beps + (\bf^* - \bf^*_{\tse})^{\top}\paren*{\bP_\tse-\bP_{\tse'}}(\bf^* - \bf^*_{\tse}) + 2\beps^{\top}\paren*{\bP_\tse-\bP_{\tse'}}(\bf^* - \bf^*_{\tse}).
    \end{equation}
    Applying \eqref{eq:conc_emp_risk_eq6} followed by \eqref{eq:conc_emp_risk_eq1} and \eqref{eq:conc_emp_risk_eq2} completes the proof.
\end{proof}
\begin{proof}[Proof of Lemma \ref{lem:log_marg_lklhd_formula}.]
    Recall that the full log likelihood satisfies
    \begin{equation} \nonumber
        p(\by|\bX, \tse,\leafparamvec ) = \paren*{2\pi\sigma^2}^{-n/2}\exp\paren*{-\frac{\norm*{\bPsi\leafparamvec - \by}_2^2}{2\sigma^2}},
    \end{equation}
    while conditioned on $\tse$, the prior on $\leafparamvec$ satisfies
    \begin{equation} \nonumber
        p(\leafparamvec|\tse) = \paren*{2\pi\sigma^2\lambda^{-1}}^{-\nleaves/2}\exp\paren*{-\frac{\lambda\norm{\leafparamvec}_2^2}{2\sigma^2}},
    \end{equation}
    where $\nleaves$ is the number of columns in $\bPsi$.
    Hence
    \begin{equation} \label{eq:cond_lik_times_cond_prior}
        p(\by|\bX, \tse,\leafparamvec)p(\leafparamvec|\tse) = \paren*{2\pi\sigma^2}^{-n/2}\paren*{2\pi\sigma^2\lambda^{-1}}^{-\nleaves/2}\exp\paren*{-\frac{1}{2\sigma^2}\paren*{\norm*{\bPsi\leafparamvec - \by}_2^2 + \lambda\norm*{\leafparamvec}_2^2}}.
    \end{equation}
    
    Consider the orthogonal decomposition
    \begin{equation} \nonumber
        \norm*{\bPsi\leafparamvec - \by}_2^2 = \norm*{\bPsi\hat\leafparamvec_{LS} - \by}_2^2 + \norm*{\bPsi(\leafparamvec - \hat\leafparamvec_{LS})}_2^2.
    \end{equation}
    We next add the second term on the right to the exponent in the prior and complete the square:
    \begin{align} \label{eq:log_marg_lik_complete_square}
        & \norm*{\bPsi(\leafparamvec - \hat\leafparamvec_{LS})}_2^2 + \lambda\norm{\leafparamvec}_2^2 
        \nn
        = ~& \leafparamvec^{\top}\paren*{\bPsi^{\top}\bPsi + \lambda \bI}\leafparamvec - 2\paren*{\bPsi\hat\leafparamvec_{LS}}^{\top}\leafparamvec + \hat\leafparamvec_{LS}^{\top}\bPsi^{\top}\bPsi\hat\leafparamvec_{LS} \nn
        = ~&  \paren*{\leafparamvec - \leafparamvec_0}^{\top}\paren*{\bPsi^{\top}\bPsi + \lambda \bI}\paren*{\leafparamvec - \leafparamvec_0} - \hat\leafparamvec_{LS}^{\top}\bPsi^{\top}\bPsi\paren*{\bPsi^{\top}\bPsi+\lambda\bI}^{-1}\bPsi^{\top}\bPsi\hat\leafparamvec_{LS} + \hat\leafparamvec_{LS}^{\top}\bPsi^{\top}\bPsi\hat\leafparamvec_{LS}.
    \end{align}
    where
    \begin{equation}\nonumber
        \leafparamvec_0 = \paren*{\bPsi^{\top}\bPsi + \lambda\bI}^{-1}\bPsi^{\top}\bPsi\hat\leafparamvec_{LS}.
    \end{equation}
    The constant term in \eqref{eq:log_marg_lik_complete_square} is
    \begin{align} \label{eq:log_marg_lik_constant_term}
        & - \hat\leafparamvec_{LS}^{\top}\bPsi^{\top}\bPsi\paren*{\bPsi^{\top}\bPsi+\lambda\bI}^{-1}\bPsi^{\top}\bPsi\hat\leafparamvec_{LS} + \hat\leafparamvec_{LS}^{\top}\bPsi^{\top}\bPsi\hat\leafparamvec_{LS} \nn
        =~& \hat\leafparamvec_{LS}^{\top}\bPsi^{\top}\paren*{\bI- \bPsi\paren*{\bPsi^{\top}\bPsi+\lambda\bI}^{-1}\bPsi^{\top}}\bPsi\hat\leafparamvec_{LS}.
    \end{align}
    
    Plugging \eqref{eq:log_marg_lik_constant_term} back into \eqref{eq:cond_lik_times_cond_prior} and integrating, we get 
    \begin{align} \label{eq:marg_lik_partial_calc}
        p(\by|\bX, \tse) & = \int p(\by|\bX, \tse,\leafparamvec)p(\leafparamvec|\tse) d\leafparamvec \nn
        & = \paren*{2\pi\sigma^2}^{-n/2}\exp\paren*{-\frac{\norm*{\bPsi\hat\leafparamvec_{LS} -\by}_2^2 + \hat\leafparamvec_{LS}^{\top}\bPsi^{\top}\paren*{\bI- \bPsi\paren*{\bPsi^{\top}\bPsi+\lambda\bI}^{-1}\bPsi^{\top}}\bPsi\hat\leafparamvec_{LS}}{2\sigma^2}} \nn
        & \quad \cdot\int \paren*{2\pi\sigma^2 \lambda^{-1}}^{-\nleaves/2}\exp\paren*{-\frac{\paren*{\leafparamvec^{\top} - \leafparamvec_0}^{\top}\paren*{\bPsi^{\top}\bPsi + \lambda \bI}\paren*{\leafparamvec^{\top} - \leafparamvec_0}}{2\sigma^2}} d\leafparamvec.
    \end{align}
    By a change of variables, the integral can be computed as
    \begin{equation} \nonumber
        \int \paren*{2\pi\sigma^2 \lambda^{-1}}^{-\nleaves/2}\exp\paren*{-\frac{\paren*{\leafparamvec^{\top} - \leafparamvec_0}^{\top}\paren*{\lambda^{-1}\bPsi^{\top}\bPsi + \bI}\paren*{\leafparamvec^{\top} - \leafparamvec_0}}{2\sigma^2\lambda^{-1}}} d\leafparamvec = \det\paren*{\lambda^{-1}\bPsi^{\top}\bPsi + \bI}^{-1/2}.
    \end{equation}
    Plugging this back into \eqref{eq:marg_lik_partial_calc} and taking logarithms yields \eqref{eq:log_marg_lklhd_formula}.
\end{proof}

\begin{proof}[Proof of Lemma \ref{lem:logdet_concentration}]
Let $\hat s_1,\hat s_2,\ldots,\hat s_\nleaves$ be the eigenvalues of $\hat\Sigma$.
Using Lemma \ref{lem:rank_emp_cov}, we have $\hat s_i = 0$ for any $i > \df(\tse)$.
We may therefore compute
\begin{align} 
    \log\det\left(\lambda^{-1}\bPsi^{\top} \bPsi + \bI\right) & = \sum_{i=1}^{\df(\tse)} \log(n \lambda^{-1}\hat s_i + 1) \nn
    & = \sum_{i=1}^{\df(\tse)} \log\paren*{\frac{\hat s_i + \lambda/n}{s_i}} + \df(\tse)\log n + \sum_{i=1}^{\df(\tse)} \log (s_i/\lambda).\nonumber
\end{align}

It remains to bound the first term.
To this end, we first condition on the $1-\delta$ probability event guaranteed by Lemma \ref{lem:emp_cov_concentration}.
Then, we observe that
\begin{align} \label{eq:emp_cov_conc_helper}
    \abs*{\frac{\hat s_i + \lambda/n}{s_i} - 1} & \leq \frac{1}{s_i} \paren*{\abs*{\hat s_i - s_i} + \frac{\lambda}{n}} \nn
    & \leq \frac{1}{s_{\min}} \paren*{\norm*{\hat\bSigma - \bSigma} + \frac{\lambda}{n}} \nn
    & \leq \frac{1}{s_{\min}}\paren*{\max\braces*{\sqrt{\frac{4\ntrees^{3/2}\log(2\df(\tse)/\delta)}{n}}, \frac{4\sqrt{m}\log(2\df(\tse)/\delta)}{n}}  + \frac{\lambda}{n}}.
\end{align}
Here, the second inequality makes use of Weyl's inequality.

Recall the elementary inequality
\begin{equation}\nonumber
    \abs*{\log x} \leq 2\abs*{x-1}
\end{equation}
for $0 < x < 1/2$.
Using this together with \eqref{eq:emp_cov_conc_helper}, we get 
\begin{equation} \nonumber
    \sum_{i=1}^{\df(\tse)} \log\paren*{\frac{\hat s_i + \lambda/n}{s_i}} = O\paren*{\sqrt{\frac{\log (1/\delta)}{n}}}
\end{equation}
when $n \geq \max\braces*{64m^{3/2}s_{\min}^{-2}\log(2\df(\tse)/\delta),4\lambda s_{\min}^{-1}}$.
\end{proof}




\begin{proof}[Proof of Lemma \ref{lem:err_term_conc}.]
Recall that $\bPsi\hat\bmu_{LS} = \bP_\tse\by$.
We thus rewrite and bound the error term as
\begin{align} \label{eq:err_term_rewrite}
    \hat\leafparamvec_{LS}^{\top}\bPsi^{\top}\paren*{\bI- \bPsi\paren*{\bPsi^{\top}\bPsi+\lambda\bI}^{-1}\bPsi^{\top}}\bPsi\hat\leafparamvec_{LS} & = \by^{\top} \bP_\tse \bM\bP_\tse \by \nn
    & \leq \norm*{\bP_\tse\by}_2^2 \norm{\bM},
\end{align}
where
\begin{equation} \nonumber
    \bM = \bP_\tse - \bPsi\paren*{\bPsi^{\top}\bPsi+\lambda\bI}^{-1}\bPsi^{\top}.
\end{equation}

Using similar arguments as in the proof of Lemma \ref{lem:con_emp_rsk_diff}, we bound
\begin{align} \label{eq:err_term_eq1}
    \norm*{\bP_\tse\by}_2^2 & = (\bf_\tse^*)^{\top}\bf_\tse^* + \beps^{\top}\bP_\tse\beps \nn
    & \leq \maxf^2\paren*{n + \sqrt{n\log(1/\delta)}} + \noisesg^2\paren*{n + \sqrt{\df(\tse)\log(1/\delta)}}.
\end{align}
Meanwhile, the nonzero eigenvalues of $\bM$ are of the form
\begin{equation} \nonumber
    1 - \frac{\hat s_i}{\hat s_i + \lambda/n} = \frac{\lambda}{n\hat s_i + \lambda}
\end{equation}
for $i =1,2,\ldots,\df(\tse)$.
These can be further bounded as
\begin{align} \label{eq:err_term_eq2}
    \frac{\lambda}{n\hat s_i + \lambda} & \leq \frac{\lambda}{n(s_i - \abs*{\hat s_i - s_i})} \nn
    & \leq \frac{\lambda}{n\paren*{s_{\min} - \norm*{\hat \bSigma - \bSigma}}}.
\end{align}
Taking $n \geq 16\ntrees^{3/2}s_{\min}^{-2}\log(2\df(\tse)/\delta)$ and conditioning on the $1-\delta$ probability event guaranteed by Lemma \ref{lem:emp_cov_concentration}, we can further bound \eqref{eq:err_term_eq2} by $\lambda/2ns_{\min}$, which gives $\norm{\bM} = O\paren*{n^{-1}}$.
Combining this with \eqref{eq:err_term_eq1} and plugging them back into \eqref{eq:err_term_rewrite} finishes the proof.
\end{proof}

\begin{lemma}[Rank of empirical covariance matrix] \label{lem:rank_emp_cov}
    With the notation of Lemma \ref{lem:logdet_concentration}, we have $\rank\paren{\hat\bSigma} \leq \rank\paren*{\bSigma}$.
\end{lemma}

\begin{proof}
    Let $\bl_1,\bl_2,\ldots,\bl_\nsamples$ denote the rows of $\bPsi$, noting that they are i.i.d. random vectors. 
    Note that if $\bSigma\bv = 0$ for some $\bv$, this implies that $\Cov\braces*{\bv^{\top}\bl} = \bv^{\top}\bSigma\bv = 0$, and so $\bv^{\top}\bl \equiv 0$ as a random variable.
    In particular, we have $\bv^{\top}\bl_i = 0$ for $i=1,2\ldots,\nsamples$, and we also get $\hat\bSigma\bv = 0$.
    As such, the nullspace for $\bSigma$ is contained within the nullspace for $\hat\bSigma$.
    The conclusion follows.
\end{proof}

\begin{lemma}[Concentration of empirical covariance matrix]
\label{lem:emp_cov_concentration}
    With the notation of Lemma \ref{lem:logdet_concentration},
    for any $0 < \delta < 1$, with probability at least $1-\delta$, we have
    \begin{equation} \nonumber
        \norm*{\frac{1}{n}\bPsi^{\top}\bPsi - \bSigma}  \leq \max\braces*{\sqrt{\frac{4\ntrees^{3/2}\log(2\df(\tse)/\delta)}{n}}, \frac{4\sqrt{m}\log(2\df(\tse)/\delta)}{n}}.
    \end{equation}
\end{lemma}

\begin{proof}
    Let $\bl_1,\bl_2,\ldots,\bl_\nsamples$ denote the rows of $\bPsi$ as before.
    Since each point can only be contained in a single leaf on each tree, we have $\norm*{\bl_j}_2 = \sqrt{\ntrees}$, while $\bSigma$ also satisfies $\norm*{\bSigma} \leq m$.
    Using Corollary 6.20\footnote{While Corollary 6.20 is stated with the assumption that the rows have mean zero, the proof in \cite{wainwright2019high} illustrates that this is unnecessary.} in \cite{wainwright2019high}, we therefore have
    \begin{equation} \nonumber
        \norm*{\frac{1}{n}\bPsi^{\top}\bPsi - \bSigma} \leq 2\df(\tse)\exp\paren*{-\frac{nt^2}{2\sqrt{\ntrees}(\ntrees+t)}}
    \end{equation}
    for any $t  > 0$.
    Rearranging this equation completes the proof.
\end{proof}

%% file: AOS/a2_mc.tex
\section{Background on Markov chains}

For the whole of this section, let $X_0,X_1,\ldots$ be an irreducible and aperiodic discrete time Markov chain on a finite state space $\statespace$, with stationary distribution $\pi$.

\subsection{Networks and voltages}

\subsubsection{Harmonic functions}
Let $P$ be the transition kernel of $(X_t)$.
We call a function $h \colon \statespace \to \R$ harmonic for $P$ at a state $x$ if
\begin{equation} \label{eq:def_harmonic}
    h(x) = \sum_{y \in \statespace} P(x,y) h(y).
\end{equation}

\begin{lemma}[Uniqueness of harmonic extensions, Proposition 9.1. in \cite{LevinPeresWilmer2006}] \label{lem:harmonic_uniqueness}
    Let $\setone \subset \statespace$ be a subset of the state space.
    Let $h_{\setone} \colon \setone \to \R$ be a function defined on $\setone$.
    The function $h \colon \statespace \to \R$ defined by $h(x) \coloneqq \E\braces{h_{\setone}(X_{\tau_{\setone}})|X_0 = x}$ is the unique extension of $h_{\setone}$ such that $h(x) = h_{\setone}(x)$ for all $x \in \setone$ and $h$ is harmonic for $P$ at all $x \in \statespace \backslash\setone$.
\end{lemma}

The values of a harmonic function can be computed by solving the system of linear equations given by \eqref{eq:def_harmonic} for each $x \in \statespace \backslash\setone$.
This is hard to do directly by hand for complicated state spaces, but when the Markov chain is symmetric, i.e. the stationary distribution satisfies $\pi(x)P(x,y) = \pi(y)P(y,x)$ for any $x, y$, we can use several operations to \emph{simplify the state space while preserving the values of the harmonic function on the remaining state space}.
Indeed, harmonic functions are equivalent to voltages on electrical circuits, and it is well-known how to simplify circuits in order to calculate voltages:
\begin{longlist}
    \item (Gluing) Points on the circuit with the same voltage can be joined.
    \item (Series law) Two resistors in series with resistances $r_1$ and $r_2$ can be merged with the new resistor having resistance $r_1 + r_2$.
    \item (Parallel law) Two resistors in parallel with resistances $r_1$ and $r_2$ can be merged with the new resistor having resistance $1/(1/r_1 + 1/r_2)$.
\end{longlist}

We make this connection rigorous by introducing the following definitions and via the subsequent lemmas.

\subsubsection{Networks, conductance, resistance}
A \emph{network} $(\statespace, c)$ is a tuple comprising a finite state space $\statespace$ and a symmetric function $c \colon \statespace \times \statespace \to \R_+$ called the \emph{conductance}.
The \emph{resistance} function is defined as $r(x,y) = \frac{1}{c(x,y)}$, and can take the value of positive infinity.
We say that $\braces{x,y}$ is an \emph{edge} in the network if $c(x,y) > 0$.
Any network $(\statespace, c)$ has an associated Markov chain whose transition probabilities are defined by
$P(x,y) = \frac{c(x,y)}{\sum_{z \in \statespace}c(x,z)}$.

\subsubsection{Voltage and current flow}
We say that a function is harmonic on the network if it is harmonic with respect to $P$.
Given $a, z \in \statespace$, a \emph{voltage} $W$ between $a$ and $z$ is a function that is harmonic on $\statespace \backslash \braces{a, z}$.
The \emph{current flow} $I \colon \statespace \times \statespace \to \R$ associated with $W$ is defined as $I(x,y) = c(x,y)\paren*{W(x) - W(y)}$.
The strength of the current flow is defined as
\begin{equation} \nonumber
    \norm{I} \coloneqq \sum_{y \in \statespace} c(a,y)\paren*{W(a) - W(y)}.
\end{equation}

\subsubsection{Effective resistance and effective conductance}
The \emph{effective resistance} between $a$ and $z$ is defined as
\begin{equation} \nonumber
    R(a \leftrightarrow z) \coloneqq \frac{W(a)-W(z)}{\norm{I}},
\end{equation}
noting that this is independent of the choice of $W$ by the uniqueness property in Lemma \ref{lem:harmonic_uniqueness}.
The \emph{effective conductance} between $a$ and $z$ is defined as $C(a \leftrightarrow z) \coloneqq 1/R(a \leftrightarrow z)$.


\begin{lemma}[Network simplification rules] \label{lem:network_simplification}
    Consider a network $(\statespace, c)$.
    Define the following operations that each produces a modified network $(\statespace', c')$.
    \begin{longlist}
        \item (Gluing) Given $u, v \in \statespace$, define $\statespace' \coloneqq \statespace\backslash\braces{v}$ and
        \begin{equation} \nonumber
            c'(x,y) = \begin{cases}
              c(x,y) & x, y \neq u, \\
              c(x,u) + c(x,v) & x \neq u, y  = u, \\
              c(u,y) + c(v,y) & x = u, y \neq u \\
              c(u,v) + c(u,u) + c(v,v) & x = y = u.
            \end{cases}
        \end{equation}
        \item
        (Parallel and series laws) Given $u, v, w \in \statespace$ with $c(v,x) = 0$ for all $x \notin \braces{u,w}$, define $\statespace' \coloneqq \statespace\backslash\braces{v}$ and
        \begin{equation} \nonumber
            c'(x,y) = \begin{cases}
              c(u,w) + \frac{c(u,v)c(v,w)}{c(u,v) + c(v,w)} & (x,y) = (u,w)~\text{or}~ (w,u),  \\
              c(x,y)  & \text{otherwise.}
            \end{cases}
        \end{equation}
    \end{longlist}
    Consider a function $h$ that is harmonic on $\statespace\backslash\cal{A}$.
    The following hold:
    \begin{longlist}
        \item If we glue states $u, v \in \statespace$ such that $h(u) = h(v)$, then $h$ remains harmonic on $\statespace'\backslash\cal{A}$ with respect to the modified transition matrix $P'$.
        \item If we apply the parallel and series laws to $u, v,w \in \statespace$ with $v \notin \cal{A}$, then $h$ remains harmonic on $\statespace'\backslash\cal{A}$ with respect to the modified transition matrix $P'$.
    \end{longlist}
    Furthermore, if $h$ is a voltage between two states $a, z \in \statespace$, then applying the operations does not change the effective conductance and resistance between them.
\end{lemma}

\begin{proof}
    The first statement is obvious as the mean value equation for harmonic functions can be repeated almost verbatim.
    For the second statement, we just have to check the mean value equation for $h$ at $u$.
    This is equivalent to the equation
    \begin{equation} \label{eq:mean_value_modified}
        \sum_{x \in \statespace'}c'(u,x)\paren*{h(x) - h(u)} = 0.
    \end{equation}
    First note that under the original network, our assumption on $c(v,x)$ and the mean value property at $v$ gives
    \begin{equation} \label{eq:series_law_step}
        c(v,w)\paren*{h(w)-h(v)} = c(u,v)\paren*{h(v)-h(u)}.
    \end{equation}
    Next, we compute
    \begin{equation} \label{eq:current_flow}
    \begin{split}
        c'(u,w)\paren*{h(w) - h(u)} & = \paren*{c(u,w) +\frac{c(u,v)c(v,w)}{c(u,v) + c(v,w)}} \paren*{h(w) - h(u)} \\
        & = c(u,w)\paren*{h(w) - h(u)} \\
        & \quad\quad + \frac{c(u,v)c(v,w)}{c(u,v) + c(v,w)}\paren*{(h(w) - h(v)) + (h(v) - h(u)} \\
        & = c(u,w)\paren*{h(w) - h(u)} + \frac{c(u,v)^2 + c(u,v)c(v,w)}{c(u,v)+c(v,w)}\paren*{h(v) - h(u)} \\
        & = c(u,w)\paren*{h(w) - h(u)} + c(v,w)\paren*{h(v) - h(u)},
    \end{split}
    \end{equation}
    where the third equality follows from \eqref{eq:series_law_step}.
    We therefore have
    \begin{equation} \nonumber
        \sum_{x \in \statespace'}c'(u,x)\paren*{h(x) - h(u)} = \sum_{x \in \statespace}c(u,x)\paren*{h(x) - h(u)},
    \end{equation}
    and the mean value equation at $u$ for the original network implies \eqref{eq:mean_value_modified}.
    
    Finally, to conclude invariance of effective conductance, we observe that it is is defined in terms of voltages and current flows.
    We have already shown that voltages are unchanged, so we just need to argue that the strength of the current flow is similarly unchanged.
    This is immediate whenever $a \notin \braces{u,w}$.
    When $a = u$, this follows from \eqref{eq:current_flow}.
\end{proof}

\begin{lemma}[Rayleigh's Monotonicity Law, Theorem 9.12 in \cite{LevinPeresWilmer2006}] \label{lem:monotonicity}
    Given a network with two resistance functions $r, r' \colon \statespace \times \statespace \to \R_+\cup\braces{\infty}$, suppose that $r \leq r'$ pointwise.
    Then we have
    \begin{equation} \nonumber
        R(a \leftrightarrow z;r) \leq R(a \leftrightarrow z;r').
    \end{equation}
\end{lemma}

\subsection{Hitting precedence probabilities}
\label{subsec:hpp}

Let $\setone, \settwo \subset \statespace$ be disjoint subsets.
The \emph{hitting precedence probability} of $\setone$ relative to $\settwo$ is defined as the following function on $\statespace$:
\begin{equation} \nonumber
    \hpp(x;\setone,\settwo) \coloneqq P\braces*{\tau_{\setone} < \tau_{\settwo} ~|~ X_0 = x}.
\end{equation}


\begin{lemma}
    $\hpp(-;\setone,\settwo)$ is a harmonic function on $\statespace \backslash \paren*{\setone \cup \settwo}$.
\end{lemma}

\begin{proof}
    Write $h_{\setone \cup \settwo}(z) = \indicator\braces{z \in \setone}$.
    It is easy to see that
    \begin{equation} \nonumber
        h_{\setone \cup \settwo}(X_{\tau_{\setone\cup \settwo}}) = \begin{cases}
            1 & \tau_{\setone} < \tau_{\settwo} \\
            0 & \tau_{\setone} > \tau_{\settwo}.
        \end{cases}
    \end{equation}
    Hence,
    \begin{equation} \nonumber
        \E\braces*{h_{\setone \cup \settwo}(X_{\tau_{\setone\cup \settwo}})|X_0 = x} = \P\braces*{\tau_{\setone} < \tau_{\settwo} ~|~ X_0 = x} = \hpp(x;\setone,\settwo).
    \end{equation}
    By Lemma \ref{lem:harmonic_uniqueness}, the left hand side is a harmonic function on $\statespace \backslash \paren*{\setone \cup \settwo}$.
\end{proof}

\begin{lemma}[HPP from a bottleneck state]
    \label{lem:reduction_to_path}
    Given a network $(\statespace,c)$ with states $a, x, z$ and such that $c(u,z) = 0$ for all $u \notin \braces{x,z}$.
    Let $x = x_0,x_1,\ldots,x_k = a$ be any sequence of states such that there is some $\rho > 0$ for which $c(x_{i-1},x_{i}) \geq \rho^{-1} c(x,z)$ for $i=1,\ldots,k$.
    Then
    \begin{equation} \nonumber
        \P\braces{\tau_a < \tau_z|X_0 = x} \geq \frac{1}{k\rho+1}.
    \end{equation}
\end{lemma}

\begin{proof}
    Let $(\statespace,c')$ be the modified network in which we set
    \begin{equation} \nonumber
        c'(u,v) =  \begin{cases}
            c(u,v) & \braces{u,v} \in \braces{\braces{x_{i-1},x_{i}} \colon i =1,\ldots,k}\cup\braces{\braces{x,z}} \\
            0 & \text{otherwise}.
        \end{cases}
    \end{equation}
    Then we have
    \begin{align} \label{eq:bcart_helper_lemma_partial_work}
        \P\braces{\tau_a < \tau_z|X_0 = x} & = \frac{r(x,z)}{R(a\leftrightarrow x; r) + r(x,z)} \nn
        & \geq \frac{r(x,z)}{R(a\leftrightarrow x;r') + r(x,z)},
    \end{align}
    where the equality uses Lemma \ref{lem:bcart_treespace_helper1} and the inequality uses Lemma \ref{lem:monotonicity}.
    Next, using the series law from Lemma \ref{lem:network_simplification}, we have
    \begin{align} 
        R(a\leftrightarrow x;r') & = \sum_{i=1}^{k} r(x_{i-1},x_{i}) \nn
        & \leq k \rho r(x,z). \nonumber
    \end{align}
    Plugging this into \eqref{eq:bcart_helper_lemma_partial_work} and cancelling $r(x,z)$ in the numerator and denominator completes the proof.
\end{proof}

\begin{lemma} 
\label{lem:bcart_treespace_helper1}
    Given a network $(\statespace,c)$ with states $a, x, z$ and such that $c(u,z) = 0$ for all $u \notin \braces{x,z}$.
    Then we have
    \begin{equation} \nonumber
        P\braces{\tau_a < \tau_z~|~X_0 = x} = \frac{C(a\leftrightarrow x)}{C(a\leftrightarrow x) + c(x,z)}.
    \end{equation}
\end{lemma}

\begin{proof}
    Let $h(y) \coloneqq \P\braces{\tau_a < \tau_z | X_0 = y}$, and note that $h$ is a voltage between $a$ and $z$.
    As such, we have
    \begin{equation} \nonumber
        R(a\leftrightarrow z) = \frac{h(a) - h(z)}{\norm{I}} = \frac{1}{\norm{I}}.
    \end{equation}
    On the other hand, $h$ is also a voltage between $a$ and $x$ on the reduced state space $\statespace\backslash\braces{z}$, which gives
    \begin{equation} \nonumber
        R(a\leftrightarrow x) = \frac{h(a) - h(x)}{\norm{I}}.
    \end{equation}
    Finally, by the series law, we have
    \begin{equation} \nonumber
        R(a\leftrightarrow z) = R(a\leftrightarrow x) + r(x,z).
    \end{equation}
    Putting everything together, we get
    \begin{align}
        P\braces{\tau_a < \tau_z|X_0 = x} & = 1- \paren*{h(a) - h(x)} \nn
        & = 1 - R(a\leftrightarrow x)\norm{I} \nn
        & = 1 - \frac{R(a\leftrightarrow x)}{R(a\leftrightarrow z)} \nn
        & = \frac{r(x,z)}{R(a\leftrightarrow x) + r(x,z)} \nn
        & = \frac{C(a\leftrightarrow x)}{C(a\leftrightarrow x) + c(x,z)}, \nonumber
    \end{align}
    as we wanted.
\end{proof}

\subsection{Mixing and hitting time upper bounds}

In this subsection, we no longer work directly with networks, but we still assume that the Markov chain $(X_t)$ is reversible.
Let $P$ denote the transition matrix of $(X_t)$.
Denote $k = |\Omega|$.
For each $x \in \Omega$ and any integer $t \geq 1$, we abuse notation and let $P^t(x,-)$ denote the distribution of $X_t$ given $X_0 = x$.
We denote the total variation distance between $P^t(x,-)$ and $\pi$ via
\begin{equation}
    d_x(t) \coloneqq \norm*{P^t(x,-) - \pi}_{\operatorname{TV}}.
\end{equation}
The $\epsilon$-mixing time of $(X_t)$ is defined as
\begin{equation}
    t_{\operatorname{mix}}(\epsilon) \coloneqq \min\braces*{t \geq 0 \colon \max_{x \in \Omega} d_x(t) \leq \epsilon}.
\end{equation}
It is convention to focus on $\epsilon=1/4$ and denote $t_{\operatorname{mix}} = t_{\operatorname{mix}}(1/4)$.
In this work, we will require more subtlety.
We define the mixing time from an initial state $x$ as
\begin{equation}
    t_{\operatorname{mix},x}(\epsilon) \coloneqq \min\braces*{t \geq 0 \colon  d_x(t) \leq \epsilon}.
\end{equation}

We say that $(X_t)$ is \emph{lazy} if all diagonal entries of $P$ satisfy $P_{ii} \geq 1/2$.
It is well-known that if $(X_t)$ is reversible and lazy, then the eigenvalues $\lambda_1,\lambda_2,\ldots,\lambda_k$ of $P$ satisfy $1 = \lambda_1 \geq \lambda_2 \geq \cdots \geq \lambda_k \geq 0$.
The \emph{spectral gap} is defined and denoted as $\gamma \coloneqq 1 - \lambda_2$.

\begin{lemma}[Mixing time from an initial value]
    \label{lem:mixing_time_upper}
    Suppose $(X_t)$ is a lazy, reversible, irreducible Markov chain, with stationary distribution $\pi$, state space $\Omega$, and spectral gap $\gamma$.
    Suppose that $x_0 \in \Omega$ is such that $\pi(x_0) \geq c^2/|\Omega|$ for some $c > 0$.
    The total variation distance between $P^t(x_0,-)$ and $\pi$ satisfies
    \begin{equation}
        d_{x_0}(t) \leq e^{-\gamma t} \abs{\Omega}/c.
    \end{equation}
\end{lemma}

\begin{proof}
    Following the proof of Theorem 12.3 in \citet{LevinPeresWilmer2006}, we have
    \begin{equation}
        \abs*{P^t(x_0,y) - \pi(y)} \leq e^{-\gamma t} \pi(x_0)^{-1/2}\pi(y)^{1/2}.
    \end{equation}
    Summing over $y$ and applying Cauchy-Schwarz, we get
    \begin{equation}
    \begin{split}
        \norm*{P^t(x_0,-) - \pi}_{\operatorname{TV}} & \leq e^{-\gamma t} \pi(x_0)^{-1/2} \sum_{y \in \Omega} \pi(y)^{1/2} \\
        & \leq e^{-\gamma t} \pi(x_0)^{-1/2} |\Omega|^{1/2} \paren*{\sum_{y \in \Omega}\pi(y)}^{1/2} \\
        & \leq e^{-\gamma t} \pi(x_0)^{-1/2}|\Omega|^{1/2}.
    \end{split}
    \end{equation}
    Finally, using the fact that $\pi(x_0) \geq c^2/|\Omega|$ gives
    \begin{equation}
        d_{x_0}(t) \leq e^{-\gamma t} \abs{\Omega}/c
    \end{equation}
    as we wanted.
\end{proof}

\begin{lemma}
    \label{lem:hitting_time_upper}
    Let $(X_t)$ be a Markov chain with transition matrix $P$ and state space $\Omega$.
    Let $B \subset \Omega$ be a subset.
    Suppose that there is some $m \geq 0$ and $\lambda > 0$ such that for every state $x \in \Omega\backslash B$, there is a path $x = x_1,x_2,\ldots,x_{l(x)}$ such that $x_{l(x)} \in B$, $l(x) \leq m$ and the probability of the path satisfies
    \begin{equation}
    \label{eq:path_prob_lower_bound}
        \prod_{i=1}^{l(x)} P(x_{i-1},x_i) \geq \lambda,
    \end{equation}
    then for any integer $t \geq 0$, the hitting time $\tau_B$ for any initialization $x \in \Omega$ satisfies
    \begin{equation}
        P\braces{\tau_B \geq tm~|~X_0=x} \leq e^{-\lambda t}.
    \end{equation}
\end{lemma}

\begin{proof}
    For any $x \in \Omega$, we may use \eqref{eq:path_prob_lower_bound} to get
    \begin{equation}
    \label{eq:path_lower_bound_helper1}
    \begin{split}
        P\braces{\tau_B \geq m~|~X_0 = x} & \leq
        P\braces{\tau_B \geq l(x)~|~X_0 = x} \\
        & \leq 1 - \prod_{i=1}^{l(x)} P(x_{i-1},x_i) \\
        & \leq 1 - \lambda.
    \end{split}
    \end{equation}
    Next, for any integer $t \geq 2$, we compute
    \begin{equation}
    \label{eq:path_lower_bound_helper2}
        \begin{split}
            & P\braces{\tau_B \geq tm~|~\tau_B \geq (t-1)m, X_0 = x} \\
            & = \sum_{y \in \Omega \backslash B} P\braces{\tau_B \geq tm, X_{(t-1)m} = y~|~\tau_B \geq (t-1)m, X_0 = x} \\
            & = \sum_{y \in \Omega \backslash B} P\braces{\tau_B \geq tm~|~X_{(t-1)m} = y,\tau_B \geq (t-1)m, X_0 = x} \\
            & \quad\quad \cdot P\braces{X_{(t-1)m} = y~|~\tau_B \geq (t-1)m, X_0 = x} \\
            & = \sum_{y \in \Omega \backslash B} P\braces{\tau_B \geq tm~|~X_{(t-1)m} = y} \cdot P\braces{X_{(t-1)m} = y~|~\tau_B \geq (t-1)m, X_0 = x},
        \end{split}
    \end{equation}
    where the last equality follows from the Markov property.
    By stationarity, we further have
    \begin{equation}
        P\braces{\tau_B \geq tm~|~X_{(t-1)m} = y} = P\braces{\tau_B \geq m~|~X_0 = y}.
    \end{equation}
    Plugging this into \eqref{eq:path_lower_bound_helper2} and using \eqref{eq:path_lower_bound_helper1} gives
    \begin{equation}
        P\braces{\tau_B \geq tm~|~\tau_B \geq (t-1)m, X_0 = x} \leq 1- \lambda..
    \end{equation}
    Finally, we expand
    \begin{equation}
    \begin{split}
        P\braces{\tau_B \geq tm~|~X_0=x}
        & = \prod_{s=1}^t P\braces{\tau_B \geq sm ~|~ \tau_B \geq (s-1)m, X_0 = x} \\
        & \leq (1-\lambda)^t,
    \end{split}
    \end{equation}
    as we wanted.
\end{proof}

\begin{lemma}[Enhanced mixing time bound]
\label{lem:enhanced_mixing_time}
    Let $(X_t)$ be a lazy, reversible, irreducible Markov chain with transition matrix $P$, stationary distribution $\pi$, state space $\Omega$, and spectral gap $\gamma$.
    Let $B \subset \Omega$ be a subset such that $\pi(x) \geq c/|\Omega|$ for all $x \in B$.
    Suppose that the assumptions of Lemma \ref{lem:hitting_time_upper} hold.
    Then for any $\epsilon > 0$, the $\epsilon$-mixing time of $(X_t)$ satisfies
    \begin{equation}
    \label{eq:enhanced_mixing_time}
        t_{\operatorname{mix}}(\epsilon) \leq 2\log(|\Omega|/2c\epsilon|) \paren*{\max\braces*{1/\gamma,m/\lambda} + 2m}.
    \end{equation}
\end{lemma}

\begin{proof}
    For any $x,y \in \Omega$ and any time index $t$, note that
    \begin{equation}
        \begin{split}
            P^t(x,y) - \pi(y) & = E\braces*{\indicator\braces{X_t = y} - \pi(y)} \\
            & = E\braces*{E\braces*{\indicator\braces{X_t = y} - \pi(y)~|~X_{\tau_B \wedge t}}} \\
            & = E\braces*{P^{(t-\tau_B)_+}(X_{\tau_B\wedge t},y) - \pi(y)},
        \end{split}
    \end{equation}
    where the second equality follows from the tower property and the strong Markov property.
    As such, we may write
    \begin{equation}
        \begin{split}
            \norm*{P^t(x,-) - \pi}_{\operatorname{TV}} & = \sum_{y \in \Omega} \abs*{P^t(x,y) - \pi(y)} \\
            & \leq \sum_{y \in \Omega}E\braces*{\abs*{P^{(t-\tau_B)_+}(X_{\tau_B\wedge t},y) - \pi(y)}} \\
            & = E\braces*{\norm*{P^{(t-\tau_B)_+}(X_{\tau_B\wedge t},-) - \pi}_{\operatorname{TV}}} \\
            & \leq E\braces*{e^{-\gamma (t-\tau_B)_+}} \cdot \abs{\Omega}/c,
        \end{split}
    \end{equation}
    where the last inequality follows from Lemma \ref{lem:mixing_time_upper}.
    To bound this expectation, we write
    \begin{equation}
        \begin{split}
            E\braces*{e^{-\gamma (t-\tau_B)_+}} & = E\braces*{e^{-\gamma (t-\tau_B)_+}\indicator\braces*{\tau_B < t/2}} + E\braces*{e^{-\gamma (t-\tau_B)_+}\indicator\braces*{\tau_B \geq t/2}} \\
            & \leq e^{-\gamma \lfloor t/2 \rfloor} + e^{- \lambda \lfloor t / 2m \rfloor},
        \end{split}
    \end{equation}
    where we have bounded the second term using Lemma \ref{lem:hitting_time_upper}.
    Plugging in the value in \eqref{eq:enhanced_mixing_time} for $t$, we obtain $d_x(t) \leq \epsilon$.
    Since $x$ is arbitrary, this completes the proof.
\end{proof}

A \emph{canonical path ensemble} for $(X_t)$ is a collection of paths $\Gamma = \braces{\gamma_{xy}}_{x,y \in \Omega}$, one for each pair of states in $\Omega$.
The congestion parameter for $\Gamma$ is defined to be
\begin{equation} \label{eq:congestion_def}
    \rho(\Gamma) \coloneqq \max_{(z,w)}\frac{1}{\pi(z)P(z,w)}\sum_{\gamma_{xy} \ni (z,w)} \pi(x)\pi(y),
\end{equation}
where the maximum is taken over all edges $(z,w)$ (i.e. with nonzero probability P(z,w)) in the network associated with $(X_t)$.

\begin{lemma}[Canonical path ensemble bound on spectral gap, Corollary 6 in \citet{sinclair1992improved}]
\label{lem:canonical_path}
    Let $(X_t)$ be a reversible and let $\Gamma$ be a canonical path ensemble.
    The spectral gap is lower bounded by
    \begin{equation}
        \gamma \geq \frac{1}{\rho(\Gamma)l(\Gamma)},
    \end{equation}
    where $l(\Gamma)$ is the maximum length of a path in $\Gamma$.
\end{lemma}

%% file: AOS/a3_hitting_times.tex
\section{Proof of Theorem \ref{thm:pure_interaction}}
\label{sec:proof_of_pure_interaction}

We will first present the proofs for Theorem \ref{thm:pure_interaction} and Theorem \ref{thm:bcart_mix} because they are relatively simple compared to that of Theorem \ref{thm:additive}.
For convenience, we repeat the relevant constructions and definitions here.

\subsection{Reachability}
Given $\tse, \tse' \in \tsespace$, we say that 
$\tse \succsim \tse'$ if $\tse$ and $\tse'$ are connected by an edge, and if either $\bias(\tse;f^*) > \bias(\tse';f^*)$ or $\bias(\tse;f^*) = \bias(\tse';f^*)$ and $\df(\tse) \geq \df(\tse')$.
Note that, because we allow only ``grow'' and ``prune'' moves, for adjacent $\tse$ and $\tse'$, $\mapone(\tse)$ and $\mapone(\tse')$ are nested subspaces, so that $\bias(\tse;f^*) = \bias(\tse';f^*)$ if and only if $\Pi_{\mapone(\tse)}[f^*] = \Pi_{\mapone(\tse')}[f^*]$.
We say that $\tse'$ is reachable from $\tse$, denoted $\tse \succeq \tse'$, if there is a sequence of TSEs $\tse = \tse^0, \tse^1,\ldots,\tse^k = \tse'$ such that $\tse^i \succsim \tse^{i+1}$ for $i=0,1,\ldots,k-1$.

\subsection{Set-up}
Without loss of generality, let $(x_1,x_2)$ be a pure interaction for $f^*$.
Let $\tsebad$ be any TSE such that
\begin{longlist}
    \item $\tsebad$ is reachable from $\emptytse$;
    \item There does not exist $\tse \in \tsespace$ such that $\tse$ is reachable from $\tsebad$ but $\tsebad$ is not reachable from $\tse$.
\end{longlist}
Note that such a TSE exists because $\tsespace$ is finite and $\succeq$ is a partial ordering on this space.
By definition, there exists a sequence of TSEs $\emptytse = \tse^0, \tse^1,\ldots,\tse^k = \tsebad$ such that $\tse^i \succeq \tse^{i+1}$ for $i=0,1,\ldots,k-1$.
We set $\setone$ to be the equivalence class of $\tsebad$ under $\succeq$ and set $\settwo$ to be the outer boundary of $\setone$.
For $n$ large enough, using Proposition \ref{prop:lml_bic_appendix} and Proposition \ref{prop:concentration_bic_diff_appendix}, there is a $1-\delta/2$ event over which, for $i=1,2,\ldots,k$,
\begin{equation}
\log p(\tse^{i} \mid \by) - \log p(\tse^{i-1} \mid \by)
=
\begin{cases}
\dfrac{n}{2\sigma^2}
\bigl( \bias(\tse^{i-1}; f^*) - \bias(\tse^i; f^*) \bigr)
+ O\!\left( \sqrt{n \log(k/\delta)} \right), \\[0.5em]
\qquad\text{if }\bias(\tse^{i-1}; f^*) > \bias(\tse^i; f^*), \\[0.8em]
\dfrac{\log n}{2}
\bigl( \df(\tse^{i-1}) - \df(\tse^i) \bigr)
+ O\!\left( \log(k/\delta) \right),
\qquad\text{otherwise}.
\end{cases}
\end{equation}
In either case, we get
\begin{equation} \label{eq:pure_interaction_posterior_ratio}
    \frac{p(\tse^i|\by)}{p(\tse^{i-1}|\by)} = \Omega(1).
\end{equation}
Using Proposition \ref{prop:concentration_bic_diff_appendix} again, there is a further $1-\delta/2$ probability event over which $\Delta\bic(\tse,\tsebad)$ satisfies either \eqref{eq:bic_diff_eq1} or \eqref{eq:bic_diff_eq2} simultaneously for all $\tse \in \settwo$ (after dividing the $\delta$ that appears in the formulas by $|\settwo|$).
Condition on these two events.

\subsection{Hitting precedence probability lower bound.}
Using Lemma \ref{lem:pure_interaction_helper1} and Lemma \ref{lem:pure_interaction_helper2}, we see that $\tse^i \notin \opt(f^*,\infty)$ for $i=0,1,\ldots,k$.
We therefore have
\begin{align} \label{eq:hpp_eq1}
    P\braces*{\tau_{\tsebad} < \tau_{\opt(f^*,\infty)}} & \geq P\braces*{\tse_i = \tse^i ~\text{for}~i=1,2,\ldots,k} \nn
    & = \prod_{i=1}^k P(\tse^{i-1},\tse^i).
\end{align}
It suffices to show that $P(\tse^{i-1},\tse^i)$ is bounded from below by a constant.
To see this, we note that
\begin{align} \label{eq:pure_interaction_hpp}
    P(\tse^{i-1},\tse^i) & = Q(\tse^{i-1},\tse^i)\min\braces*{\frac{Q(\tse^{i},\tse^{i-1})p(\tse^i|\by)}{Q(\tse^{i-1},\tse^i)p(\tse^{i-1}|\by)}, 1} \nn
    & = \Omega\paren*{\min\braces*{\frac{p(\tse^i|\by)}{p(\tse^{i-1}|\by)}}, 1} \nn
    & = \Omega(1), 
\end{align}
where the first two inequalities follow because the proposal distributions do not depend on the training sample size $n$, while the last equality follows from equation \eqref{eq:pure_interaction_posterior_ratio}.

\subsection{BIC lower bound}
Consider $\tse' \in \settwo$.
By definition of $\settwo$, there exists $\tse \in \setone$ such that $\tse$ and $\tse'$ are connected by an edge, but $\tse \not\succsim \tse'$.
This implies that either $\bias(\tse;f^*) < \bias(\tse';f^*)$ or $\bias(\tse;f^*) = \bias(\tse';f^*)$ and $\df(\tse;f^*) < \df(\tse';f^*)$.
Since $\tse$ and $\tsebad$ are mutually reachable, we have $\bias(\tse;f^*) = \bias(\tsebad;f^*)$ and $\df(\tse) = \df(\tsebad)$.
We therefore conclude that $\tse'$ either has a larger squared bias or larger degrees of freedom compared to $\tsebad$.
Applying equations \eqref{eq:bic_diff_eq1} and \eqref{eq:bic_diff_eq2} and taking the minimum sample size $N$ large enough gives
\begin{equation} \label{eq:pure_interaction_bic}
    \Delta\bic(\tse,\tsebad) \geq \log n + O\paren*{\log(|\settwo|/\delta)}.
\end{equation}

\subsection{Conclusion}
Applying Proposition \ref{prop:recipe} with equations \eqref{eq:pure_interaction_hpp} and \eqref{eq:pure_interaction_bic}, we get a $1-2\delta$ probability event over which
\begin{equation}
    E\braces*{\tau_{\opt(f^*,\infty)}} = \Omega\paren*{n^{1/2}}.
\end{equation}


\begin{lemma}
    \label{lem:pure_interaction_helper1}
    For $i=0,1\ldots,k$, no tree in $\tse^i$ contains a split on either $x_1$ or $x_2$.
\end{lemma}

\begin{proof}
    Suppose otherwise.
    By changing the labeling of $x_1$ and $x_2$ if necessary, there exists
    \begin{equation}
        i \coloneqq \min\braces*{1 \leq j \leq l \colon \tse_j ~\text{contains split on}~x_1}.
    \end{equation}
    Since only ``grow'' and ``prune'' moves are allowed, $\tse_i$ is obtained from $\tse_{i-1}$ via a ``grow'' move that splits a leaf node $\leaf$ into:
    \begin{equation*}
        \leaf_L \coloneqq \braces*{\bx \in \leaf \colon x_i \leq t}, \quad\quad \leaf_R \coloneqq \braces*{\bx \in \leaf \colon x_i > t}.
    \end{equation*}
    Define the function $\psi \coloneqq \indicator_{\leaf_L} - \indicator_{\leaf_R}$.
    Then the span of $\braces*{\indicator_{\leaf_L}, \indicator_{\leaf_R}}$ is the same as that of $\braces*{\indicator_{\leaf}, \psi}$, which implies that
    \begin{equation}
        \mapone(\tse_i) = \linspan\braces*{\mapone(\tse_{i-1}),\psi}.
    \end{equation}
    Furthermore, we have $\psi \notin \mapone(\tse_{i-1})$ since all functions in $\mapone(\tse_{i-1})$ do not depend on $x_1$.
    This implies that $\df(\tse_i) = \df(\tse_{i-1}) + 1$.
    On the other hand, since $x_i \indep y ~|~ \bx \in \leaf$, we have $\psi \perp y$, which means that $\bias(\tse_i;f^*) = \bias(\tse_{i-1};f^*)$.
    As such, we have $\tse_{i-1} \not\succsim \tse_i$, which gives a contradiction.
\end{proof}

\begin{lemma}
    \label{lem:pure_interaction_helper2}
    For any $\tse \in \tsespace$, if $\tse$ does not contain splits on both $x_1$ and $x_2$, then $\bias(\tse;f^*) > 0$.
\end{lemma}

\begin{proof}
    Since $(x_1,x_2) \not\indep y$, there exist values $(a_1,a_2)$, $(b_1,b_2)$ and $(c_3,c_4,\ldots,c_d)$ such that $f^*(a_1,a_2,c_3,\ldots,c_d) \neq f^*(b_1,b_2,c_3,\ldots,c_d)$.
    On the other hand, all functions in $\mapone(\tse)$ are constant with respect to $x_1$ and $x_2$, so that $f^* \notin \mapone(\tse)$.
\end{proof}


\section{Proof of Theorem \ref{thm:bcart_mix}}
\label{sec:proof_of_bcart_mix}

\subsection{Set-up}
For convenience, we repeat the relevant construction here.
Without loss of generality, let $x_1$ be the feature that gives $f^*$ root dependence.
By assumption, there is a threshold $t$ such that splitting the trivial tree on $x_1$ at $t$ gives a decrease in squared bias.
We set
\begin{equation*}
    \setone = \braces*{\tree \in \tsespace[1] \colon \tree~\text{has root split on}~x_1~\text{at}~t},
\end{equation*}
and 
\begin{equation*}
    \treebad = \arg\min \braces*{\bic(\tse) \colon \tse \in \setone}.
\end{equation*}
Note that the outer boundary of $\setone$ is a singleton set comprising the trivial tree $\emptytree$.
By assumption, we have $\setone \cap \opt[1](f^*,0) = \emptyset$.

We continue this construction by denoting $\tree_0 = \emptytree$ and letting $\tree_1$ be a tree structure comprising a single root split at $(x_1,t)$.
By adding the splits in $\treebad$ iteratively, we get a sequence $\tree_1,\tree_2,\ldots,\tree_k$ of nested tree structures with $\tree_k = \treebad$.
By Proposition \ref{prop:concentration_bic_diff_appendix} and Proposition \ref{prop:lml_bic_appendix}, for any training sample size $n$ large enough, there is a $1-\delta$ probability event with respect to $\P_n$ such that
\begin{equation} \label{eq:bcart_log_posterior_diff}
    \log p(\tree_j|\by) - \log p(\tree_0|\by) = \frac{n}{2\sigma^2}\paren*{\bias(\tree_0;f^*) - \bias(\tree_j;f^*)} + O\paren*{\sqrt{n\log(2k/\delta)}} 
\end{equation}
for $j=1,2,\ldots,k$.
Condition on this event.

We note the following:
\begin{equation} \label{eq:bcart_empty_bias}
    \bias(\tree_0;f^*) = \int (f^*)^2d\nu - \paren*{\int f^* d\nu}^2
\end{equation}
\begin{equation} \label{eq:bcart_no_bias}
    \bias(\tree_k;f^*) = 0
\end{equation}
\begin{equation} \label{eq:bcart_path_bias}
    \bias(\tree_j;f^*) \leq \bias(\tree_1;f^*) < \bias(\tree_0;f^*) \quad\quad\text{for}~j=1,2,\ldots,k.
\end{equation}
Here, the last statement follows from the fact that
\begin{equation*}
    \frac{\bias(\tree_0;f^*) - \bias(\tree_1;f^*)}{\bias(\tree_0;f^*)} = \text{Corr}^2\paren*{f^*(\bx),\indicator\braces*{x_1 \leq t}} > 0,
\end{equation*}
and because $\tree_j$ is a refinement of $\tree_1$ for each $j=1,2,\ldots,k$.

\subsection{Hitting precedence probability lower bound}
    We define a conductance function on $\Omega_{\operatorname{TSE},1}$ via
    \begin{equation}
        c(\tree,\tree') = p(\tree|\by) P(\tree,\tree'),
    \end{equation}
    where $P$ is the transition kernel of the BART sampler.
    It is clear that the Markov chain for Bayesian CART is equivalent to the Markov chain associated with the network $(\Omega_{\operatorname{TSE},1}, c)$.
    By ignoring quantities that do not depend on the training sample size $n$, we get
    \begin{equation}
    \begin{split}
        c(\tree,\tree') & = \min\braces*{p(\tree|\by) Q(\tree,\tree'), p(\tree'|\by) Q(\tree',\tree)} \\
        & \asymp \min\braces*{p(\tree|\by), p(\tree'|\by)}.
    \end{split}
    \end{equation}
        
    Consider the set $\setthree \coloneqq \statespace\backslash\paren*{\setone\cup\braces*{\emptytree}}$.
    Then $\opt[1](f^*,0) \subset \setthree$ by assumption, and we have $\tau_\setthree \leq \tau_{\opt[1](f^*,0)}$.
    This means that it suffices to consider
    \begin{equation*}
         P\braces*{\tau_{\treebad} < \tau_{\setthree}} = \hpp(\emptytree;\setthree,\treebad)
    \end{equation*}
    and to bound it from below.
    To calculate values for the harmonic function $\hpp(-;\setthree,\treebad)$, we use Lemma \ref{lem:network_simplification} to glue all states in $\setthree$ together without changing the values the function values.
    We will abuse notation and denote the new glued state using $\setthree$ while continuing to use $c$ to denote the conductance function on the new state space.

    Notice that the only edge from $\setthree$ connects to $\emptytree$.
    As such, we can compute
    \begin{align}
        c(\emptytree,\setthree) & \leq \sum_{\tree \sim \emptytree} \min\braces*{p(\emptytree|\by) Q(\emptytree,\tree), p(\tree|\by) Q(\tree,\emptytree)} \nn
        & \leq p(\emptytree|\by) \sum_{\tree \sim \emptytree}  Q(\emptytree,\tree) \nn
        & \leq p(\emptytree|\by). \nonumber
    \end{align}
    This implies, using equations \eqref{eq:bcart_log_posterior_diff} and \eqref{eq:bcart_path_bias}, that we can bound the ratios of conductances for $j=1,2\ldots,k$ as:
    \begin{align} \label{eq:bcart_conductance_ratio_lower_bound}
        \log\frac{c(\tree_j,\tree_{j-1})}{c(\tree_0,\setthree)} & \gtrsim  \min\braces*{\log p(\tree_j|\by)-\log p(\tree_0|\by),\log p(\tree_{j-1}|\by)-\log p(\tree_0|\by)} \nn
        & \geq \begin{cases}
            \frac{n}{2\sigma^2}\paren*{\bias(\tree_0;f^*) - \bias(\tree_{j-1};f^*)} + O\paren*{\sqrt{n\log(2k/\delta)}} & j \geq 2 \\
            0 & j = 1.
        \end{cases}
    \end{align}
    By \eqref{eq:bcart_path_bias}, there is some minimum training sample size $N$ so that for all $n \geq N$, the right hand side side of \eqref{eq:bcart_conductance_ratio_lower_bound} is nonnegative.
    In this case, the assumptions of Lemma \ref{lem:reduction_to_path} hold, and we get
    \begin{equation} \label{eq:bcart_hpp_bound}
        P\braces*{\tau_{\treebad} < \tau_{\setthree}} = \Omega(1)
    \end{equation}
    as desired.

\subsection{BIC lower bound}
From equations \eqref{eq:bcart_empty_bias} and \eqref{eq:bcart_no_bias}, we have
\begin{align} \label{eq:bcart_bic_bound}
    \Delta\bic(\emptytree,\treebad) & = \frac{n}{\sigma^2}\paren*{\bias(\emptytree;f^*) - \bias(\treebad;f^*)} + O\paren*{\sqrt{n\log(2k/\delta)}} \nn
    & = \frac{n}{\sigma^2}\paren*{\int (f^*)^2 d\nu -  \paren*{\int f^* d\nu}^2} + O\paren*{\sqrt{n\log(2k/\delta)}}.
\end{align}

\subsection{Conclusion}
Applying Proposition \ref{prop:recipe} with equations \eqref{eq:bcart_hpp_bound} and \eqref{eq:bcart_bic_bound}, we get a $1-2\delta$ probability event over which
\begin{equation*}
    E\braces*{\tau_{\opt[1](f^*,0)}} = \Omega\paren*{\exp\paren*{\frac{n}{2\sigma^2}\paren*{\int (f^*)^2 d\nu -  \paren*{\int f^* d\nu}^2} + O\paren*{\sqrt{n\log(2k/\delta)}}}}.
\end{equation*}
Taking logarithms, dividing by $n$ and applying Markov's inequality gives
\begin{equation*}
    \E_n\braces*{\frac{\log E\braces*{\tau_{\opt[1](f^*,0)}}}{n}} \geq \frac{(1-\delta)(1-O(n^{-1/2})}{2\sigma^2}\paren*{\int (f^*)^2 d\nu -  \paren*{\int f^* d\nu}^2}.
\end{equation*}
Letting $\delta \to 0$ and taking $n \to \infty$ finishes the proof.

\section{Splitting rules, local decision stumps and coverage}

Before proving Theorem \ref{thm:additive}, we first introduce some required machinery.

\subsection{Local decision stump basis}
Let $\tree$ be a tree structure. 
Let $(v_j,\tau_j)$, $j=1,\ldots,l$ denote the splits (or splitting rules) on $\tree$ (the labels of its internal nodes).
Every node on the tree corresponds to rectangular region $\node \subset \cal{X}$ that is obtained by recursively partitioning the covariate space using the splits further up the tree.
If $\node$ is an internal node, it has two children nodes denoted $\node_L$ and $\node_R$ defined by
\begin{equation} \nonumber
    \node_L \coloneqq \braces*{\bx \in \node ~\colon~ x_{v} \leq \tau}
\end{equation}
\begin{equation} \nonumber
    \node_R \coloneqq \braces*{\bx \in \node ~\colon~ x_{v} > \tau}
\end{equation}
where $(v,\tau)$ is the split on $\node$.
For each internal node $(\node_j, v_j,\tau)$, define a \emph{local decision stump function}
\begin{equation} \label{eq:local_decision_stumps}
    \psi_j(\bx) \coloneqq \frac{\nu(\node_R)\indicator\braces{\bx \in \node_L} - \nu(\node_L)\indicator\braces{\bx \in \node_R}}{\sqrt{\nu(\node_L)\nu(\node_R)}},
\end{equation}
where $\node_L$ and $\node_R$ denote the children of $\node$.
It is easy to check (see for instance \cite{agarwal2022hierarchical}) that $\psi_1,\psi_2,\ldots,\psi_l$ are orthogonal, and that, together with the constant function $\psi_0 \equiv 1$, form a basis for $\mapone(\tree)$.
Using this basis makes it more convenient to analyze the difference between $\mapone(\tree)$ and $\mapone(\tree')$ when $\tree'$ is obtained from $\tree$ via a ``grow'' move.
Indeed, let $\psi_{l+1}$ denote the local decision stump corresponding to the new split.
We then have $\mapone(\tree') = \mapone(\tree)\oplus\linspan(\psi_{l+1})$.

Now consider a TSE $\tse = (\tree_1,\tree_2,\ldots,\tree_m)$.
We may also write a basis for $\mapone(\tse)$ by concatenating the bases $\braces*{\psi_{i,1},\psi_{i,2},\ldots,\psi_{i,l_i}}$ for each tree $\tree_i$, together with the constant function.
If $\tse'$ is obtained from TSE via a ``grow'' move, we likewise have the property $\mapone(\tse') = \mapone(\tse)\oplus\linspan(\psi')$, where $\psi'$ is the local decision stump corresponding to the new split.
On the other hand, the basis functions from different trees need not be orthogonal to each other.
To regain orthogonality, we use the following lemma:

\begin{lemma}[Conditions for orthogonality]
\label{lem:independence_of_decision_stumps}
    Suppose $\nu = \nu_1\times\nu_2\times\cdots\times\nu_\nfeats$ is a product measure on $\xspace$.
    Let $\tree_1$ and $\tree_2$ be two trees, and let $I_1, I_2 \subset \braces*{1,2,\ldots,\nfeats}$ be two disjoint subsets of indices such that $\tree_1$ and $\tree_2$ contain splits only on features in $I_1$ and $I_2$ respectively.
    Then the local decision stumps for both trees, $\braces*{\psi_{1,1},\psi_{1,2},\ldots,\psi_{1,l_1}}$ and $\braces*{\psi_{2,1},\psi_{2,2},\ldots,\psi_{2,l_2}}$, are orthogonal to each other.
\end{lemma}

\begin{proof}
    Consider two stumps from different trees: $\psi_{1,k_1}$ and $\psi_{2,k_2}$.
    Under the assumption of a product measure, $\braces{x_i \colon i \in I_1}$ is independent of $\braces{x_i \colon i \in I_2}$.
    Since $\psi_{1,k_1}$ is a function of the first set of variables and $\psi_{2,k_2}$ is a function of the second set, they are thus independent of each other.
    We therefore have
    \begin{equation} \nonumber
        \int \psi_{1,k_1}\psi_{2,k_2}d\nu = \int \psi_{1,k_1}d\nu_{I_1}\int \psi_{2,k_2}d\nu_{I_2} = 0,
    \end{equation}
    where the second equality follows from the fact that all local decision stumps have mean zero.
\end{proof}

\begin{lemma}[Existence of informative split]
\label{lem:additive_existence_of_informative_split}
    For any finite contiguous subset of integers $I$, let $g\colon I \to \R$ be non-constant.
    Let $\nu$ be any measure on $I$.
    Then there exists a split at a threshold $t$ with associated decision stump $\psi$ such that $t$ is a knot for $g$ and
    \begin{equation} \nonumber
        \paren*{\int \phi g d\nu}^2 > 0.
    \end{equation}
    As such, there is a sequence of recursive splits with associated local decision stumps $\psi_1,\psi_2,\ldots,\psi_q$, where $q$ is the number of knots of $g$, such that
    \begin{longlist}
        \item $g \in \linspan\paren*{1,\psi_1,\psi_2,\ldots,\psi_q}$;
        \item $\paren*{\int \psi_ig d\nu}^2 > 0$ for $i=1,2,\ldots,q$;
        \item $\psi_i$ splits on a knot of $g$ for $i=1,2,\ldots,q$.
    \end{longlist}
\end{lemma}

\begin{proof}
    Let $i_1,i_2,\ldots,i_k$ denote the knots of $g$, 
    and let $\tilde\psi_j$ denote the decision stump functions corresponding to a split at threshold $x=i_j$ for $j=1,2,\ldots,k$, using the formula \eqref{eq:local_decision_stumps}.
    Let $\tilde\psi_0 \equiv 1$ denote the constant function as usual.
    Then it is easy to see that $g \in \linspan\paren*{\tilde\psi_0,\tilde\psi_1,\ldots,\tilde\psi_{k}}$.
    As such, if $g\perp\tilde\psi_k$ for $k > 1$, then $g \in \linspan(\tilde\psi_0)$, i.e. $g$ is a constant function.
    To conclude the second statement, we apply the first part recursively to leaves obtained by making each split.
\end{proof}

\begin{lemma}[Formula for decrease in bias]
\label{lem:additive_dec_in_bias}
    Suppose $\tse'$ is obtained from $\tse$ via a ``grow'' move.
    Let $\psi$ be the local decision stump associated with the new split.
    Let $\phi \coloneqq \frac{\psi - \Pi_{\mapone(\tse)}[\psi]}{\norm*{\psi - \Pi_{\mapone(\tse)}[\psi]}_{L^2(\nu)}}$.
    Then for any regression function $f$, we have
    \begin{equation} \nonumber
        \bias(\tse';f) = \bias(\tse;f) - \paren*{\int \phi f d\nu}^2.
    \end{equation}
\end{lemma}

\begin{proof}
    We have
    \begin{equation} \nonumber
        \mapone(\tse') = \mapone(\tse)\oplus\linspan(\psi') =  \mapone(\tse)\oplus\linspan(\phi),
    \end{equation}
    with the last expression comprising an orthogonal decomposition.
    As such, we have
    \begin{align}
        \bias(\tse;f) - \bias(\tse';f) & = \norm*{f - \Pi_{\mapone(\tse)}[f]}_{L^2(\nu)}^2 - \norm*{f - \Pi_{\mapone(\tse)}[f] - \Pi_{\linspan(\phi)}[f]}_{L^2(\nu)}^2 \nn
        & = \norm*{\Pi_{\linspan(\phi)}[f]}_{L^2(\nu)}^2 \nn
        & = \paren*{\int \phi f d\nu}^2 \nonumber
    \end{align}
    as we wanted.
\end{proof}

\subsection{Coverage}
We first introduce some useful notation.
Given a coordinate index $i$, let $\bx^{-i} \in \braces*{1,2,\ldots,\ncats}^{\nfeats-1}$, $x_i \in \braces*{1,2,\ldots,\ncats}$.
Combining these, we let $(\bx^{-i},x_i) \in \braces*{1,2,\ldots,\ncats}^\nfeats$ have $i$-th coordinate equal to $x_i$ and all other coordinates given by $\bx^{-i}$.
Also use $\be_i \in \R^\nfeats$ to denote the $i$-th coordinate vector.
Given a real-valued function $f$ defined on $\braces*{1,2,\ldots,\ncats}^\nfeats$, we say that $\bx \in \braces*{1,2,\ldots,\ncats}^\nfeats$ is a \emph{jump location} for $f$ with respect to feature $i$ if $f(\bx) \neq f(\bx+\be_i)$.
Let $\pem$ be a PEM.
For each feature $i=1,2,\ldots,\nfeats$ and $t = 1,2,\ldots,b$, the \emph{coverage} of $\pem$ of the split $(i,t)$ is defined as
\begin{equation} \label{def:coverage}
    \coverage(i,t;\pem) \coloneqq \braces*{\bx^{-i} \in \braces*{1,2,\ldots,\ncats}^{\nfeats-1} \colon 
    \begin{aligned}
        &\exists f \in \pem, (\bx^{-i},t)~\text{is a jump location for} \\
        &f~\text{with respect to feature}~i
    \end{aligned}}.
\end{equation}
If $\coverage(i,t;\pem) = \braces*{1,2,\ldots,\ncats}^{\nfeats-1}$, we say that $\pem$ has \emph{full coverage} of the split $(i,t)$.
Note that given a collection of PEMs $\pem_1,\pem_2,\ldots,\pem_\ntrees$, we have
\begin{equation} \nonumber
    \coverage(i,t;\pem_1+\pem_2,\ldots,\pem_\ntrees) = \bigcup_{j=1}^\ntrees \coverage(i,t;\pem_j).
\end{equation}

\begin{lemma}[Zero bias requires full coverage of all knots]
    \label{lem:zero_bias_coverage}
    Suppose $f^* \in \pem$, where $f^*(\bx) = f_1(x_1) + f_2(x_2) + \cdots + f_{m'}(x_{m'})$ for some univariate functions $f_1,f_2,\ldots,f_{m'}$.
    Then for any feature $1 \leq i \leq m'$ and any knot $t$ of $f_i$, i.e. a value for which $f_i(t) \neq f_i(t+1)$, $\pem$ has full coverage of the split $(i,t)$.
\end{lemma}

\begin{proof}
    For any $\bx^{-i} \in \braces*{1,2,\ldots,\ncats}^{\nfeats-1}$, we have
    \begin{equation} \nonumber
        f^*((\bx^{-i},t+1)) - f^*((\bx^{-i},t) = f_i(t+1) - f_i(t) \neq 0,
    \end{equation}
    so that $(\bx^{-i},t)$ is a jump location for $f^*$ with respect to feature $i$.
    Hence, if $f^* \in \pem$, it provides the desired function in the definition \eqref{def:coverage}.
\end{proof}

\begin{lemma}[Full coverage implies inclusion of grid cells]
    \label{lem:full_coverage_grid}
    Suppose $\pem = \mapone(\tree)$ for a single tree $\tree$.
    Suppose $\pem$ has full coverage of splits $(i,\xi_{i,1}), (i,\xi_{i,2}),\ldots,(i,\xi_{i,q_i})$ for $i = 1,2, \ldots,k$.
    For any choice of $1 \leq j_i \leq q_i$, $i=1,2,\ldots,k$, denote the cell
    \begin{equation} \nonumber
    \cell \coloneqq \braces*{\bx \colon \xi_{i,j_{i-1}} < x_i \leq \xi_{i,j_{i}}~\text{for}~i=1,2,\ldots,k}.
\end{equation}
    We then have $\indicator_\cell \in \pem$.
\end{lemma}

\begin{proof}
    Let $\bx \in \cell$ be any point, and let $\leaf(\bx)$ be the leaf of $\tree$ containing $\bx$.
    We claim that $\leaf(\bx) \subset \cell$.
    Suppose not, then there exists a coordinate direction $i$ in which $\leaf(\bx)$ exceeds $\cell$.
    By reordering if necessary, we may thus assume that $(\bx^{-i},\xi_{i,j_i}+1) \in \leaf(\bx)$. 
    For any $f \in \pem$, we may write $f = a_0\indicator_{\leaf(\bx)} + \sum_{l=1}^L a_l \indicator_{\leaf_l}$, where $\leaf_1,\leaf_2,\ldots,\leaf_L$ are leaves in $\tree$.
    We then have
    \begin{equation} \nonumber
        f(\bx^{-i},\xi_{i,j_i}+1) = a_0 = f(\bx^{-i},\xi_{i,j_i}),
    \end{equation}
    which contradicts our assumption that $\pem$ has full coverage of the split at $(i,\xi_{i,j_i})$.
    We thus have
    \begin{equation} \nonumber
        \cell = \bigcup_{\bx \in \cell} \leaf(\bx) = \bigcup_{\leaf_l \subset \cell}\leaf_l.
    \end{equation}
    Since the right hand side is union over a collection of disjoint sets, we have $\indicator_\cell = \sum_{\leaf_l \subset \cell} \indicator_{\leaf_l} \in \pem$.
\end{proof}

\begin{lemma}[Sufficient condition for lack of coverage]
    \label{lem:lack_of_coverage}
    Let $\tree$ be a tree structure.
    Consider a split $(i,t)$, and let $\node$ be any node in $\tree$.
    Suppose the following hold:
    \begin{longlist}
        \item $(\bx^{-i},t) \in \node$ for some $\bx^{-i} \in \braces{1,2\ldots,\ncats}^{\nfeats-1}$;
        \item No ancestor or descendant of $\node$, including $\node$ itself, uses the splitting rule $(i,t)$;
    \end{longlist}
    Then if we let $\node^{-i}$ denote the projection of $\node$ onto all but the $i$-th coordinate, we have
    \begin{equation} \nonumber
        \coverage(i,t;\mapone(\tree)) \cap \node^{-i} = \emptyset.
    \end{equation}
\end{lemma}

\begin{proof}
    Let $\bz^{-i} \in \node^{-i}$ be any point.
    By the first assumption, $t$ is within the bounds of $\node$ along direction $i$, so $(\bz^{-i},t) \in \node$.
    Let $\leaf_0$ be the leaf node containing $(\bz^{-i},t)$.
    By the second assumption, $(\bz^{-i},t+1) \in \leaf_0$, otherwise some parent of $\leaf$ would have made the split $(i,t)$.
    To show that $\bz^{-i} \notin \coverage(i,t;\mapone(\tree))$, it suffices to show that $f(\bz^{-i},t+1) = f(\bz^{-i},t)$ for all $f \in \mapone(\tree)$.
    We may write $f = \sum_{j=0}^k a_j \indicator_{\leaf_j}$ where $\leaf_1,\leaf_2,\ldots,\leaf_j$ are other leaves in $\tree$.
    We then have $f(\bz^{-i},t+1) = a_0 = f(\bz^{-1},t)$ as we wanted.
\end{proof}

\begin{proposition}[Dimension of additive functions]
    \label{prop:additive_dim}
    Suppose $f^* \in \pem$, where $f^*(\bx) = f_1(x_1) + f_2(x_2) + \cdots + f_m(x_{m})$ for some univariate functions $f_1,f_2,\ldots,f_{m}$.
    Then for any $l < m$ we have $\dim_l(f^*) > \dim_m(f^*)$, while for any $l \geq m$, we have
    \begin{equation}
        \dim_l(f^*) = \dim_m(f^*) = \sum_{i=1}^m \dim_1(f_i) - m + 1.
    \end{equation}
\end{proposition}

\begin{proof}
    Let $\pem=\mapone(\tse)$ for some $\tse$ with $l$ trees, and suppose $f^*  \in \pem$.
    For $i = 1,2, \ldots,m$, let $(i,\xi_{i,1}), (i,\xi_{i,2}),\ldots,~$ $(i,\xi_{i,q_i-1})$ denote the knots of $f_i$, where $q_i = \dim_1(f_i)$
    We claim that $\tse$ must contain a node with splitting rule $(i,\xi_{i,j})$ for every choice of $i$ and $j$.
    If not, then $\pem$ does not cover $(i,\xi_{i,j})$, contradicting Lemma \ref{lem:zero_bias_coverage}.
    
    Construct $\tse$ via a sequence of grow moves in arbitrary order, thereby deriving a sequence $\emptytse = \tse^0, \tse^1,\ldots,\tse^k$.
    Consider the first time $t_{i,j}$ for which a split using rule $(i,\xi_{i,j})$ is added to the TSE.
    Let $\psi_{i,j}$ denote the local decision stump associated with the split.
    Since $\mapone(\tse^{t_{i,j}-1})$ has zero coverage of $(i,\xi_{i,j})$, we have $\df(\tse^{t_{i,j}}) = \df(\tse^{t_{i,j}-1})+1$.
    Adding this up over all splits and adding the constant function gives the lower bound $\df(\pem) \geq \sum_{i=1}^m q_i - m+1$.
    By putting all splits on feature $i$ on the $i$-th tree for $i=1,2,\ldots,m$, we see that this lower bound is achievable whenever $l \geq m$, thereby giving
    \begin{equation} \nonumber
        \dim_l(f^*) = \sum_{i=1}^m q_i - m + 1 = \dim_m(f^*).
    \end{equation}

    When $l < m$, then by the pigeonhole principle, there exists splits on different features $(i,s)$ and $(j,t)$ that occur on the same tree.
    We claim that this implies that either there exists two linearly independent local decision stumps $\phi$ and $\psi$ splitting on $(i,s)$ or the same applies to $(j,t)$.
    Suppose not, then the unique local decision stump splitting on $(i,s)$ must not depend on any other feature, while the same applies to that splitting on $(j,t)$.
    However, because nodes depend on the feature used in the root split, this cannot be simultaneously true if they both correspond to splits on the same tree.
    By repeating the argument above, we therefore get $\df(\tse) \geq \sum_{i=1}^m q_i - m + 2$.
    Since this holds for any $\tse$, we have
    \begin{equation} \nonumber
        \dim_l(f^*) \geq \sum_{i=1}^m q_i - m  + 2 > \dim_m(f^*)
    \end{equation}
    as we wanted.
\end{proof}

\section{Proof of Theorem \ref{thm:additive} Part 1}

\subsection{Set-up}
For $i=1,2,\ldots,m'$, denote $q_i = \dim_1(f_i)$, and let $0=\xi_{i,0} < \xi_{i,1} < \cdots < \xi_{i,q_{i}} = b$ denote the \emph{knots} of $f_j$, i.e. the values for which $f_i(\xi_{i,j}) \neq f_i(\xi_{i,j}+1)$, together with the endpoints.\footnote{ $\dim_1(f_i)$ is simply the number of constant pieces of $f_i$ or alternatively, one larger than the number of knots of $f_i$.}
Without loss of generality, assume that $f_1,\ldots,f_{m'}$ are ordered in descending order of their 1-ensemble dimension, i.e. $q_1 \geq q_2 \geq \cdots \geq q_{m'}$.

We now define $\tsebad$ and a ``bad set'' $\setone$.
To this end, we define a collection of partition models $\pem_1,\pem_2,\ldots,\pem_m$ (spans of indicators in a single partition) as follows.
First, for $i=3,4,\ldots,m-1$, $j=1,2,\ldots,q_i$, define the cells $\leaf_{i,j} \coloneqq \braces{\bx \colon \xi_{i,{j-1}} < x_i \leq \xi_{i,j}}$,
and set $\pem_i = \linspan\paren*{\braces*{\indicator_{\leaf_{i,j}} \colon j=1,2,\ldots,q_i}}$, i.e. for each $i$, $\pem_i$ contains splits only on feature $i$, and only at the knots of $f_i$. 
This also implies that $f_i \in \pem_i$. 
Next, for the remaining component functions, $i=m,m+1,\ldots,m'$, $j_i=1,2,\ldots,q_i$, define the cells
\begin{equation}
    \label{eq:leaves_for_product_tree}
    \leaf_{m,j_m,j_{m+1},\ldots,j_{m'}} \coloneqq \braces*{\bx \colon \xi_{i,j_{i-1}} < x_i \leq \xi_{i,j_{i}}~\text{for}~i=m,m+1,\ldots,m'},
\end{equation}
and set $\pem_m$ to be the span of their indicators.
Observe that $\pem_m$ comprises a grid contains splits only on features $m,{m+1},\ldots,{m'}$, and only at the knots of 
$f_m,f_{m+1},\ldots,f_{m'}$ respectively.
This also implies $f_m,f_{m+1},\ldots,f_{m'} \in \pem_m$. 

Finally, to introduce inefficiency, we define each of $\pem_1$ and $\pem_2$ to have splits on both features 1 and 2.
This construction is fairly involved and will be detailed in full in the next section.
For now, it suffices to assume that $f_1 + f_2 \in \pem_1 + \pem_2$, which implies that $f^* \in \pem_1 + \pem_2 + \cdots + \pem_m$.
Define $\setone$ via
\begin{equation} \nonumber
    \setone \coloneqq \braces*{(\tree_1,\tree_2,\ldots,\tree_m) \colon \mapone(\tree_i) = \pem_i ~\text{for}~i=1,2,\ldots,m}.
\end{equation}
It is clear that for any $\tse \in \setone$, we have $\mapone(\tse) = \pem_1 + \pem_2 + \cdots + \pem_m$, so that $\setone$ comprises a collection of TSEs with zero bias.
We will set $\tsebad$ to be a particular element of $\setone$.
We set $\settwo$ to be the outer boundary of $\setone$.

\subsection{Hitting precedence probability lower bound}
\textit{Step 1: Construction of path.}
We construct a path in $\tsespace$ comprising $\emptytse = \tse^0,\tse^1,\ldots,\tse^s = \tsebad$ such that $\tse^{i+1}$ is obtained from $\tse^i$ via a ``grow'' move for $i=0,1,\ldots,s-1$.
To do this, we first apply Lemma \ref{lem:additive_existence_of_informative_split} to obtain, for each feature $j=1,2,\ldots,m'$, a sequence of recursive splits $\psi_{j,1},\psi_{j,2},\ldots,\psi_{j,q_j-1}$ such that
\begin{longlist}
    \item $f_j \in \linspan(\psi_{j,1},\psi_{j,2},\ldots,\psi_{j,q_j-1})$;
    \item $\paren*{\int f_j \psi_{j,l} d\nu_j}^2 > 0$ for $l=1,2,\ldots,q_j-1$;
    \item $\psi_{j,l}$ splits on a knot of $f_j$ for $l=1,2,\ldots,q_j-1$.
\end{longlist}
Note that for convenience, we have identified the splits with their associated local decision stumps.
We break up the path into $m$ segments, $0 = s_0 < s_1 < \cdots < s_m = k$, with the $j$-th segment comprising $\tse^{s_{j-1}+1},\tse^{s_{j-1}+2}, \ldots, \tse^{s_j}$, possessing the following desired properties:
\begin{longlist}
    \item During this segment, the $j$-th tree is grown from the trivial tree $\emptytree$ to its final state $\tree_j^*$, while no other trees are modified;
    \item $\mapone(\tree_j^*) = \pem_j$;
    \item $\bias(\tse^{i-1};f^*) > \bias(\tse^i;f^*)$ for $i=s_{j-1}+1,s_{j-1}+2,\ldots,s_{j}$.
\end{longlist}


For $j=3,4,\ldots,m-1$, the splits $\psi_{j,1},\psi_{j,2},\ldots,\psi_{j,q_j-1}$ immediately yield a sequence of trees
$\emptytree = \tree_j^0, \tree_j^1,\ldots,\tree_j^{q_j-1} = \tree_j^*$ such that each $\tree_j^l$ is obtained from $\tree^{l-1}$ via adding the split $\psi_{j,l}$.
Next, assuming correctness of the construction up to $\tse^{s_{j-1}}$, $\psi_{j,l}$ is orthogonal to $\mapone(\tse^{s_{j-1}})$ and also to the previous splits $\psi_{j,1},\psi_{j,2},\ldots,\psi_{j,l-1}$, by Lemma \ref{lem:independence_of_decision_stumps}.
This allows us to apply Lemma \ref{lem:additive_dec_in_bias} to get
\begin{equation} \nonumber
    \bias(\tse^{s_{j-1}+i};f^*) = \bias(\tse^{s_{j-1}+i-1};f^*) - \paren*{\int \psi_i f^* d\nu}^2.
\end{equation}
We then notice that by the independence of different features,
\begin{equation} \nonumber
    \paren*{\int \psi_i f^* d\nu}^2 = \paren*{\int \psi_i f_j d\nu}^2 > 0.
\end{equation}

Let us construct $\tree^*_m$ inductively.
First, fix $0 \leq p < m'-m$,
and assume that we have constructed a tree $\tree_{m,p}^*$ such that the leaves of $\tree_{m,p}^*$ comprise the collection (see equation \eqref{eq:leaves_for_product_tree}):
\begin{equation} \nonumber
    \braces*{\leaf_{m,j_m,j_{m+1},\ldots,j_{m+p}} \colon 1 \leq j_{m+i} \leq q_{m+i}, i=0,1,2,\ldots,p}
\end{equation}
We now construct a sequence $\tree_{m,p}^* = \tree_{m,p}^0, \tree_{m,p}^1,\ldots,\tree_{m,p}^r = \tree_{m,p+1}^*$ by looping over the leaves, and for each leaf $\leaf = \leaf_{m,j_m,j_{m+1},\ldots,j_{m+p}}$, iteratively adding the splits $\tilde\psi_1,\tilde\psi_2,\ldots,\tilde\psi_{q_{m+p+1}-1}$, where $\tilde\psi_l = \nu(\leaf)^{-1/2}\psi_{m+p+1,l} \indicator_{\leaf}$ for $l=1,2,\ldots,q_{m+p+1}-1$.
Note that these are orthonormal.
When adding the split $\tilde\psi_l$, the decrease in squared bias can thus be computed as
\begin{equation} \label{eq:bias_dec_product_tree}
    \bias(\tree_{m,p}^{i-1};f^*) - \bias(\tree_{m,p}^{i};f^*) = \paren*{\int \tilde\psi_l f^* d\nu}^2 = \nu(\leaf)\paren*{\int \psi_{m+p+1,l} f_j d\nu}^2 > 0.
\end{equation}
We lift the sequence of tree structures to a sequence of TSEs, using Lemma \ref{lem:independence_of_decision_stumps} and Lemma \ref{lem:additive_dec_in_bias} as before to translate \eqref{eq:bias_dec_product_tree} to be in terms of the sequence of TSEs.

It remains to define $\pem_1$ and $\pem_2$ and construct $\tree_1^*$ and $\tree_2^*$.
Let $a_i$ be the index of the knot $\xi_{i,a_i}$ forming the threshold for $\psi_{i,q_i}$ for $i=1,2$.
For $k=0,1,\ldots,q_1-1$, define $b_k = k$ for $k < a_i$ and $b_k = k+1$ for $k \geq a_i$.
Similarly, for $l=0,1,\ldots,q_{2}-1$,define $c_l = l$ for $l < a_2$ and $c_l = l+1$ for $k \geq a_2$.
Set $\pem_1$ to be the span of indicators of the cells
\begin{equation} \nonumber
    \leaf_{1,k,l} \coloneqq \braces*{\bx \colon \xi_{1,b_{k-1}} < x_1 \leq \xi_{1,b_k}~\text{and}~ \xi_{2,c_{l-1}} < x_2 \leq \xi_{2,c_l}}
\end{equation}
for $k=1,2,\ldots,q_1-1$ and $l=1,2$.
Next, define $d_0=0$, $d_1 = a_1$, $d_2 = q_1$, $e_1 = 0$, $e_1 = a_2$, $e_2 = q_2$.
We set $\pem_2$ to be the span of indicators of the cells
\begin{equation} \nonumber
    \leaf_{2,k,l} \coloneqq \braces*{\bx \colon \xi_{1,d_{k-1}} < x_1 \leq \xi_{1,d_k}~\text{and}~ \xi_{2,e_{l-1}} < x_2 \leq \xi_{2,e_l}}
\end{equation}
for $k, l = 1,2$.
We construct $\tree_1^*$ similarly to $\tree_m^*$, through recursive partitioning on $x_1$ and $x_2$, but without making the final split on both features.
Using the same argument as before, we obtain the desired sequence of trees.
Finally, to construct $\tree_2^*$, we simply make the omitted splits on the new tree.
More precisely, we define the splits
\begin{equation} \nonumber
    \psi_1 = \frac{\nu\braces*{x_1 > \xi_{1,a_1}}\indicator\braces*{x_1 \leq \xi_{1,a_1}} - \nu\braces*{x_1 \leq \xi_{1,a_1}}\indicator\braces*{x_1 > \xi_{1,a_1}}}{\sqrt{\nu\braces*{x_1 \leq \xi_{1,a_1}}\nu\braces*{x_1 > \xi_{1,a_1}}}}
\end{equation}
and $\psi_2 = \nu\braces*{x_1 \leq \xi_{1,a_1}}^{1/2}\phi\indicator\braces*{x \leq \xi_{1,a_1}}$, $\psi_3 = \nu\braces*{x_1 > \xi_{1,a_1}}^{1/2}\phi\indicator\braces*{x > \xi_{1,a_1}}$, where
\begin{equation} \nonumber
    \phi = \frac{\nu\braces*{x_2 > \xi_{2,a_2}}\indicator\braces*{x_2 \leq \xi_{2,a_2}} - \nu\braces*{x_2 \leq \xi_{2,a_2}}\indicator\braces*{x_2 > \xi_{2,a_2}}}{\sqrt{\nu\braces*{x_2 \leq \xi_{2,a_2}}\nu\braces*{x_2 > \xi_{2,a_2}}}}.
\end{equation}
It is clear that these, together with the constant function, span $\pem_2$.
Furthermore, it is easy to see that we have
\begin{equation} \nonumber
    \frac{\psi_1 - \Pi_{\pem_1}[\psi_1]}{\norm{\psi_1 - \Pi_{\pem_1}[\psi_1]}_{L^2(\nu)}} = \psi_{1,q_1},
\end{equation}
\begin{equation} \nonumber
    \frac{\psi_2 - \Pi_{\pem_1\oplus\linspan(\psi_1)}[\psi_2]}{\norm{\psi_2 - \Pi_{\pem_1\oplus\linspan(\psi_1)}[\psi_2]}_{L^2(\nu)}} = \nu\braces*{x_1 \leq \xi_{1,a_1}}^{-1/2}\psi_{2,q_2}\indicator\braces*{x \leq \xi_{1,a_1}},
\end{equation}
\begin{equation} \nonumber
    \frac{\psi_3 - \Pi_{\pem_1\oplus\linspan(\psi_1,\psi_2)}[\psi_3]}{\norm{\psi_3 - \Pi_{\pem_1\oplus\linspan(\psi_1,\psi_2)}[\psi_3]}_{L^2(\nu)}} = \nu\braces*{x_1 > \xi_{1,a_1}}^{-1/2}\psi_{2,q_2}\indicator\braces*{x > \xi_{1,a_1}}.
\end{equation}
By construction, we have
\begin{equation} \nonumber
    \paren*{\int \psi_{1,q_1}f_1 d\nu_1}^2, \paren*{\int \psi_{2,q_2}f_2 d\nu_2}^2 > 0.
\end{equation}
Applying Lemma \ref{lem:additive_dec_in_bias} then shows that the desired property for the 2nd segment of TSEs is satisfied.
It is clear from the construction that $\psi_{j1},\psi_{j2},\ldots,\psi_{jq_{j-1}} \in \pem_1$ and that $\psi_{jq_j} \in \pem_2$ for $j=1,2$.
This implies that $f_1 + f_2 \in \pem_1 + \pem_2$ as desired.

\textit{Step 2: Disjointness from optimal set.}
We have already proved, in this construction, that every ``grow'' move adds a split that decreases bias, and therefore must be linearly independent from the existing PEM.
In other words, we have $\df(\tse^{i+1}) = \df(\tse^i) + 1$ for $i=0,1,\ldots,{k-1}$.
Expanding this gives $\df(\tsebad) = k + 1 = \sum_{j=1}^m (s_j - s_{j-1}) + 1$.
Since each split increments the number of leaf nodes by one, each $s_j - s_{j-1} + 1$ is equal to the number of cells in $\pem_j$, which we now compute as follows.
For $j=3,4,\ldots,m-1$, we have $s_j - s_{j-1} = \dim_1(f_j) - 1$,
while for $j=m$, we have $s_m - s_{m-1} = \prod_{l=m}^{m'} \dim_1(f_l) - 1$.
Likewise, we have $s_1 = \paren*{\dim_1(f_1)- 1}\paren*{\dim_1(f_2)-1} - 1$ and $s_2 - s_1 = 3$.
Putting these together gives
\begin{equation} \nonumber
    \df(\tsebad) = \sum_{j=3}^{m-1} \dim_1(f_j) + \prod_{l=m}^{m'} \dim_1(f_l) + \paren*{\dim_1(f_1)- 1}\paren*{\dim_1(f_2)-1} + 5 - m.
\end{equation}

Let us now show that this is suboptimal, by constructing a TSE $\tsegood$ with zero bias but fewer degrees of freedom.
To do so, we simply set $\tsegood = \paren*{\tree_1'\tree_2',\tree_3^*,\ldots,\tree_m^*}$, where $\tree_j^*$ has the same structure as in $\treebad$ for $j=3,4,\ldots,m$.
On the other hand, we define $\tree_1'$ using $\psi_{1,1},\psi_{1,2},\ldots,\psi_{1,q_1}$ and $\tree_2'$ using $\psi_{2,1},\psi_{2,2},\ldots,\psi_{2,q_2}$.
By assumption, we have $f_j \in \mapone(\tree_j')$ for $j=1,2$, which yields unbiasedness.
The degrees of freedom is given by
\begin{equation} \label{eq:tsegood_dof}
    \df(\tsegood) = \sum_{j=1}^{m-1} \dim_1(f_j) + \prod_{l=m}^{m'} \dim_1(f_l) - m + 1.
\end{equation}
Taking the difference gives
\begin{align}
    \df(\tsebad) - \df(\tsegood) & = \paren*{\dim_1(f_1)- 1}\paren*{\dim_1(f_2)-1} - \dim_1(f_1) - \dim_1(f_2) + 4\nn
    & = \paren*{\dim_1(f_1)- 2}\paren*{\dim_1(f_2)-2} + 1, \nonumber
\end{align}
which is strictly larger than 0 if $\dim_1(f_1),\dim_1(f_2) \geq 2$.
This condition is satisfied since we have assumed all component functions are non-constant.
Combining this with the fact that $\bias(\tse^i) > 0$ for all $i < s$, we therefore have $\tse^i \notin \opt(f^*,k)$ for all $i=0,1,\ldots,s$, with $k = \paren*{\dim_1(f_1)- 2}\paren*{\dim_1(f_2)-2}$.

\textit{Step 3: Conclusion.}
Using Proposition \ref{prop:concentration_bic_diff_appendix} and Proposition \ref{prop:lml_bic_appendix}, for $n$ large enough, there is a $1-\delta/2$ event over which
\begin{equation} \label{eq:additive_hpp_lower_bound_p1}
    \log p(\tse^i|\by) - \log p(\tse^{i-1}|\by) = \frac{n}{2\sigma^2}\paren*{\bias(\tse^{i-1};f^*) - \bias(\tse^i;f^*)} + O\paren*{\sqrt{n\log(s/\delta)}}.
\end{equation}
Conditioning on this set, we get $\frac{p(\tse^i|\by)}{p(\tse^{i-1}|\by)} = \Omega(1)$ for $i=1,2,\ldots,s$ for all $n$ large enough.
We may then repeat the calculations in the proof of Theorem \ref{thm:pure_interaction}, specifically equations \eqref{eq:hpp_eq1} and \eqref{eq:pure_interaction_hpp}, to get
\begin{equation} \nonumber
    P\braces*{\tau_{\tsebad} < \tau_{\opt(f^*,k)}} \geq P\braces*{\tse_i = \tse^i ~\text{for}~ i=1,2,\ldots,s} = \Omega(1).
\end{equation}



\subsection{BIC lower bound}
Consider $\tse'= (\tree_1',\tree_2',\ldots,\tree_m') \in \settwo$.
By definition of $\settwo$, there exists $\tse = (\tree_1,\tree_2,\ldots,\tree_m) \in \setone$ such that $\tse'$ is obtained from $\tse$ via a ``grow'', ``prune'', ``change'', or ``swap'' move.
We now consider each type of move and show that either $\bias(\tse';f^*) > 0$ or $\df(\tse') > \df(\tse) = \df(\tsebad)$.
To prove this, we make use of a few key observations.
\begin{longlist}
    \item Since a move only affects one tree with index $i_0$, we have $\tree_i' = \tree_i$ and $\mapone(\tree_i') = \pem_i$ for all $i \neq i_0$;
    \item Every split $(i,\xi_{i,j})$ necessary for $\mapone(\tse')$ to be unbiased (see Lemma \ref{lem:zero_bias_coverage}) has full coverage in a single tree and has zero coverage in all other trees;
    \item Since no move is allowed to result in an empty leaf node, if $\node$ is an internal node in $\tree_{i_0}$ or $\tree_{i_0}'$ with a split $(i,\xi)$, no ancestor or descendant of $\node$ makes the same split $(i,\xi)$.
\end{longlist}
For convenience, we denote $\pem_{-i_0} = \pem_1 + \cdots \pem_{i_0-1} + \pem_{i_0+1} + \cdots + \pem_m$.

\textit{Case 1: ``grow'' move.}
Let $\psi$ denote the local decision stump associated with the new split, and let $\leaf$ denote the leaf that is split.
As shown earlier, we have $\psi \perp \pem_{i_0}$.
Next, for any boundary point $(\bx^{-j},\xi)$ of $\leaf$ in direction $j$ (that is not an a boundary point of the entire space $\xspace$), $\psi$ has a jump location at $(\bx^{-j},\xi)$ with respect to feature $j$.
On the other hand, by our construction, this means that $(j,\xi)$ is a split fully covered by $\pem_{i_0}$ and has zero coverage in $\pem_{-i_0}$.
By Lemma \ref{lem:zero_bias_coverage}, this means $\psi \notin \pem_{-i_0}$.
Together, this implies that $\psi \notin \mapone(\tse)$ and $\df(\tse') = \df(\tse) + 1$ as desired.

\textit{Case 2: ``prune'' move.}
Suppose the pruned split is $(k,t)$ and occurs on a leaf $\leaf$.
Let $\leaf^{-k}$ denote the projection of the leaf onto all but the $k$-th coordinate.
Applying the third observation above together with Lemma \ref{lem:lack_of_coverage} gives
\begin{equation} \nonumber
    \coverage(k,t;\mapone(\tse'))\cap \leaf^{-k} = \coverage(k,t;\mapone(\tree_{i_0}')) \cap \leaf^{-k} = \emptyset.
\end{equation}
Lemma \ref{lem:zero_bias_coverage} then implies that $\bias(\tse';f^*) > 0$.

\textit{Case 3: ``change'' move.}
Suppose the ``changed'' split occurs on a node $\node$ and is with respect to a feature $k$ at threshold $\xi_{k,j}$.
Then since no descendant of $\node$ makes the same split, if we let $\node^{-k}$ denote the projection of $\node$ onto all but the $k$-th coordinate, then as before, we have
\begin{equation} \nonumber
    \coverage(k,t;\mapone(\tse'))\cap \node^{-k} = \coverage(k,t;\mapone(\tree_{i_0}')) \cap \node^{-k} = \emptyset,
\end{equation}
and $\bias(\tse';f^*) > 0$.

\textit{Case 4: ``swap'' move.}
The swap move can either be performed on a pair of parent-child nodes, or on a parent with both of its children, if both children have the same splitting rule.
In the latter case, $\mapone(\tree') = \mapone(\tree)$, contradicting our assumption that $\tse' \notin \setone$.
In the former case, let $\node$ denote the parent, and let $\node_L$ and $\node_R$ denote its two children.
Suppose without loss of generality that the splitting rules $(j,s)$ and $(k,t)$, of $\node$ and $\node_L$ respectively, are to be swapped.
Then by construction of $\tree_{i_0}$, $t$ must be a knot for $f_k$.
By the third observation, no ancestor of $\node_L$ uses the splitting rule $(k,t)$, which implies that a descendant of $\node_R$ must split on $(k,t)$.
However, this would mean that this swap move is not allowed, giving a contradiction.

Finally, using Proposition \ref{prop:concentration_bic_diff_appendix} and Proposition \ref{prop:lml_bic_appendix}, for $n$ large enough, there is a $1-\delta/2$ event over which
\begin{align} \label{eq:additive_bic_lower_bound}
    & \Delta\bic(\tse',\tse) \nn
    & \quad = \begin{cases}
        \frac{n}{\sigma^2}\bias(\tse';f^*) + O\paren*{\sqrt{n\log(|\settwo|/\delta)}} & \text{if}~ \bias(\tse';f^*) > 0 \\
        \log n\paren*{\df(\tse') - \df(\tsebad)} + O\paren*{\log(|\settwo|/\delta)} & \text{otherwise}.
    \end{cases}
\end{align}
Condition further on this event.

\subsection{Conclusion}
Applying Proposition \ref{prop:recipe} together with equations \eqref{eq:additive_hpp_lower_bound_p1} and \eqref{eq:additive_bic_lower_bound}, while taking $n$ large enough, we get a $1-2\delta$ probability event over which
\begin{equation} \nonumber
    E\braces*{\tau_{\opt(f^*,k)}} = \Omega\paren*{n^{1/2}},
\end{equation}
where $k = \paren*{\dim_1(f_1)- 2}\paren*{\dim_1(f_2)-2}$.



\section{Proof of Theorem \ref{thm:additive} Part 2}
\label{sec:proof_of_additive}

\subsection{Set-up}
We first use the same definition of $\tsegood = (\tree_1,\tree_2,\ldots,\tree_\ntrees)$ as in the previous section (see the paragraph immediately preceding equation \eqref{eq:tsegood_dof}).
To define $\tsebad = (\tree_1^*,\tree_2^*,\ldots,\tree_\ntrees^*)$, we start with $\tsegood$ and simply swap the roles of $f_1$ and $f_{m'}$.
More precisely, we let $\tree_j^* = \tree_j$ for $j=2,3,\ldots,\ntrees-1$.
We define $\tree_1^*$ as comprising the local decision stumps $\phi_{\ntrees',1},\phi_{\ntrees',2},\ldots,\phi_{\ntrees',q_{\ntrees'}}$, which were defined in the proof of the hitting precedence probability lower bound in the previous section.
Define $\pem_j = \mapone(\tree_j^*)$ for $j=1,2,\ldots,\ntrees-1$.
We then define $\pem_\ntrees$ to be a grid on features $\braces*{1,\ntrees,\ntrees+1,\ldots,\ntrees'-1}$, or in other words, $\mapone(\tree_\ntrees^*)$ is the span of indicators of the cells
\begin{equation}
    \leaf_{m,j_m,j_{m+1},\ldots,j_{m'}} \coloneqq \braces*{\bx \colon \xi_{i,j_{i-1}} < x_i \leq \xi_{i,j_{i}}~\text{for}~i=1,m,m+1,\ldots,m'-1},
\end{equation}
as we vary $i=1, m, m+1,\ldots,m'-1$ and $j_i=1,2,\ldots,q_i$.
We construct $\tree_\ntrees^*$ such that $\mapone(\tree_\ntrees^*)$, in the manner described in the proof of the hitting precedence probability lower bound in the previous section.
Here, recall that $\dim_1(f_1) > \dim_1(f_\ntrees)$, which will introduce suboptimality into $\tsebad$.

Notice that $f^* \in \pem \coloneqq \pem_1 + \pem_2 + \cdots \pem_\ntrees$ via the same argument as in the previous section.
Next, we define the set $\setone$ via
\begin{equation}\nonumber
    \setone \coloneqq \braces*{\paren*{\tse_1',\tse_2',\ldots,\tse_m'} \colon \mapone(\tse_j) \supseteq \pem_j ~\text{for}~j=1,2,\ldots,\ntrees}.
\end{equation}
We define $\settwo$ to be the outer boundary of $\setone$.

\subsection{Hitting precedence probability lower bound}
This follows the proof of the hitting precedence probability lower bound in the previous section almost exactly. 
First, using the same construction, there is path in $\tsespace$ comprising $\emptytse = \tse^0,\tse^1,\ldots,\tse^s = \tsebad$ such that $\tse^{i+1}$ is obtained from $\tse^i$ via a ``grow'' move and
\begin{equation}\nonumber
    \bias(\tse^{i};f^*) > \bias(\tse^{i-1};f^*)
\end{equation}
for $i=0,1,\ldots,s-1$.

Next, we compute
\begin{equation}\nonumber
    \df(\tsebad) = \sum_{j=2}^{\ntrees-1} \dim_1(f_j) + \dim_1(f_{m'}) + \dim_1(f_1)\prod_{j=\ntrees}^{\ntrees'-1}\dim_1(f_j) - \ntrees + 1.
\end{equation}
Taking the difference between this and \eqref{eq:tsegood_dof} gives
\begin{align} \label{eq:dof_diff}
    \df(\tsebad) - \df(\tsegood) & = \dim_1(f_{m'}) - \dim_1(f_1) + \dim_1(f_1)\prod_{j=\ntrees}^{\ntrees'-1}\dim_1(f_j) - \prod_{j=\ntrees}^{\ntrees'}\dim_1(f_j) \nn
    & = \paren*{\dim_1(f_1) - \dim_1(f_{\ntrees'})}\paren*{\prod_{j=\ntrees}^{\ntrees'-1}\dim_1(f_j) - 1},
\end{align}
which is strictly larger than $k \coloneqq \dim_1(f_1) - \dim_1(f_{\ntrees'}) - 1$ if $\dim_1(f_j) \geq 2$ for $j=2,3,\ldots,m'-1$.
This condition is satisfied since we have assumed all component functions are non-constant.
Combining this with the fact that $\bias(\tse^i) > 0$ for all $ i < s$, we therefore have $\tse^i \notin \opt(f^*,k)$.

Finally, using Proposition \ref{prop:concentration_bic_diff_appendix} and Proposition \ref{prop:lml_bic_appendix}, for $n$ large enough, there is a $1-\delta/2$ event over which
\begin{equation} \label{eq:additive_hpp_lower_bound}
    \log p(\tse^i|\by) - \log p(\tse^{i-1}|\by) = \frac{n}{2\sigma^2}\paren*{\bias(\tse^{i-1};f^*) - \bias(\tse^i;f^*)} + O\paren*{\sqrt{n\log(s/\delta)}}.
\end{equation}
Conditioning on this set, we get $\frac{p(\tse^i|\by)}{p(\tse^{i-1}|\by)} = \Omega(1)$ for $i=1,2,\ldots,s$ for all $n$ large enough.
We may then repeat the calculations in the proof of Theorem \ref{thm:pure_interaction}, specifically equations \eqref{eq:hpp_eq1} and \eqref{eq:pure_interaction_hpp}, to get
\begin{equation} \label{eq:additive_hpp_lower_bound_p2}
    P\braces*{\tau_{\tsebad} < \tau_{\opt(f^*,k)}} \geq P\braces*{\tse_i = \tse^i ~\text{for}~ i=1,2,\ldots,s} = \Omega(1).
\end{equation}

\subsection{BIC lower bound}
Consider $\tse'= (\tree_1',\tree_2',\ldots,\tree_m') \in \settwo$.
By definition of $\settwo$, there exists $\tse = (\tree_1,\tree_2,\ldots,\tree_m) \in \setone$ such that $\tse'$ is obtained from $\tse$ via a ``prune'' move (a ``grow'' move will not break the defining constraint of $\setone$.)
We now show that either $\bias(\tse';f^*) > 0$ or $\df(\tse') \geq \df(\tsebad) + \min_{1 \leq i \leq \ntrees'} \dim_1(f_i) - 2$.

Let $\tree_{i_0}$ be the tree that is pruned, and supposed the prune split is $(j,t)$, occurring on a node $\node$.
Since $\mapone(\tree_{i_0}') \not\supseteq \pem_{i_0}$, $(j,t)$ must be on a feature split on in $\tree_{i_0}^*$ and on a knot $\xi_{j,k}$ for $f_j$.
Furthermore, because only ``grow'' and ``prune'' moves are allowed, and this is the first time $\mapone(\tree_{i_0}') \not\supseteq \pem_{i_0}$, $\tree_{i_0}^*$ must be a subtree of $\tree_{i_0}$, and $\node$ is an internal node of $\tree_{i_0}^*$.
Using Lemma \ref{lem:lack_of_coverage}, we have $\coverage(j,t;\mapone(\tree_{i_0}'))\cap\node^{-j} = \emptyset$.

Suppose that $\bias(\tse';f^*) = 0$, i.e. $f^* \in \mapone(\tse')$.
Let $I_1,I_2,\ldots,I_\ntrees$ be the subsets of feature indices split on in trees $\tree_1^*,\tree_2^*,\ldots,\tree_\ntrees^*$ respectively.
We claim that there exists some tree $\tree_i$, $i\neq i_0$ such that for all choices of coordinates $\bx_{I_i}$ in the index set $I_i$, there exists a choice of coordinates $\bz_{-I_i\cup\braces{j}}$ in the index set $\braces*{1,2,\ldots,\ncats}\backslash (I_i\cup\braces{j})$ such that $(\bx_{I_i},\bz_{-I_i\cup\braces{j}}) \in \coverage(j,t;\mapone(\tree_i))$.
This claim is proved as Lemma \ref{lem:additive_p2_helper} below.
Assuming it for now, let $\leaf_1,\leaf_2,\ldots,\leaf_L$ be the leaves of $\tree_i^*$.
Consider one such leaf $\leaf_l$.
Since $\leaf_l$ does not depend on the features in $I_i$, we may pick $(\bx_{I_i},\bz_{-I_i}) \in \coverage(j,t;\mapone(\tree_i))$ such that $(\bx_{I_i},\bz_{-I_i}) \in \leaf_l$.
By Lemma \ref{lem:lack_of_coverage}, $\leaf_l$ in $\tree_i$ must contain a descendant that uses the splitting rule $(j,t)$.
In particular, $\leaf_l$ is split in $\tree_i$, and this split can be represented by a local decision stump $\psi_l$.
In this manner, we obtain $\psi_1,\psi_2,\ldots,\psi_L$.
Denote $\mathbb{U} = \linspan(\psi_1,\psi_2,\ldots,\psi_{L-1})$.
We claim that $\mathbb{U}\cap\pem = \emptyset$.
To see this, take any $f = \sum_{l=1}^{L-1} a_l \indicator_{\leaf_l} \in \mathbb{U}$ and assume that $f \neq 0$.
First note that $f\perp\pem_i$ by orthogonality of local decision stump features from a single tree.
Furthermore, $f$ has jump locations with respect to features in $I_i$.
But no function in $\pem_{-i}$ depends on features in $I_i$, which means that $f \notin \pem_{-i}$.
This gives $f \notin \pem$.

Let $\psi'$ denote the local decision stump associated to the pruned split, and let $\pem'$ denote its orthogonal complement in $\pem$.
We have $\pem' \subset \mapone(\tse')$ and have further shown that $\mathbb{U} \subset \mapone(\tse')$.
We therefore have
\begin{align}
    \df(\tse') & = \dim(\mapone(\tse')) \nn
    & \geq \dim(\pem') + \dim(\mathbb{U}) \nn
    & = \df(\tsebad) - 1 + L - 1 \nn
    & \geq \df(\tsebad) + \min_{1\leq i \leq m'}\dim_1(f_i) - 2.\nonumber
\end{align}
Here, the first inequality follows from the trivial intersection of $\pem'$ and $\mathbb{U}$.

Using Proposition \ref{prop:concentration_bic_diff_appendix} and Proposition \ref{prop:lml_bic_appendix}, there is a $1-\delta/2$ event over which
\begin{align} \label{eq:additive_bic_lower_bound_p2}
    & \Delta\bic(\tse',\tse) \nn
    & \quad = \begin{cases}
        \frac{n}{\sigma^2}\bias(\tse';f^*) + O\paren*{\sqrt{n\log(|\settwo|/\delta)}} & \text{if}~ \bias(\tse';f^*) > 0 \\
        \log n\paren*{\df(\tse') - \df(\tsebad)} + O\paren*{\sqrt{\log(|\settwo|/\delta)}} & \text{otherwise}.
    \end{cases}
\end{align}
Condition further on this event.

\subsection{Conclusion}
Applying Proposition \ref{prop:recipe} together with equations \eqref{eq:additive_hpp_lower_bound_p2} and \eqref{eq:additive_bic_lower_bound_p2}, while taking $n$ large enough, we get a $1-2\delta$ probability event over which
\begin{equation}\nonumber
    E\braces*{\tau_{\opt(f^*,k)}} = \Omega\paren*{n^{a/2 - 1}},
\end{equation}
where $k = \max_{1 \leq i \leq \ntrees'}\dim_1(f_i) - \min_{1\leq i \leq \ntrees'}\dim_1(f_i) - 1$ and $a = \min_{1 \leq i \leq \ntrees'} \dim_1(f_i)$.

\begin{lemma}[Existence of tree covering a cylinder]
    \label{lem:additive_p2_helper}
    There exists some tree $\tree_i$, $i\neq i_0$ such that for all choices of coordinates $\bx_{I_i}$ in the index set $I_i$, there exists a choice of coordinates $\bz_{-I_i\cup\braces{j}}$ in the index set $\braces*{1,2,\ldots,\ncats}\backslash (I_i\cup\braces{j})$ such that $(\bx_{I_i},\bz_{-I_i\cup\braces{j}}) \in \coverage(j,t;\mapone(\tree_i))$.
\end{lemma}

\begin{proof}
    Let $r_1,r_2,\ldots,r_\ntrees$ be a permutation of $\braces{1,2,\ldots,\ntrees}$, with $r_1 = i_0$ (this implies that $j \in I_{r_1}$).
    For $k=2,3,\ldots,m-1$, set $J_k = \cup_{i \geq k} I_{r_i}$, and set $\pem^k = \mapone(\tree_{r_k}) + \mapone(\tree_{r_{k+1}}) + \cdots + \mapone(\tree_{r_m})$.

    We make the following observation:
    For some $k$, suppose there exists a cylinder set $\cell_k \subset \coverage(j,t;\pem^k)$ that does not depend on any feature in $J_k$.
    Suppose that there exists some $\bx_{I_{r_k}}$ in the projection of $\cell$ to coordinates in $I_{r_k}$ such that $(\bx_{I_i},\bz_{-I_i\cup\braces{j}}) \notin \coverage(j,t;\mapone(\tree_{r_k}))$ for all choices of $\bz_{-I_i\cup\braces{j}}$.
    Then taking the intersection of $\cell_k$ and $\braces{\bx_{I_{r_k}}} \times \braces*{1,2,\ldots,\ncats}^{-I_{r_k}\cup\braces{j}}$ gives a cylinder set $\cell_{k+1}$ that does not depend on any feature in $J_{k+1}$ and such that $\cell_{k+1} \subset \coverage(j,t;\pem^{k+1})$.

    Now, since $\node^{-j}\cap \coverage(j,t;\mapone(\tree_{r_1}')) = \emptyset$, we have $\node^{-j} \subset \coverage(j,t;\pem^2)$.
    Since $\node^{-j}$ is an internal node of $\tree_{r_1}^*$, it is a cylinder set that does not depend on any feature in $J_2$.
    By applying the above observation inductively on $k=2,3,\ldots,m-1$, we obtain the statement of the lemma.
\end{proof}

\section{Proof of Proposition \ref{prop:recipe}}

Proposition \ref{prop:recipe} will follow almost immediately from the following more general result.

\begin{proposition}[General statement for recipe]
    \label{prop:recipe_general}
    Let $X_0,X_1,\ldots$ be an irreducible and aperiodic discrete time Markov chain on a finite state space $\Omega$, with stationary distribution $\pi$.
    Let $x \in \Omega$ be a state and $\setthree \subset \Omega$ be a subset such that their hitting times from an initial state $x_0 \in \Omega$ satisfy
    \begin{equation}\nonumber
        P\braces*{\tau_x < \tau_{\setthree}~|~X_0=x_0} \geq c
    \end{equation}
    for some constant $c$.
    Let $\settwo \subset \Omega$ be a subset such that every path from $x$ to $\setthree$ intersects $\settwo$.
    Then the hitting time of $\setthree$ satisfies
    \begin{equation}\nonumber
        E\braces*{\tau_{\setthree}~|~X_0 = x_0} \geq \frac{c\pi(x)}{\pi(\settwo)}.
    \end{equation}
\end{proposition}

\begin{proof}
    By conditioning on the event $\braces{\tau_x < \tau_{\setthree}}$, we calculate
    \begin{align}
        E\braces*{\tau_{\setthree} ~|~ X_0 = x_0} & \geq E\braces*{\tau_{\setthree} \indicator\braces*{\tau_x < \tau_{\setthree}} ~|~ X_0 = x_0} \nn
        & = E\braces*{\tau_{\setthree}  ~|~ X_0 = x_0, \tau_x < \tau_{\setthree}}P\braces*{\tau_x < \tau_{\setthree}~|~X_0 = x_0}.\nonumber
    \end{align}
    The second multiplicand on the right is lower bounded by $c$ by assumption, so we just need to bound the first one.
    Using the strong Markov property, we first lower bound this as:
    \begin{align} \label{eq:recipe_pf_eq1}
        E\braces*{\tau_{\setthree}  ~|~ X_0 = x_0, \tau_x < \tau_{\setthree}} & = E\braces*{\tau_{\setthree}  ~|~ X_0 = x} + E\braces*{ \tau_x | \tau_x < \tau_{\setthree}} \nn
        & \geq E\braces*{\tau_{\setthree}  ~|~ X_0 = x}.
    \end{align}
    
    Let $Z$ denote the number of times $(X_t)$ returns to $x$ before hitting $\settwo$.
    Then, assuming that $X_0 = x$, we have the inequalities
    \begin{equation*}
        \tau_{\setthree} \geq \tau_{\settwo} \geq Z+1.
    \end{equation*}
    Note that $Z + 1$ is a geometric random variable with success probability
    \begin{equation*}
        p = \braces*{\tau_x^+ > \tau_{\settwo} ~|~ X_0 = x},
    \end{equation*}
    where $\tau_x^+$ is the first return time to $x$, i.e.
    \begin{equation*}
        \min \braces*{t > 0 \colon X_t = x}.
    \end{equation*}
    We therefore continue \eqref{eq:recipe_pf_eq1} to get
    \begin{equation} \label{eq:recipe_pf_eq2}
        E\braces*{\tau_{\setthree}  ~|~ X_0 = x} \geq \frac{1}{P\braces*{\tau_x^+ > \tau_{\settwo} ~|~ X_0 = x}}.
    \end{equation}

    We next write this probability in terms of another random variable $W$, which we define to be the number of visits to states in $\settwo$ before returning to $x$, when the chain is started at $X_0 = x$.
    We then have
    \begin{equation} \label{eq:recipe_pf_eq3}
        P\braces*{\tau_x^+ > \tau_{\settwo} ~|~ X_0 = x} = P\braces*{W \geq 1} \leq E\braces*{W}.
    \end{equation}
    To bound this expectation, for each $y \in \settwo$, let $W_y$ denote the number of visits to $y$ before returning to $x$, and observe that $W = \sum_{y \in \settwo} W_y$.
    Let $\pi$ denote the unique stationary distribution of $(X_t)$.
    Using Lemma \ref{lem:num_visits_bound}, we then have $E\braces{W_y} \leq \pi(y)/\pi(x)$.
    Adding up these inequalities, we get
    \begin{equation} \label{eq:recipe_pf_eq4}
        E\braces{W} = \sum_{y \in \settwo} E\braces{W_y} \leq \frac{\sum_{y \in \settwo} \pi(y)}{\pi(x)} = \frac{\pi(\settwo)}{\pi(x)}.
    \end{equation}
    Combining equations \eqref{eq:recipe_pf_eq2}, \eqref{eq:recipe_pf_eq3}, and \eqref{eq:recipe_pf_eq4} completes the proof.
\end{proof}

\begin{proof}[Proof of Proposition \ref{prop:recipe}]
    It is clear that the Markov chain induced by a run of the BART sampler is irreducible and aperiodic, with stationary distribution given by the marginal posterior $p(\tse|\by)$.
    We hence use Proposition \ref{prop:recipe_general} to get
    \begin{equation}\nonumber
        E\braces*{\tau_{\opt(f^*,k)}} = \Omega\paren*{\frac{p(\tsebad|\by)}{p(\settwo|\by)}}.
    \end{equation}
    We now compute
    \begin{align}
        \frac{p(\tsebad|\by)}{p(\settwo|\by)} & \geq \frac{1}{\abs*{\settwo}}\min_{\tse \in \settwo}\frac{p(\tsebad|\by)}{p(\tse|\by)} \nn
        & = \frac{1}{\abs*{\settwo}}\exp\paren*{\min_{\tse \in \settwo}\braces*{\log p(\tsebad|\by) - \log p(\tse|\by)} } \nn
        & \geq \frac{1}{\abs*{\settwo}}\exp\paren*{\frac{1}{2}\min_{\tse \in \settwo}\Delta\bic(\tse,\tsebad) - O\paren*{\sqrt{\log(\abs*{\settwo}/\delta)}}} \nn
        & = \Omega\paren*{\exp\paren*{\frac{1}{2}\min_{\tse \in \settwo}\Delta\bic(\tse,\tsebad)}},\nonumber
    \end{align}
    where the second inequality follows from the proof of Proposition \ref{prop:lml_bic_appendix} and holds with probability at least $1-\delta$.
\end{proof}

\begin{lemma}[Bounding number of visits] \label{lem:num_visits_bound}
    Let $X_0,X_1,\ldots$ be an irreducible and aperiodic discrete time Markov chain on a finite state space $\Omega$, with stationary distribution $\pi$.
    For any two states $x, y \in \Omega$, we have
    \begin{equation} \label{eq:num_visits_bound}
        \E\braces*{\textnormal{number of visits to $y$ before returning to $x$}|X_0 = x} = \frac{\pi(y)}{\pi(x)}.
    \end{equation}
\end{lemma}

\begin{proof}
    Fix $x$ and denote the quantity on the left side of equation \eqref{eq:num_visits_bound} by $\tilde \pi(y)$.
    We may then rewrite equations (1.25) and (1.26) of \cite{LevinPeresWilmer2006} in our notation as follows:
    \begin{equation*}
        \pi(y) = \frac{\tilde{\pi}(y)}{E\braces*{\tau_x^+|X_0 = x}},
    \end{equation*}
    \begin{equation*}
        \pi(x) = \frac{1}{E\braces*{\tau_x^+|X_0 = x}},
    \end{equation*}
    where $\tau_x^+$ is the first return time to $x$.
    Taking the ratio of the two equations completes the proof.
\end{proof}

%% file: AOS/a4_upper_bounds.tex
\section{Proofs for Section \ref{sec:upper_bounds}}
\label{sec:upper_bounds_proofs}

We first discuss some preliminaries that will be used in all of our proofs.
As described in Section \ref{sec:proof_sketch}, for $i=1,2,\ldots,m'$, denote $q_i = \dim_1(f_i)$, and let $0=\xi_{i,0} < \xi_{i,1} < \cdots < \xi_{i,q_{i}} = b$ denote the \emph{knots} of $f_i$, i.e. the values for which $f_i(\xi_{i,j}) \neq f_i(\xi_{i,j}+1)$, together with the endpoints.
Recall that $\tsespace[1]$ denotes the space of all tree structures.
Furthermore, we use the following adaptation of our notation from Section \ref{sec:proof_sketch}:
Given $\tse, \tse' \in \tsespace$, we say that 
$\tse \succsim \tse'$ if $\tse$ and $\tse'$ are connected by a \emph{grow/prune} edge and if either $\bias(\tse;f^*) > \bias(\tse';f^*)$ or $\Pi_\tse[f^*] = \Pi_{\tse'}[f^*]$ and $\df(\tse) \geq \df(\tse')$.
We call a path $(\tse^1,\tse^2,\ldots,\tse^l)$ in which $\tse^{j-1} \succsim \tse^j$ for $j=1,2,\ldots,l$ a \emph{monotonic} path.

\begin{lemma}
    \label{lem:many_trees_helper2}
    Suppose $\tse \succsim \tse'$ and $\tse \precsim \tse'$.
    Then $\bic(\tse) = \bic(\tse')$.
    Furthermore, 
    there exists $0 < c_1 \leq 1$ such that
    \begin{equation} \label{eq:tse_comparison_upper_bound_helper}
        \lim_{n \to \infty}\P_n\braces{p(\tse|\by) \geq c_1 p(\tse'|\by) ~\text{for all}~ \tse,\tse' \in \tsespace~\text{such that}~\tse \precsim \tse'} = 1.
    \end{equation}
\end{lemma}

\begin{proof}
    For the first statement, we note that if $\tse$ and $\tse'$ are connected by a grow/prune move, then $\mathcal{F}(\tse)$ and $\mathcal{F}(\tse')$ differ by at most a 1-dimensional subspace.
    If $\tse \succsim \tse'$ and $\tse \precsim \tse'$ simultaneously, then we must have $\df(\tse) = \df(\tse')$, which implies that $\mathcal{F}(\tse) = \mathcal{F}(\tse')$ and thus $\bic(\tse) = \bic(\tse')$.
    For the second statement, we use Proposition \ref{prop:lml_bic_appendix} with $\delta = 1/n$, while noting Remark \ref{rem:lml_bic}.
    Taking a union bound over all pairs $\tse, \tse' \in \tsespace$, the conclusion follows.
\end{proof}




\begin{proof}[Proof of Theorem \ref{thm:mixing_upper_bound_ntrees}]
The idea is to apply Lemma \ref{lem:enhanced_mixing_time} to the Markov chain, with the high posterior region being $\opt(f^*,0)$.
To actualize this, we need to lower bound the spectral gap $\gamma$ as well as the path probability assumed in Lemma \ref{lem:hitting_time_upper}.
Both of these bounds will stem from the construction of a canonical path ensemble $\Gamma = \braces{\gamma_{\tse,\tse'}}_{\tse,\tse' \in \tsespace}$.
In this construction, every path $\gamma_{\tse,\tse'} = (\tse = \tse^0,\tse^1,\ldots,\tse^l=\tse')$ can be divided into 3 segments $(\tse^0,\tse^1,\ldots,\tse^{i_1-1})$, $(\tse^{i_1},\tse^{i_1+1},\ldots,\tse^{i_2-1})$, and $(\tse^{i_2},\tse^{i_2+1},\ldots,\tse^{l})$, such that (1) $\tse^{j-1} \succsim \tse^{j}$ for $1 \leq j \leq i_1$, (2) $\tse^j \in \opt(f^*,0)$ for $i_1 \leq j \leq i_2-1$, and (3) $\tse^{j-1} \precsim \tse^{j}$ for $i_2 \leq j \leq l$.
In other words, segments (1) and (3) are monotonic, while segment (2) is contained entirely in $\opt(f^*,0)$.
One or more of these segments are potentially empty.
Note that since the Markov chain is reversible, the ordering of the path does not matter.

\emph{Step 1: Lower bound for path probability to HDPR.}
Assume for now that this construction is possible and condition on the event in equation \eqref{eq:tse_comparison_upper_bound_helper}.
For any path $\gamma_{\tse,\tse'}$, the transition probability $P(\tse^{j-1},\tse^j)$ for $1 \leq j \leq i_1$ satisfies
\begin{equation}
\label{eq:more_trees_upper_helper1}
\begin{split}
    P(\tse^{j-1},\tse^j) & = Q(\tse^{j-1},\tse^j)\alpha(\tse^{j-1},\tse^j) \\
    & = \min\braces*{Q(\tse^{j},\tse^{j-1}) \cdot \frac{p(\tse^j|\by)}{p(\tse^{j-1}|\by)}, Q(\tse^{j-1},\tse^j)} \\
    & \geq c_1\min\braces*{Q(\tse^{j},\tse^{j-1}) , Q(\tse^{j-1},\tse^j)}.
\end{split}
\end{equation}
Now let
\begin{equation} \label{eq:c2_def}
    c_2 \coloneqq \min_{\substack{\tse,\tse' \in \tsespace \\ \tse \sim \tse'}}Q(\tse,\tse')
\end{equation}
where the minimum is taken over all pairs of TSEs that are connected by an edge under $Q$.
Recall that $l(\Gamma)$ denotes the maximum length of a path in the ensemble $\Gamma$.
Then the construction together with the calculation \eqref{eq:more_trees_upper_helper1} implies that there is a path from any $\tse$ to $\opt(f^*,0)$ whose path probability is lower bounded by $(c_1c_2)^{l(\Gamma)}$.
For the lazified chain, note that the same path has probability at least $(c_1c_2/2)^{l(\Gamma)}$.

\emph{Step 2: Lower bound for spectral gap.}
Meanwhile, to bound the congestion parameter $\rho(\Gamma)$, consider an edge $(\widetilde{\tse},\widetilde{\tse}')$.
Its edge probability under the lazy transition $\check{P}$ is
\begin{equation}
\begin{split}
    p(\widetilde{\tse}|\by) \check{P}(\widetilde{\tse},\widetilde{\tse}') & =
    \frac{1}{2}p(\widetilde{\tse}|\by) P(\widetilde{\tse},\widetilde{\tse}') \\
    & = \frac{1}{2}\min\braces*{p(\widetilde{\tse}'|\by)Q(\widetilde{\tse}',\widetilde{\tse}), p(\widetilde{\tse}|\by)Q(\widetilde{\tse},\widetilde{\tse}')} \\
    & \geq \frac{c_2}{2} \min\braces*{p(\widetilde{\tse}'|\by), p(\widetilde{\tse}|\by)}.
\end{split}
\end{equation}
Suppose that this edge occurs along $\gamma_{\tse,\tse'}$, the canonical path connecting $\tse$ to $\tse'$.
Then by the construction, we have
\begin{equation}
    \min\braces*{p(\widetilde{\tse}'|\by), p(\widetilde{\tse}|\by)} \geq c_1\min\braces*{p(\tse|\by),p(\tse'|\by)}.
\end{equation}
This implies that
\begin{equation}
\begin{split}
    \frac{1}{p(\widetilde{\tse}|\by) \check{P}(\widetilde{\tse},\widetilde{\tse}')}\sum_{\gamma_{\tse,\tse'} \ni (\widetilde{\tse},\widetilde{\tse}')} p(\tse|\by)p(\tse'|\by) & \leq \frac{2}{c_2}\sum_{\gamma_{\tse,\tse'} \ni (\widetilde{\tse},\widetilde{\tse}')} \frac{\min\braces*{p(\tse|\by),p(\tse'|\by)}}{\min\braces*{p(\widetilde{\tse}'|\by), p(\widetilde{\tse}|\by)}} \\
    & \leq \frac{2\abs*{\Gamma}}{c_1c_2}.
\end{split}
\end{equation}
Using the definition \eqref{eq:congestion_def}, we thus obtain the congestion upper bound $\rho(\Gamma) \leq 2\abs*{\Gamma}/c_1c_2$.

We now put the ingredients together.
Using the canonical path lemma (Lemma \ref{lem:canonical_path}), we can lower bound the spectral gap by
\begin{equation}
    \gamma \geq \frac{c_1c_2}{2\abs*{\Gamma}l(\Gamma)},
\end{equation}
where $l(\Gamma)$ is the maximum length of a path in $\Gamma$.
The assumptions of Lemma \ref{lem:enhanced_mixing_time} are hence satisfied with this lower bound, together with the the path probability lower bound $(c_1c_2/2)^{l(\Gamma)}$ and path length upper bound $l(\Gamma)$.
Since none of these parameters depend on $n$, the resulting mixing time bound from \eqref{eq:enhanced_mixing_time} also does not depend on $n$ and can be taken to be a fixed constant $C > 0$.
This implies the desired conclusion \eqref{eq:more_trees_mixing}.

\emph{Step 3: Canonical path ensemble construction.}
We first use induction to show that, for any $\tse$, there exists a path to $\opt(f^*,0)$ satisfying (1), with segments (2) and (3) empty.
Our induction will be triply nested on $\df(\tse)$ (outer induction), $\bias(\tse;f^*)$ (middle induction), as well as $\#\operatorname{leaves}(\tse)$ (inner induction), with the base case $\df(\tse) = \dim_m(f^*)$, $\bias(\tse;f^*) = 0$, and $\#\operatorname{leaves}(\tse) = \sum_{i=1}^{m'} q_i - m' + m$, where the minimum leaf count is obtained from Proposition \ref{prop:additive_dim}.
Note that since $\tsespace$ is finite, the number of possible values for all three quantities forms a finite grid.
Under the base case, the empty path suffices.
Otherwise, consider $\tse \notin \opt(f^*,0)$ and write $\tse = (\tree_1,\tree_2,\ldots,\tree_m)$.
Let us consider three cases:

\begin{enumerate}
\item[(i)] There exists a tree $\tree_k$ containing a leaf that can be pruned without changing $\Pi_\tse[f^*]$ to get a TSE $\tse'$.
Then $\tse \succsim \tse'$ and $\#\operatorname{leaves}(\tse') \leq \#\operatorname{leaves}(\tse)$ or $\df(\tse') < \df(\tse)$, so that we are done by the induction hypothesis.
\item[(ii)] Suppose Case (i) does not hold and $\bias(\tse;f^*) > 0$.
For $i=1,2,\ldots,m'$, $j=1,2,\ldots,q_i-1$, let $\psi_{i,j}$ be the decision stump that splits on feature $i$ at threshold $\xi_{i,j}$.
Furthermore, let $\psi_0 \equiv 1$ be the constant function.
Note that
\begin{equation}
    f^* \in \operatorname{span}\paren*{\braces*{\psi_{i,j} \colon 1 \leq i \leq m', 1 \leq j \leq q_i-1}\cup\braces{\psi_0}}.
\end{equation}
Let $\pem = \mathcal{F}(\tse)$.
By Lemma \ref{lem:many_trees_helper}, there must be some $\psi_{i,j}$ such that
\begin{equation}
    \norm*{\Pi_{\pem + \operatorname{span(\psi_{i,j})}}[f^*]}_2^2 < \norm*{\Pi_\pem[f^*]}_2^2.
\end{equation}
Since $m \geq 2\abs*{\tsespace[1]}$, by the pigeonhole principle, there must be at least $\abs*{\tsespace[1]}$ trees that are trivial, otherwise they could be pruned according to Case (i).
Hence, by making a split corresponding to $\psi_{i,j}$ on one of these trivial root nodes gives a TSE $\tse'$ such that $\tse \succsim \tse'$ and $\bias(\tse';f^*) < \bias(\tse;f^*)$.
We may then use the induction hypothesis.
\item[(iii)] Suppose Case (i) and Case (ii) do not hold.
Then $f^* \in \mathcal{F}(\tse)$.
Since there are at least $\abs*{\tsespace[1]}$ trees that are trivial in $\tse$, we may sequentially make root splits corresponding to all of $\braces*{\psi_{i,j} \colon 1 \leq i \leq m', 1 \leq j \leq q_i-1}$.
Each of these root splits does not increase $\df(\tse)$ since $\psi_{i,j} \in \mathcal{F}(\tse)$, so they give a monotonic path.
Subsequently, we may sequentially prune all of the original nontrivial trees in $\tse$, which also gives a monotonic path.
Concatenating these together directly gives a monotonic path from $\tse$ to $\opt(f^*,0)$.
\end{enumerate}

Next, define
\begin{equation}
    \mathcal{O} \coloneqq \braces*{\tse \in \opt(f^*,0) \colon \max_{\tree \in \tse} \operatorname{depth}(\tree) \leq 1}.
\end{equation}
Note that for any $\tse \in \opt(f^*,0)$, following the same construction as in Case (iii) above gives a path within $\opt(f^*,0)$ connecting it with $\mathcal{O}$.
It is also clear that we can perform any transposition of tree indices using a path lying in $\opt(f^*,0)$, so that all elements of $\mathcal{O}$ (and hence $\opt(f^*,0)$) are connected.

We have now shown individual constructions for segments (1), (2), and (3) for any pair $\tse, \tse' \in \tsespace$.
By concatenating these segments together for each such pair, we get the desired canonical path ensemble $\Gamma$.
\end{proof}


\begin{proof}[Proof of Theorem \ref{thm:mixing_upper_bound_multistep}]
    For notational simplicity, denote $D \coloneqq \text{diam}(\tsespace))$.
    Let $\check{P}$ denote the transition kernel of the lazified Markov chain as before.
    Let $c_2$ be defined as in equation \eqref{eq:c2_def}, noting that $\tse \sim \tse$ (i.e. a self-loop exists) for all $\tse \in \tsespace$.
    If $r \geq D$, then for any $\tse,\tse' \in \tsespace$, we have $Q^r(\tse,\tse') \geq c_2^r$.
    This implies that
    \begin{equation} \label{eq:multistep_helper}
    \begin{split}
        p(\tse|\by)\check{P}(\tse,\tse')
        & = \frac{1}{2} \min\braces*{p(\tse|\by)Q^r(\tse,\tse'), p(\tse'|\by)Q^r(\tse',\tse)} \\
        & \geq \frac{c_2^r}{2} \min\braces*{p(\tse|\by),p(\tse'|\by)}.
    \end{split}
    \end{equation}
    Since every pair of TSEs are connected by an edge, we can therefore define a canonical path ensemble $\Gamma$ such that the path between any pair $\tse, \tse'$ is simply the edge between them.
    Using \eqref{eq:multistep_helper}, we can show that this ensemble has congestion parameter bounded from above by
    \begin{equation}
        \rho(\Gamma) \coloneqq \max_{\tse,\tse' \in \tsespace}\frac{p(\tse|\by)p(\tse'|\by)}{p(\tse|\by) \check{P}(\tse,\tse')}  \leq \frac{2}{c_2^r}.
    \end{equation}
    As such, using Lemma \ref{lem:canonical_path}, the spectral gap is lower bounded as $\gamma \geq c_2^r/2$.
    
    Next, condition on the event in equation \eqref{eq:tse_comparison_upper_bound_helper}. For any $\tse \in \tsespace$, pick any $\tse' \in \opt(f^*,0)$.
    Similar to \eqref{eq:more_trees_upper_helper1}, we have
    \begin{equation}
        \check P(\tse,\tse') \geq \frac{c_1}{2}Q^r(\tse,\tse') \geq \frac{c_1 c_2^r}{2},
    \end{equation}
    with this edge forming a path from $\tse$ to $\opt(f^*,0)$.
    The assumptions of Lemma \ref{lem:enhanced_mixing_time} are hence satisfied with these lower bounds for the path probability and for $\gamma$.
    Since none of these parameters depend on $n$, the resulting mixing time bound from \eqref{eq:enhanced_mixing_time} also does not depend on $n$ and can be taken to be a fixed constant $C > 0$, which implies the desired conclusion.
\end{proof}

\begin{proof}[Proof of Theorem \ref{thm:mixing_upper_bound_temperature}]
    \emph{Step 1: Canonical path ensemble construction.}
    We perform a similar construction as that in the proof of Theorem \ref{thm:mixing_upper_bound_ntrees}.
    Towards, this we first claim that $\opt(f^*,\infty)$ is fully connected as an induced subgraph of $\tsespace$.
    We will prove this claim by induction, where the inductive hypothesis is that there is a path in $\opt(f^*,\infty)$ between any pair $\tse,\tse' \in \opt(f^*,\infty)$, with $\tse = (\tree_1,\tree_2,\ldots,\tree_m), \tse' = (\tree_1',\tree_2',\ldots,\tree_m')$ with $\tree_i = \tree_i'$ for $i=1,2,\ldots,m-j$.
    The claim is equivalent to the inductive hypothesis for $j=m$ and we use $j=0$ as the base case.
    Assuming the inductive hypothesis for some index $j-1$, consider the pair $\tse,\tse'$ described above.
    We perform the following sequence of moves: first place all relevant splits of $\tree_{m-j+1}$ onto $\tree_m$, then prune $\tree_j$ to the trivial tree, then grow it into a copy of $\tree_{m-j+1}'$.
    If $j=1$, then we instead first place all relevant splits of $\tree_m$ onto $\tree_1$.
    After pruning and regrowing the tree into a copy of $\tree_m'$, we then prune all the additional splits on the modified version of $\tree_1$.
    It is easy to check that this gives a path in $\opt(f^*,\infty)$.
    Concatenating this to the path given by the inductive hypothesis for $j-1$ proves the hypothesis for index $j$.
    This completes the proof of the claim.
    
    For any $\tse \in \tsespace\backslash\opt(f^*,\infty)$, it is clear that there is a path $(\tse = \tse^0,\tse^1,\ldots,\tse^l)$, where $\tse^l \in \opt(f^*,\infty)$ and $\tse^{i}$ is obtained from $\tse^{i-1}$ using a ``grow'' move for $i=1,2,\ldots,l$.
    For instance, a naive way to obtain such a path is to iteratively split the leaves of $\tree_1$ along each relevant feature $j$ and on the knots of $f_j$.
    Although we may not have $\tse^{i-1} \gtrsim \tse^i$ (because the degrees of freedom might increase), this path is still monotonic with respect to bias.
    Finally, for each pair $\tse, \tse' \in \tsespace$, we define their canonical path $\gamma_{\tse,\tse'}$ to be the concatenation of their respective paths to $\opt(f^*,\infty)$ together with an arbitrary path in $\opt(f^*,\infty)$ connecting the $\opt(f^*,\infty)$ endpoints of the two paths.

    \emph{Step 2: Lower bound for path probability to HDPR.}
    We begin this step with some preliminaries.
    Fix $\delta > 0$.
    Choose the failure probability to be $\delta/2|\tsespace|^2$ in Propositions \ref{prop:concentration_bic_diff_appendix} and \ref{prop:lml_bic_appendix}, and condition on the intersection of the high probability events guaranteed by these two propositions as we vary over pairs $\tse, \tse' \in \tsespace$.
    Using the union bound, this event, which we denote as $A_n$, has probability at least $1-\delta$.
    Let $P$ and $\check P$ denote the transition matrices of the original and lazified Markov chains as before.
    Furthermore, define the following constants:
    \begin{itemize}
        \item Let $c_2$ be defined as in equation \eqref{eq:c2_def};
        \item Let $c_3 \coloneqq \min_{\tse,\tse' \in \tsespace} p(\tse)/p(\tse')$ be the minimum ratio of prior probabilities;
        \item Let $c_4 \coloneqq \max_{\tse \in \tsespace} \df(\tse)$.
    \end{itemize}
    
    Now consider any $\tse \in \tsespace$ and let $(\tse = \tse^0,\tse^1,\ldots,\tse^l = \tse')$ be its canonical path to an element $\tse' \in \opt(f^*,0)$.
    For $i=1,2,\ldots,l$, we compute
    \begin{equation} \label{eq:tempered_path_prob_lower_bound_helper}
        \begin{split}
            \check{P}(\tse^{i-1},\tse^i) & \geq \frac{c_2}{2} \min\braces*{\frac{p(\tse^{i};\by,T)}{p(\tse^{i-1};\by,T)}, 1} \\
            & = \frac{c_2}{2}\min\braces*{ \paren*{\frac{p(\by|\tse^i)}{p(\by|\tse^{i-1})}}^{1/T} \frac{p(\tse^i)}{p(\tse^{i-1})}, 1} \\
            & \geq \frac{c_2c_3}{2} \min\braces*{ \paren*{\frac{p(\by|\tse^i)}{p(\by|\tse^{i-1})}}^{1/T}, 1}.
        \end{split}
    \end{equation}
    Because we conditioned on $A_n$, the fractional marginal likelihood ratio satisfies
    \begin{equation} \label{eq:lmlkhd_ratio_tempered}
            \log\paren*{\paren*{\frac{p(\by|\tse^i)}{p(\by|\tse^{i-1})}}^{1/T}} = \frac{\Delta\bic(\tse^{i-1},\tse^i) + O(1)}{2T}.
    \end{equation}
    Next, since the canonical path was constructed to be non-increasing in bias, we have three possibilities:
    \begin{equation}
        \Delta\bic(\tse^{i-1},\tse^i) = 
        \begin{cases}
            \frac{\beta n}{\sigma^2} + O(\sqrt{n}) & \text{if}~\beta \coloneqq \bias(\tse^{i-1};f^*) - \bias(\tse^i;f^*)) > 0 \\
            -\log n + O(1) & \text{if}~\beta = 0~\text{and}~\df(\tse^i) > \df(\tse^{i-1}) \\
            O(1) & \text{otherwise}.
        \end{cases}
    \end{equation}
    Here, the big-$O$ notation elides dependence on $\delta$.
    Nonetheless, plugging these formulas back into \eqref{eq:lmlkhd_ratio_tempered} and using the assumption $T \geq \log^\beta n$, we get the bound
    \begin{equation}
        \log\paren*{\paren*{\frac{p(\by|\tse^i)}{p(\by|\tse^{i-1})}}^{1/T}} \geq - 2\log^{1-\beta}n
    \end{equation}
    for all $n$ large enough, irrespective of the choice of $\delta$.
    Combining this bound with \eqref{eq:tempered_path_prob_lower_bound_helper} and taking the product over all edges in the path, we are able to lower bound the probability of the path under $\check P$ by $(c_2c_3/2)^{l(\Gamma)} \exp(-2l(\Gamma)\log^{1-\beta} n)$.

    \emph{Step 3: Lower bound for spectral gap.}
    Let $Z \coloneqq \sum_{\tse \in \tsespace}p(\tse)p(\by|\tse)^{1/T}$.
    Consider an edge $(\widetilde{\tse},\widetilde{\tse}')$.
    Its edge probability under $\check{P}$ is
    \begin{equation} \label{eq:tempered_path_congestion_helper1}
    \begin{split}
        p(\widetilde{\tse};\by,T) \check{P}(\widetilde{\tse},\widetilde{\tse}') 
        & = \frac{1}{2}\min\braces*{p(\widetilde{\tse}';\by,T)Q(\widetilde{\tse}',\widetilde{\tse}), p(\widetilde{\tse};\by,T)Q(\widetilde{\tse},\widetilde{\tse}')} \\
        & \geq \frac{c_2}{2Z} \min\braces*{p(\widetilde{\tse}'|\by)^{1/T}p(\widetilde{\tse}'), p(\widetilde{\tse}|\by)^{1/T}p(\widetilde{\tse})} \\
        & \geq \frac{c_2}{2Z} \min\braces*{p(\widetilde{\tse}'|\by)^, p(\widetilde{\tse}|\by)}^{1/T}\min\braces*{p(\widetilde{\tse}'),p(\widetilde{\tse})}.
    \end{split}
    \end{equation}
    Suppose that this edge occurs along $\gamma_{\tse,\tse'}$, the canonical path connecting $\tse$ to $\tse'$.
    By our construction, the edge either along one of the monotonic segments or is contained within $\opt(f^*,\infty)$.
    Hence, we must have
    \begin{equation}
        \max\braces*{\bias(\widetilde{\tse}';f^*), \bias(\widetilde{\tse};f^*)} \leq \max\braces*{\bias({\tse}';f^*), \bias({\tse};f^*)}.
    \end{equation}
    Without loss of generality, we may assume that the maximum on left and right hand sides occur on $\widetilde{\tse}$ and $\tse$ respectively.
    Because of our conditioning on the set $A_n$, for all $n$ large enough, we have
    \begin{equation} \label{eq:tempered_path_congestion_helper}
    \begin{split}
        \log\paren*{\frac{\min\braces*{p({\tse}'|\by)^, p({\tse}|\by)}^{1/T}}{\min\braces*{p(\widetilde{\tse}'|\by)^, p(\widetilde{\tse}|\by)}^{1/T}}} & = \frac{1}{T}\log\paren*{\frac{p(\tse|\by)}{p(\widetilde{\tse}|\by)}} \\
        & = \frac{\Delta\bic(\widetilde{\tse},\tse) + O(1)}{2T}.
    \end{split}
    \end{equation}
    
    Similar to Step 2, we have
    \begin{equation}
        \Delta\bic(\widetilde{\tse},\tse) = 
        \begin{cases}
            -\frac{\beta_1 n}{\sigma^2} + O(\sqrt{n}) & \text{if}~\beta_1 \coloneqq \bias(\tse;f^*) - \bias(\widetilde{\tse};f^*)) > 0 \\
            (\df(\widetilde{\tse}) - \df(\tse))\log n + O(1) & \text{otherwise}.
        \end{cases}
    \end{equation}
    Plugging this back into \eqref{eq:tempered_path_congestion_helper} and using our assumption on $T$ gives the upper bound
    \begin{equation}
        \log\paren*{\frac{\min\braces*{p({\tse}'|\by)^, p({\tse}|\by)}^{1/T}}{\min\braces*{p(\widetilde{\tse}'|\by)^, p(\widetilde{\tse}|\by)}^{1/T}}} \leq 2c_4\log^{1-\beta}n
    \end{equation}
    for all $n$ large enough.
    Combining this with \eqref{eq:tempered_path_congestion_helper} allows us to compute
    \begin{equation}
    \begin{split}
        & \frac{1}{p(\widetilde{\tse}|\by) \check{P}(\widetilde{\tse},\widetilde{\tse}')}\sum_{\gamma_{\tse,\tse'} \ni (\widetilde{\tse},\widetilde{\tse}')} p(\tse|\by)p(\tse'|\by) \\
        & \leq \frac{2}{c_2}\sum_{\gamma_{\tse,\tse'} \ni (\widetilde{\tse},\widetilde{\tse}')} \frac{\min\braces*{p({\tse}'|\by), p({\tse}|\by)}^{1/T}\max\braces*{p({\tse}'),p({\tse})}}{\min\braces*{p(\widetilde{\tse}'|\by), p(\widetilde{\tse}|\by)}^{1/T}\min\braces*{p(\widetilde{\tse}'),p(\widetilde{\tse})}} \\
        & \leq \frac{2\abs*{\Gamma}\exp(2c_4\log^{1-\beta}n)}{c_2c_3}.
    \end{split}
    \end{equation}
    Using the definition \eqref{eq:congestion_def}, we see that the right hand side is an upper bound for the congestion parameter $\rho(\Gamma)$.
    As such, using Lemma \ref{lem:canonical_path}, the spectral gap is lower bounded as 
    \begin{equation}
        \gamma \geq \frac{c_2c_3}{2\abs*{\Gamma}l(\Gamma)} \exp(-2c_4\log^{1-\beta}n).
    \end{equation}

    \emph{Step 4: Putting everything together.}
    We apply Lemma \ref{lem:enhanced_mixing_time} using our lower bounds for the HDPR path probability and spectral gap, together with the path length upper bound $l(\Gamma)$.
    Preserving only the dependence on $n$ and omitting all lower order terms, we get the mixing time upper bound
    \begin{equation}
        t_{\operatorname{mix}}(\epsilon) = O(\exp(\max\braces*{l(\Gamma),1/c_4}\log^{1-\beta}n),
    \end{equation}
    which holds on the $1-\delta$ probability event $A_n$ as $n, T \to \infty$.
    Since this holds for any $\delta > 0$ and since the elided constants to not depend on $\delta$, the desired conclusion follows.
\end{proof}


\begin{proof}[Proof of Proposition \ref{prop:tempered_posterior_conc_treespace}]
    Condition on the set $A_n$ from the proof of Theorem \ref{thm:mixing_upper_bound_temperature}.
    Let $\tse^* \in \opt(f^*,0)$ and $\tse \notin \opt(f^*,0)$.
    We have
    \begin{equation}
    \begin{split}
        p(\tse;\by,T) & \leq \frac{p(\tse;\by,T)}{p(\tse^*;\by,T)} \\
        & = \frac{p(\tse)p(\by|\tse)^{1/T}}{p(\tse^*)p(\by|\tse^*)^{1/T} } \\
        & \leq c_3^{-1} \exp\paren*{-\frac{\Delta\bic(\tse,\tse^*) + O(1)}{2T}}.
    \end{split}
    \end{equation} 
    Since $\Delta\bic(\tse,\tse^*) \geq \log n + O(1)$ when $A_n$ holds and $T \leq \log^{\beta}\nsamples$ by assumption, the right hand side is bounded by $c_3^{-1}\exp(-\log^{1-\beta} n/2 + O(1))$.
    Adding up these inequalities over $\tse \in \opt(f^*,0)^c$ gives
    \begin{equation}
        p(\opt(f^*,0)^c;\by,T) = \exp\paren*{-\frac{1}{2}\log^{1-\beta}n + O(1)}.
    \end{equation}
    Since this holds for any $\delta > 0$, we get the desired result.
\end{proof}

\begin{lemma}
    \label{lem:many_trees_helper}
    Let $U$ and $W$ be linear subspaces of an inner product space $V$.
    Let $u \in W$ and let $\braces*{w_1,w_2,\ldots,w_k}$ be a basis for $W$.
    Suppose $P_{U+\operatorname{span}\paren*{w_j}}u = P_Uu$ for $j=1,2,\ldots,k$.
    Then $u \in U$.
\end{lemma}

\begin{proof}
    By taking the quotient with respect to $U$, one may assume without loss of generality that $U = 0$.
    Then the condition on $u$ is equivalent to saying that $u \perp w_j$ for $j=1,2,\ldots,k$.
    This implies that $u \in W^\perp$.
    Since we also assumed that $u \in W$, we must have $u = 0$.
\end{proof}

%% file: AOS/a4_pem_dim.tex
\section{Invariance of PEM dimension to change of measure}

\begin{lemma}[Characterization of subspace dimension]
    Let $\bv_1,\bv_2,\ldots,\bv_n$ be vectors in an inner product space.
    Let $\bG$ be the Gram matrix of these vectors, in other words, its $(i,j)$ entry satisfies $G_{ij} = \inprod{\bv_i,\bv_j}$ for $1 \leq i, j \leq n$.
    Then the dimension of the subspace spanned by $\bv_1,\bv_2,\ldots,\bv_n$ is equal to the number of nonzero eigenvalues of $\bG$.
\end{lemma}

\begin{proof}
    By restricting to a linearly independent set, it suffices to show that $\bv_1,\bv_2,\ldots,\bv_n$ are linearly independent if and only if $\bG$ is invertible.
    For the forward direction, suppose $\bG$ is not invertible, then there exists a vector of coefficients $\balpha = (\alpha_1,\ldots,\alpha_n)$ such that $\balpha^T\bG\balpha = 0$.
    But in that case, we have
    \begin{equation}\nonumber
        0 = \balpha^T\bG\balpha = \inprod*{\sum_{i=1}^n \alpha_i \bv_i, \sum_{i=1}^n \alpha_i \bv_i}.
    \end{equation}
    By definition of the inner product, this means that $\sum_{i=1}^n \alpha_i \bv_i = 0$, contradicting linear independence.
    The reverse direction is similar.
\end{proof}

\begin{lemma}[Invariance of dimension to covariate distribution] \label{lem:invariance_of_dim}
    Let $\nu$ and $\nu'$ be two measures on a compact covariate space $\cal{X}$ that are absolutely continuous with respect to each other.
    Let $v_1,v_2,\ldots,v_n \in L^2(\cal{X},\nu)$.
    Then $v_1,v_2,\ldots,v_n$ span the same subspace in both $L^2(\cal{X},\nu)$ and $L^2(\cal{X},\nu')$.
\end{lemma}

\begin{proof}
    By restricting to a linearly independent set, it suffices to show that $v_1,v_2,\ldots,v_n$ are linearly independent in $L^2(\cal{X},\nu)$ if and only if they are linearly independent in $L^2(\cal{X},\nu')$.
    By definition of absolute continuity, there exists a constant $c > 0$ such that
    \begin{equation}\nonumber
        c^{-1}\int_{\cal{X}} f(x) d\nu(x) \leq \int_{\cal{X}} f(x) d\nu'(x)  \leq c\int_{\cal{X}} f(x) d\nu(x)
    \end{equation}
    for any $f \in $
    Let $\bG$ and $\bG'$ be their Gram matrices in $L^2(\cal{X},\nu)$ and $L^2(\cal{X},\nu')$ respectively.
    Let $\balpha = (\alpha_1,\ldots,\alpha_n)$ be any vector of coefficients.
    Then we have
    \begin{align}
        \balpha^T \bG \balpha & = \int_{\cal{X}} \paren*{\sum_{i=1}^n \alpha_iv_i(x)}^2 d\nu(x) \nn
        & \leq c\int_{\cal{X}} \paren*{\sum_{i=1}^n \alpha_iv_i(x)}^2 d\nu'(x) \nn
        & = c\balpha^T \bG' \balpha.\nonumber
    \end{align}
    Similarly, we get $\balpha^T \bG' \balpha \leq c \balpha^T \bG \balpha$.
\end{proof}

This shows that whenever the covariate distribution $\nu$ has full support, it is equivalent to the uniform measure.

%% file: AOS/a6_extra_sims.tex
\section{Additional simulation results}
\label{sec:additional_sims}

\subsection{Evaluation metrics}
\label{subsec:metrics}

Details on how RMSE, coverage, and the Gelman-Rubin $\hat R$ statistic are computed are provided below.
\begin{longlist}
    \item \textit{RMSE}: 
    Let $\tse_0,\tse_1,\tse_2,\ldots$ denote the Markov chain induced by a run of BART.
    For $j=0,1,2,\ldots$, let $h_j$ be a draw from the conditional posterior on regression functions, $p(f|\tse_j,\by)$.
    The BART algorithm returns the collection $\mathcal{S}_{\operatorname{post}} \coloneqq \braces*{h_{t_\text{burn-in}+1}, h_{t_\text{burn-in}+2},\ldots h_{t_{\text{max}}}}$,
     where $t_{\text{burn-in}}$ is the number of burn-in iterations to be discarded.
     For each dataset used in our experiments, we create a test set $\mathcal{D}_{\operatorname{test}} \coloneqq \{(\bx_i,y_i)\}_{i=1}^r$.
     If the dataset is a real-world dataset, the MSE of the BART sampler over the test set is defined to be
     \begin{equation} \label{eq:MSE_def}
         \frac{1}{r}\sum_{i=1}^n \paren*{\frac{1}{t_{\text{max}} - t_{\text{burn-in}}}\sum_{t=t_\text{burn-in}+1}^{t_{\text{max}}}h_t(\bx_i) - y_i}^2,
     \end{equation}
     with RMSE defined to be the square root of this quantity.
     If the dataset is synthetic, we define RMSE by replacing $y_i$ in \eqref{eq:MSE_def} with the true function values $f(\bx_i)$.
    \item \textit{Coverage}: For each real dataset,  we consider coverage using the fitted posterior predictive intervals. 
    To do this, for each $(\bx_i,y_i)$ in $\mathcal{D}_{\operatorname{test}}$, we consider the collection 
    $$
    \braces*{h_{t_\text{burn-in}+1}(\bx_i) + e_{i,t_\text{burn-in}+1}, h_{t_\text{burn-in}+2}(\bx_i) + e_{i,t_\text{burn-in}+2},\ldots h_{t_{\text{max}}}(\bx_i) +  e_{i,t_\text{max}}},
    $$
    where for each $t$, $e_{i,t} \sim_{i.i.d.} \mathcal{N}(0,\sigma_t^2)$ and $\sigma_t^2$ is the variance parameter sampled at the $t$th iteration of the BART sampler.
    Let $\underline{q}_{i}$ and $\overline{q}_{i}$ denote the $0.025$ and $0.975$ quantile values from this collection.
    The coverage of the BART sampler is defined to be
    \begin{equation}
        \frac{1}{r}\sum_{i=1}^n \indicator\braces*{y_i \in \left[\underline{q}_{i}, \overline{q}_{i}\right]}.
    \end{equation}
    For each simulated dataset, we instead consider the coverage of credible intervals for the true regression function's values.
    To do this, for each $\bx_i$ in $\mathcal{D}_{\operatorname{test}}$, we consider the collection of values obtained by evaluating $\mathcal{S}_{\operatorname{post}}$ on $\bx_i$ and define $\underline{q}_{i}$ and $\overline{q}_{i}$ denote the $0.025$ and $0.975$ quantile values from this collection.
    Coverage is defined to be
    \begin{equation}
        \frac{1}{r}\sum_{i=1}^n \indicator\braces*{f(\bx_i) \in \left[\underline{q}_{i}, \overline{q}_{i}\right]}.
    \end{equation}
    \item \textit{Gelman-Rubin $\hat R$}: For $m$ independent Markov chains each of length $n$, let $\bar{z}_j$ and $s_j^2$ denote the sample mean and standard deviation for the samples in chain $j$ and let $\bar{\bar{z}}$ denote the overall mean.
    Let $B \coloneqq \frac{n}{m-1}\sum_{j=1}^m (\bar{z}_j - \bar{\bar{z}})^2$ and $W \coloneqq \frac{1}{m}\sum_{j=1}^m s_j^2$ denote the between-chain and average within-chain variances respectively. 
    $\hat{R}$ is defined via
    $\hat{R} \coloneqq \sqrt{\frac{\hat{V}}{W}}$, where $\hat{V} = \frac{n-1}{n}W + \frac{1}{n}B$.
    In our experiments, to derive a scalar stochastic process from the Markov chain $\tse_0,\tse_1,\tse_2,\ldots$ induced by a run of BART, we consider the sequence of values
    \begin{equation}
        \paren*{\frac{1}{r}\sum_{i=1}^n \paren*{h_t(\bx_i) - y_i}^2}^{1/2}
    \end{equation}
    for $t=t_\text{burn-in}+1, t_\text{burn-in}+2,\ldots, t_\text{max}$.
    We then calculate $\hat R$ using these sequences obtained from 8 independent chains from the BART sampler.
\end{longlist}

\subsection{Experiment 3: Dirichlet vs uniform prior on split feature probabilities}

The original prior for BART selects split features uniformly at random (Section \ref{subsec:differences_bart}).
\citet{linero2018bayesianJASA} proposed using non-uniform split feature probabilities with a Dirichlet prior so that the model can adapt better to sparse structure in high-dimensional settings.
In this experiment, we compare the mixing performance of the sampler under both choices of priors.
The Dirichlet prior is applied with the choice $\alpha = 1$ (using \citet{linero2018bayesianJASA}'s notation).




The results of the experiment are shown in Figure \ref{fig:experiment5}.
Use of a Dirichlet prior on split feature probabilities either exacerbates the increasing trend for $\hat R$, or at best has ambiguous effect.
This is despite improving prediction performance on datasets exhibiting sparsity such as Low Dimensional Smooth and Piecewise Linear.

\subsection{Experiment 4: Varying burn-in length}
\label{subsec:experiment_burnin}


We compare burn-in lengths of $n_{\text{skip}} \in \{5,000, 10,000\}$ iterations against the baseline of 1,000 iterations, leaving the number of posterior sampling iterations fixed at 10,000.

The results of the experiment are shown in Figure \ref{fig:experiment_burnin}.
We observe that increasing the number of burn-in iterations can exacerbate the upward trend in the Gelman-Rubin statistic $\hat{R}$. 
This behavior suggests that the true mixing time of the sampler is much larger than 10{,}000 iterations, and points toward the MCMC chain getting stuck around different local modes in the target posterior. 
Specifically, under this hypothesis, having a longer burn-in leaves the between-chain differences largely unchanged because each individual chain does not escape its local mode.
On the other hand, a longer burn-in removes early, highly variable samples, thereby reducing within-chain variance.
As a result, $\hat{R}$ increases.
Likewise, increasing the number of burn-in iterations seems to make coverage worse.
Its effect of RMSE seems to be negligible.


\subsection{Experiment 5: XGBoost vs. default initialization}

The default implementation of the BART sampler initializes with all trees being trivial (Algorithm \ref{alg:method}).
\citet{he2021stochastic} found that initializing instead from a more predictive ensemble, for instance one derived via their XBART algorithm, can lead to improved prediction performance.
Specifically, they showed that combining 100 posterior draws from each of 25 independent XBART initializations has lower RMSE and higher coverage than 2500 posterior draws from a single chain with 1000 burn-in iterations.
In this experiment, we continue this line of inquiry and compare the mixing performance of the sampler under an initialization from an XGBoost \citep{chen2016xgboost} ensemble against that using a default initialization.
The XGBoost ensemble was tuned via randomized search over hyperparameters including tree depth (ranging from 2 to 8) and learning rate (from 0.01 to 0.3).
For this experiment, we use a smaller ensemble of $m=20$ trees to emphasize the initialization effect, as larger ensembles may dilute the impact of initial conditions. 




The results of the experiment are shown in Figure \ref{fig:experiment3}.
We observe that using an XGBoost initialization exacerbates the upward trend in the Gelman-Rubin statistic $\hat{R}$. 
Similar to Experiment 4, we believe that these results point toward the MCMC chain getting stuck around different local modes in the target posterior. 
Meanwhile, the XGBoost initialization seems to produce worse coverage but better RMSE.





\subsection{Experiment 6: Varying the number of splitting thresholds}

This experiment evaluates how the number of candidate splitting thresholds affects the performance of the BART sampler. We compare the default strategy—which uses all unique feature values as thresholds—against two sparser alternatives: using only the 100 and 200 evenly-spaced quantiles of each feature as candidate thresholds.

    
    


The results of the experiment are shown in Figure \ref{fig:experiment_thresholds}.
We observe that the increasing trend in $\hat R$, as well as RMSE and coverage, is relatively robust to the different numbers of candidate splitting thresholds.


\subsection{Experiment 7: Restricted vs. full move set}

While the default sampler proposes randomly from all four move types, some of our theoretical results require constraining the proposals to just ``grow'' and ``prune''.
In this experiment, we investigate the effect of such a constraint on the BART sampler.
The results for the experiment are shown in Figure \ref{fig:experiment_moveset}.
We see that restricting the move set does not substantially affect $\hat R$, coverage, or RMSE.

\subsection{Experiment 8: Fully marginalized vs. default posterior in Metropolis-Hastings filter acceptance probability}
\label{sec:marg_lik_ablation}

Recall that in our analyzed sampler, we change the acceptance probability in the Metropolis-Hastings filter to be in terms of the fully marginalized posterior on TSEs instead of being conditioned on the parameters of all other trees apart from the one being updated.
We refer to this version of the sampler as the marginalized sampler and in this experiment compare its performance to that of the default sampler, given the same number of iterations.
Note that per-iteration computational complexity for the marginalized sampler is much higher---it requires inverting a linear system where the number of columns is the total number of leaves in the ensemble, whereas the number of columns in the linear system corresponding to the default sampler is the number of leaves in the tree being updated.
Due to the substantial computational overhead of the marginalized sampler, we conduct this experiment with reduced scope: 1000 posterior samples instead of 10,000, and training sample sizes $n \in \{100, 200, 500\}$ instead of the full range.

The results of the experiment are shown in Figure \ref{fig:experiment_marginal}.
We observe that the increasing trend in $\hat R$ seems weaker under the marginalized sampler, providing numerical evidence that our hitting time lower bounds should also apply to the in-practice BART sampler.
The marginalized sampler also seems to produce better coverage compared to the default sampler.
The results for RMSE are somewhat ambiguous, with neither sampler dominating the other.

\begin{figure}[H]
    \centering
    \begin{tabular}{c}
    \begin{minipage}{0.9\textwidth}
        \centering
        \includegraphics[width=0.32\linewidth]{plots_publication_v2/Temperature_Schedule/Credible_Intervals/Temperature_Schedule_california_housing_gr_rmse.pdf}%
        \hfill
        \includegraphics[width=0.32\linewidth]{plots_publication_v2/Temperature_Schedule/Credible_Intervals/Temperature_Schedule_california_housing_coverage.pdf}
        \hfill
        \includegraphics[width=0.32\linewidth]{plots_publication_v2/Temperature_Schedule/Credible_Intervals/Temperature_Schedule_california_housing_rmse.pdf}%
    \end{minipage} \\
    \begin{minipage}{0.9\textwidth}
        \centering
        \includegraphics[width=0.32\linewidth]{plots_publication_v2/Temperature_Schedule/Credible_Intervals/Temperature_Schedule_low_lei_candes_gr_rmse.pdf}%
        \hfill
        \includegraphics[width=0.32\linewidth]{plots_publication_v2/Temperature_Schedule/Credible_Intervals/Temperature_Schedule_low_lei_candes_coverage.pdf}
        \hfill
        \includegraphics[width=0.32\linewidth]{plots_publication_v2/Temperature_Schedule/Credible_Intervals/Temperature_Schedule_low_lei_candes_rmse.pdf}%
    \end{minipage} \\
    \begin{minipage}{0.9\textwidth}
        \centering
        \includegraphics[width=0.32\linewidth]{plots_publication_v2/Temperature_Schedule/Credible_Intervals/Temperature_Schedule_1199_BNG_echoMonths_gr_rmse.pdf}%
        \hfill
        \includegraphics[width=0.32\linewidth]{plots_publication_v2/Temperature_Schedule/Credible_Intervals/Temperature_Schedule_1199_BNG_echoMonths_coverage.pdf}
        \hfill
        \includegraphics[width=0.32\linewidth]{plots_publication_v2/Temperature_Schedule/Credible_Intervals/Temperature_Schedule_1199_BNG_echoMonths_rmse.pdf}%
    \end{minipage} \\
        \begin{minipage}{0.9\textwidth}
        \centering
        \includegraphics[width=0.32\linewidth]{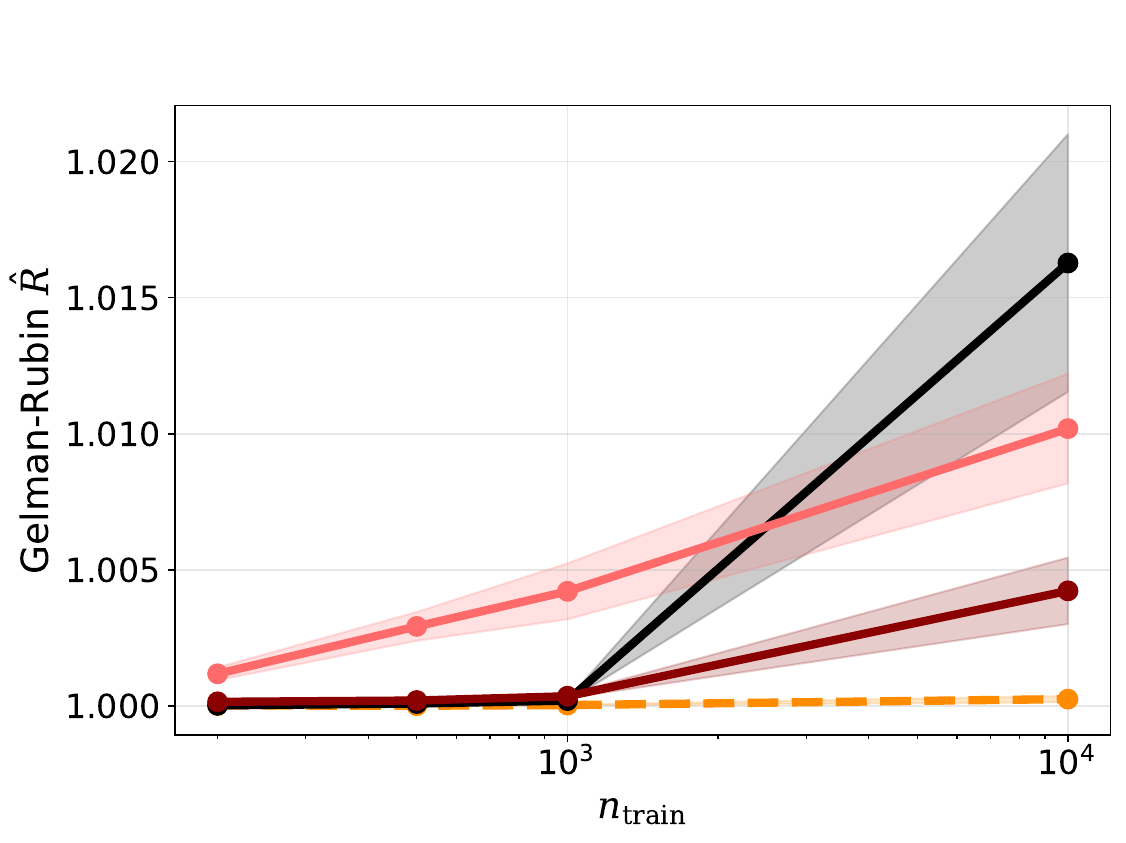}%
        \hfill
        \includegraphics[width=0.32\linewidth]{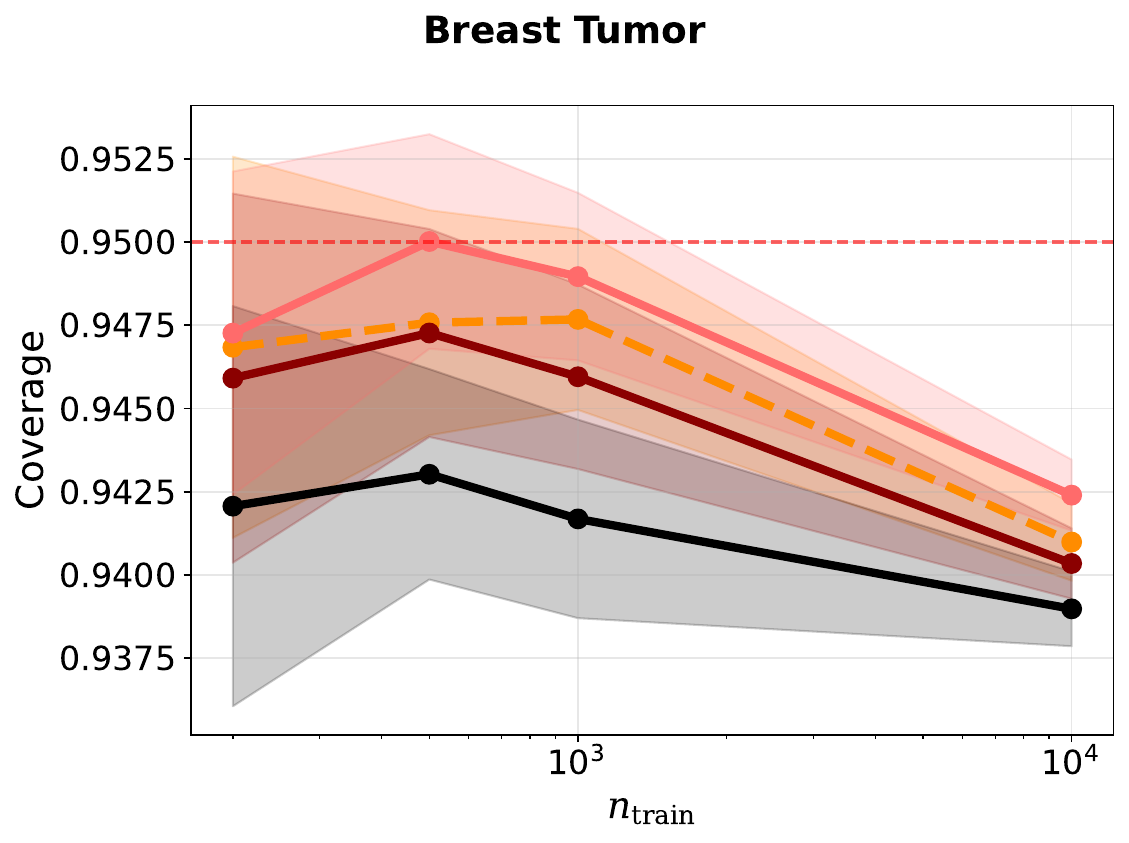}
        \hfill
        \includegraphics[width=0.32\linewidth]{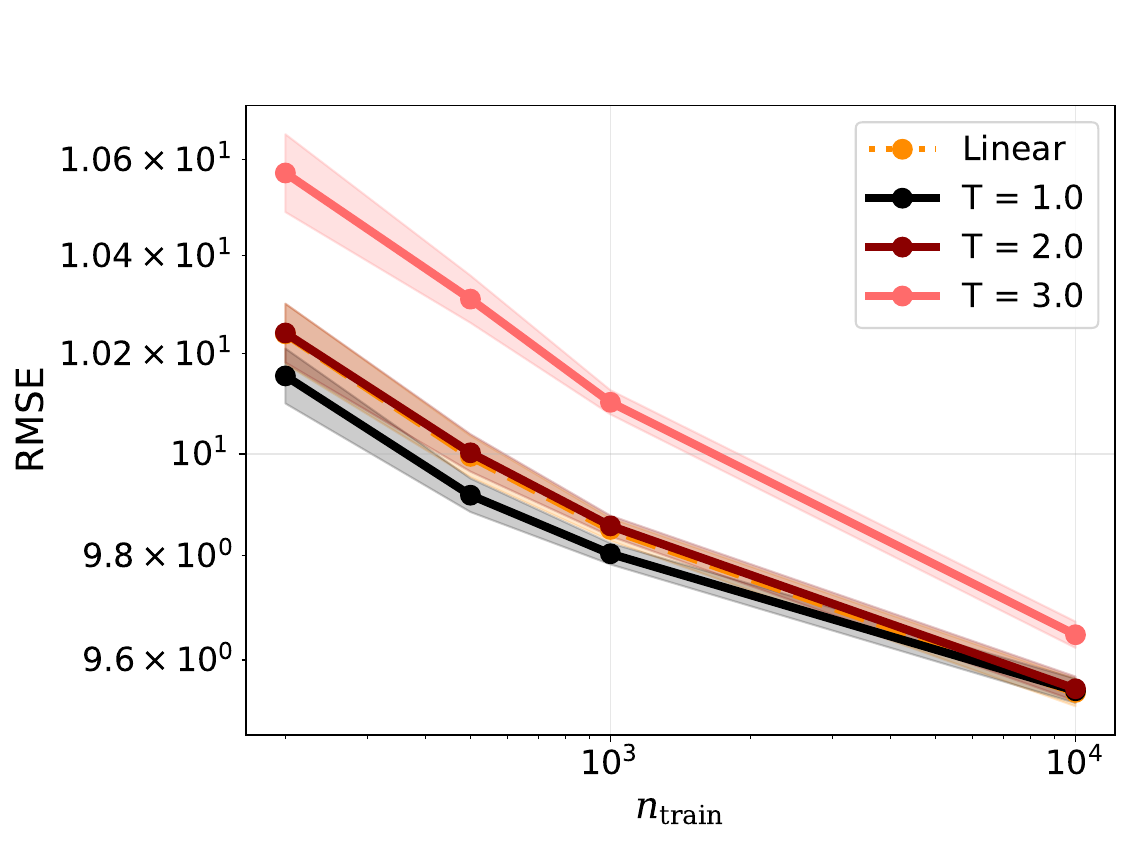}%
    \end{minipage} \\
    \begin{minipage}{0.9\textwidth}
        \centering
        \includegraphics[width=0.32\linewidth]{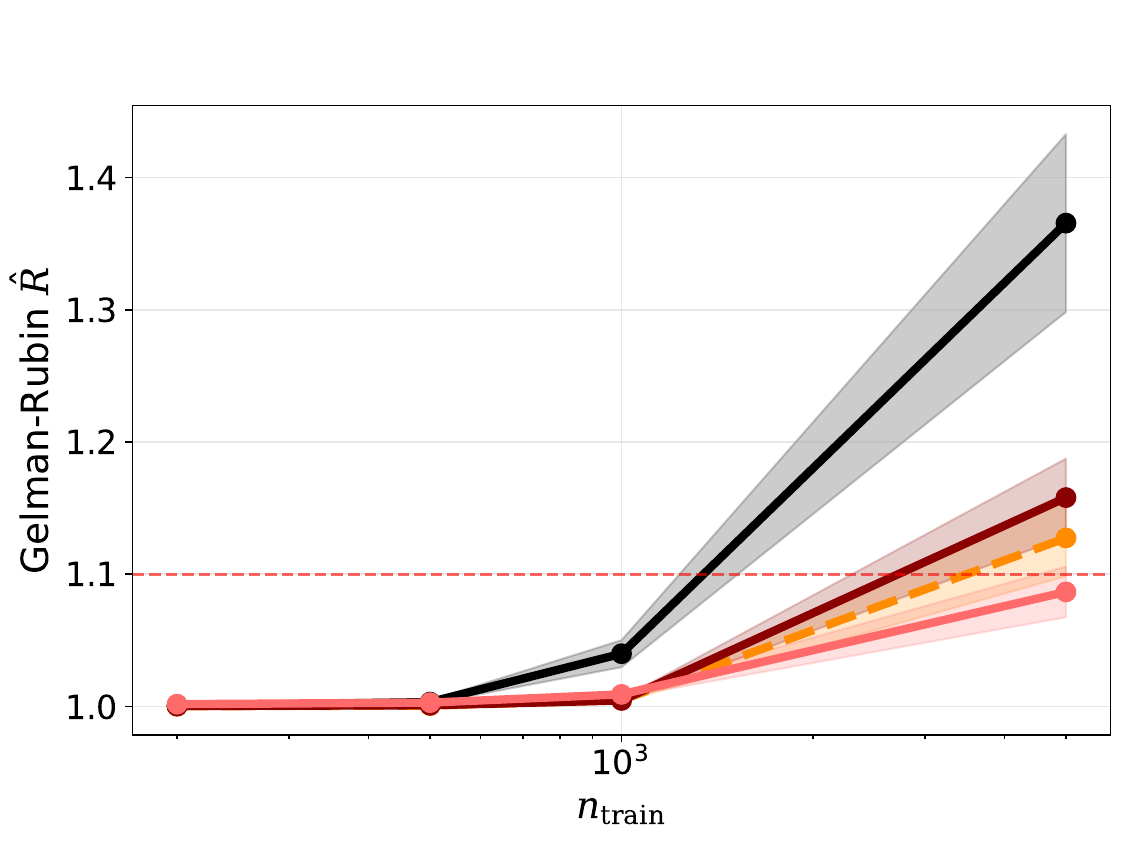}%
        \hfill
        \includegraphics[width=0.32\linewidth]{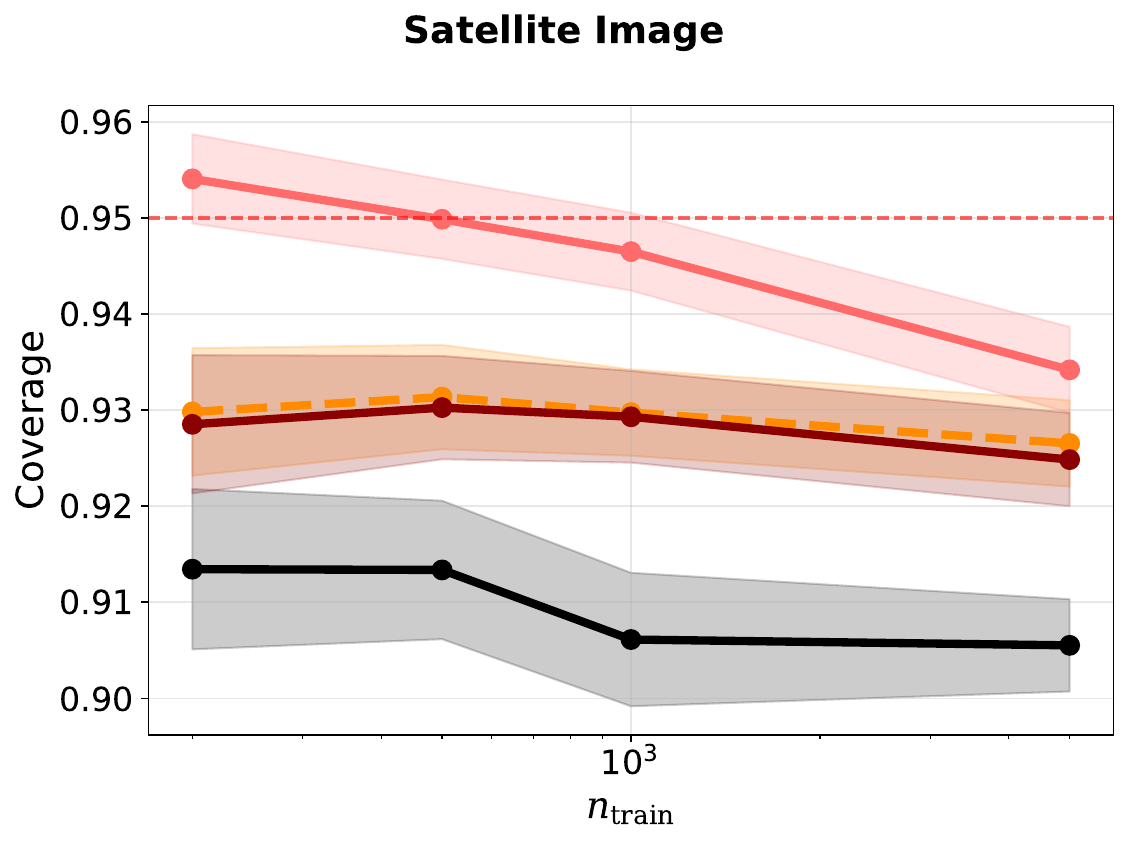}
        \hfill
        \includegraphics[width=0.32\linewidth]{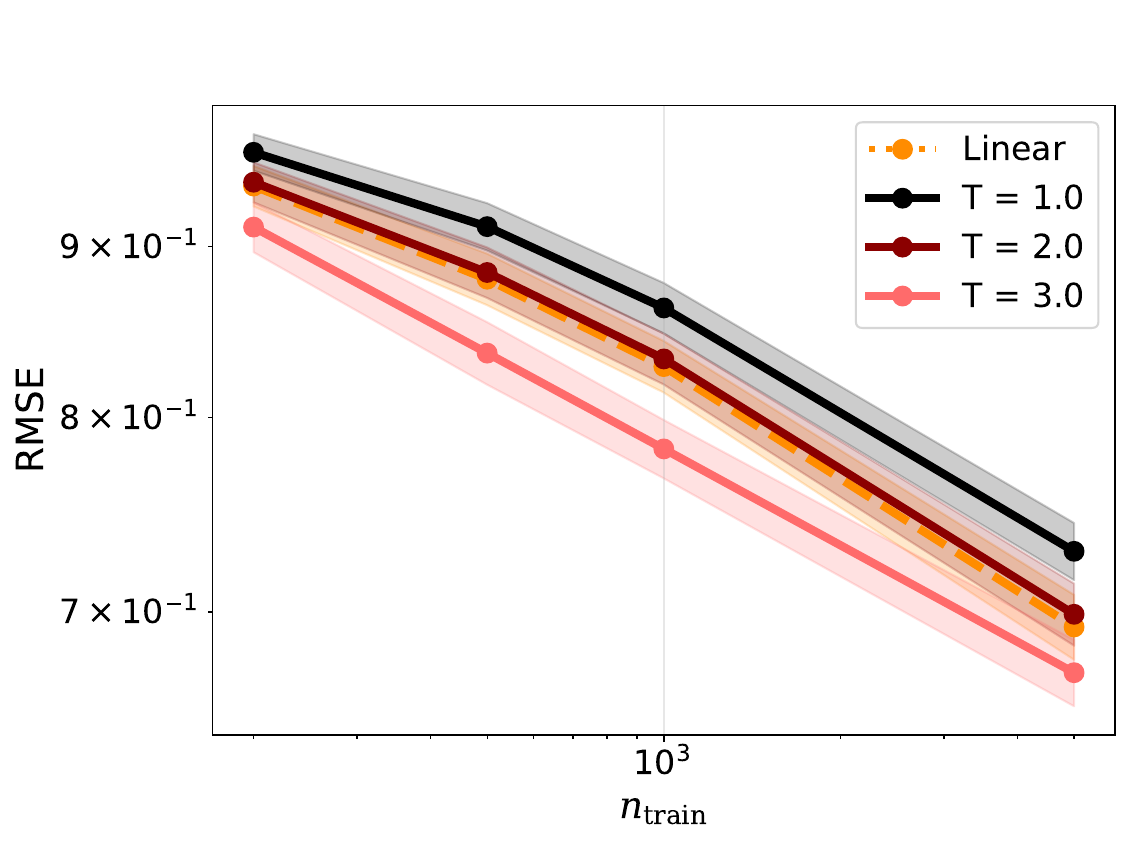}%
    \end{minipage} \\
    \begin{minipage}{0.9\textwidth}
        \centering
        \includegraphics[width=0.32\linewidth]{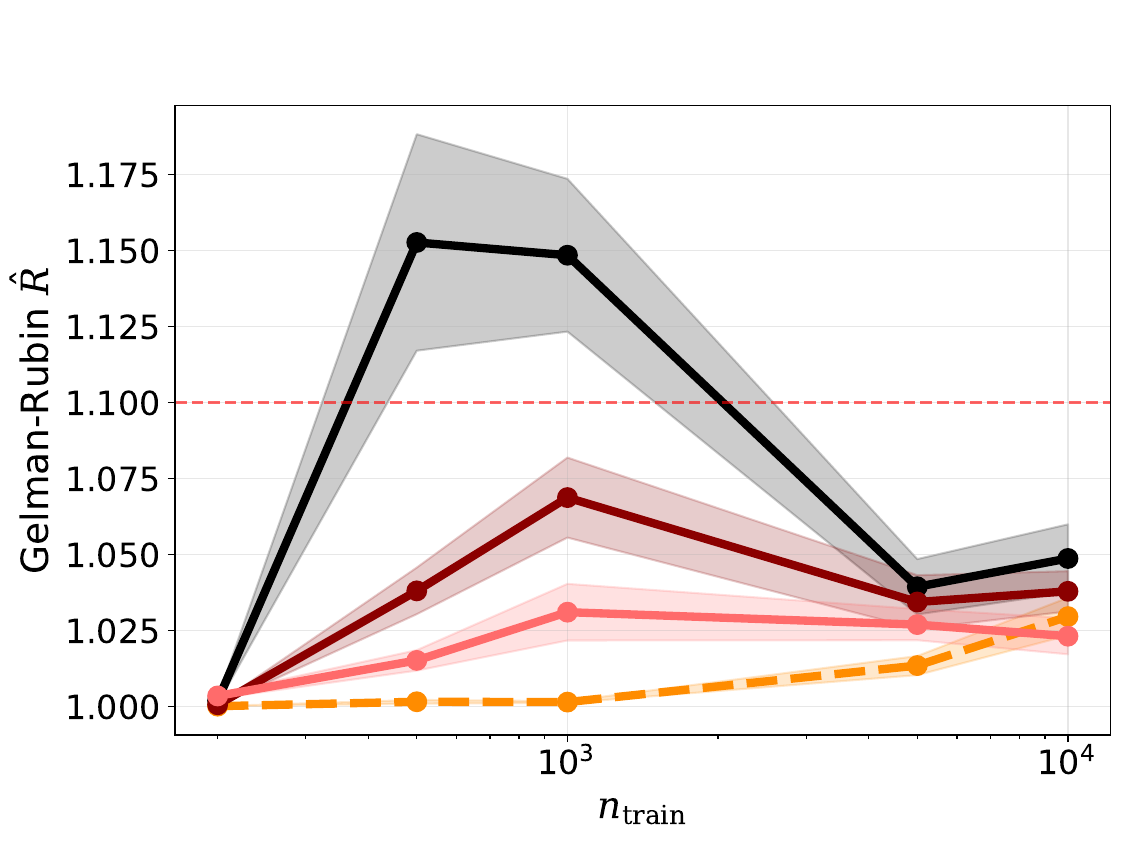}%
        \hfill
        \includegraphics[width=0.32\linewidth]{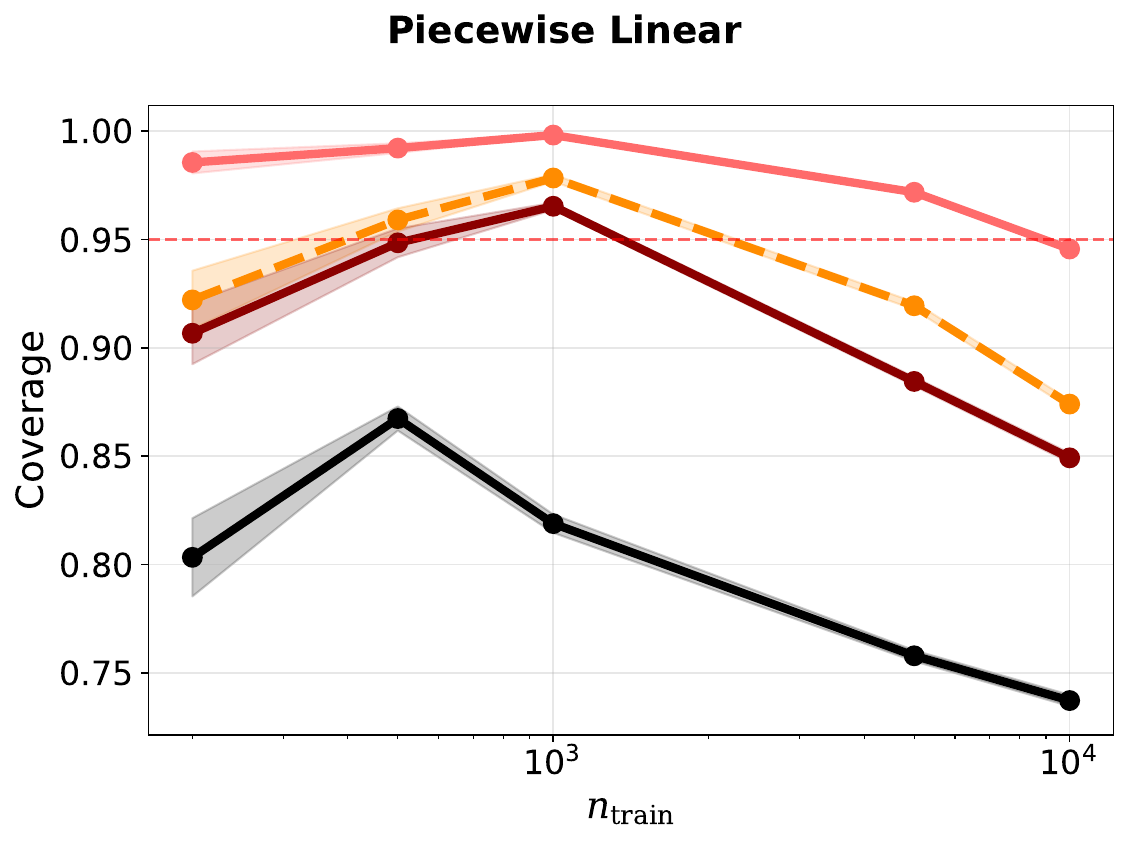}
        \hfill
        \includegraphics[width=0.32\linewidth]{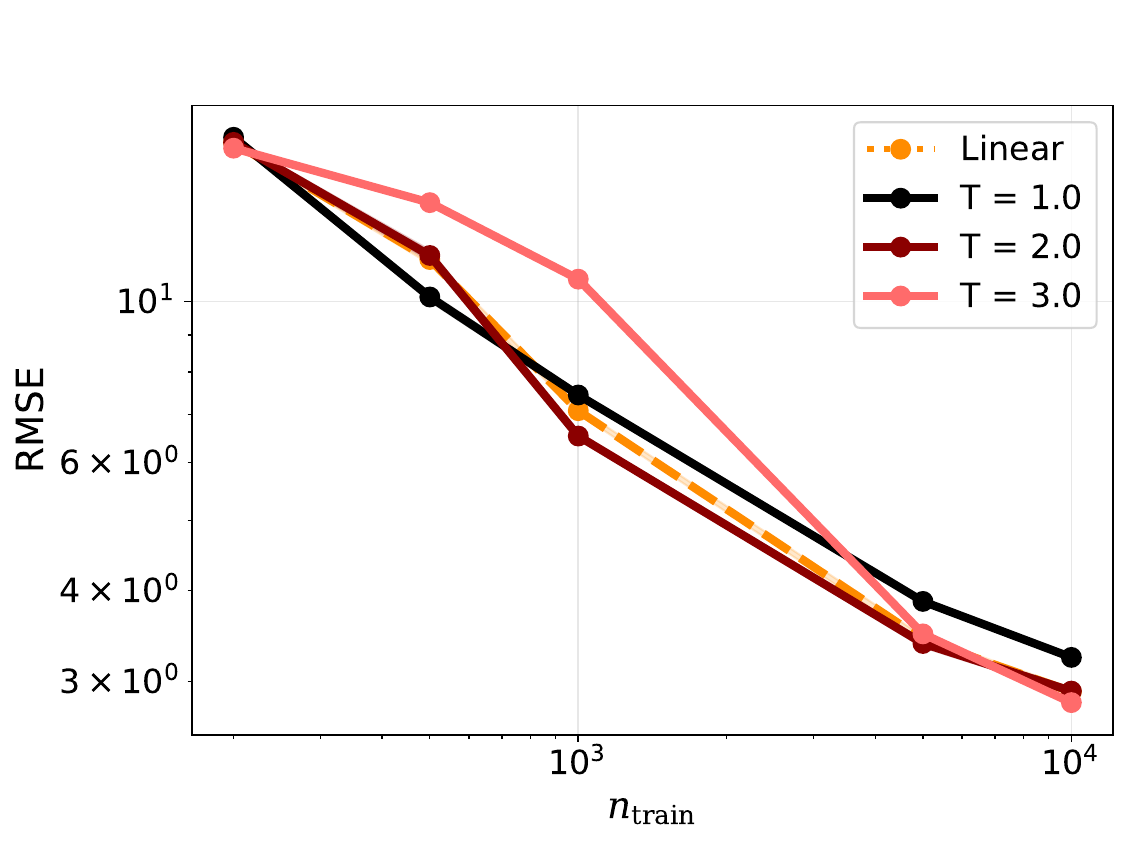}%
    \end{minipage} \\
    \end{tabular}
    \caption{
    Values for Gelman-Rubin $\hat R$ (left), coverage (center), and RMSE (right) for the BART sampler under different fixed temperatures ($T \in \braces{1,2,3}$) as well as a linear temperature schedule.
    Results are plotted for California Housing, Low Dimensional Smooth, Echo Months, Breast Tumor, Satellite Image, and Piecewise Linear datasets (from top to bottom).
    Error bars represent $\pm 1.96$ standard errors from 25 replicates.
    }
    \label{fig:experiment1_additional}
\end{figure}

\newpage
\begin{figure}[H]
    \centering
    \begin{tabular}{c}
      \begin{minipage}{0.9\textwidth}
        \centering
        \includegraphics[width=0.32\linewidth]{plots_publication_v2/plots_gr_quantiles/gr_quantiles_california_housing_ntrain10k.png}%
        \hfill
        \includegraphics[width=0.32\linewidth]{plots_publication_v2/plots_gr_quantiles/gr_quantiles_low_lei_candes_ntrain10k.png}
        \hfill
        \includegraphics[width=0.32\linewidth]{plots_publication_v2/plots_gr_quantiles/gr_quantiles_1199_BNG_echoMonths_ntrain10k.png}%
    \end{minipage} \\
      \begin{minipage}{0.9\textwidth}
        \centering
        \includegraphics[width=0.32\linewidth]{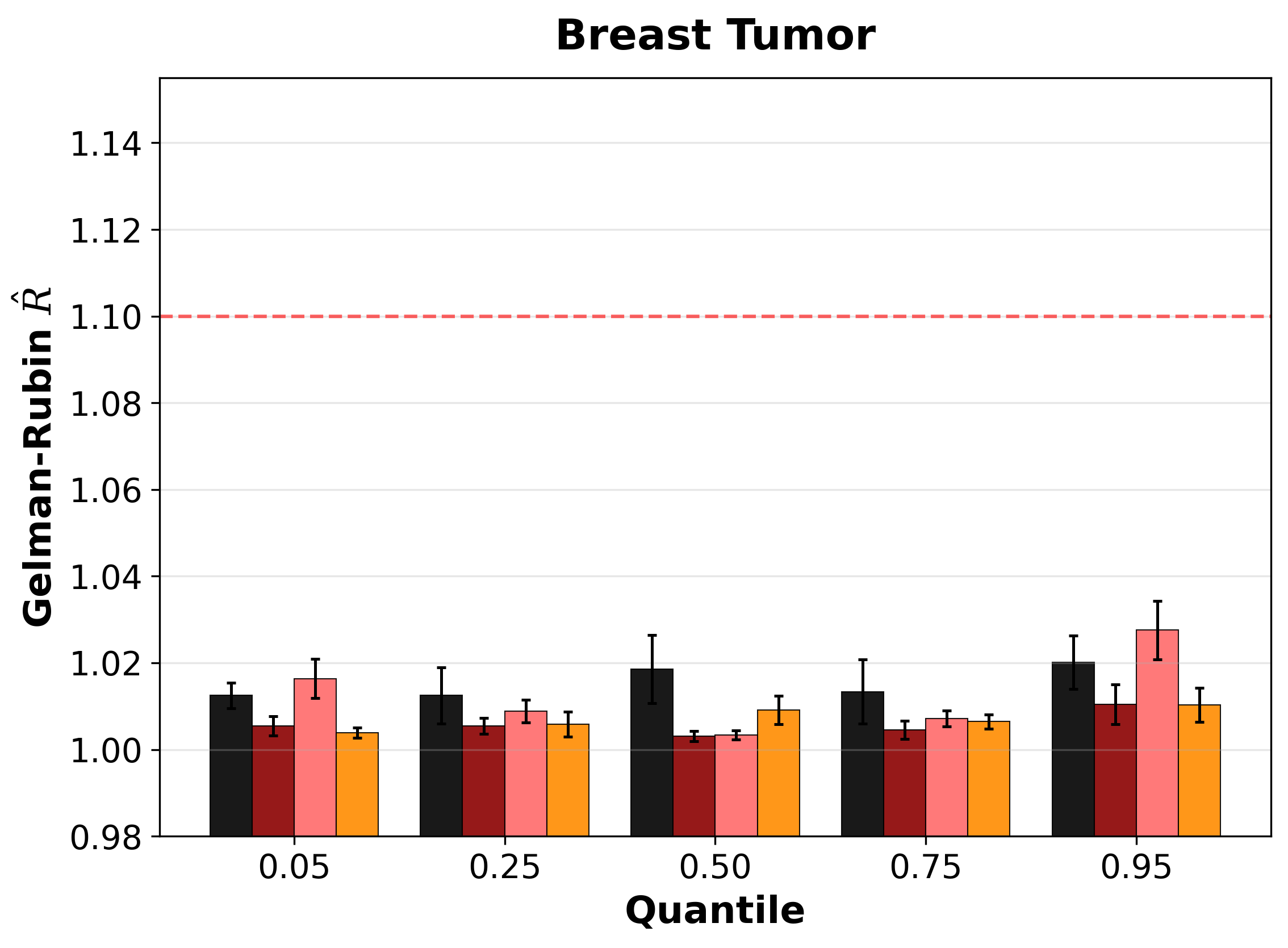}%
        \hfill
        \includegraphics[width=0.32\linewidth]{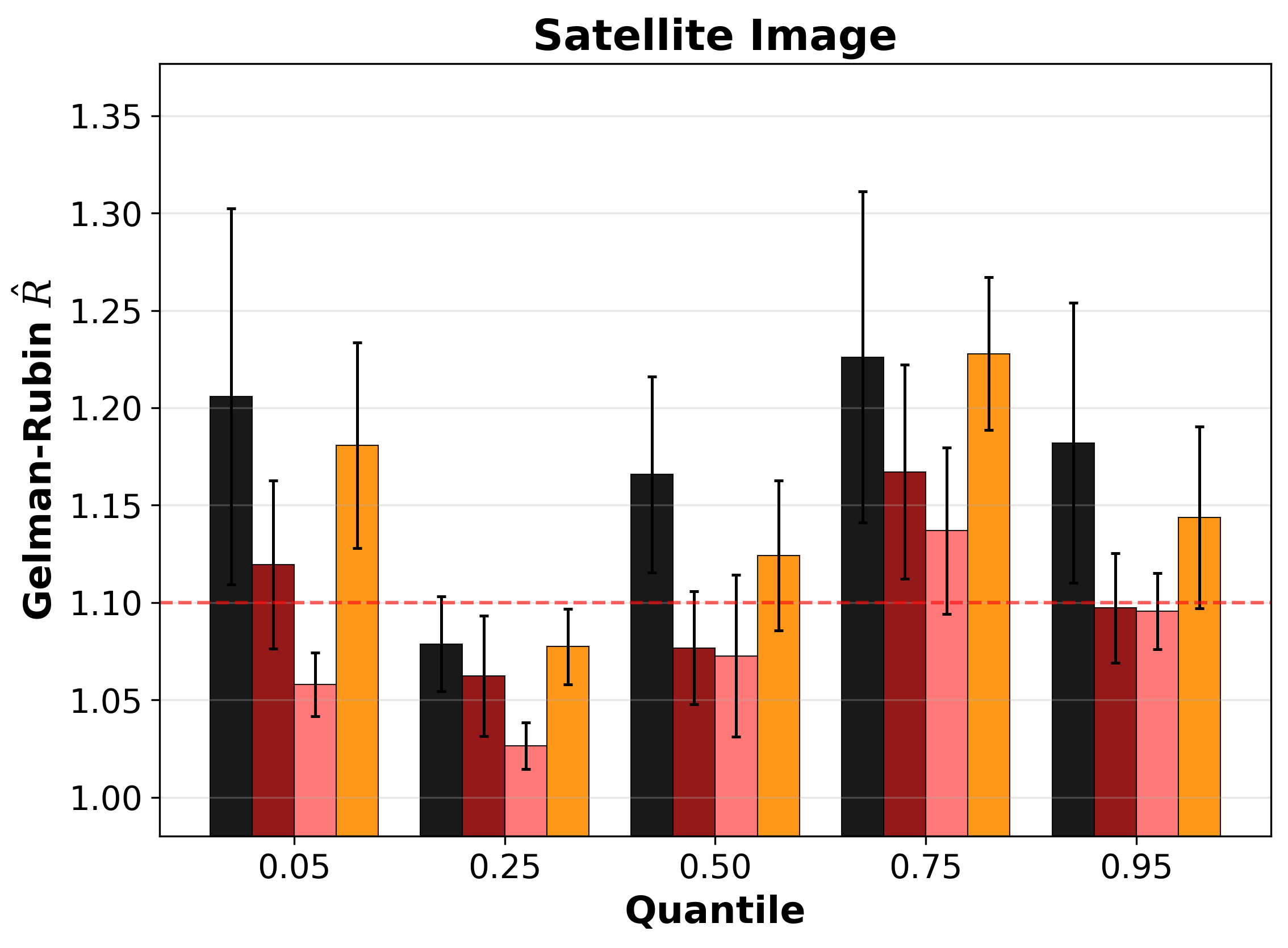}
        \hfill
        \includegraphics[width=0.32\linewidth]{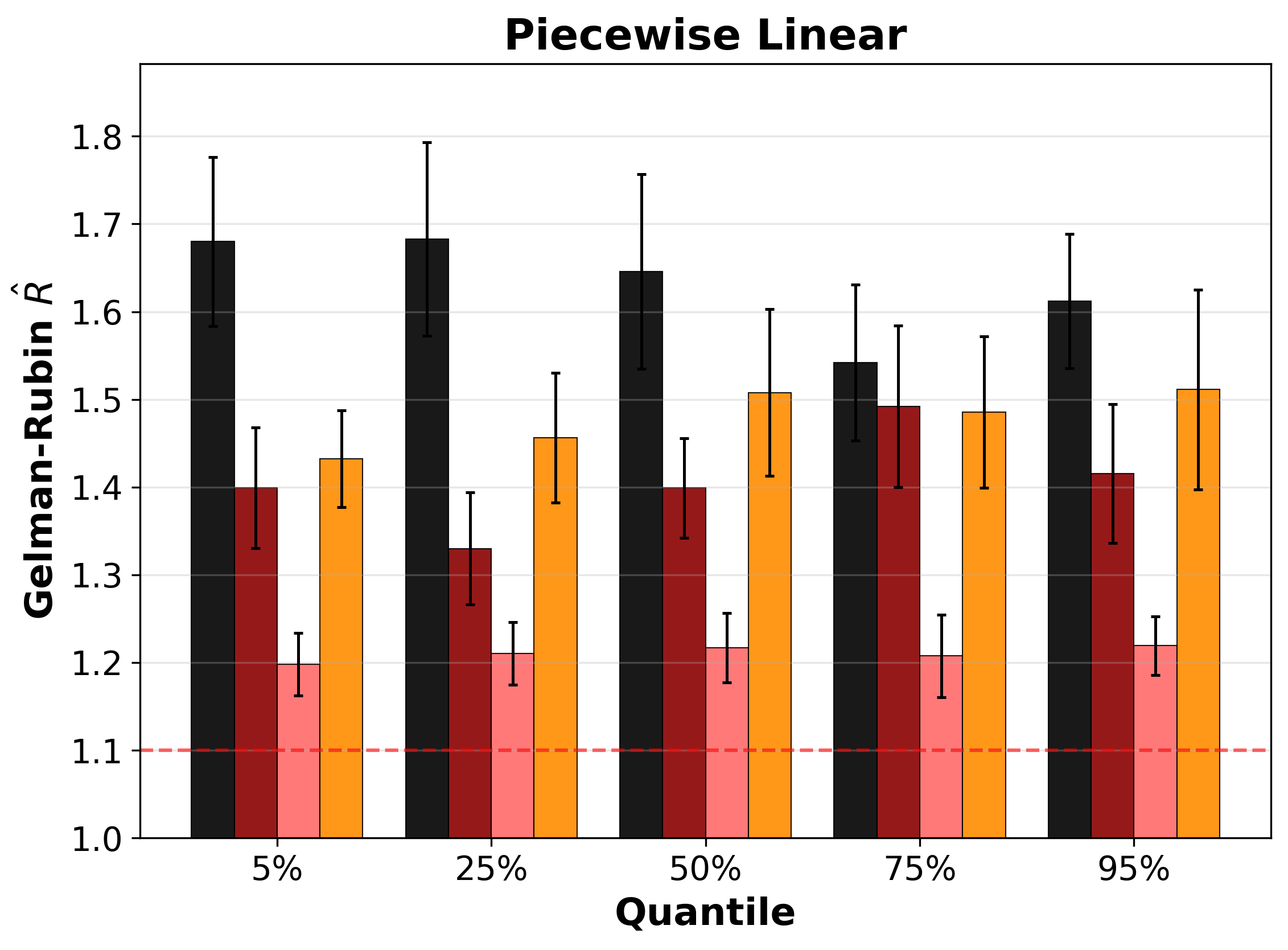}%
    \end{minipage} \\
    \end{tabular}
    \caption{
    Values for Gelman-Rubin $\hat R$ for the BART sampler under different fixed temperatures ($T \in \braces{1,2,3}$) when $\hat R$ is computed with respect to the 0.05, 0.25, 0.50, 0.75, and 0.95 quantiles for the fitted responses on the held-out test set.
    Results are plotted for the California Housing (leftow Dimensional Smooth (center, top), Echo Months (right, top), Breast Tumor (left, bottom), Satellite Image (center, bottom), and Piecewise Linear (right, bottom) datasets. 
        }
    \label{fig:experiment1_bar_additional}
\end{figure}

\newpage

\begin{figure}[H]
    \centering
    \begin{tabular}{c}
    \begin{minipage}{0.9\textwidth}
        \centering
        \includegraphics[width=0.32\linewidth]{plots_publication_v2/Trees/Credible_Intervals/Trees_california_housing_gr_rmse.pdf}
        \hfill
        \includegraphics[width=0.32\linewidth]{plots_publication_v2/Trees/Credible_Intervals/Trees_california_housing_coverage.pdf}
        \hfill
        \includegraphics[width=0.32\linewidth]{plots_publication_v2/Trees/Credible_Intervals/Trees_california_housing_rmse.pdf}%
    \end{minipage} \\
    \begin{minipage}{0.9\textwidth}
        \centering
        \includegraphics[width=0.32\linewidth]{plots_publication_v2/Trees/Credible_Intervals/Trees_low_lei_candes_gr_rmse.pdf}%
        \hfill
        \includegraphics[width=0.32\linewidth]{plots_publication_v2/Trees/Credible_Intervals/Trees_low_lei_candes_coverage.pdf}
        \hfill
        \includegraphics[width=0.32\linewidth]{plots_publication_v2/Trees/Credible_Intervals/Trees_low_lei_candes_rmse.pdf}%
    \end{minipage} \\
    \begin{minipage}{0.9\textwidth}
        \centering
        \includegraphics[width=0.32\linewidth]{plots_publication_v2/Trees/Credible_Intervals/Trees_1199_BNG_echoMonths_gr_rmse.pdf}%
        \hfill
        \includegraphics[width=0.32\linewidth]{plots_publication_v2/Trees/Credible_Intervals/Trees_1199_BNG_echoMonths_coverage.pdf}
        \hfill
        \includegraphics[width=0.32\linewidth]{plots_publication_v2/Trees/Credible_Intervals/Trees_1199_BNG_echoMonths_rmse.pdf}%
    \end{minipage} \\
    \begin{minipage}{0.9\textwidth}
        \centering
        \includegraphics[width=0.32\linewidth]{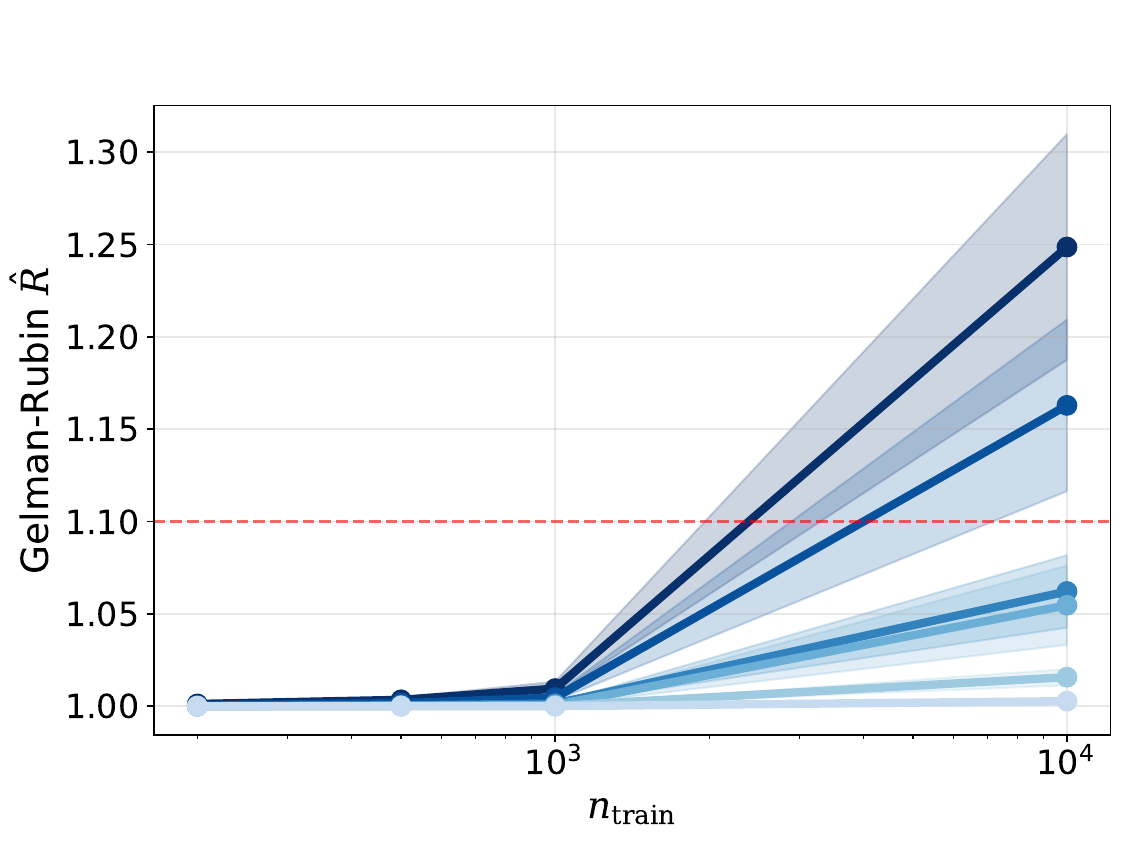}
        \hfill
        \includegraphics[width=0.32\linewidth]{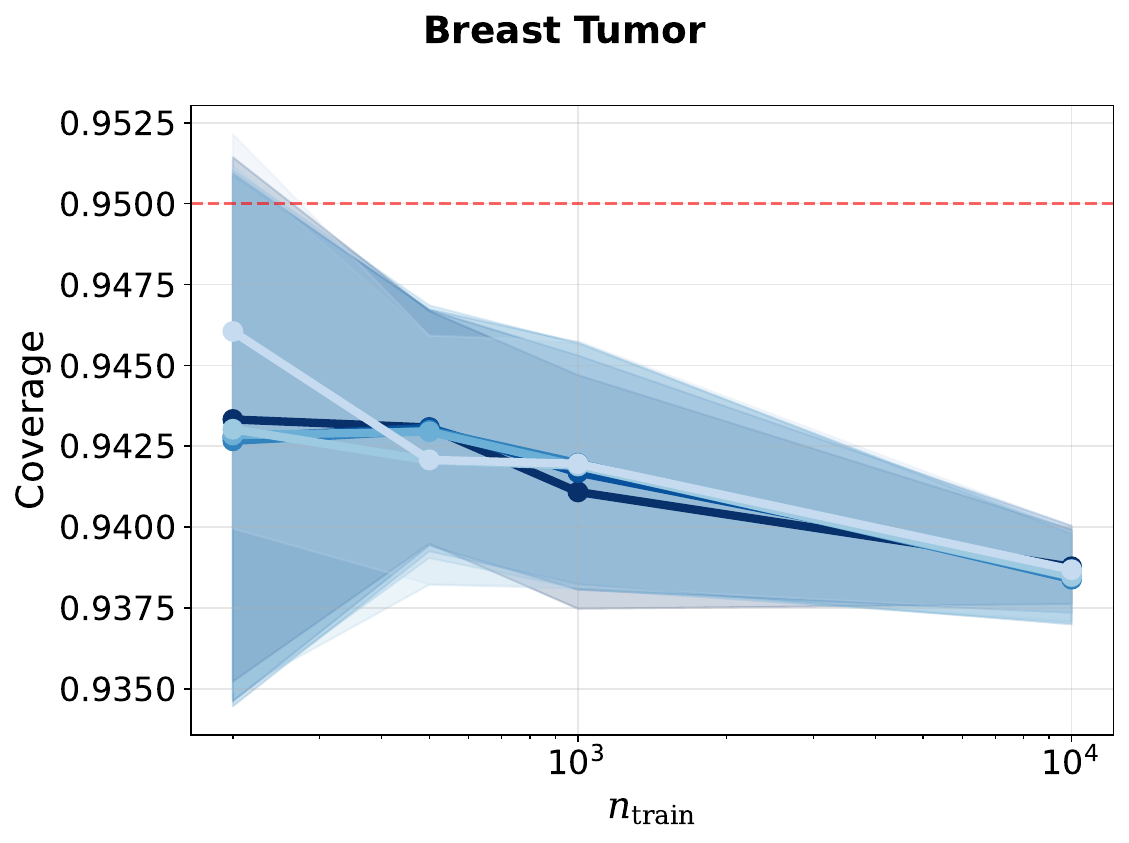}
        \hfill
        \includegraphics[width=0.32\linewidth]{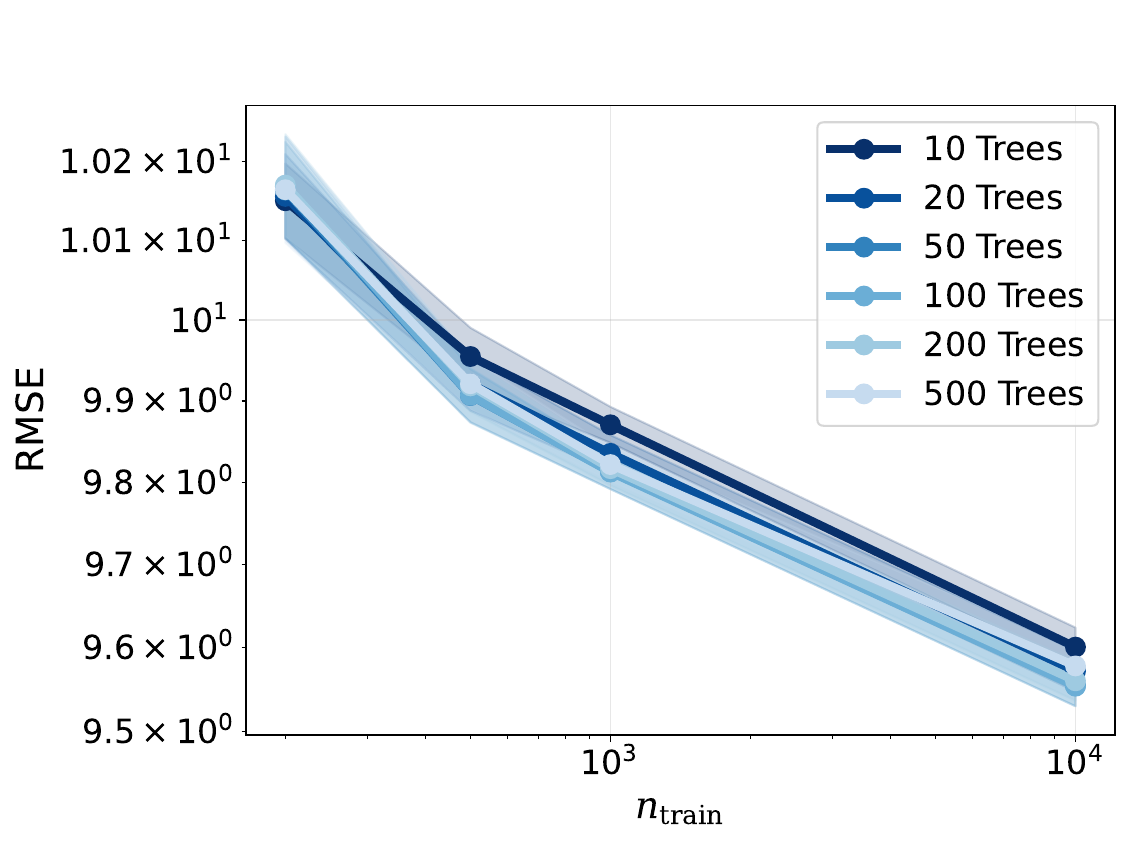}%
    \end{minipage} \\
    \begin{minipage}{0.9\textwidth}
        \centering
        \includegraphics[width=0.32\linewidth]{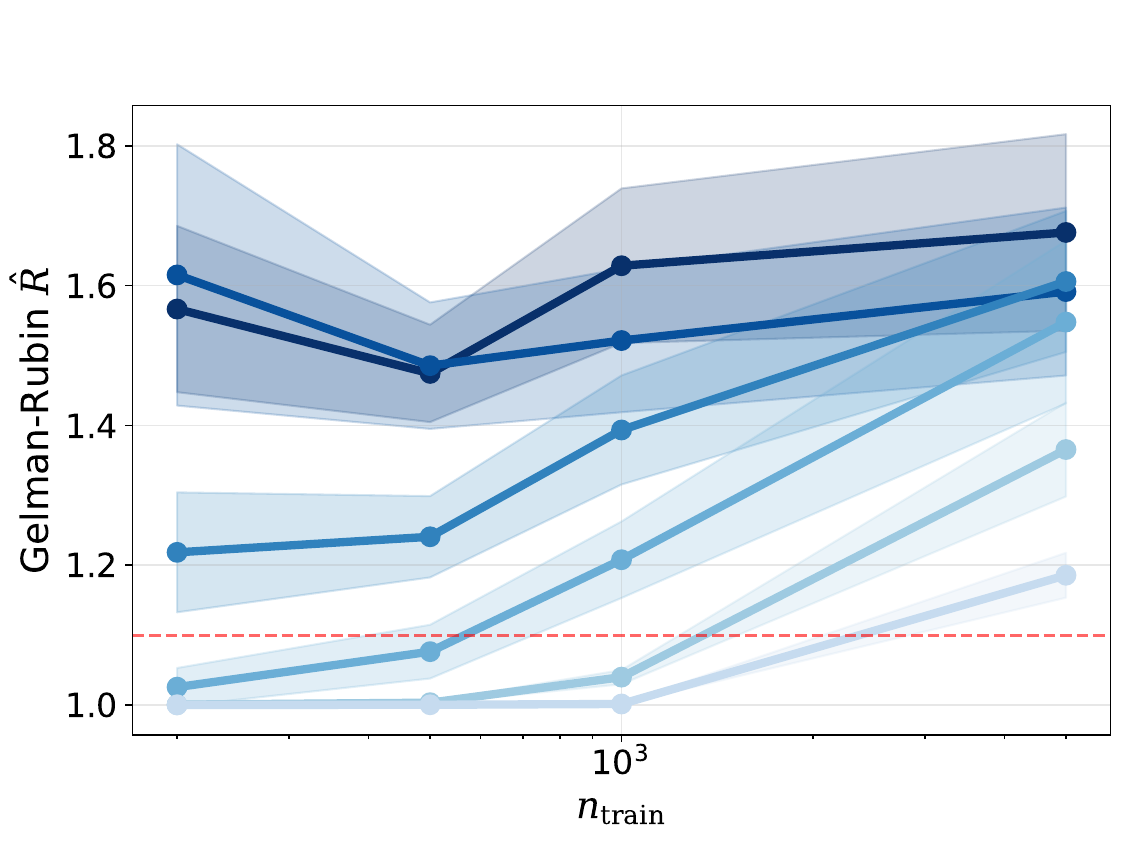}%
        \hfill
        \includegraphics[width=0.32\linewidth]{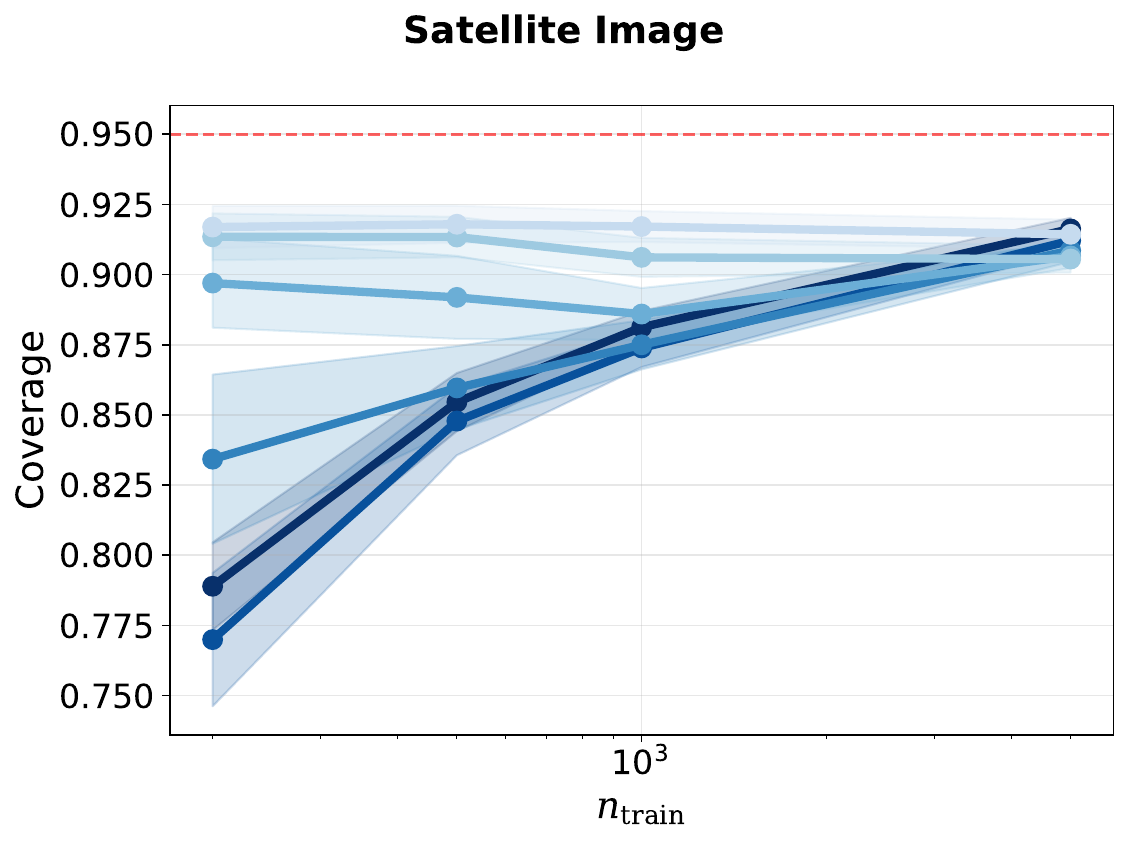}
        \hfill
        \includegraphics[width=0.32\linewidth]{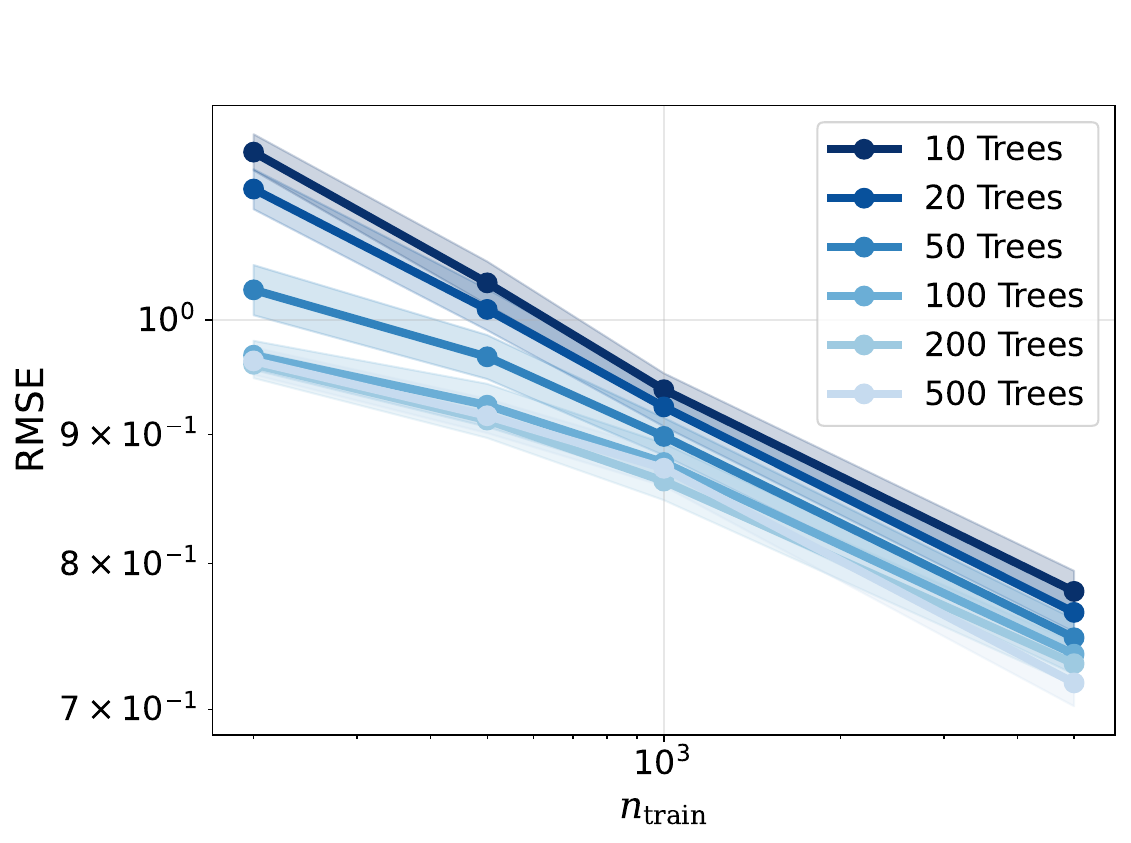}%
    \end{minipage} \\
    \begin{minipage}{0.9\textwidth}
        \centering
        \includegraphics[width=0.32\linewidth]{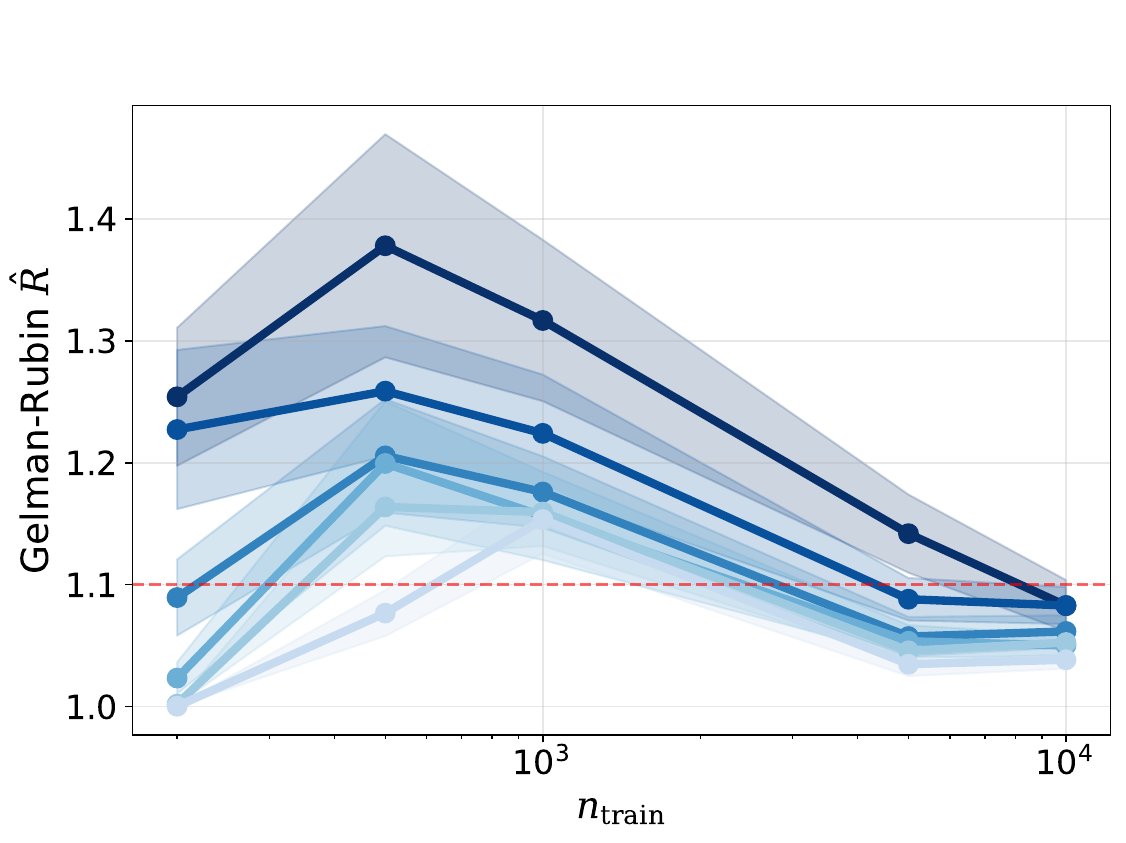}%
        \hfill
        \includegraphics[width=0.32\linewidth]{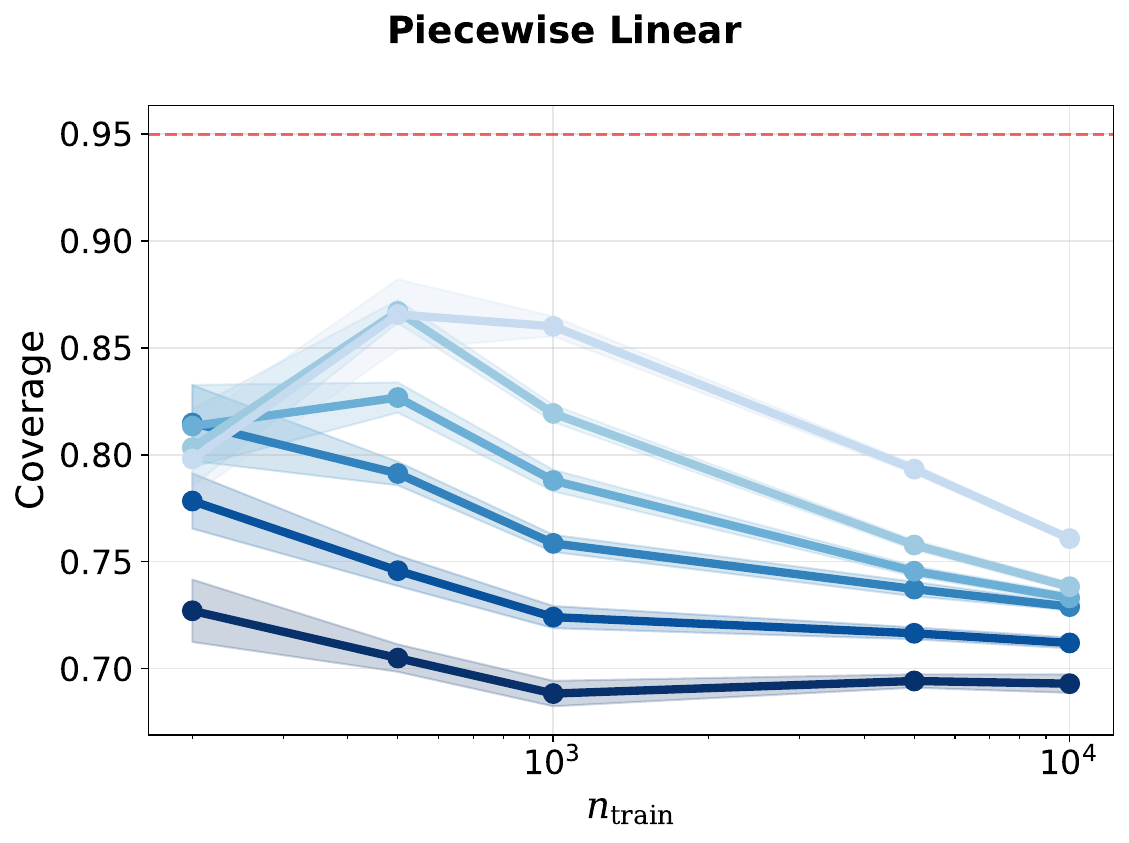}
        \hfill
        \includegraphics[width=0.32\linewidth]{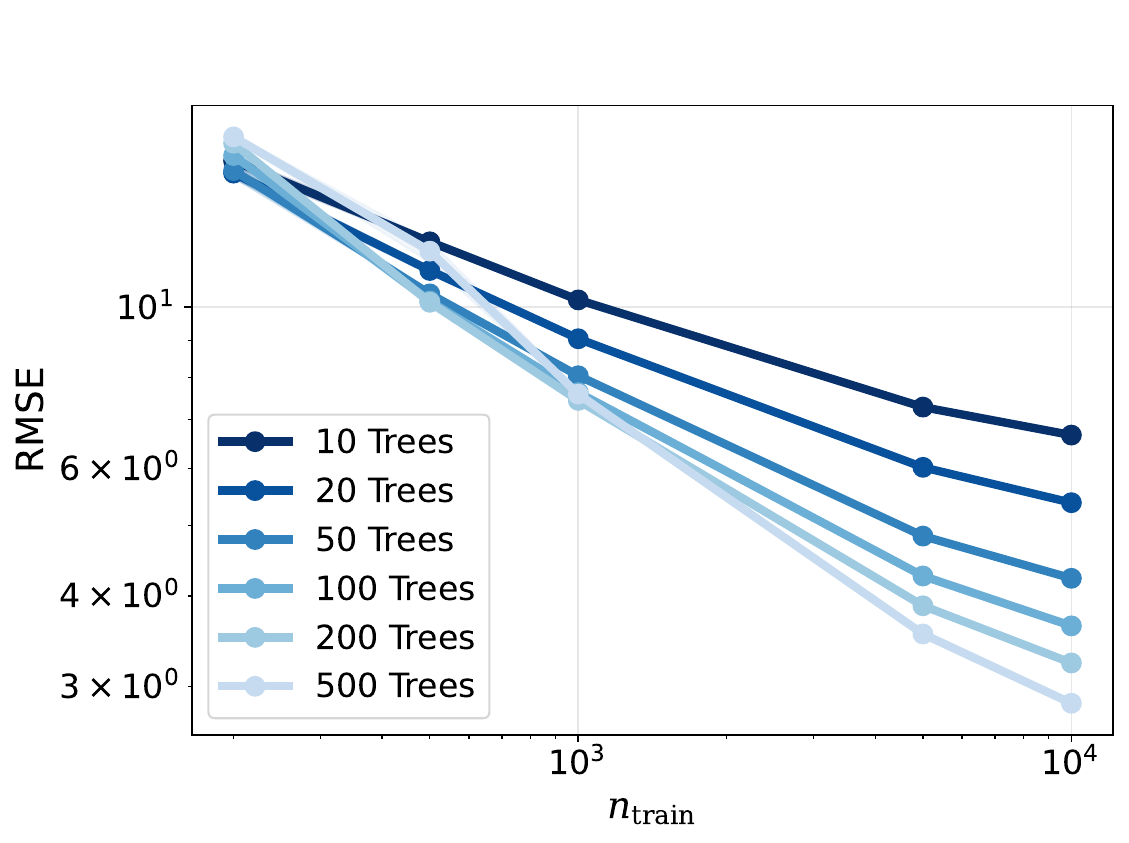}%
    \end{minipage} \\
    \end{tabular}
    \caption{
    Values for Gelman-Rubin $\hat R$ (left), coverage (center), and RMSE (right) for the BART sampler under different numbers of trees in the BART model.
    Results are plotted for California Housing, Low Dimensional Smooth, Echo Months, Breast Tumor, Satellite Image, and Piecewise Linear datasets (from top to bottom).
    Error bars represent $\pm 1.96$ standard errors from 25 replicates.
    }
    \label{fig:experiment2_additional}
\end{figure}

\newpage

\begin{figure}[H]
    \centering
    \begin{tabular}{c}
    \begin{minipage}{0.9\textwidth}
        \centering
        \includegraphics[width=0.32\linewidth]{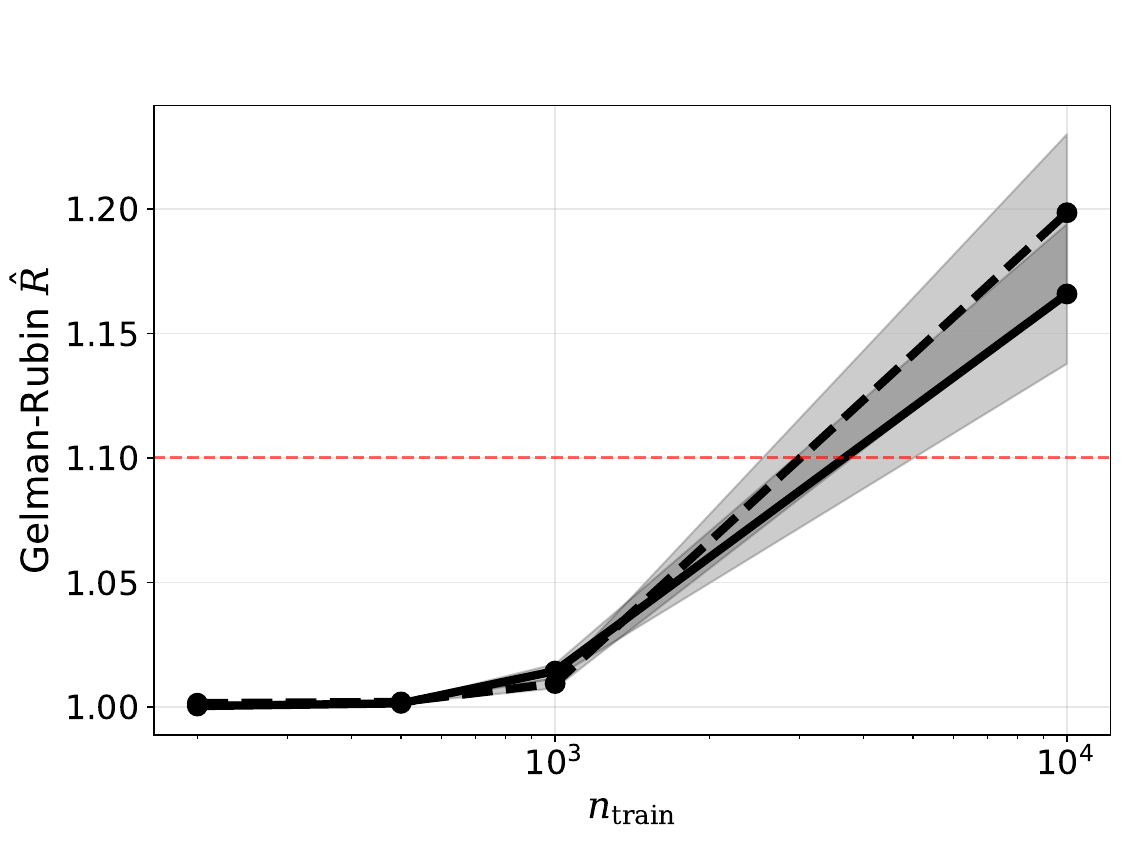}%
        \hfill
        \includegraphics[width=0.32\linewidth]{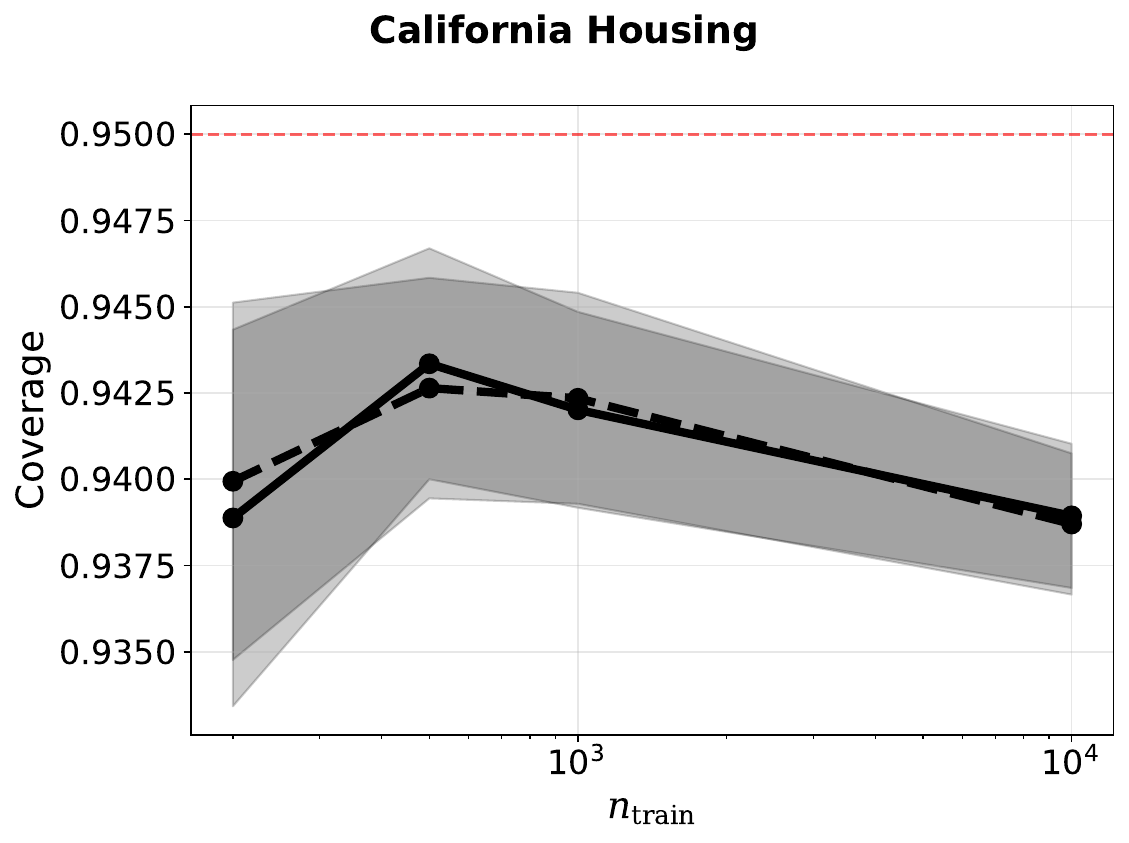}%
        \hfill
        \includegraphics[width=0.32\linewidth]{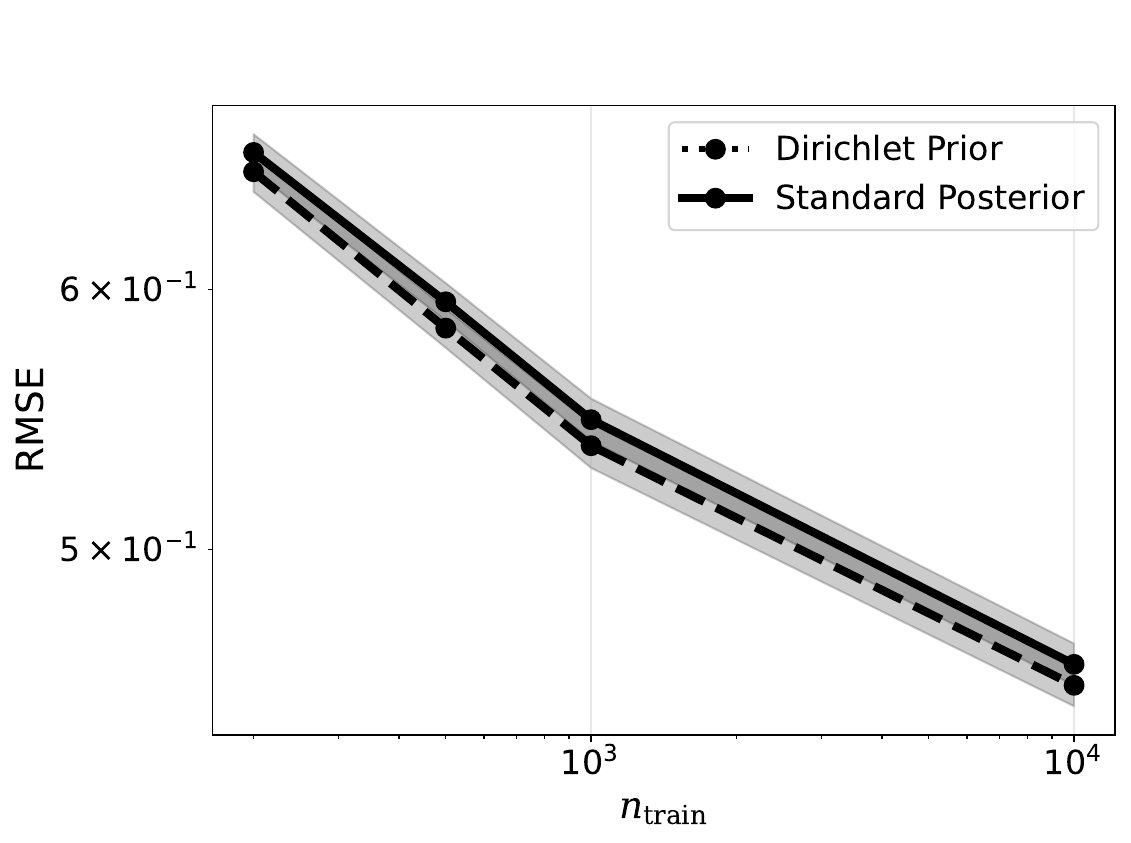}
    \end{minipage} \\
    \begin{minipage}{0.9\textwidth}
        \centering
        \includegraphics[width=0.32\linewidth]{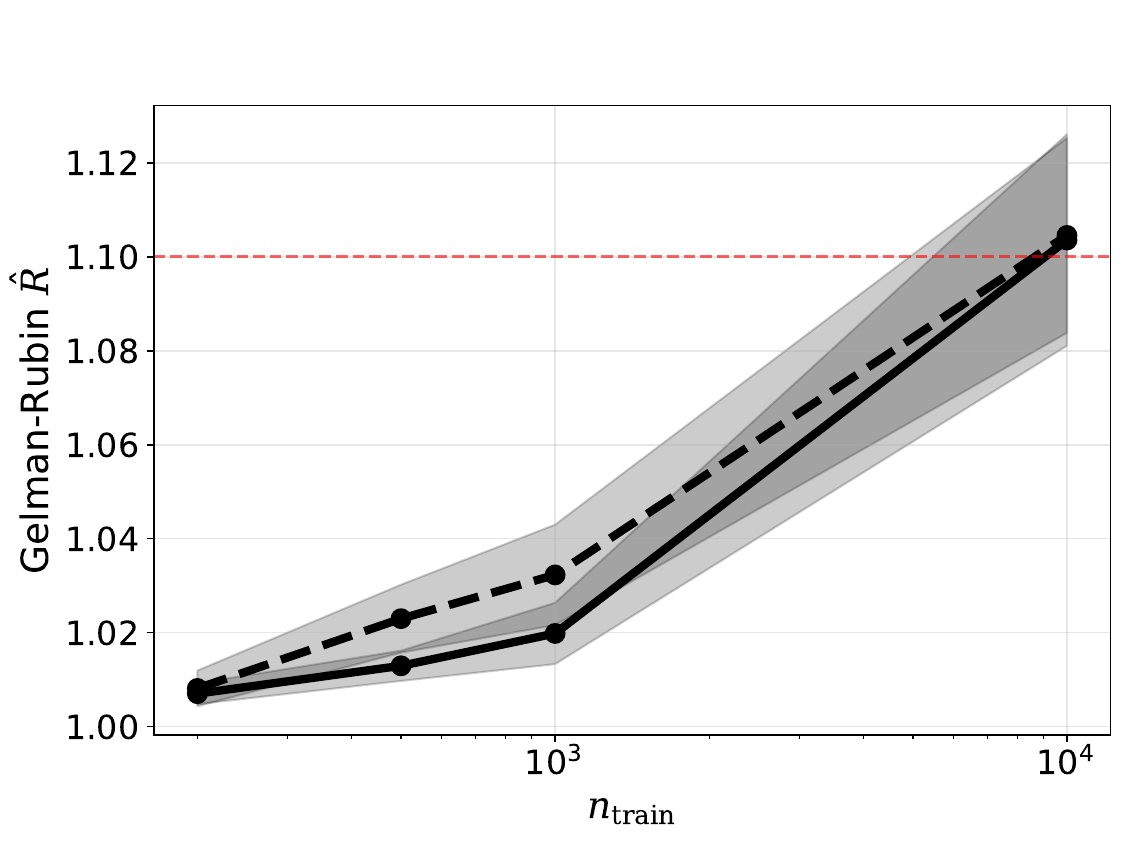}%
        \hfill
        \includegraphics[width=0.32\linewidth]{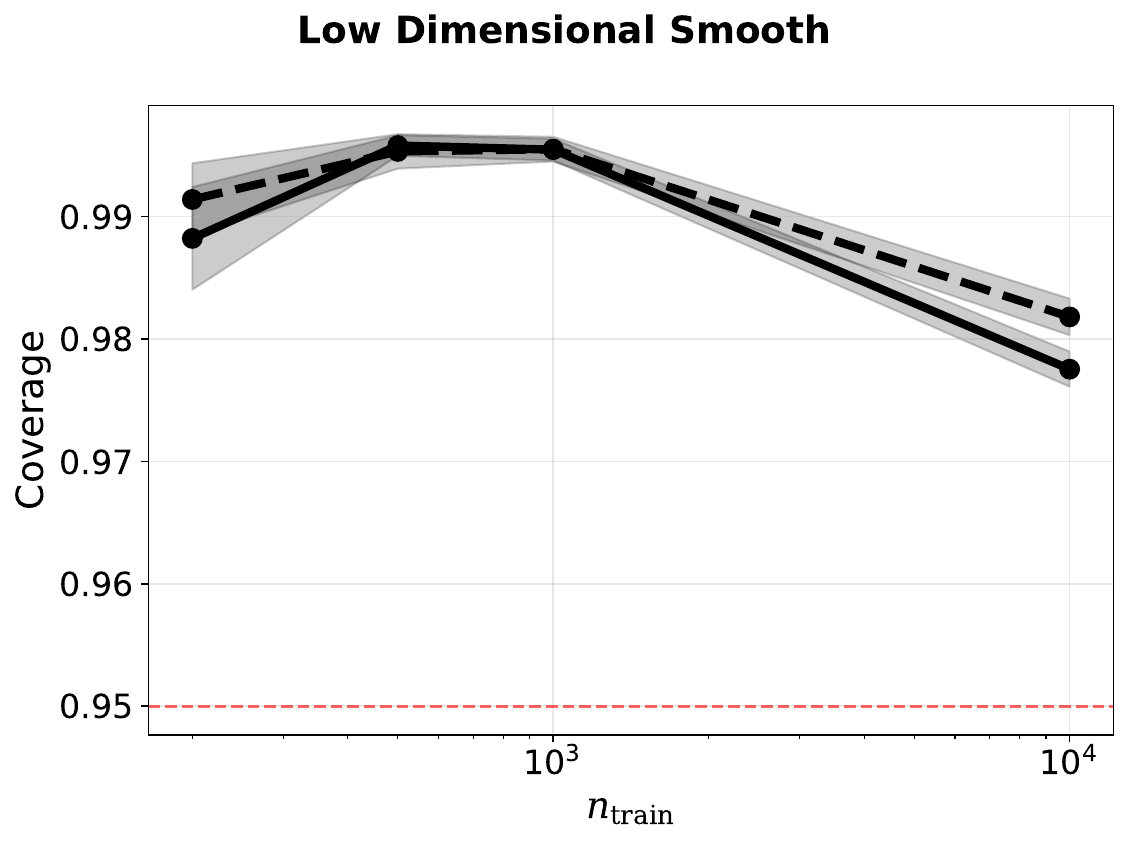}%
        \hfill
        \includegraphics[width=0.32\linewidth]{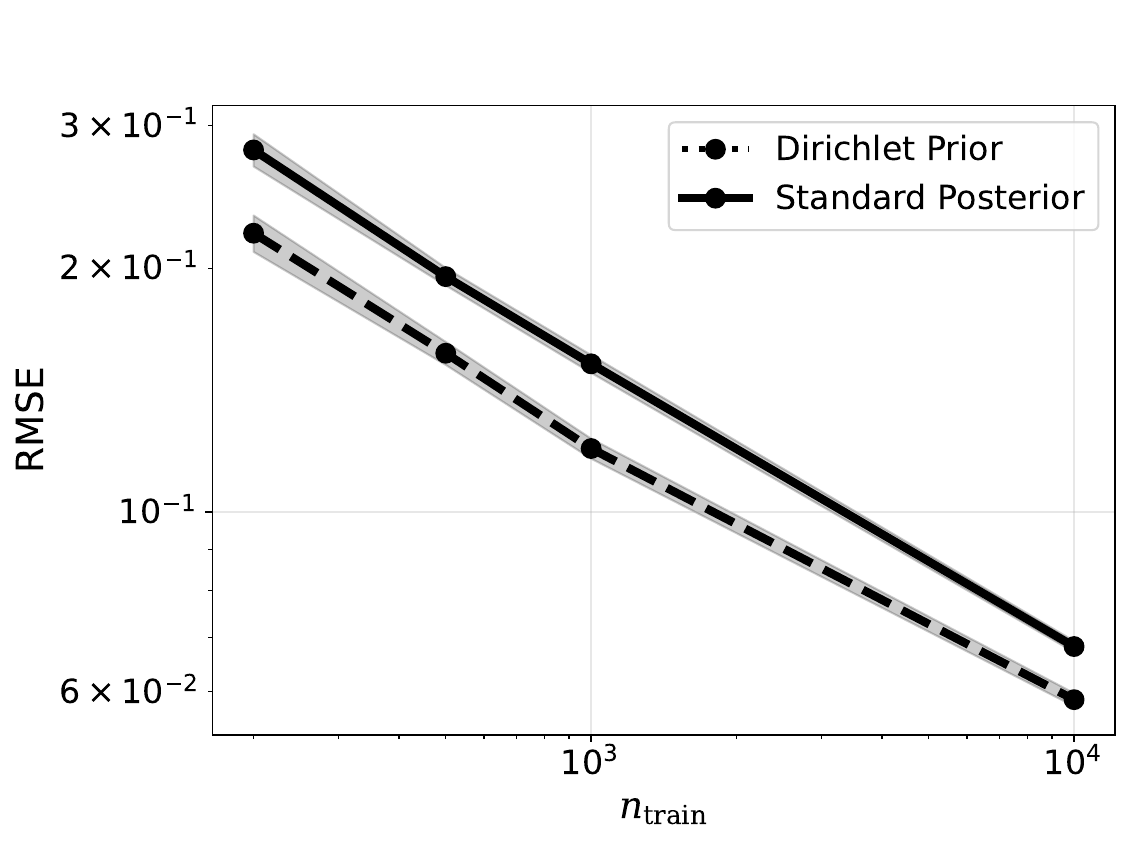}
    \end{minipage} \\
    \begin{minipage}{0.9\textwidth}
        \centering
        \includegraphics[width=0.32\linewidth]{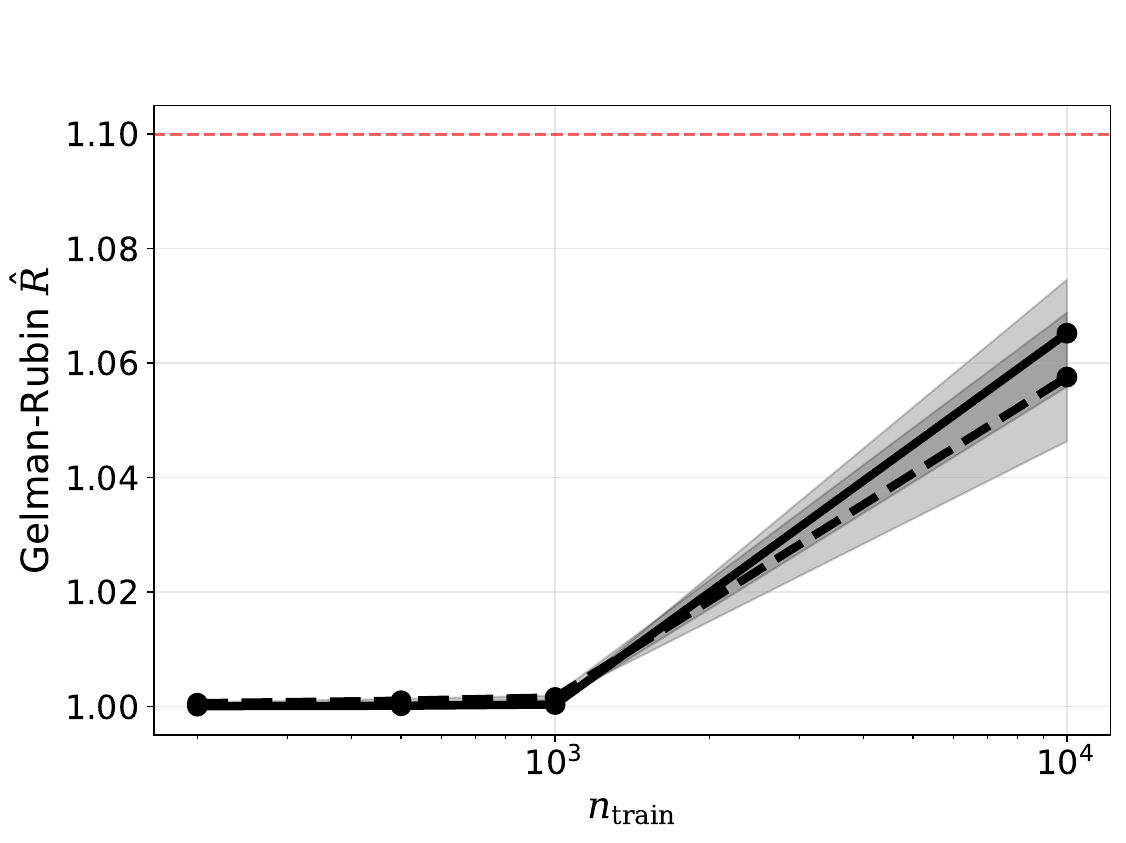}%
        \hfill
        \includegraphics[width=0.32\linewidth]{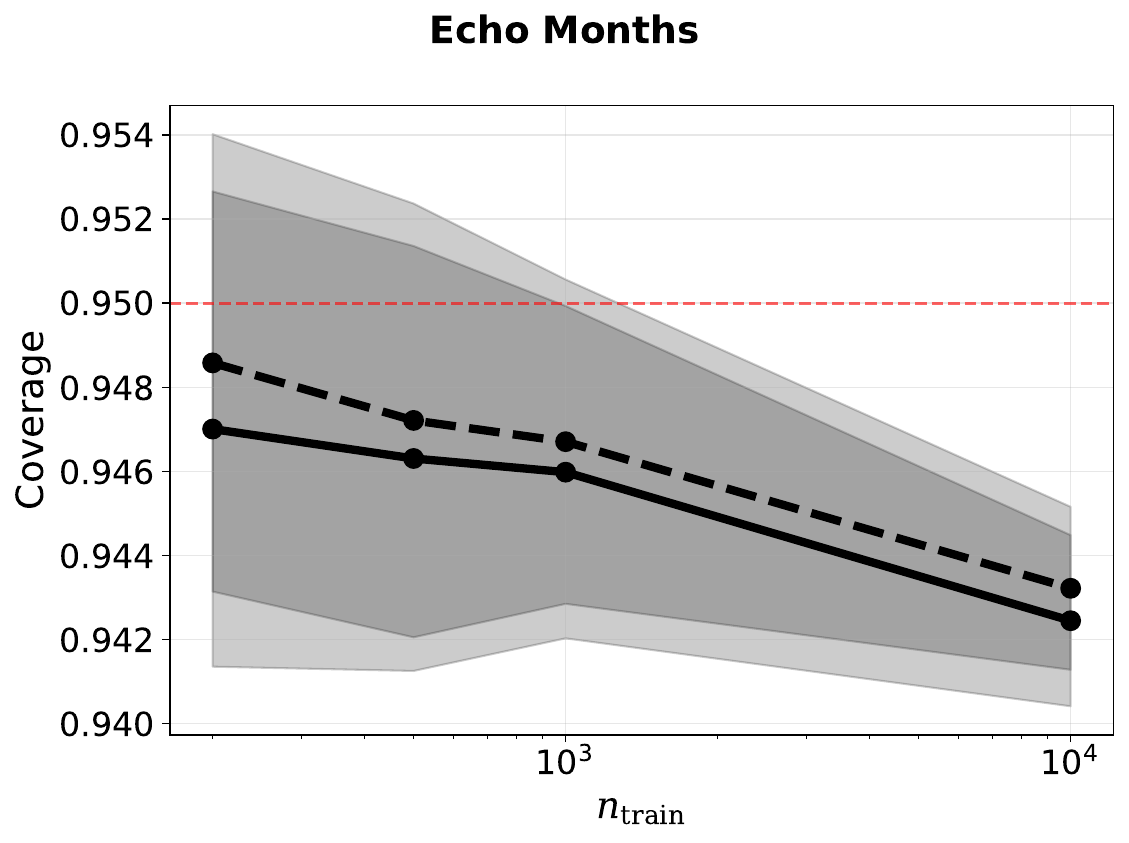}%
        \hfill
        \includegraphics[width=0.32\linewidth]{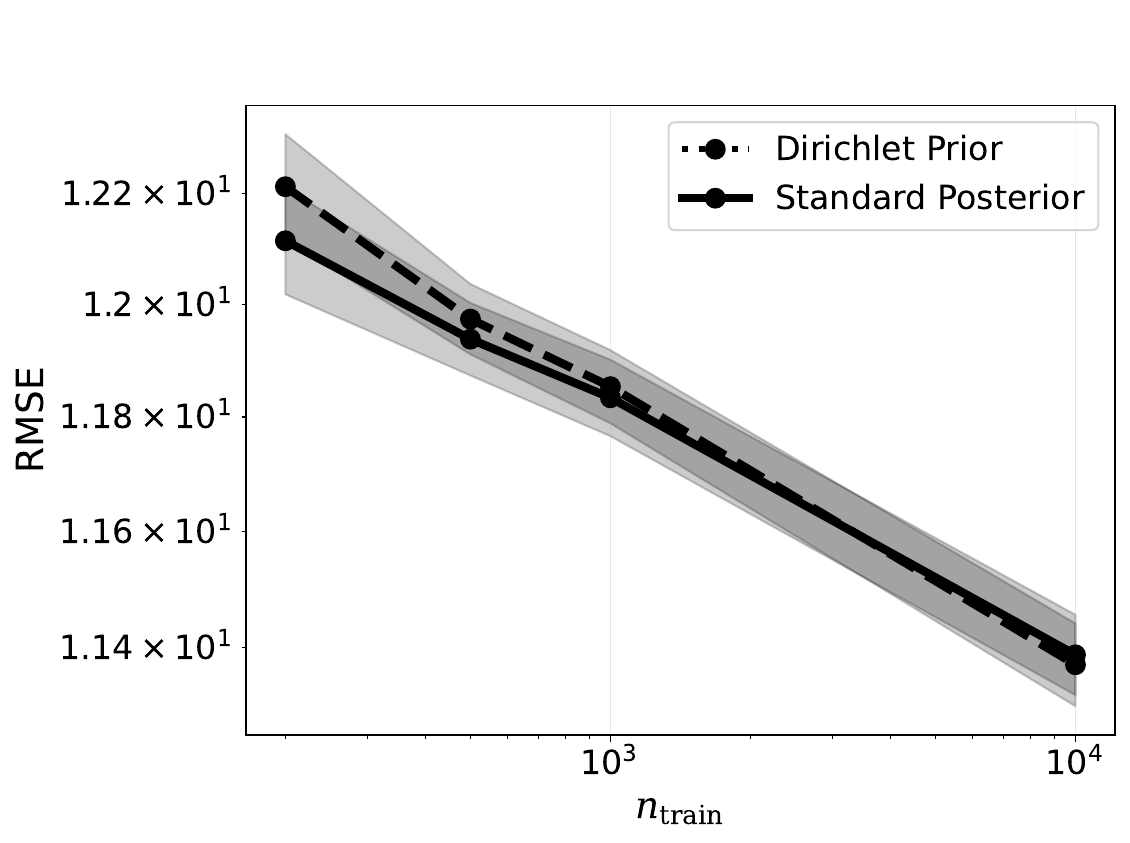}
    \end{minipage} \\
    \begin{minipage}{0.9\textwidth}
        \centering
        \includegraphics[width=0.32\linewidth]{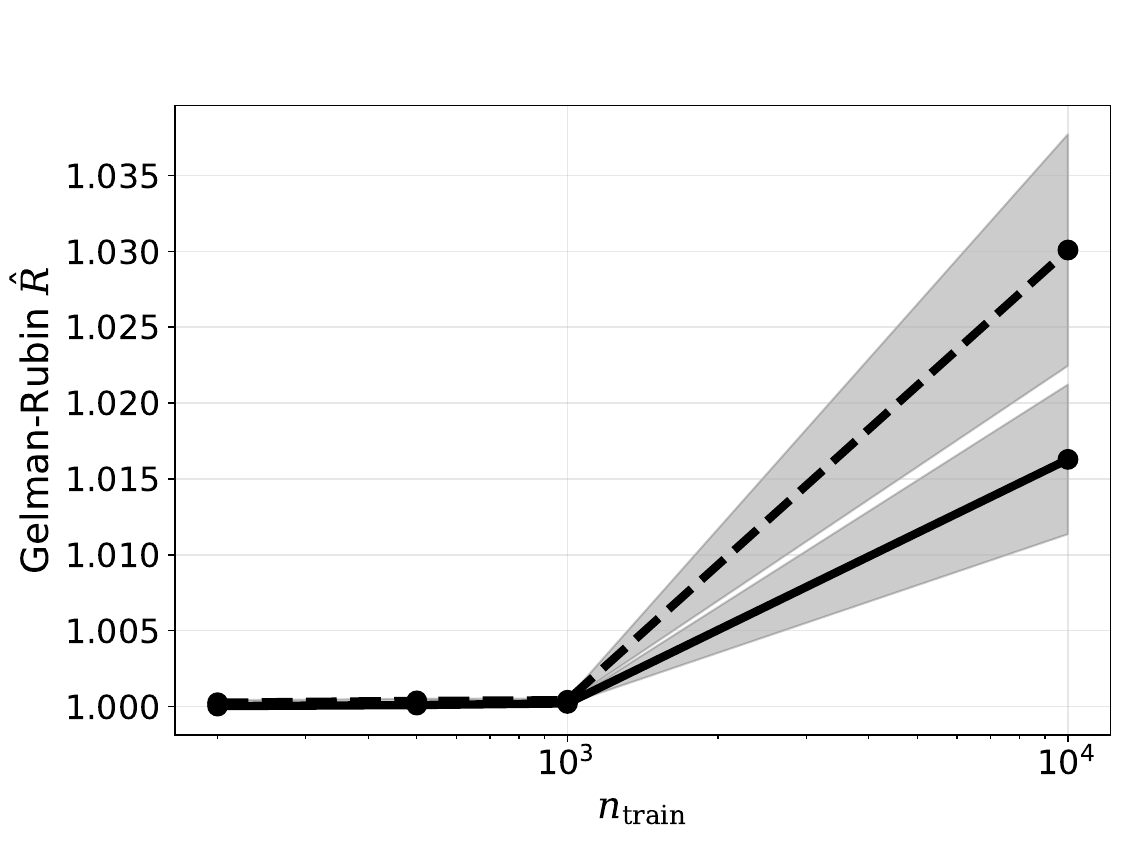}
        \hfill
        \includegraphics[width=0.32\linewidth]{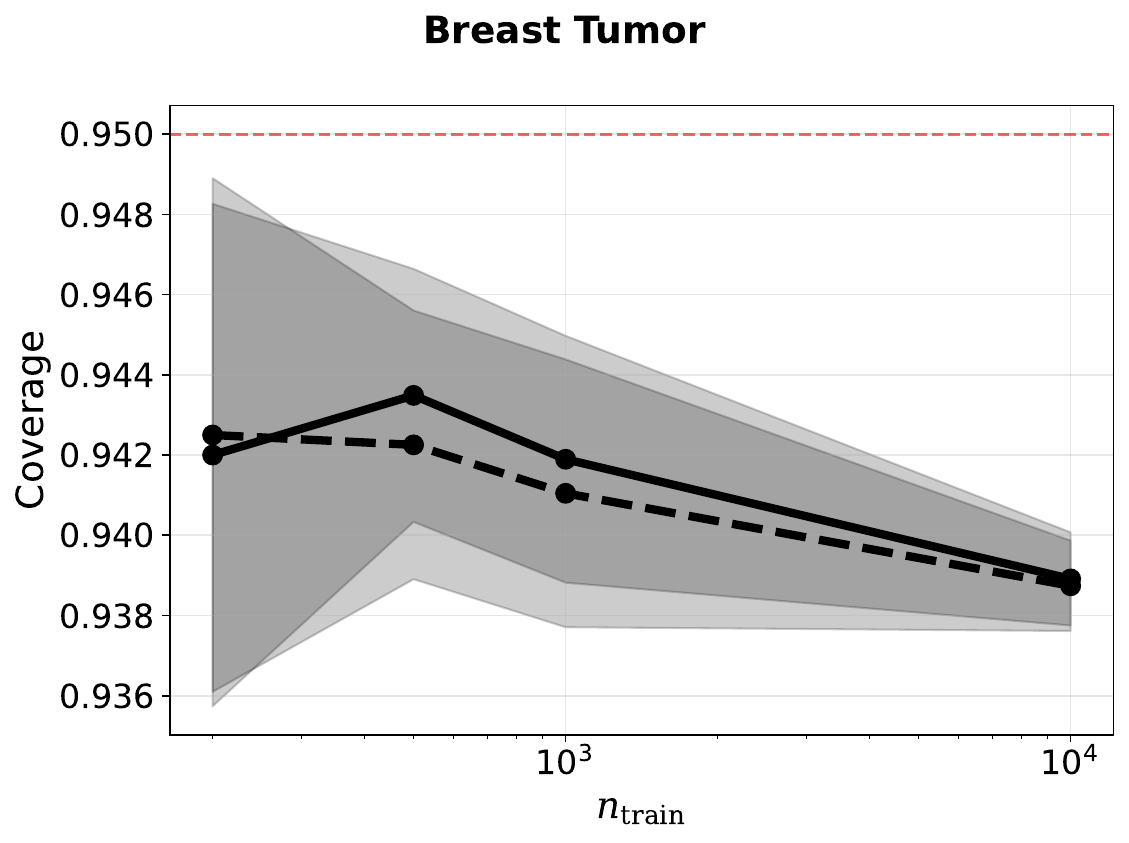}
        \hfill
        \includegraphics[width=0.32\linewidth]{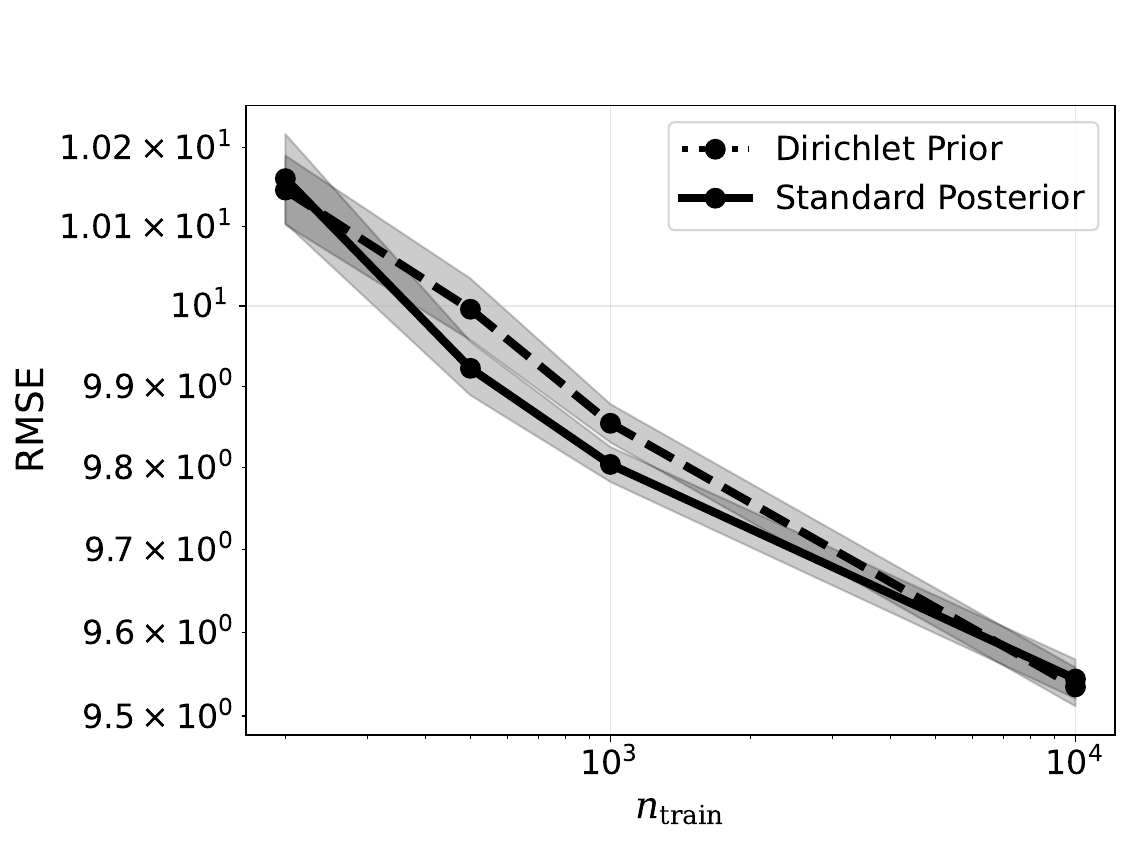}%
    \end{minipage} \\
    \begin{minipage}{0.9\textwidth}
        \centering
        \includegraphics[width=0.32\linewidth]{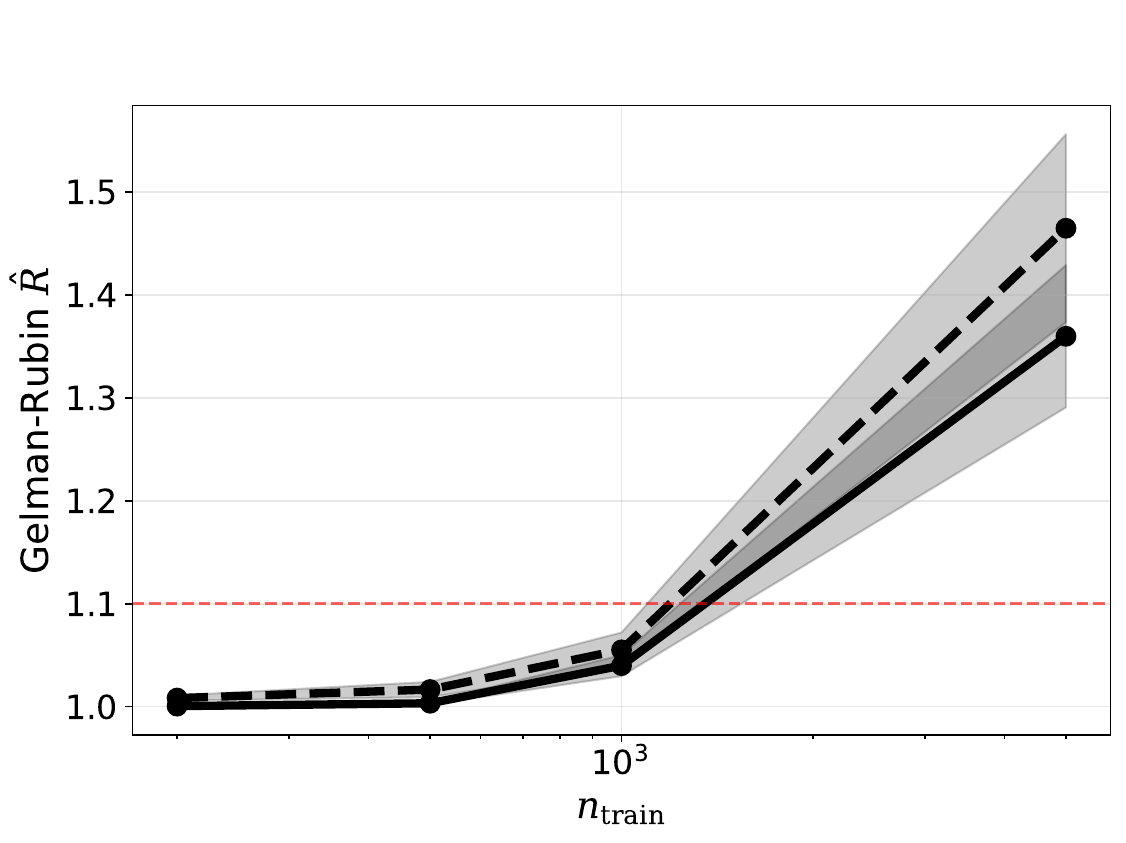}%
        \hfill
        \includegraphics[width=0.32\linewidth]{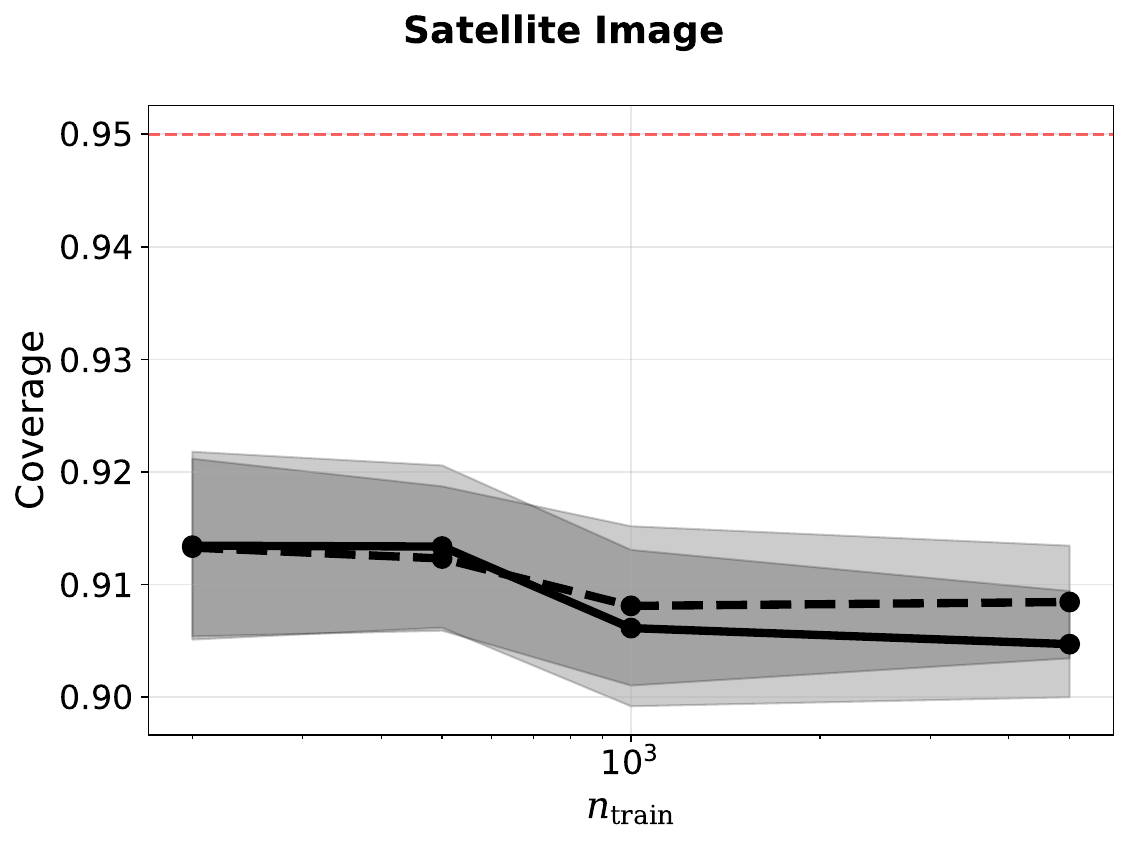}
        \hfill
        \includegraphics[width=0.32\linewidth]{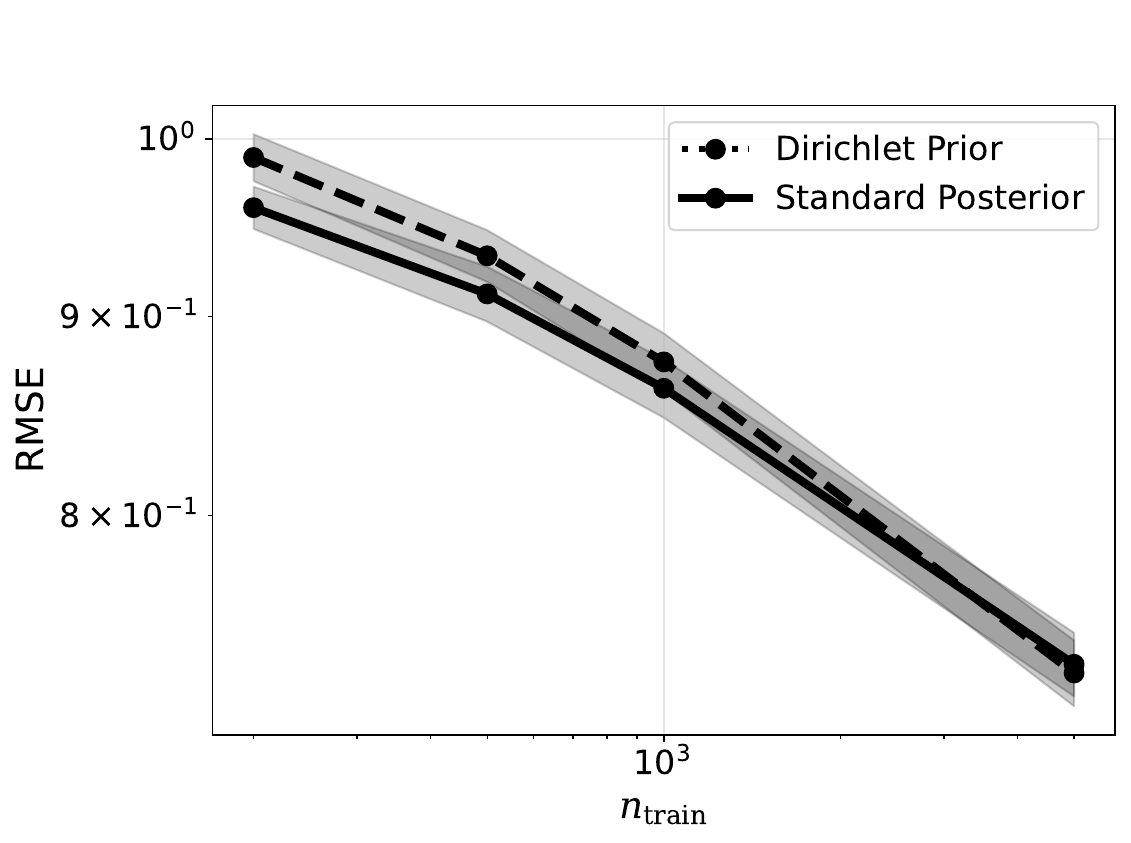}%
    \end{minipage} \\
    \begin{minipage}{0.9\textwidth}
        \centering
        \includegraphics[width=0.32\linewidth]{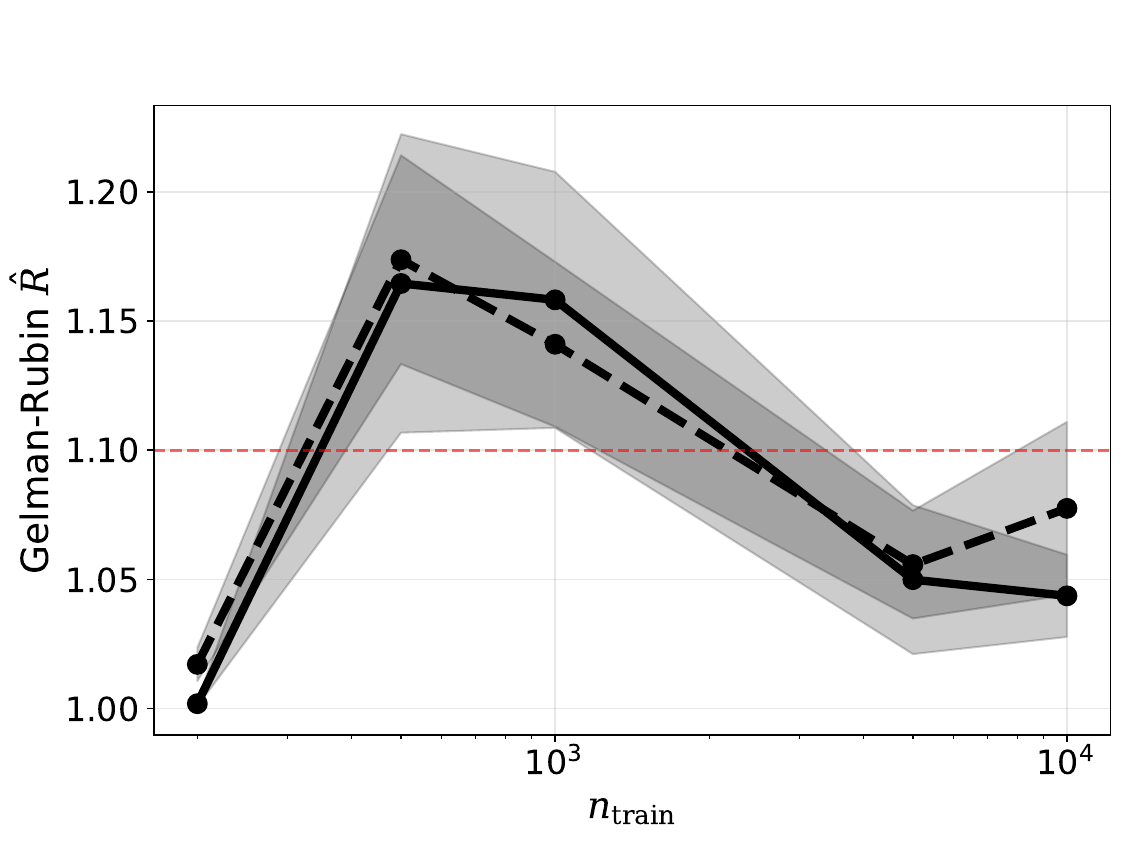}%
        \hfill
        \includegraphics[width=0.32\linewidth]{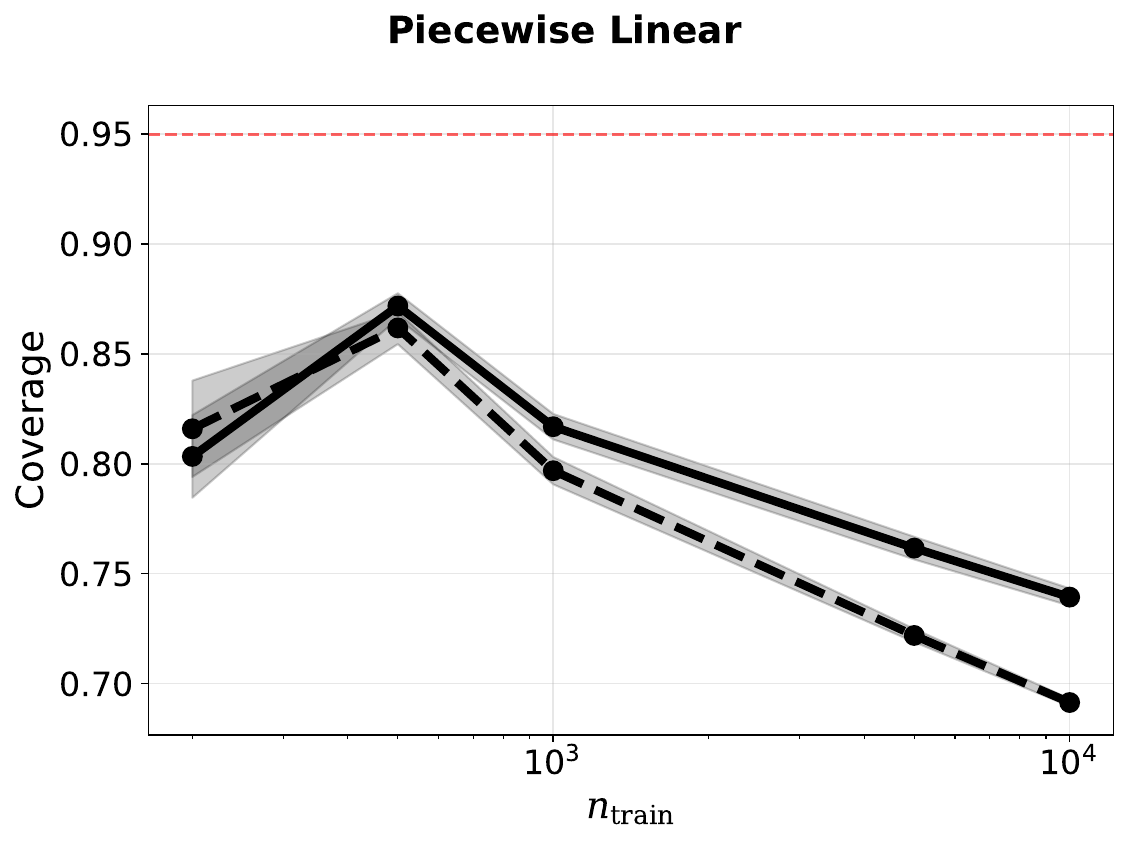}
        \hfill
        \includegraphics[width=0.32\linewidth]{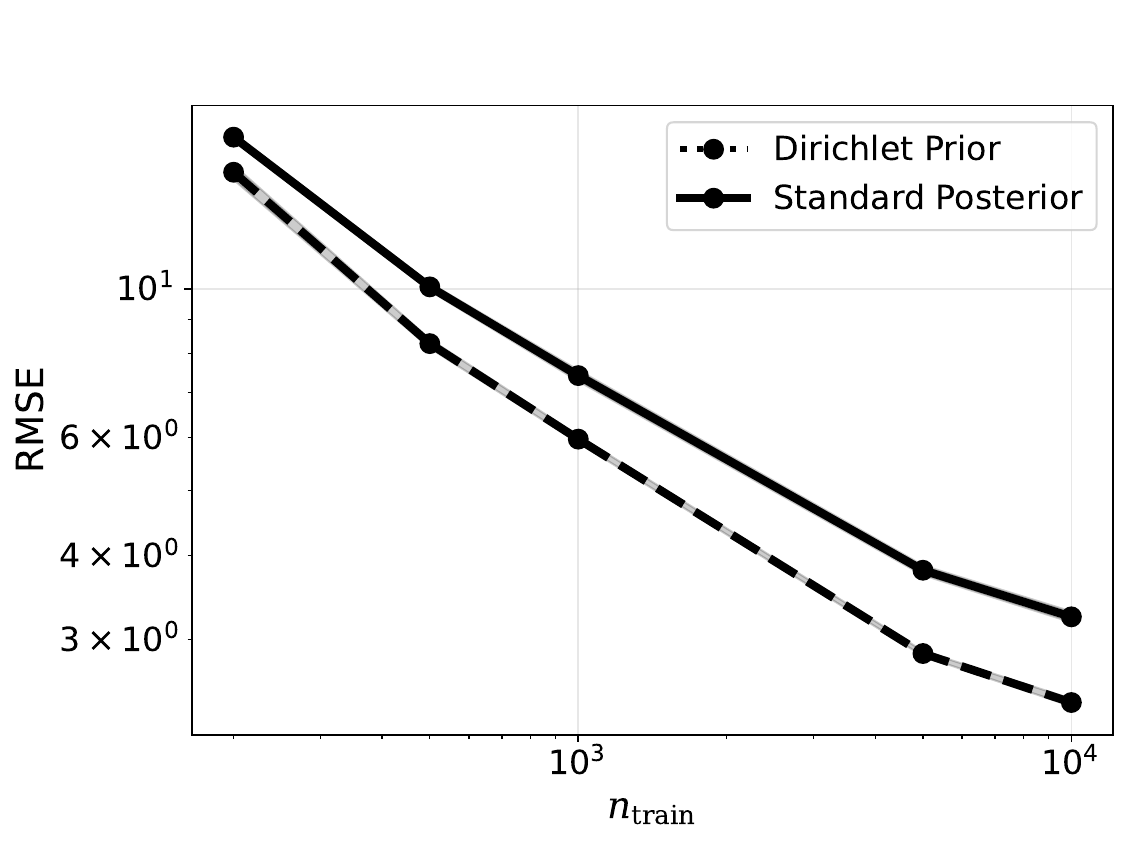}%
    \end{minipage} \\
    \end{tabular}
    \caption{
    Values for Gelman-Rubin $\hat R$ (left), coverage (center), and RMSE (right) for the BART sampler under the Dirichlet prior and the standard (uniform) prior for split feature probabilities.
    Results are plotted for California Housing, Low Dimensional Smooth, Echo Months, Breast Tumor, Satellite Image, and Piecewise Linear datasets (from top to bottom).
    Error bars represent $\pm 1.96$ standard errors from 25 replicates.
    }
    \label{fig:experiment5}
\end{figure}

\newpage

\begin{figure}[H]
    \centering
    \begin{tabular}{c}
    \begin{minipage}{0.9\textwidth}
        \centering
        \includegraphics[width=0.32\linewidth]{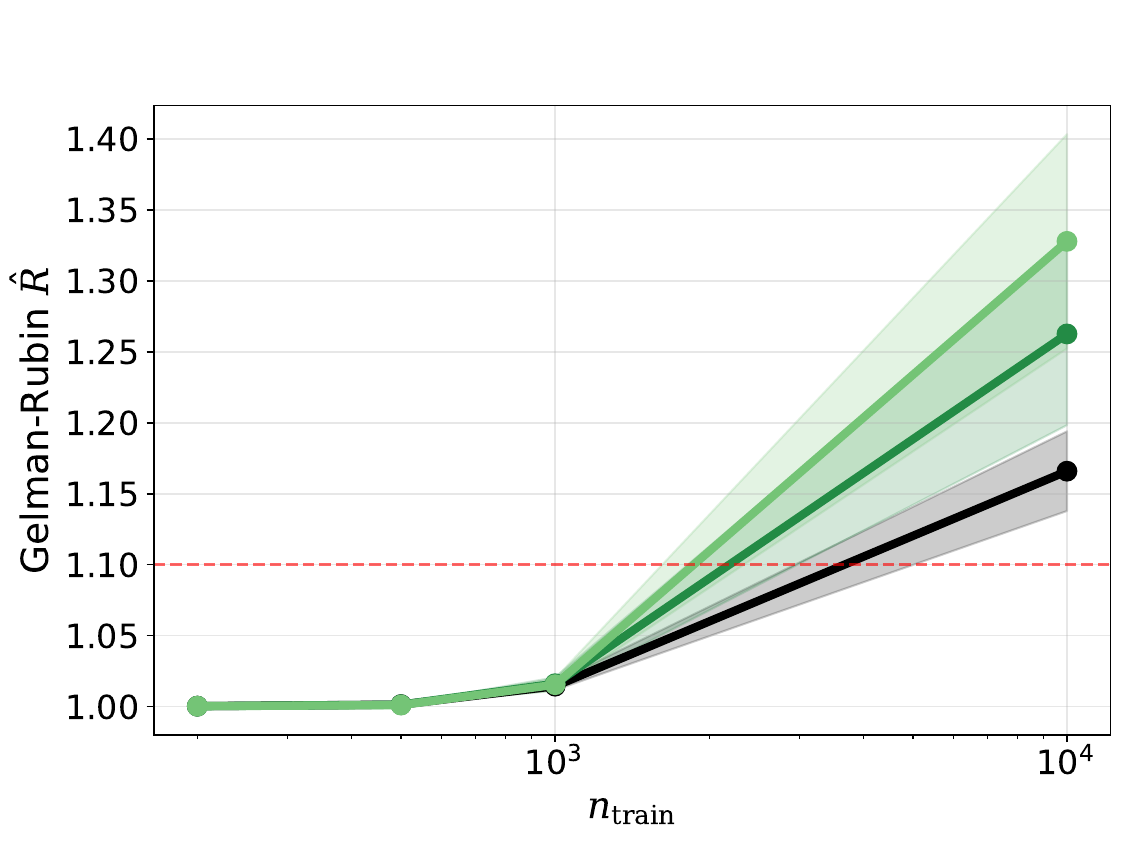}
        \hfill
        \includegraphics[width=0.32\linewidth]{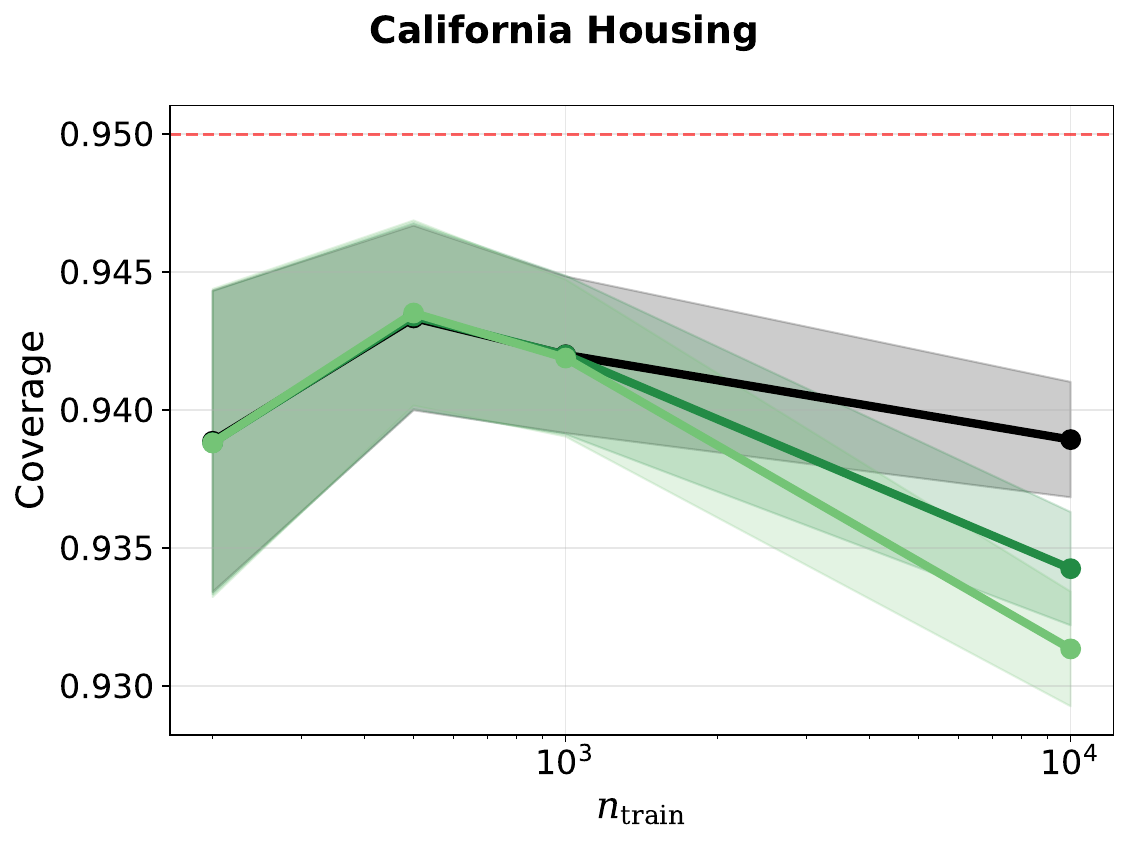}
        \hfill
        \includegraphics[width=0.32\linewidth]{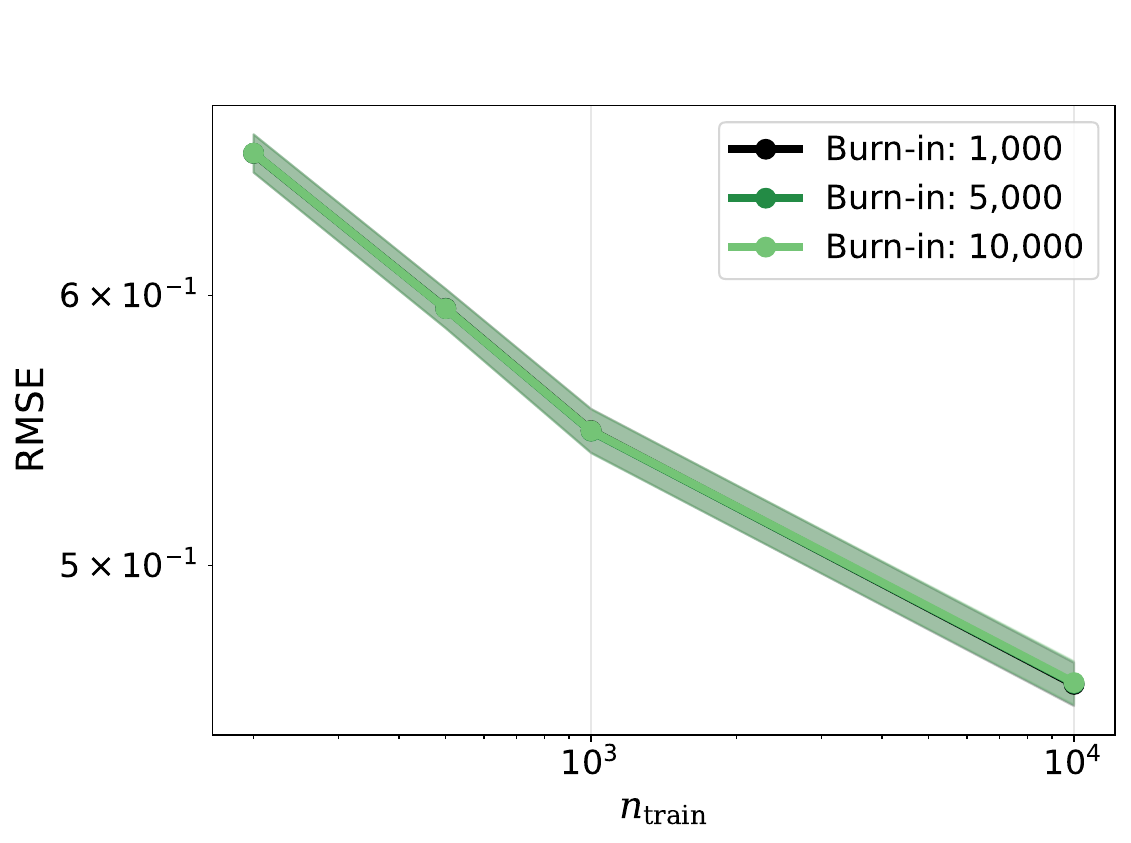}%
    \end{minipage} \\
    \begin{minipage}{0.9\textwidth}
        \centering
        \includegraphics[width=0.32\linewidth]{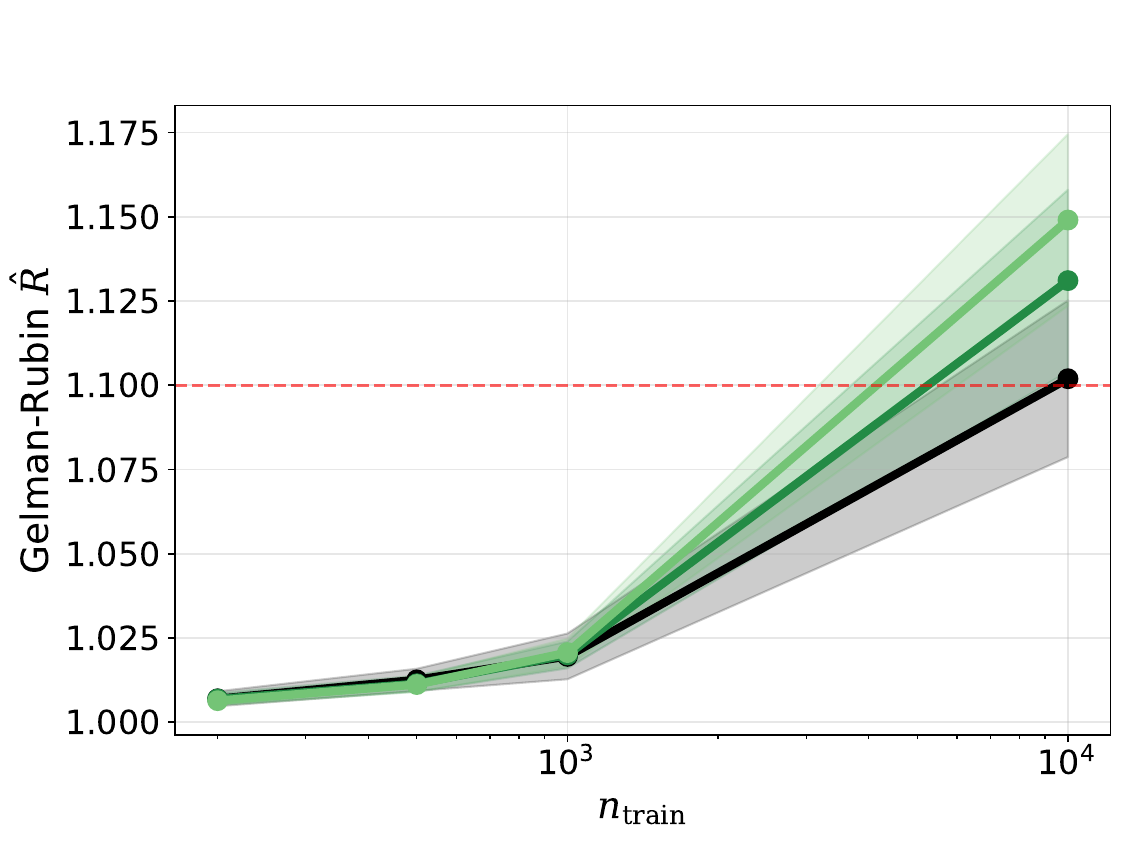}%
        \hfill
        \includegraphics[width=0.32\linewidth]{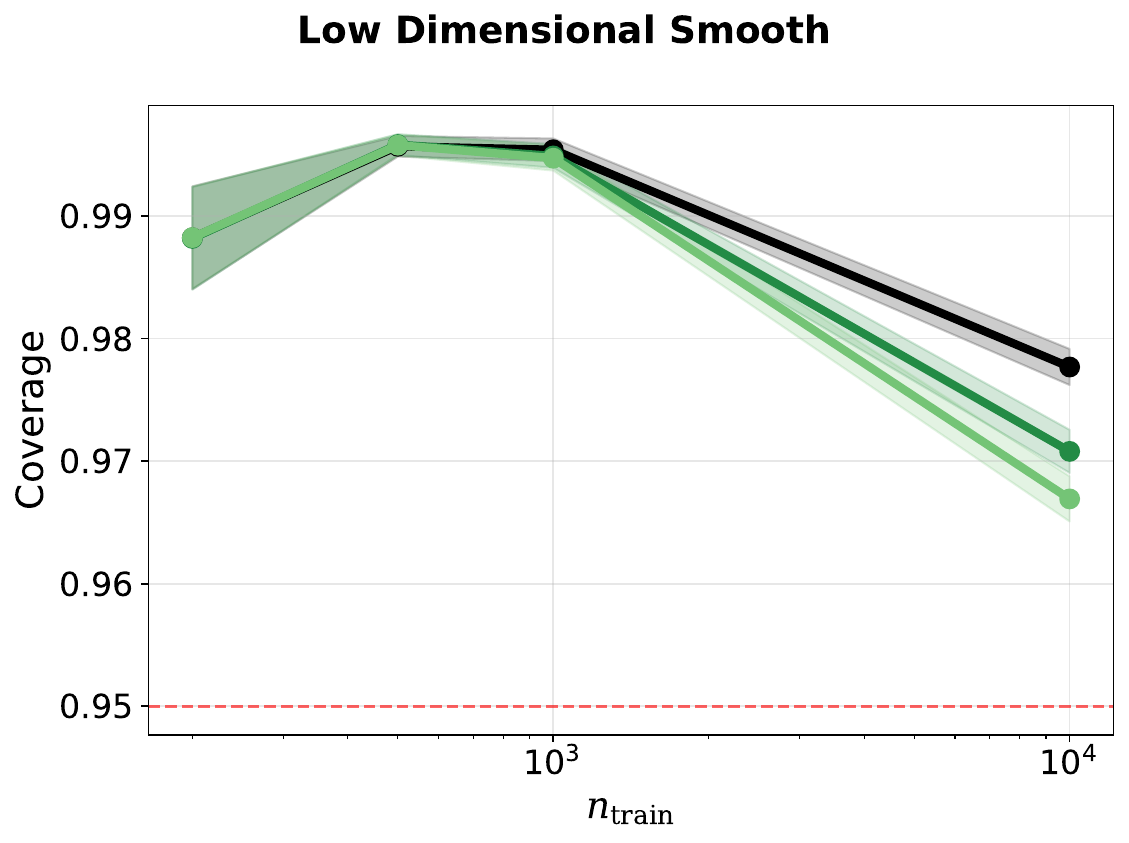}
        \hfill
        \includegraphics[width=0.32\linewidth]{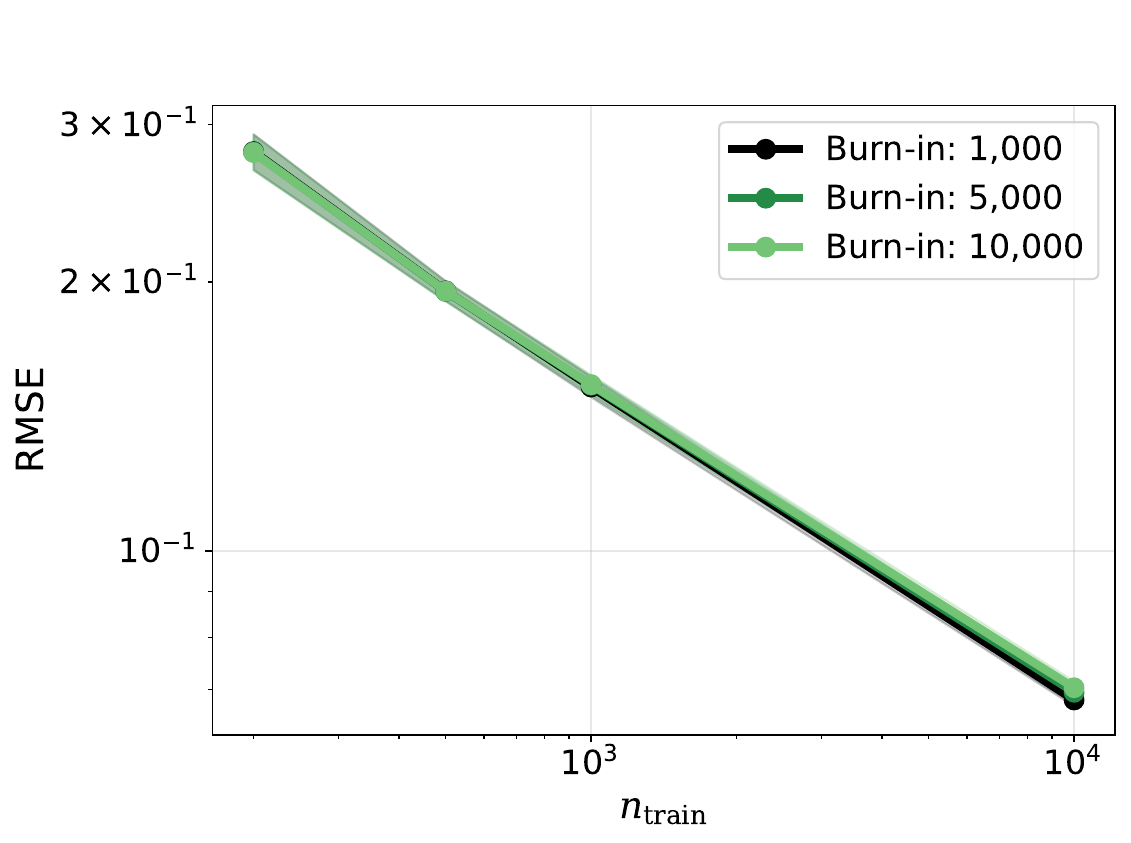}%
    \end{minipage} \\
    \begin{minipage}{0.9\textwidth}
        \centering
        \includegraphics[width=0.32\linewidth]{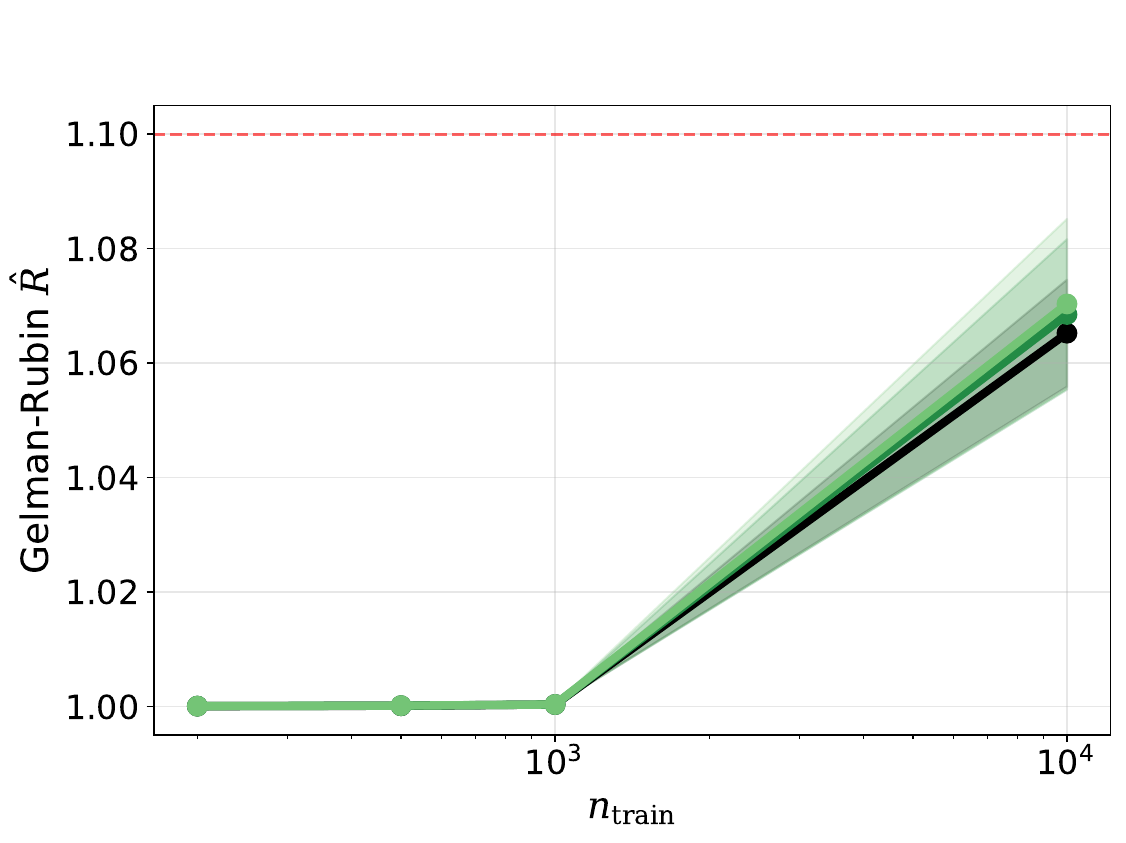}%
        \hfill
        \includegraphics[width=0.32\linewidth]{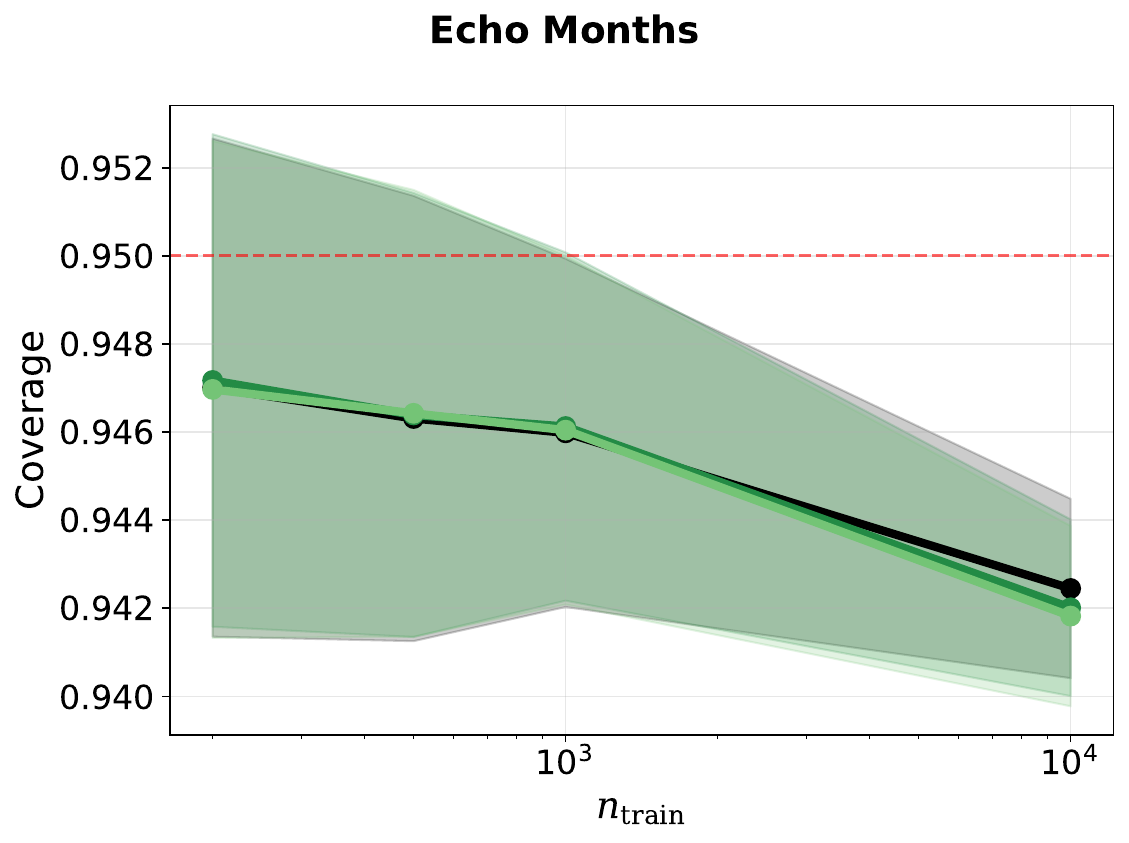}
        \hfill
        \includegraphics[width=0.32\linewidth]{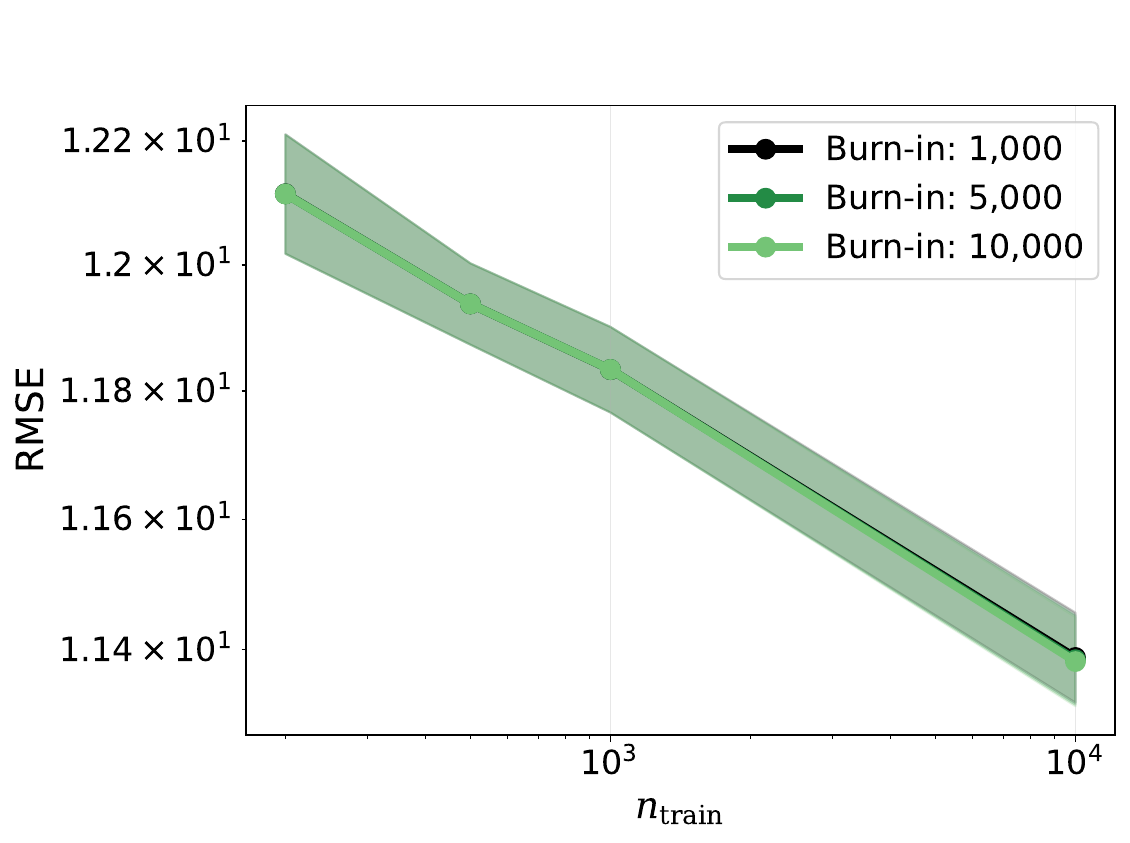}%
    \end{minipage} \\
    \begin{minipage}{0.9\textwidth}
        \centering
        \includegraphics[width=0.32\linewidth]{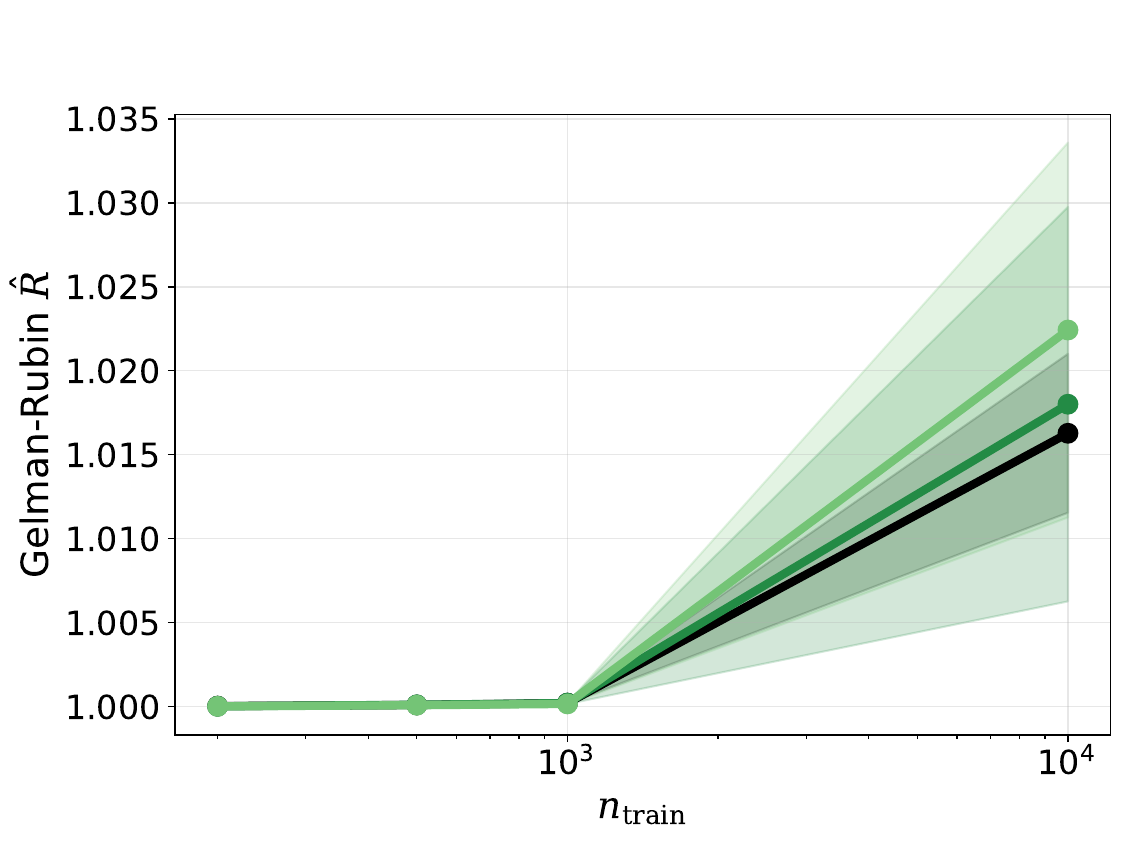}
        \hfill
        \includegraphics[width=0.32\linewidth]{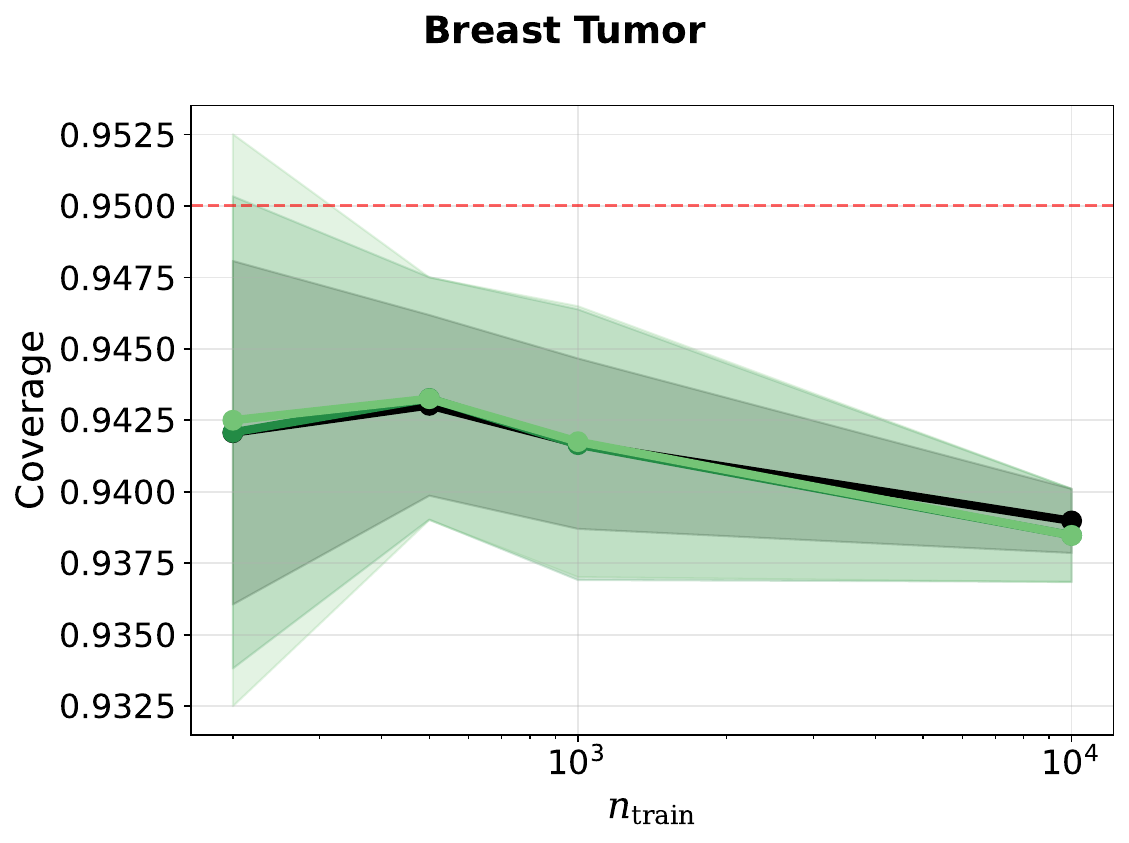}
        \hfill
        \includegraphics[width=0.32\linewidth]{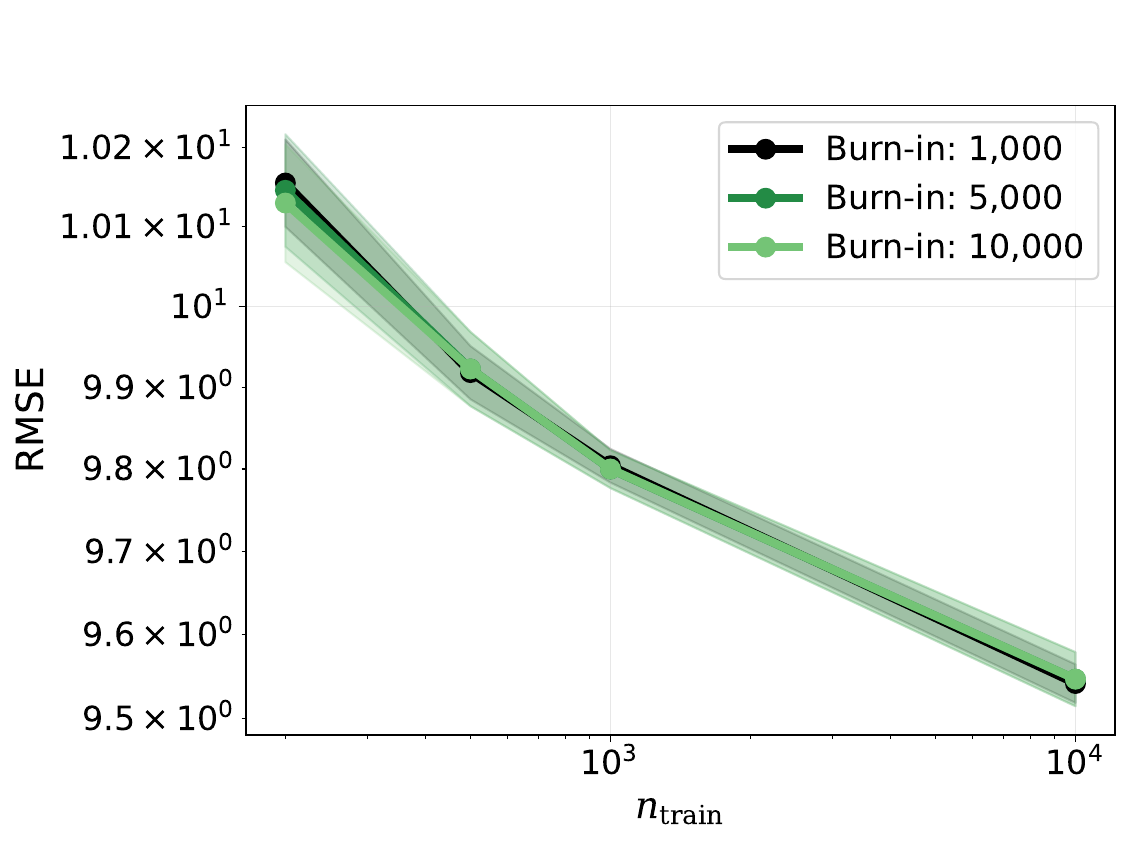}%
    \end{minipage} \\
    \begin{minipage}{0.9\textwidth}
        \centering
        \includegraphics[width=0.32\linewidth]{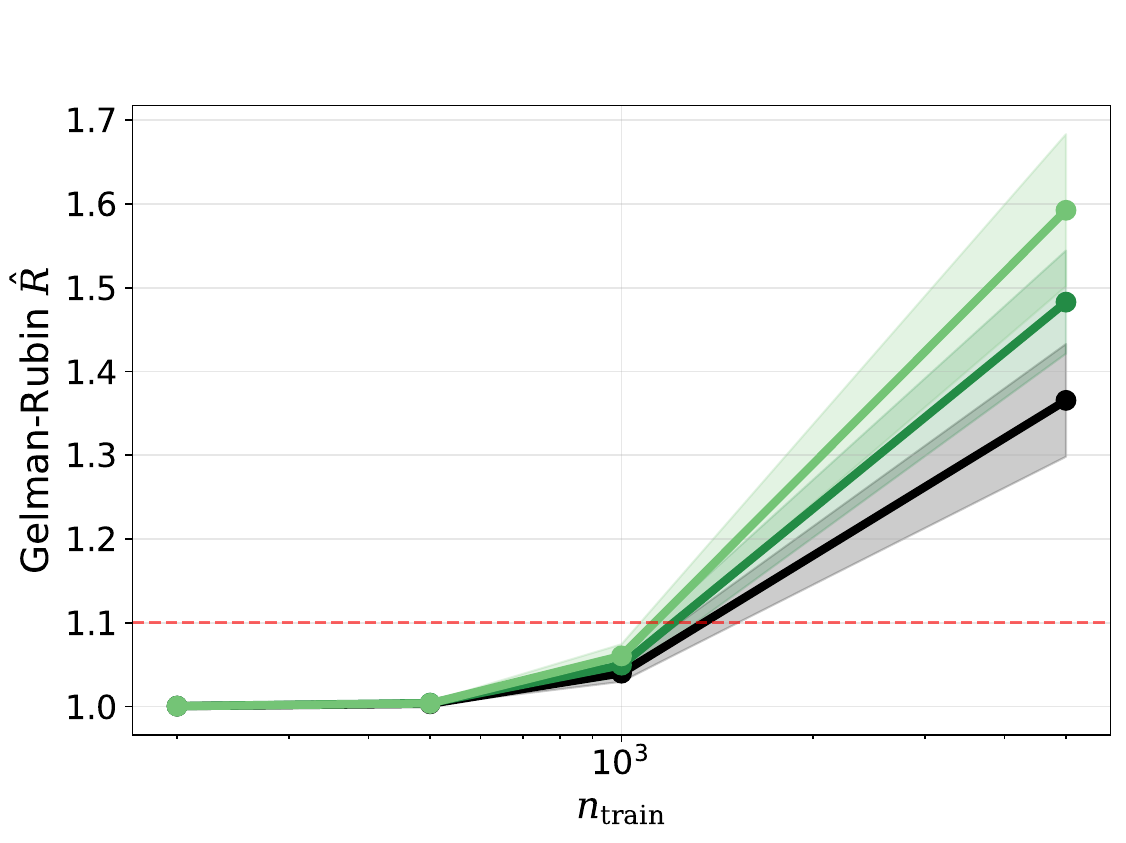}%
        \hfill
        \includegraphics[width=0.32\linewidth]{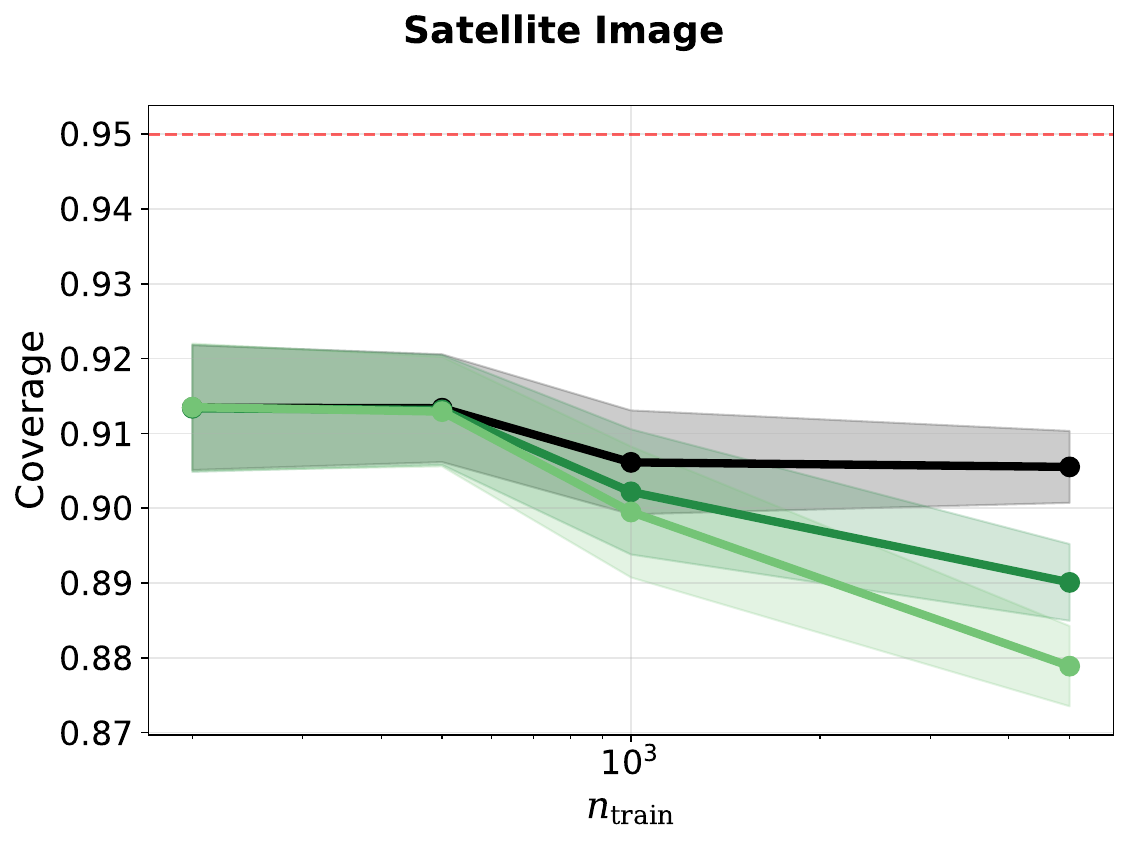}
        \hfill
        \includegraphics[width=0.32\linewidth]{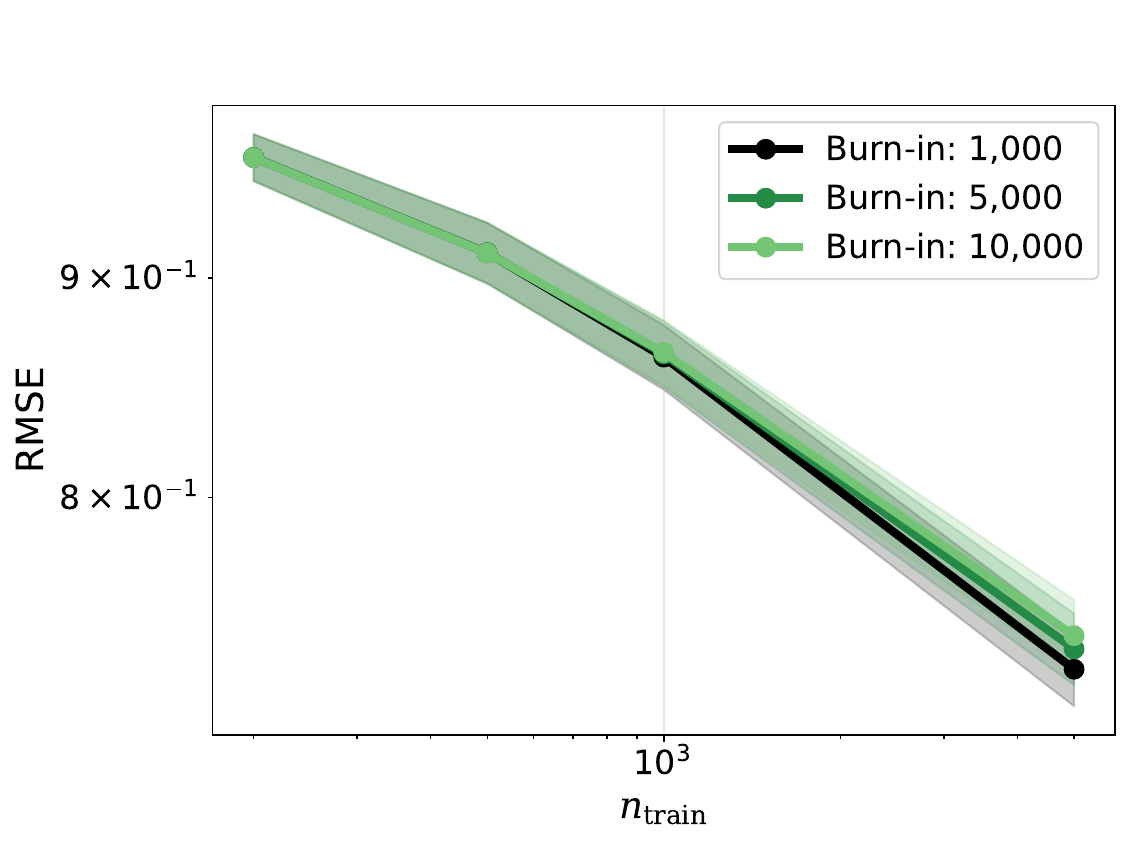}%
    \end{minipage} \\
    \begin{minipage}{0.9\textwidth}
        \centering
        \includegraphics[width=0.32\linewidth]{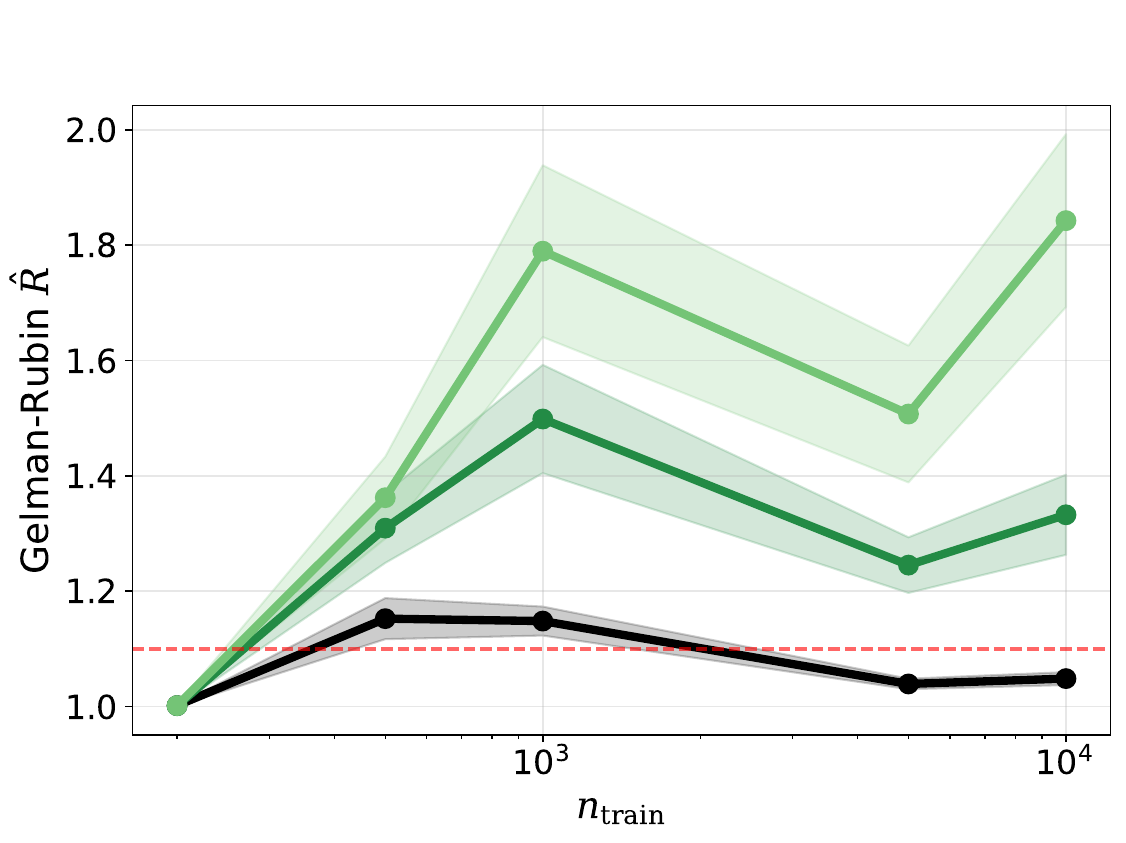}%
        \hfill
        \includegraphics[width=0.32\linewidth]{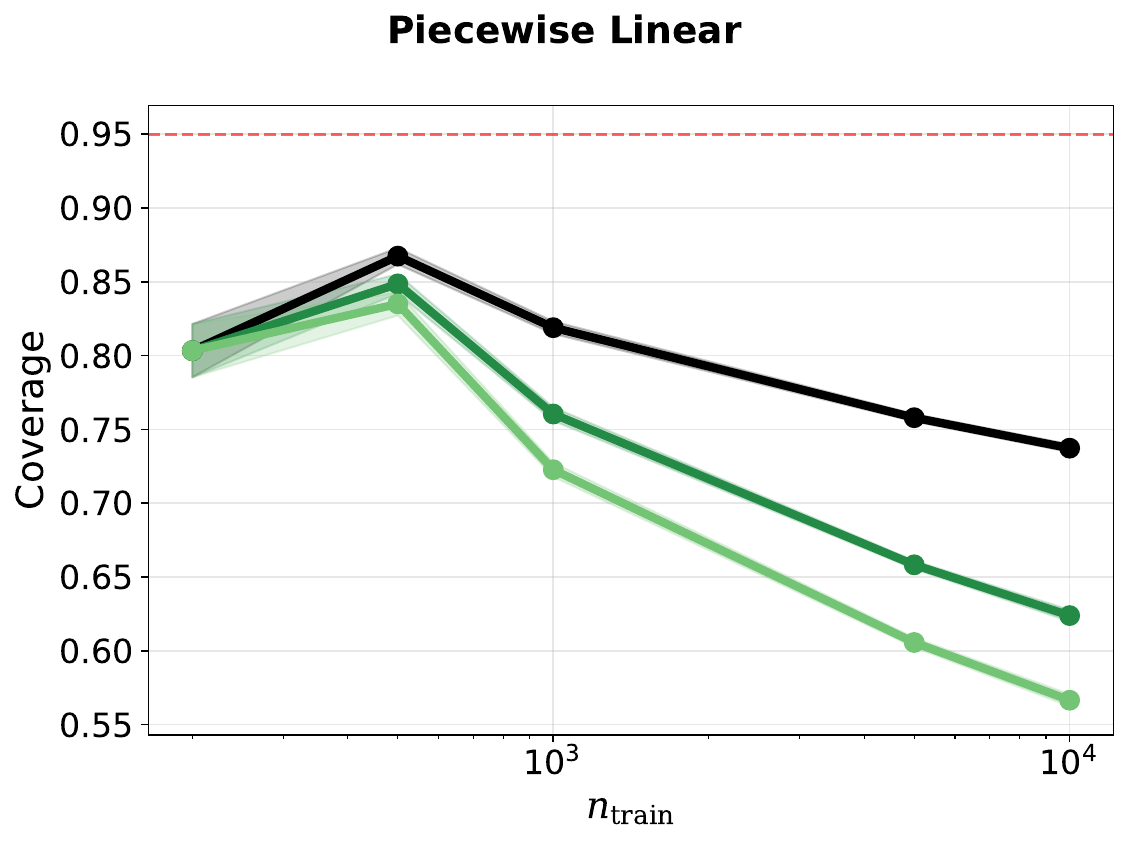}
        \hfill
        \includegraphics[width=0.32\linewidth]{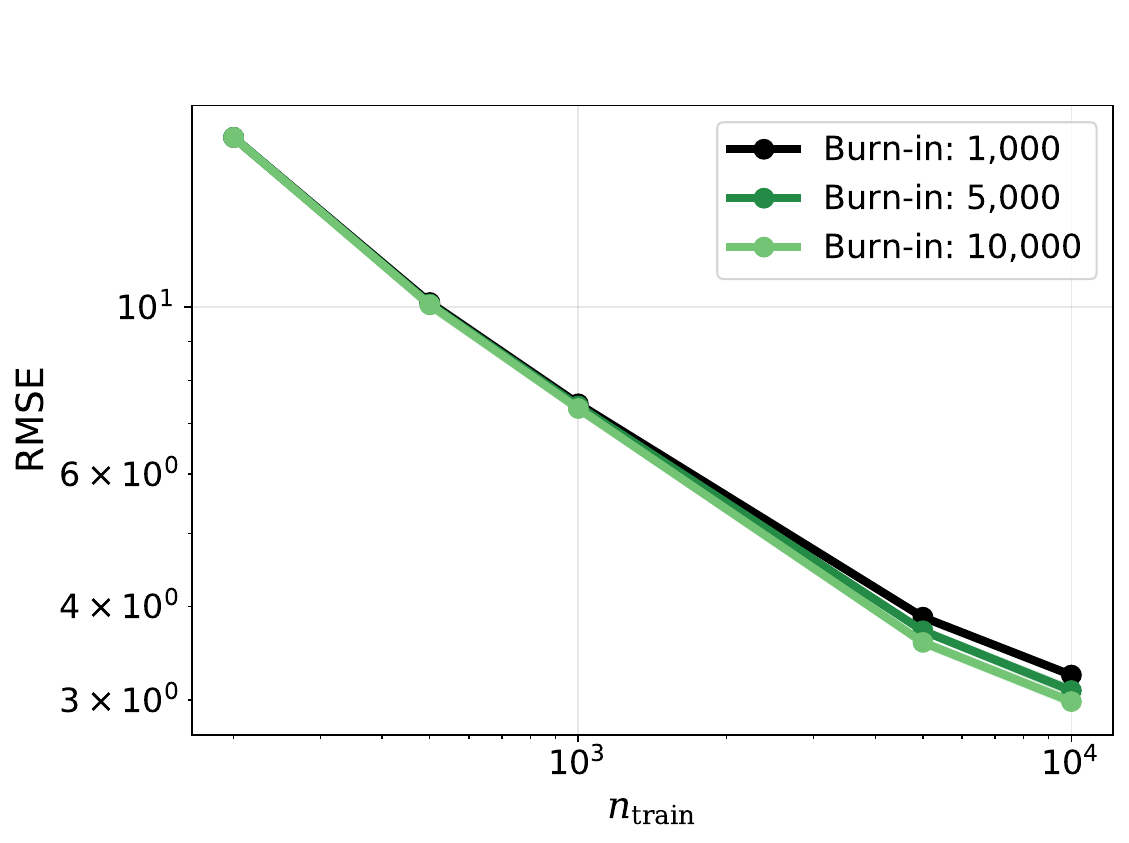}%
    \end{minipage} \\
    \end{tabular}
    \caption{
    Values for Gelman-Rubin $\hat R$ (left), coverage (center), and RMSE (right) for the BART sampler under different burn-in lengths.
    Results are plotted for California Housing, Low Dimensional Smooth, Echo Months, Breast Tumor, Satellite Image, and Piecewise Linear datasets (from top to bottom).
    Error bars represent $\pm 1.96$ standard errors from 25 replicates.
    }
    \label{fig:experiment_burnin}
\end{figure}

\newpage

\begin{figure}[H]
    \centering
    \begin{tabular}{c}
    \begin{minipage}{0.9\textwidth}
        \centering
        \includegraphics[width=0.32\linewidth]{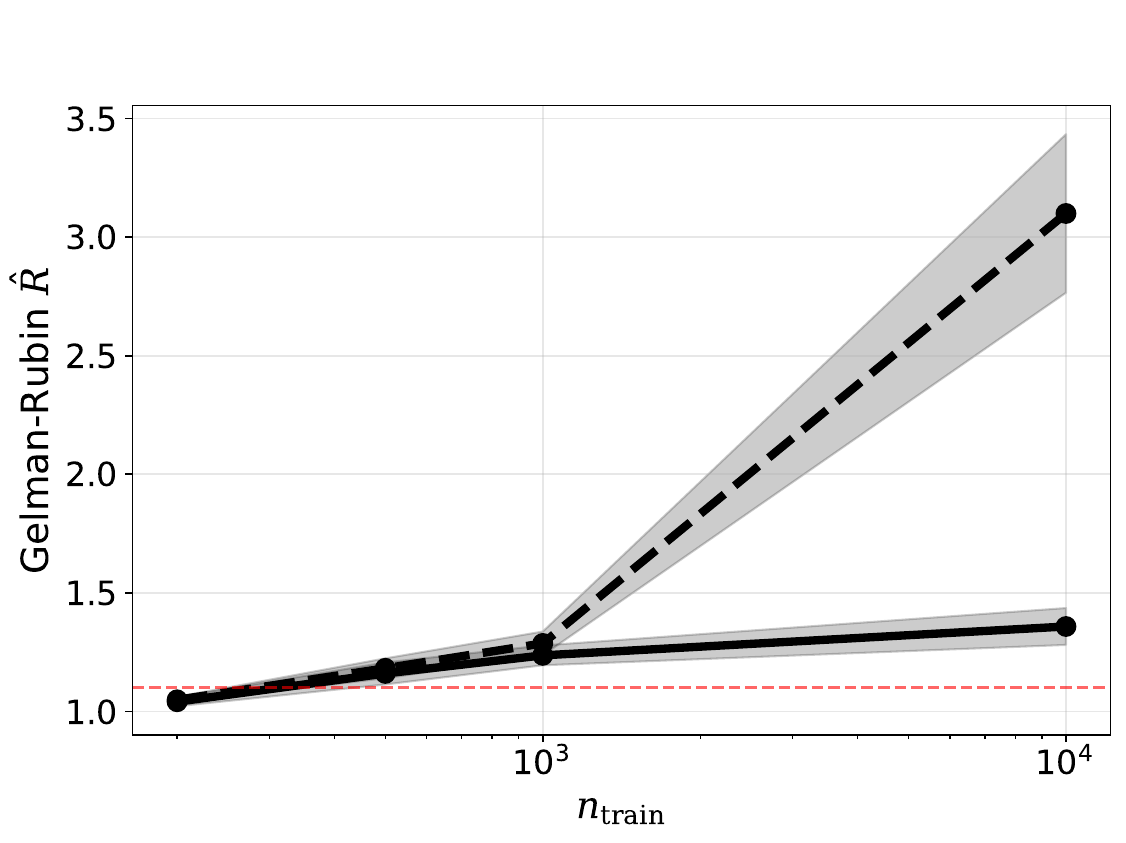}%
        \hfill
        \includegraphics[width=0.32\linewidth]{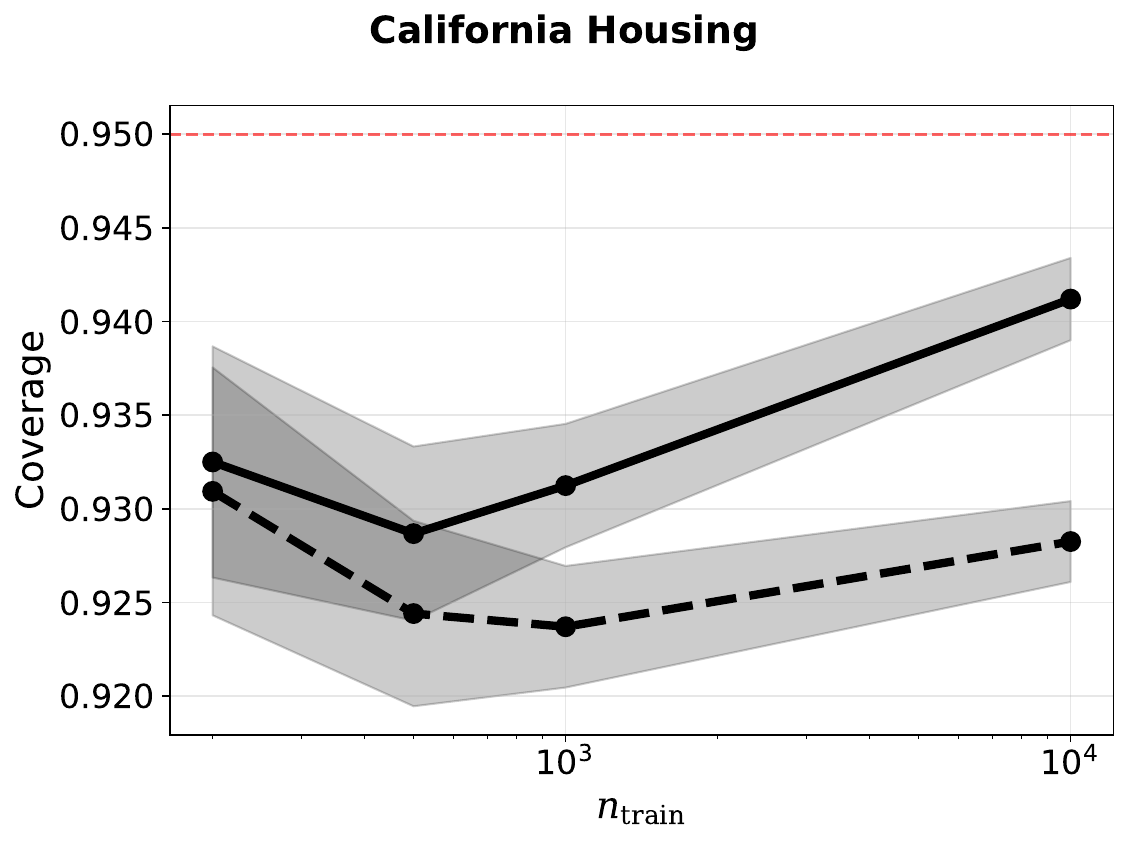}%
        \hfill
        \includegraphics[width=0.32\linewidth]{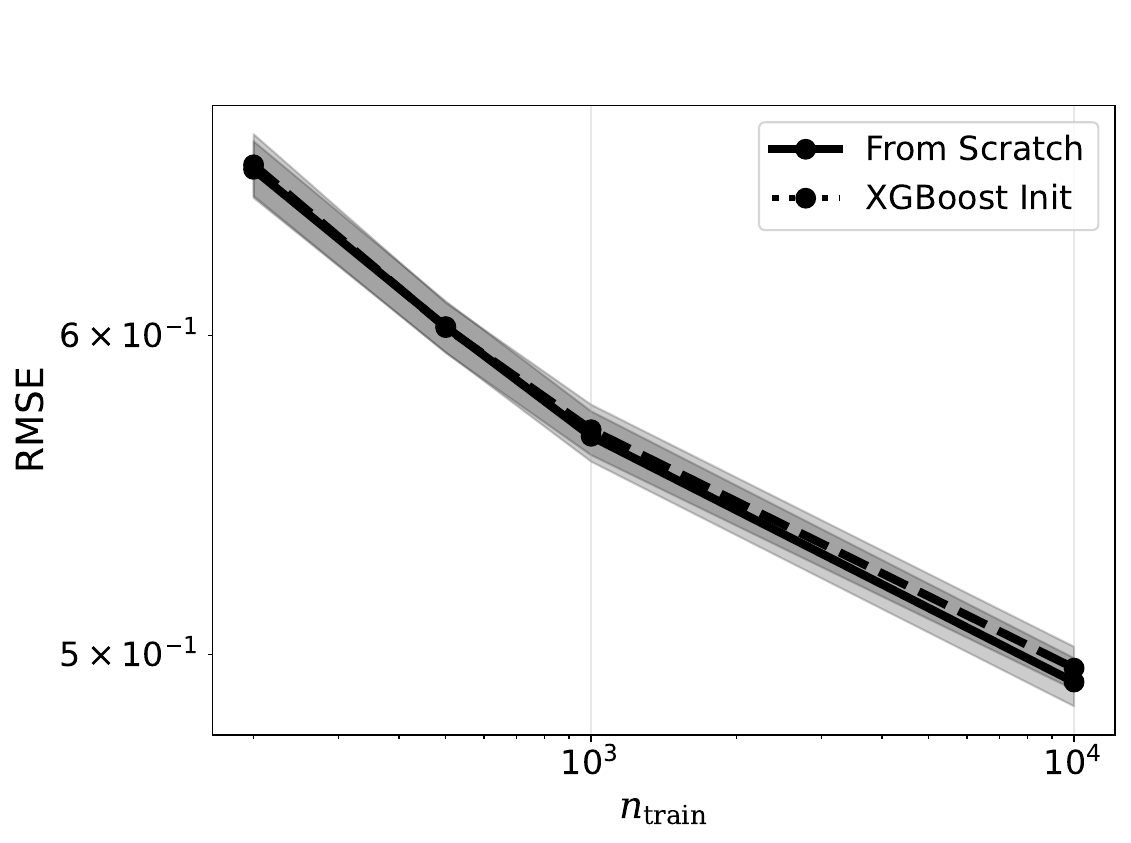}
    \end{minipage} \\
    \begin{minipage}{0.9\textwidth}
        \centering
        \includegraphics[width=0.32\linewidth]{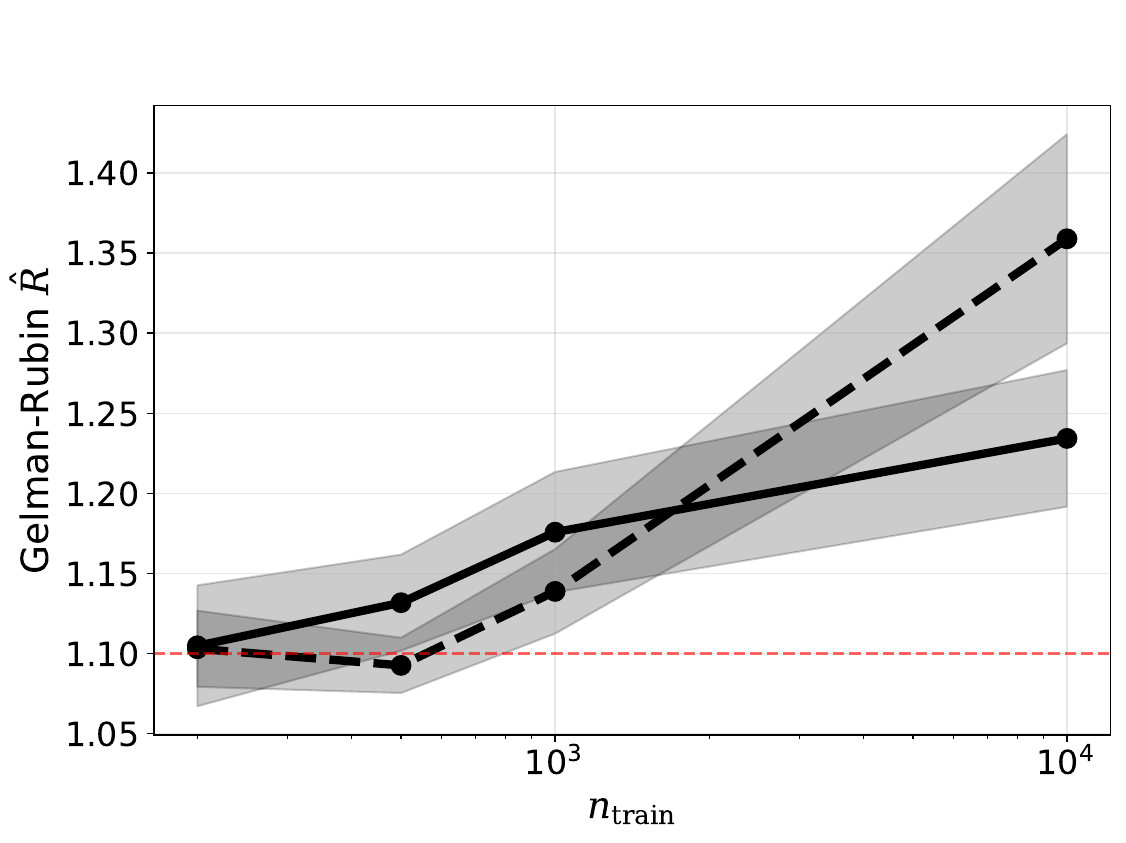}%
        \hfill
        \includegraphics[width=0.32\linewidth]{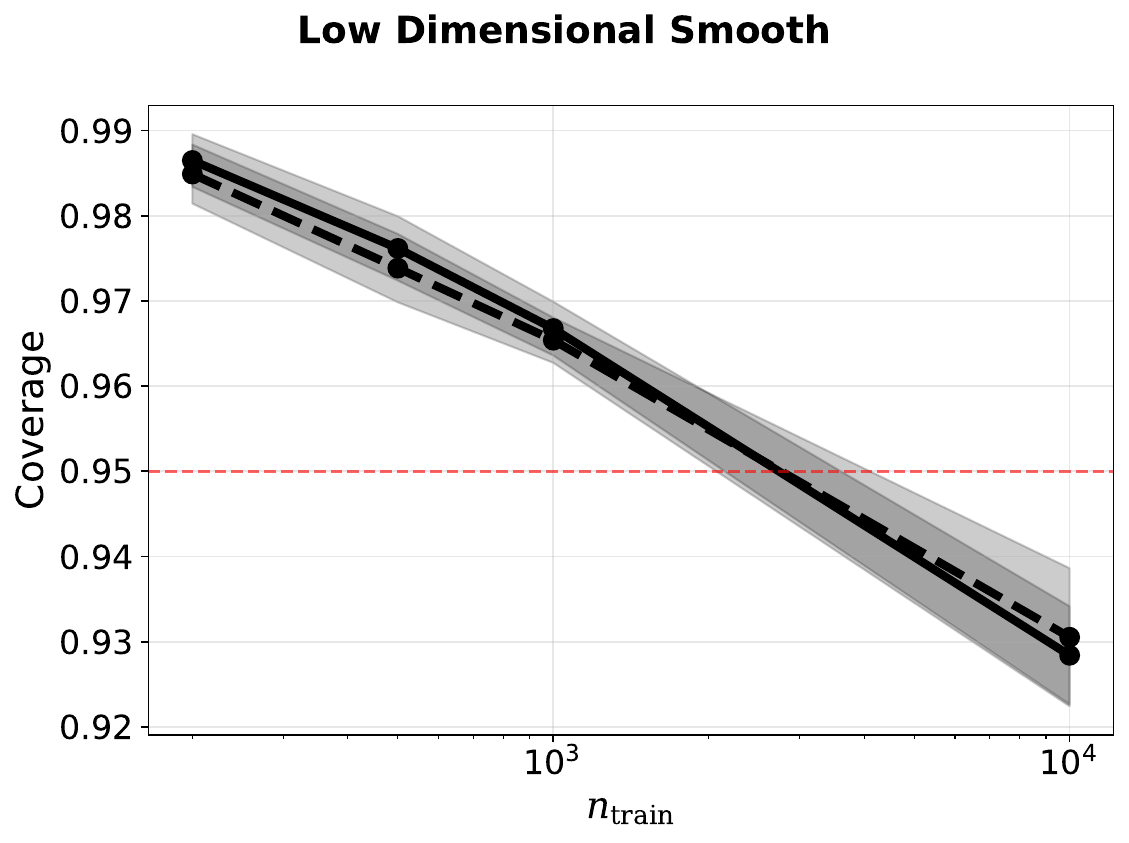}%
        \hfill
        \includegraphics[width=0.32\linewidth]{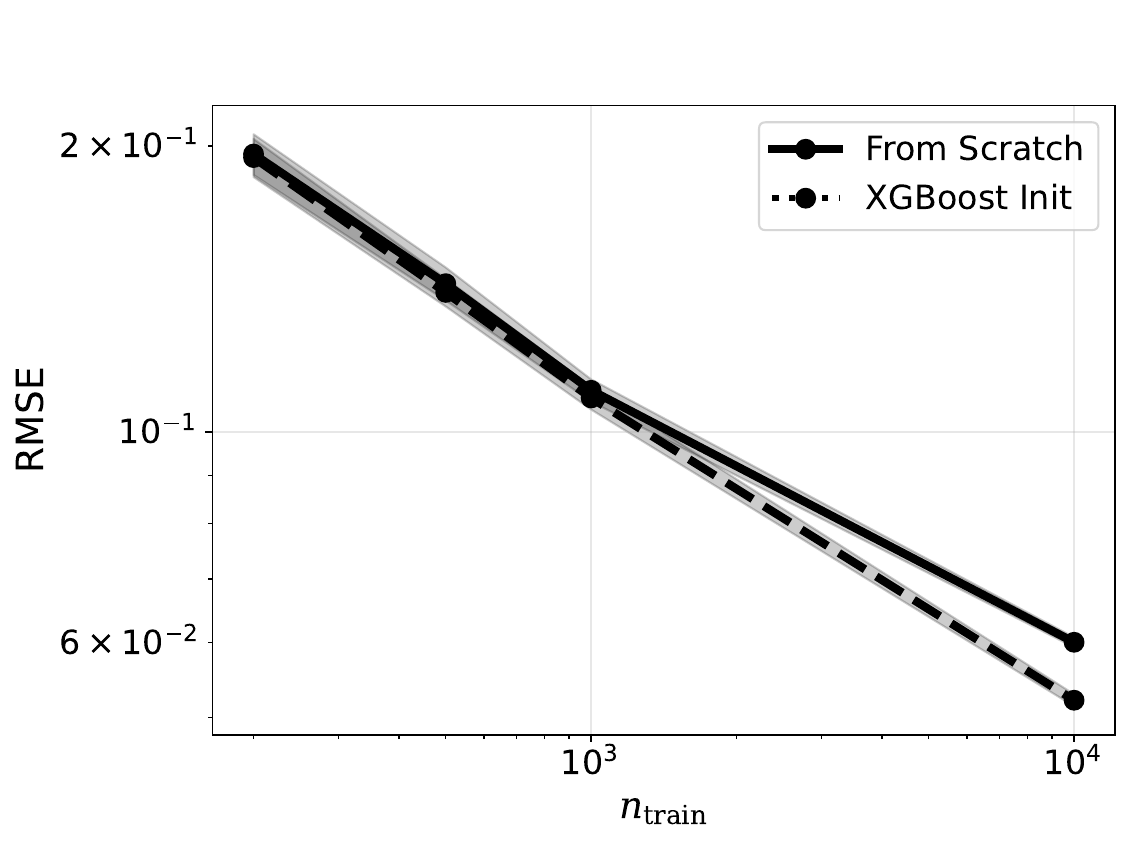}
    \end{minipage} \\
    \begin{minipage}{0.9\textwidth}
        \centering
        \includegraphics[width=0.32\linewidth]{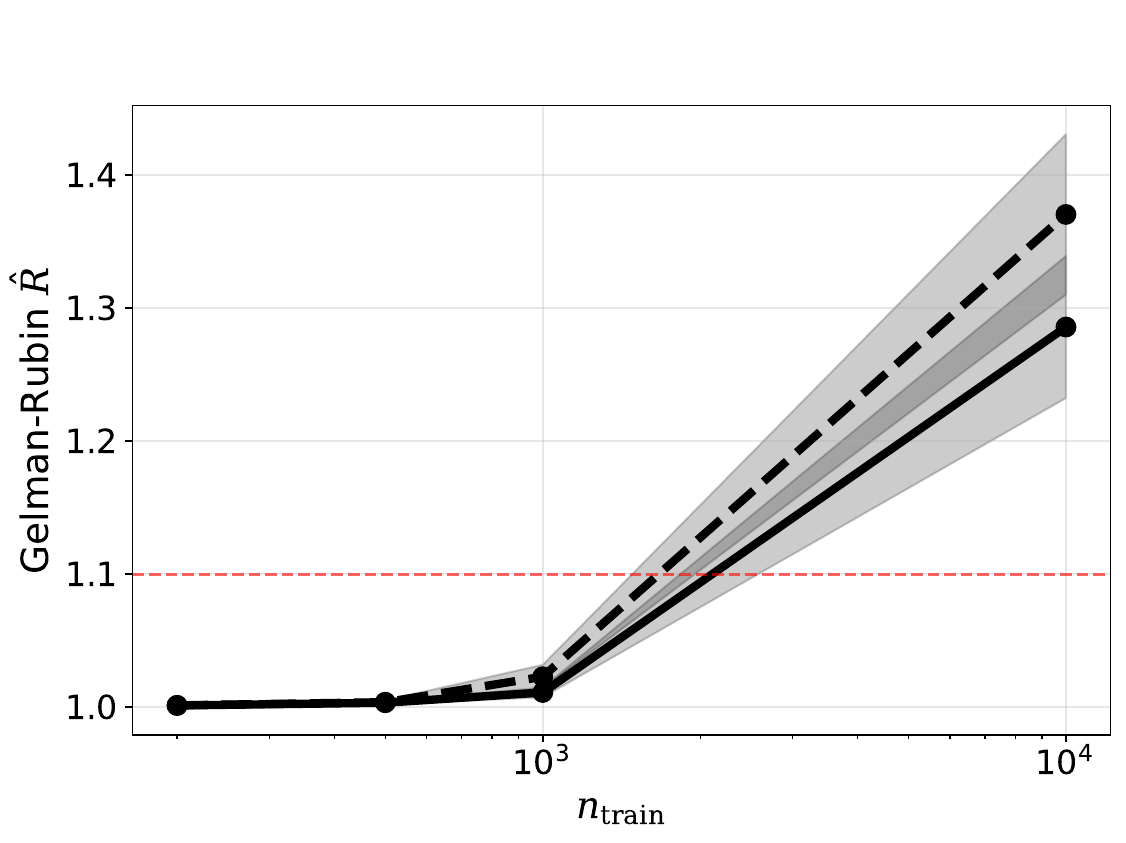}%
        \hfill
        \includegraphics[width=0.32\linewidth]{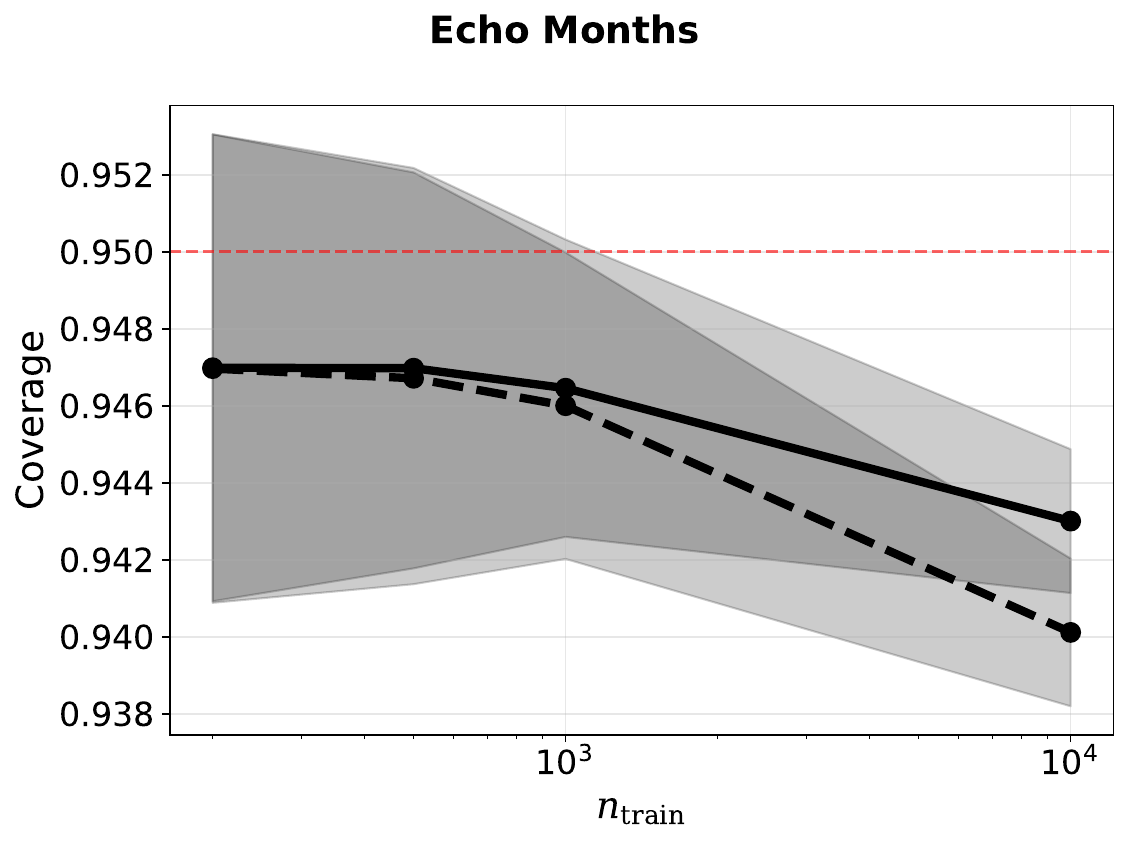}%
        \hfill
        \includegraphics[width=0.32\linewidth]{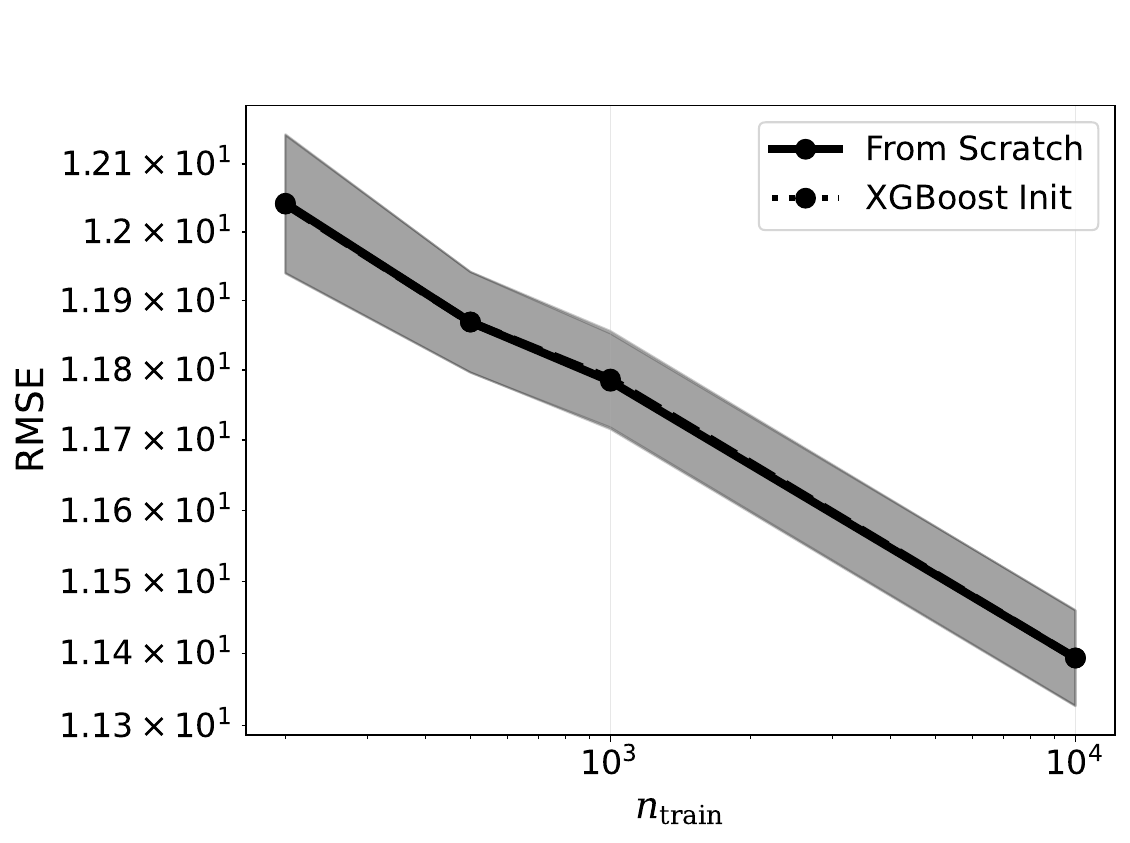}
    \end{minipage} \\
    \begin{minipage}{0.9\textwidth}
        \centering
        \includegraphics[width=0.32\linewidth]{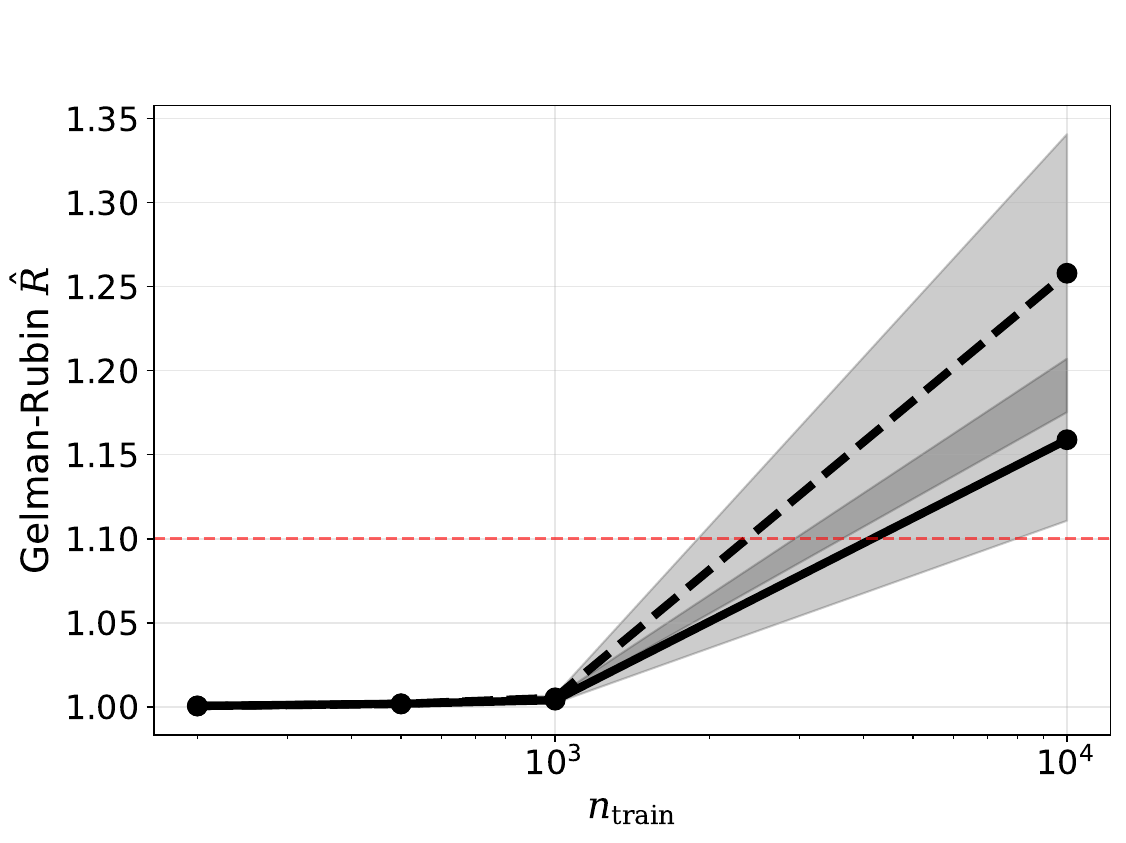}
        \hfill
        \includegraphics[width=0.32\linewidth]{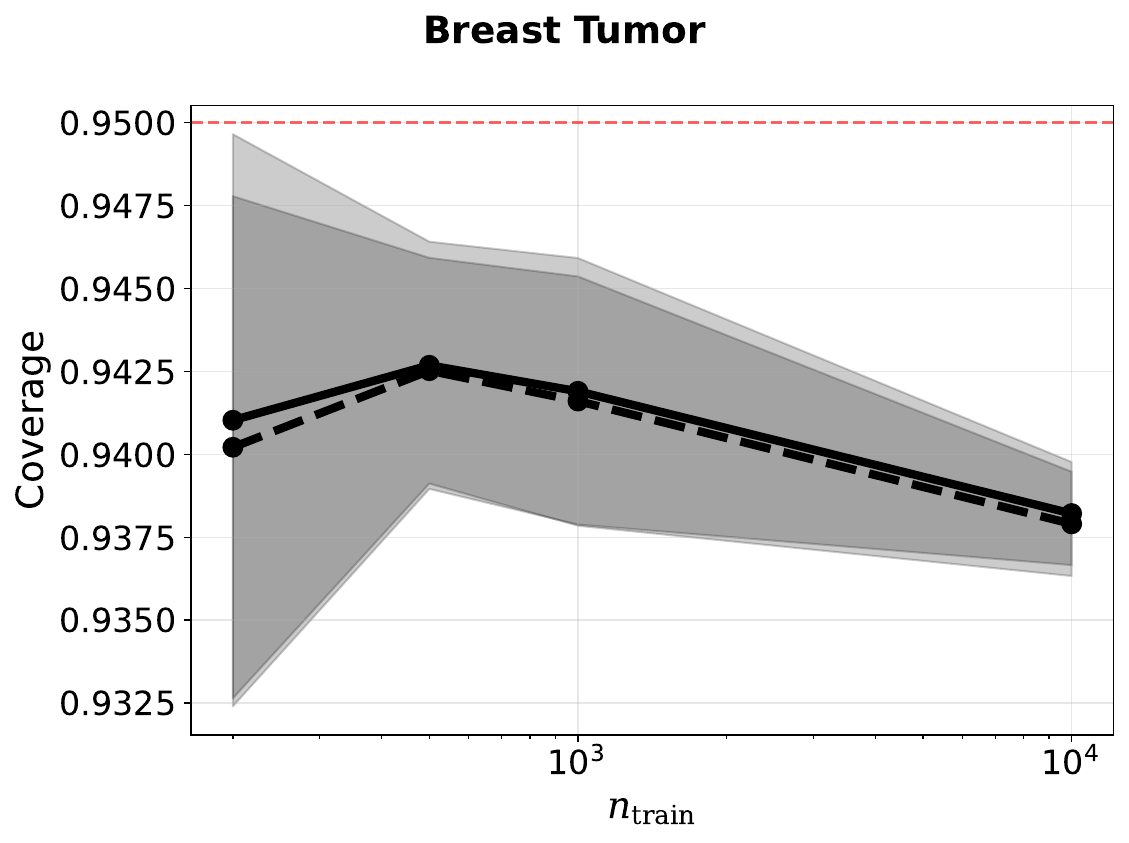}
        \hfill
        \includegraphics[width=0.32\linewidth]{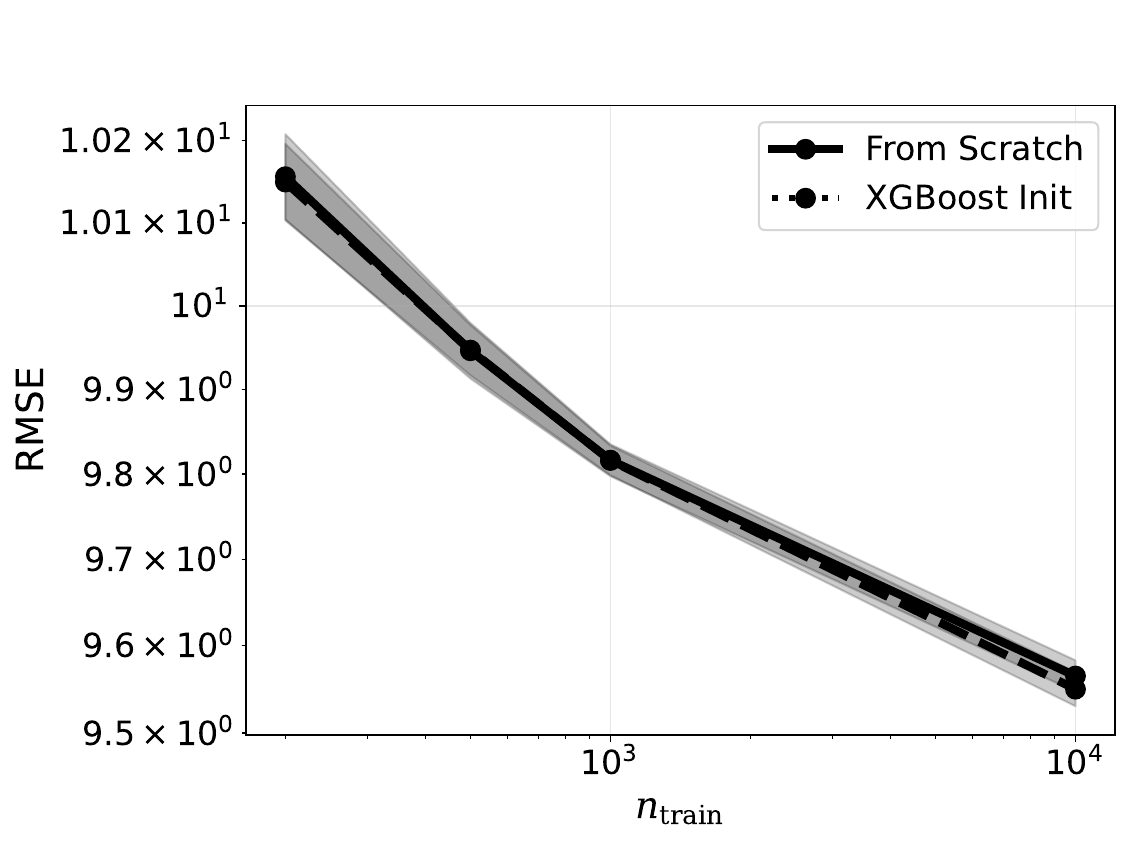}%
    \end{minipage} \\
    \begin{minipage}{0.9\textwidth}
        \centering
        \includegraphics[width=0.32\linewidth]{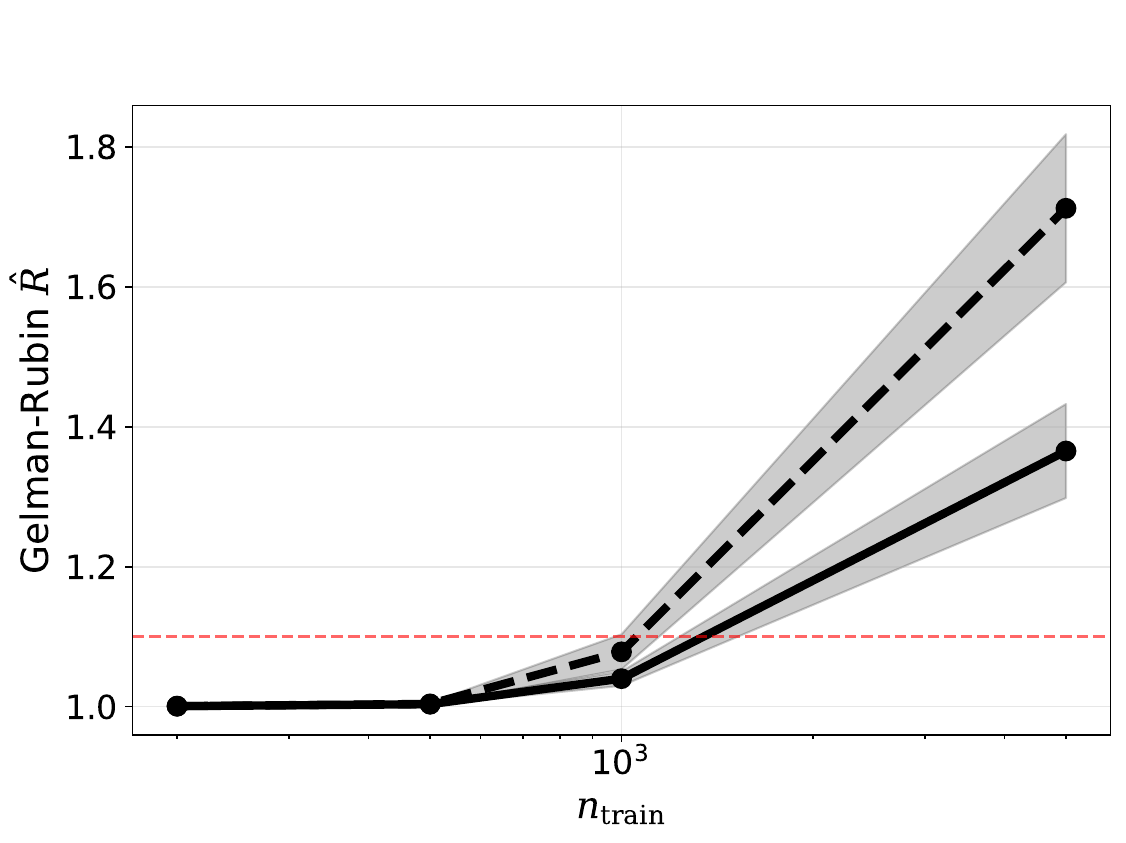}%
        \hfill
        \includegraphics[width=0.32\linewidth]{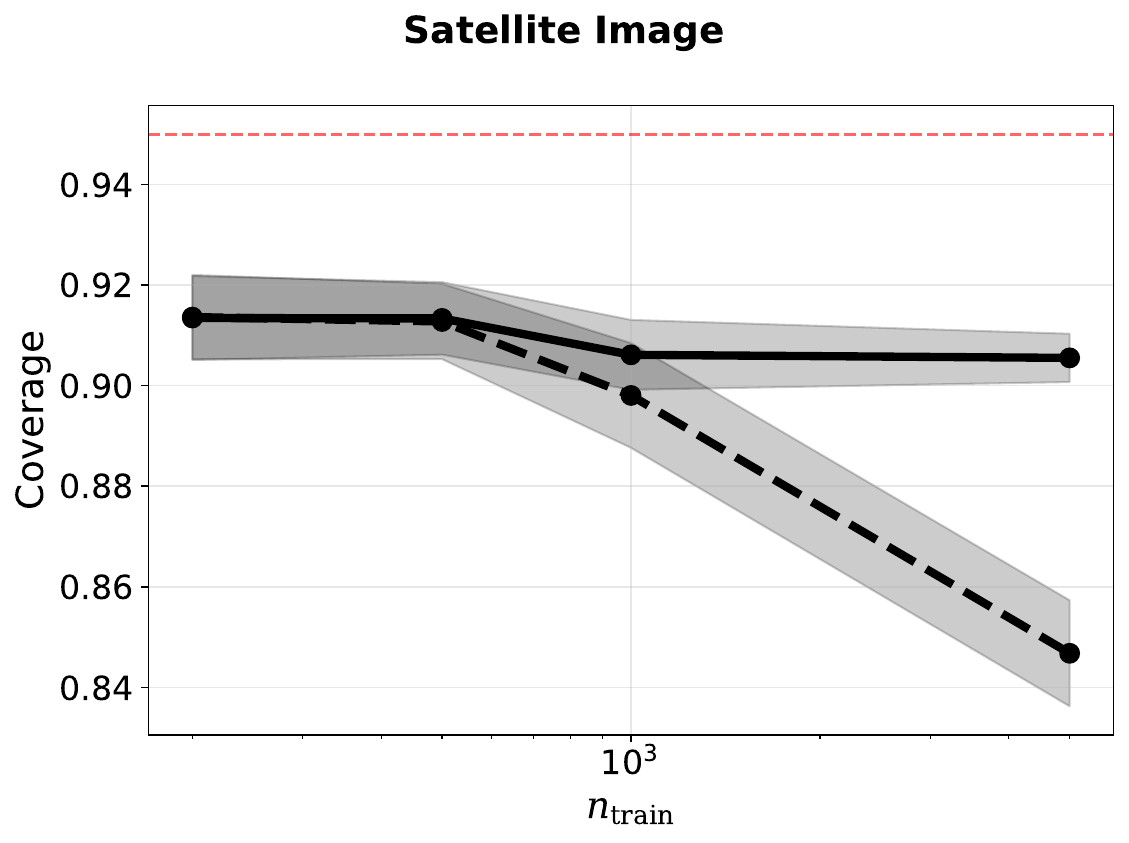}
        \hfill
        \includegraphics[width=0.32\linewidth]{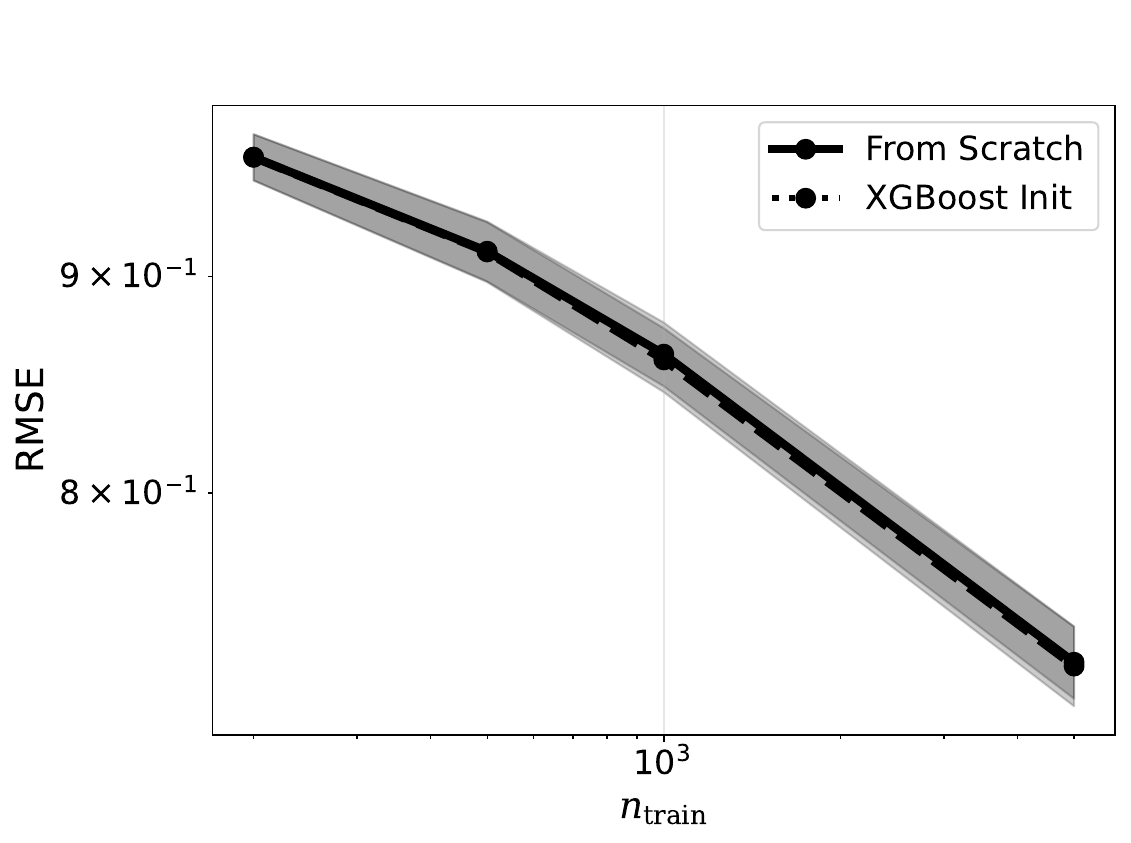}%
    \end{minipage} \\
    \begin{minipage}{0.9\textwidth}
        \centering
        \includegraphics[width=0.32\linewidth]{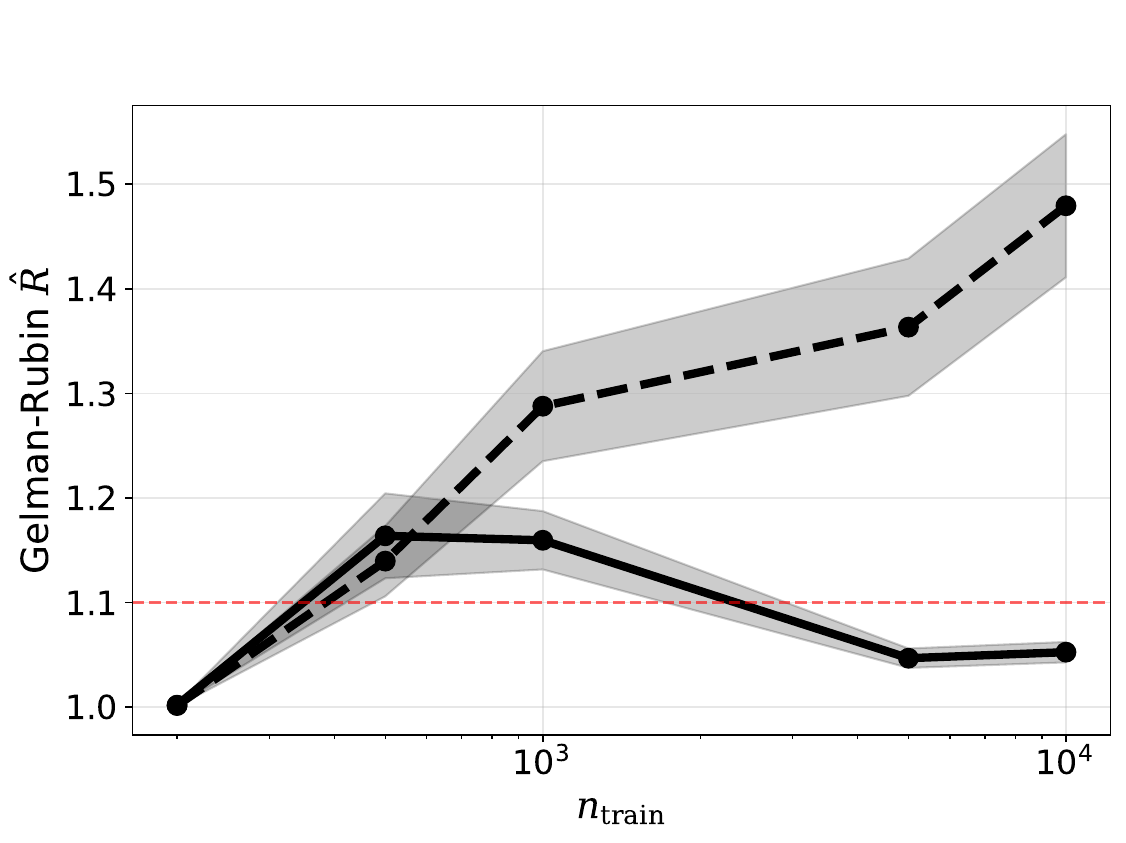}%
        \hfill
        \includegraphics[width=0.32\linewidth]{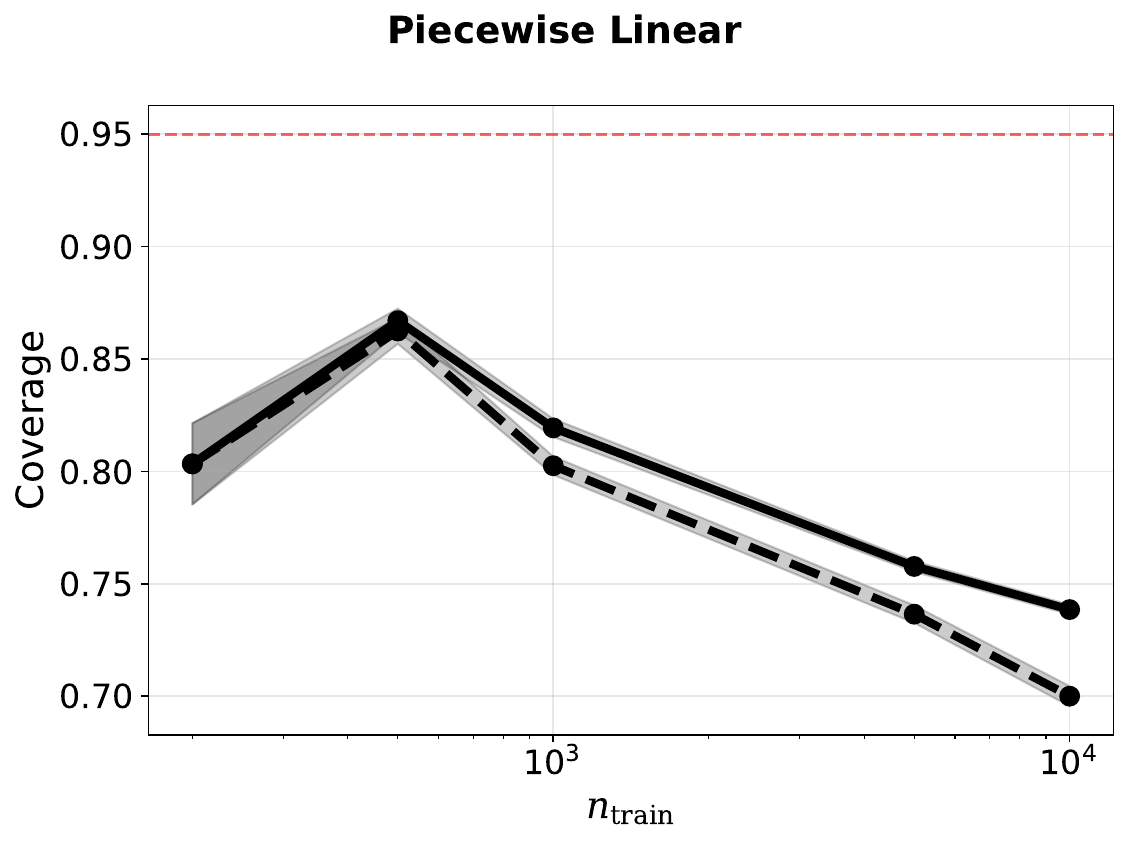}
        \hfill
        \includegraphics[width=0.32\linewidth]{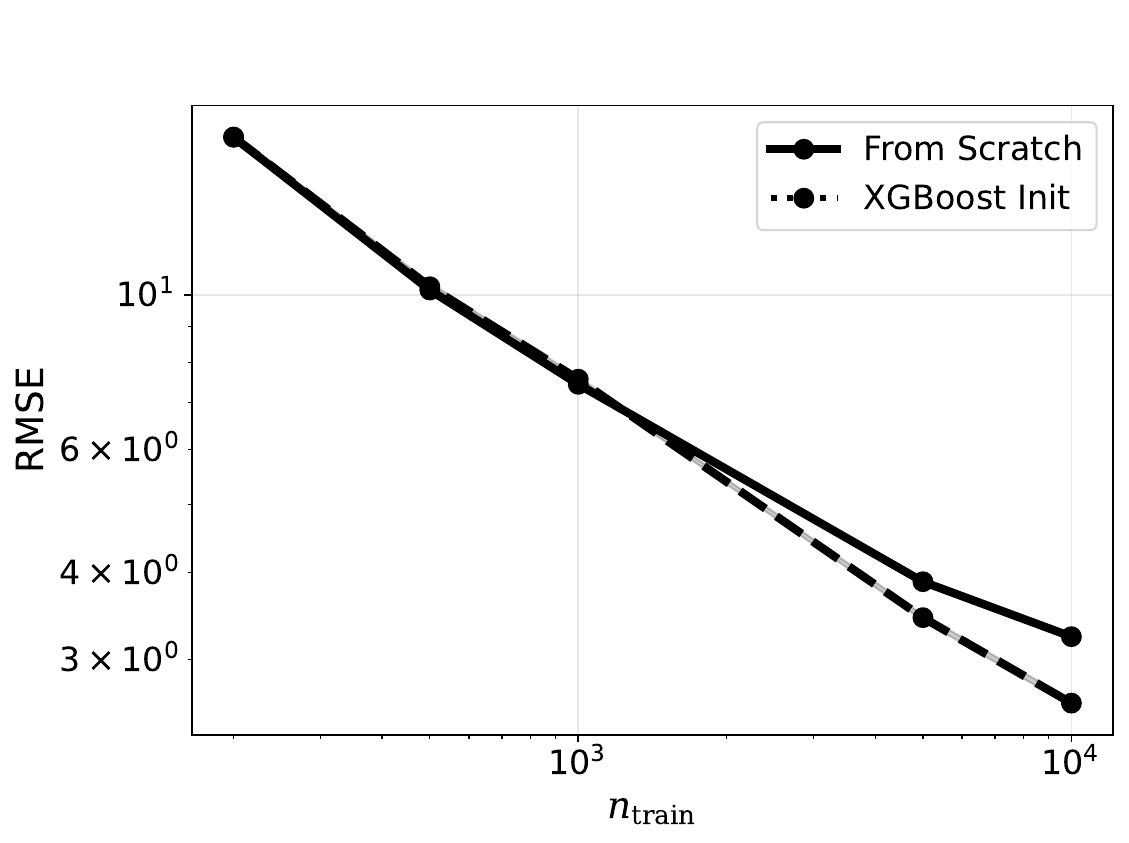}%
    \end{minipage} \\
    \end{tabular}
    \caption{
    Values for Gelman-Rubin $\hat R$ (left), coverage (center), and RMSE (right) for the BART sampler under an XGBoost initialization and the standard BART initialization.
    Results are plotted for California Housing, Low Dimensional Smooth, Echo Months, Breast Tumor, Satellite Image, and Piecewise Linear datasets (from top to bottom).
    Error bars represent $\pm 1.96$ standard errors from 25 replicates.
    }
    \label{fig:experiment3}
\end{figure}

\newpage

\begin{figure}[H]
    \centering
    \begin{tabular}{c}
    \begin{minipage}{0.9\textwidth}
        \centering
        \includegraphics[width=0.32\linewidth]{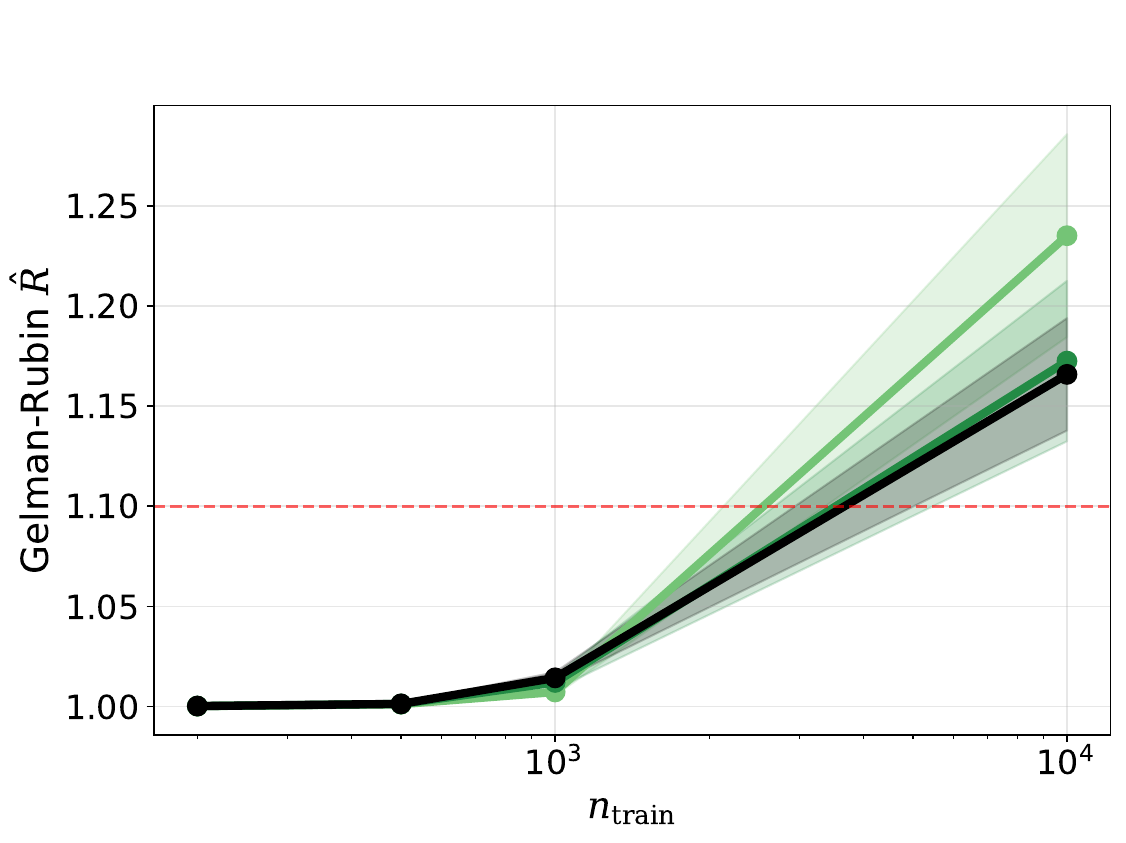}%
        \hfill
        \includegraphics[width=0.32\linewidth]{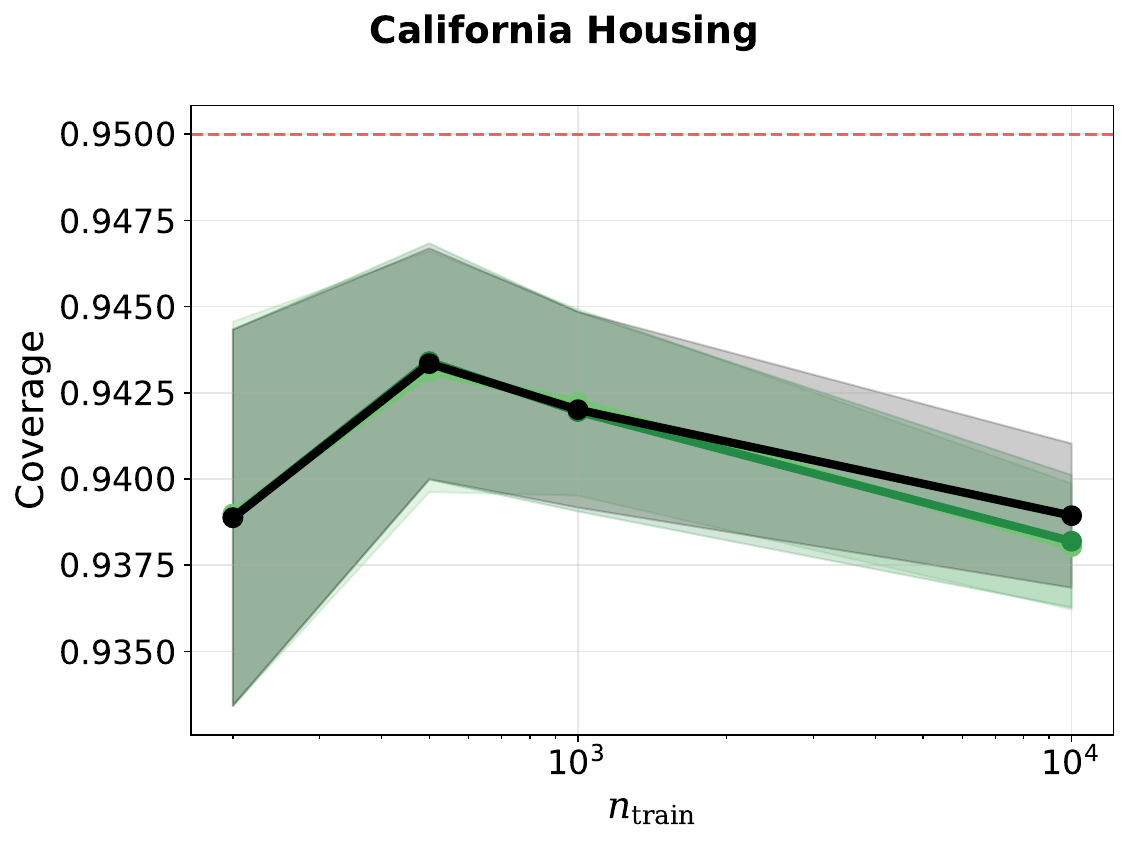}%
        \hfill
        \includegraphics[width=0.32\linewidth]{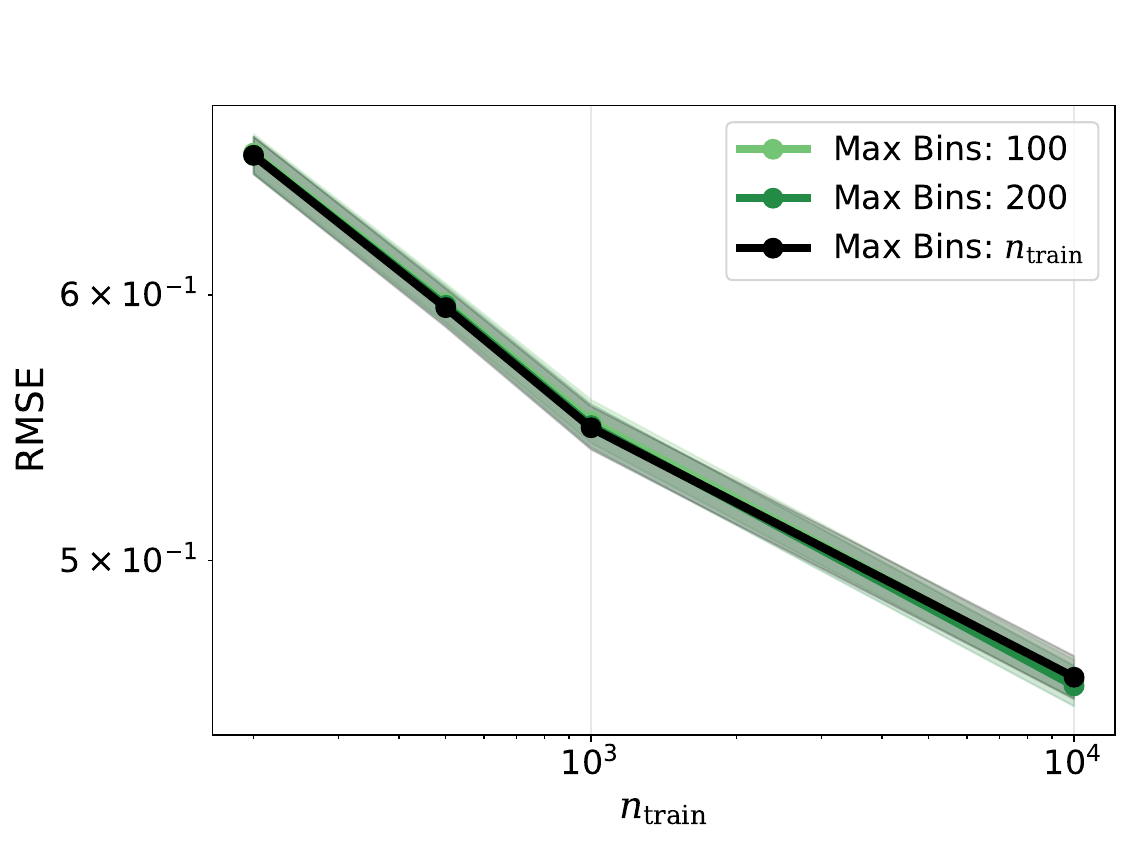}
    \end{minipage} \\
    \begin{minipage}{0.9\textwidth}
        \centering
        \includegraphics[width=0.32\linewidth]{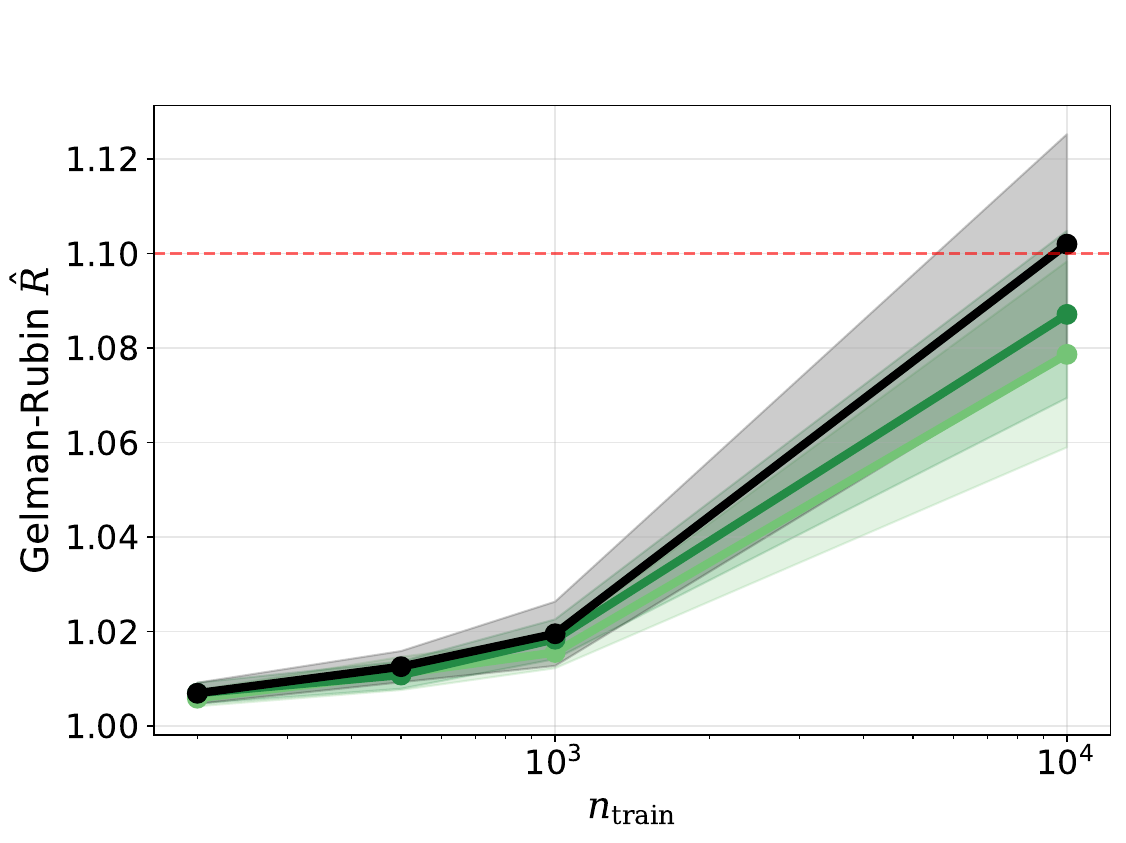}%
        \hfill
        \includegraphics[width=0.32\linewidth]{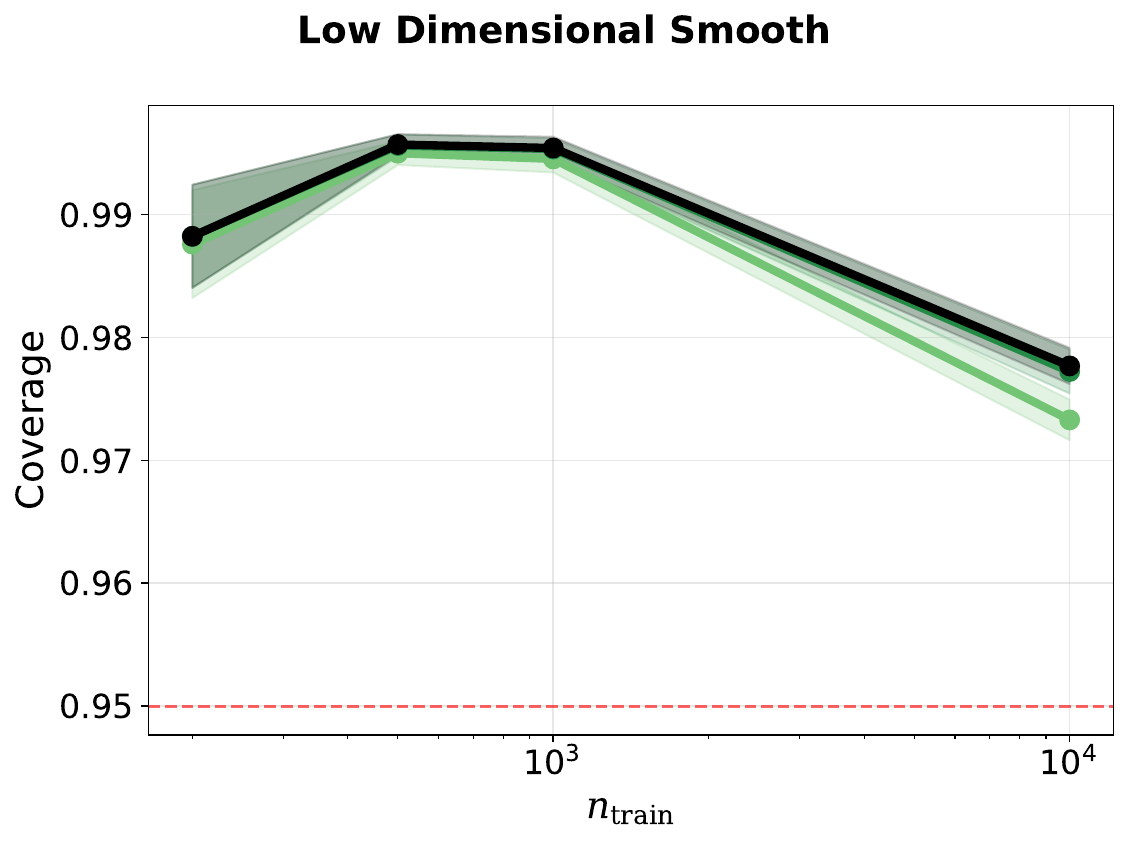}%
        \hfill
        \includegraphics[width=0.32\linewidth]{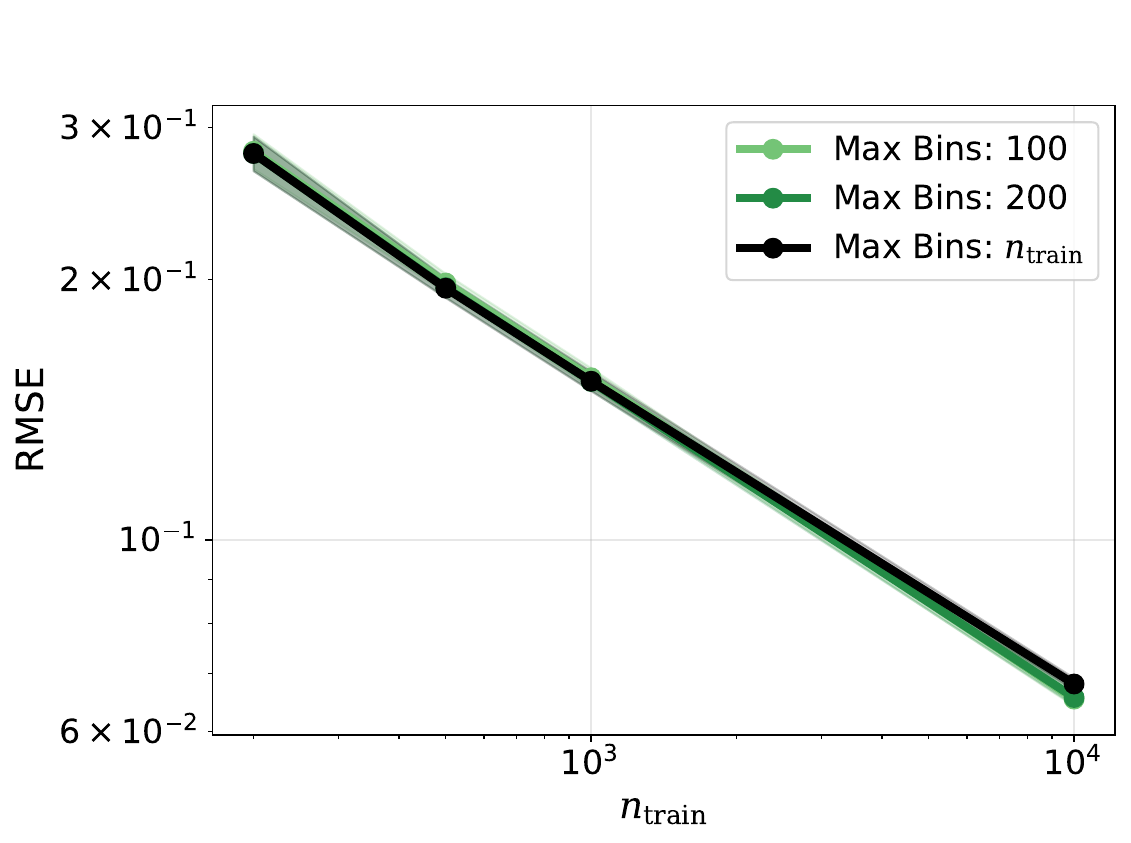}
    \end{minipage} \\
    \begin{minipage}{0.9\textwidth}
        \centering
        \includegraphics[width=0.32\linewidth]{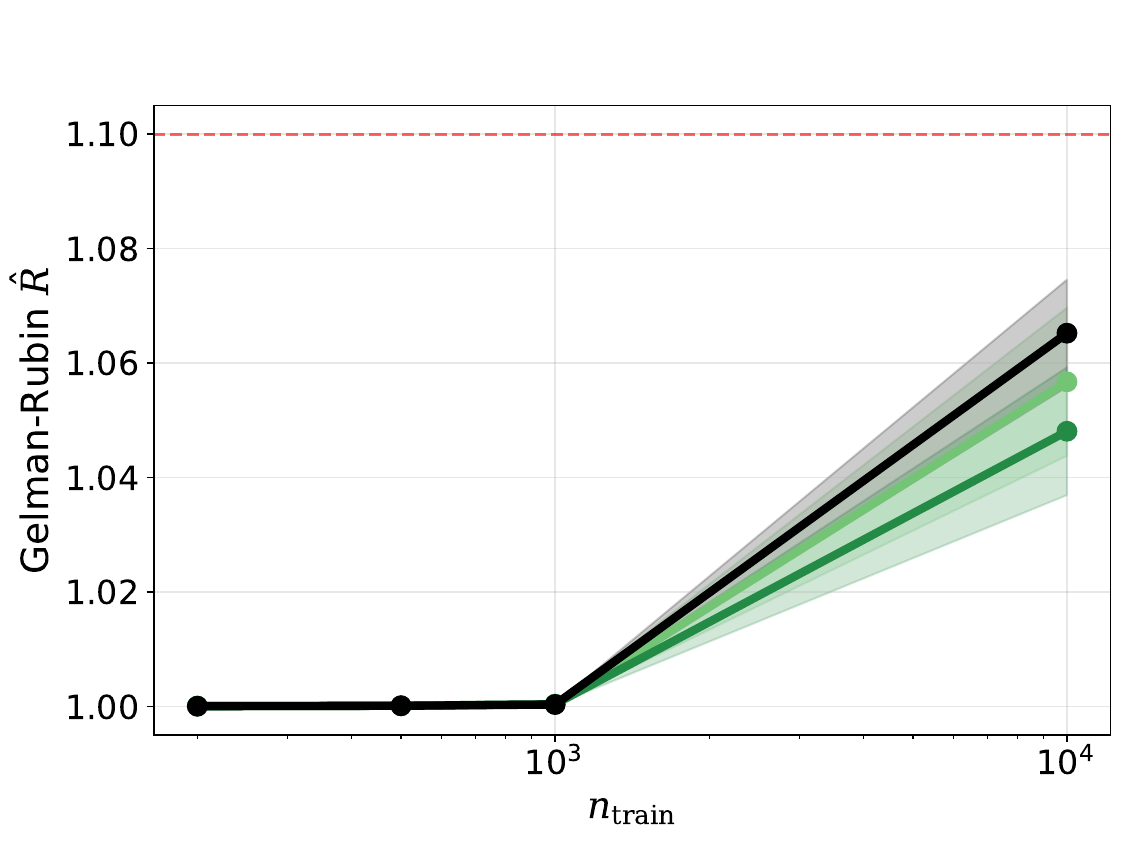}%
        \hfill
        \includegraphics[width=0.32\linewidth]{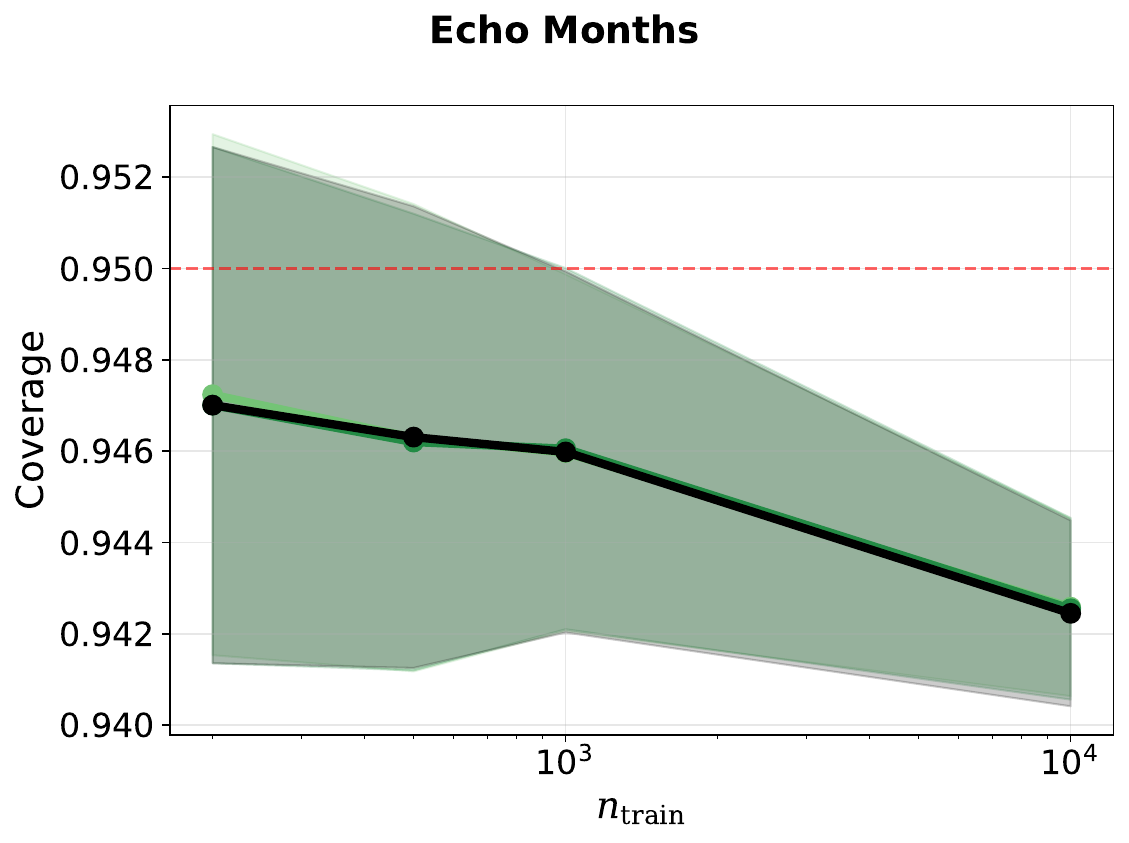}%
        \hfill
        \includegraphics[width=0.32\linewidth]{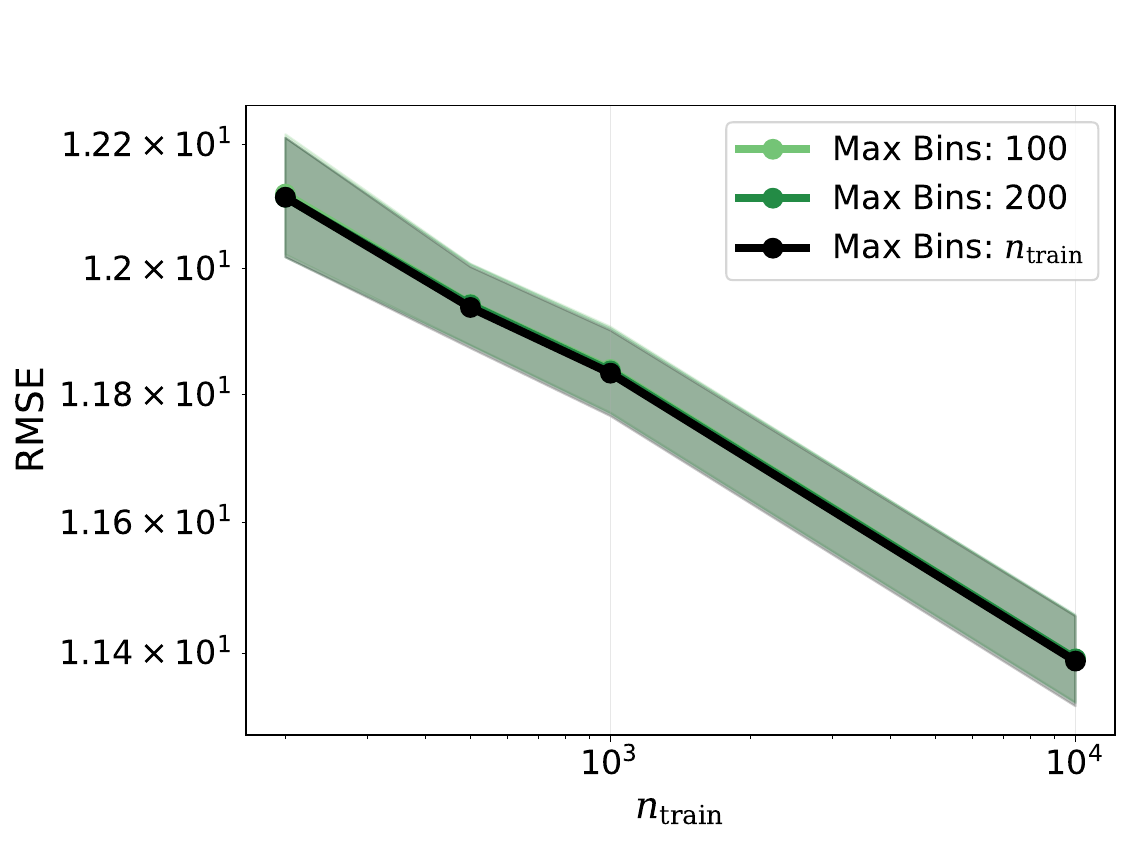}
    \end{minipage} \\
      \begin{minipage}{0.9\textwidth}
        \centering
        \includegraphics[width=0.32\linewidth]{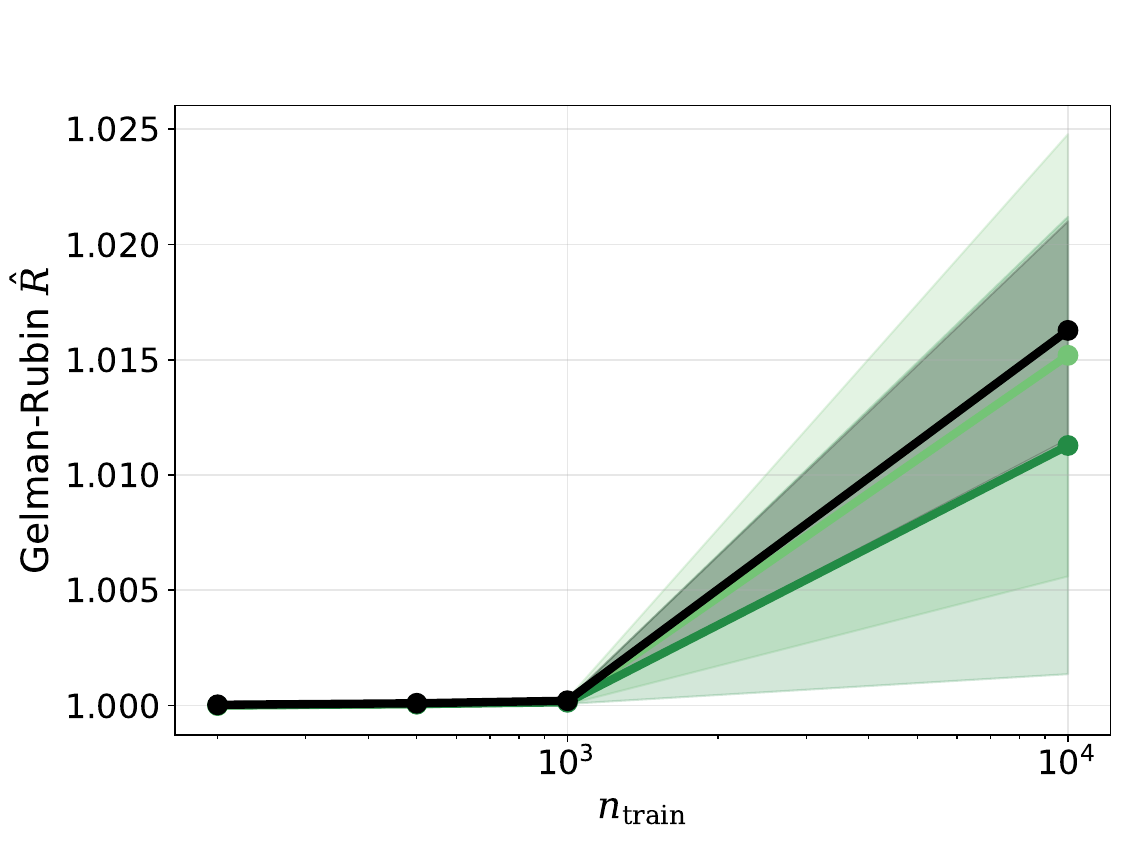}%
        \hfill
        \includegraphics[width=0.32\linewidth]{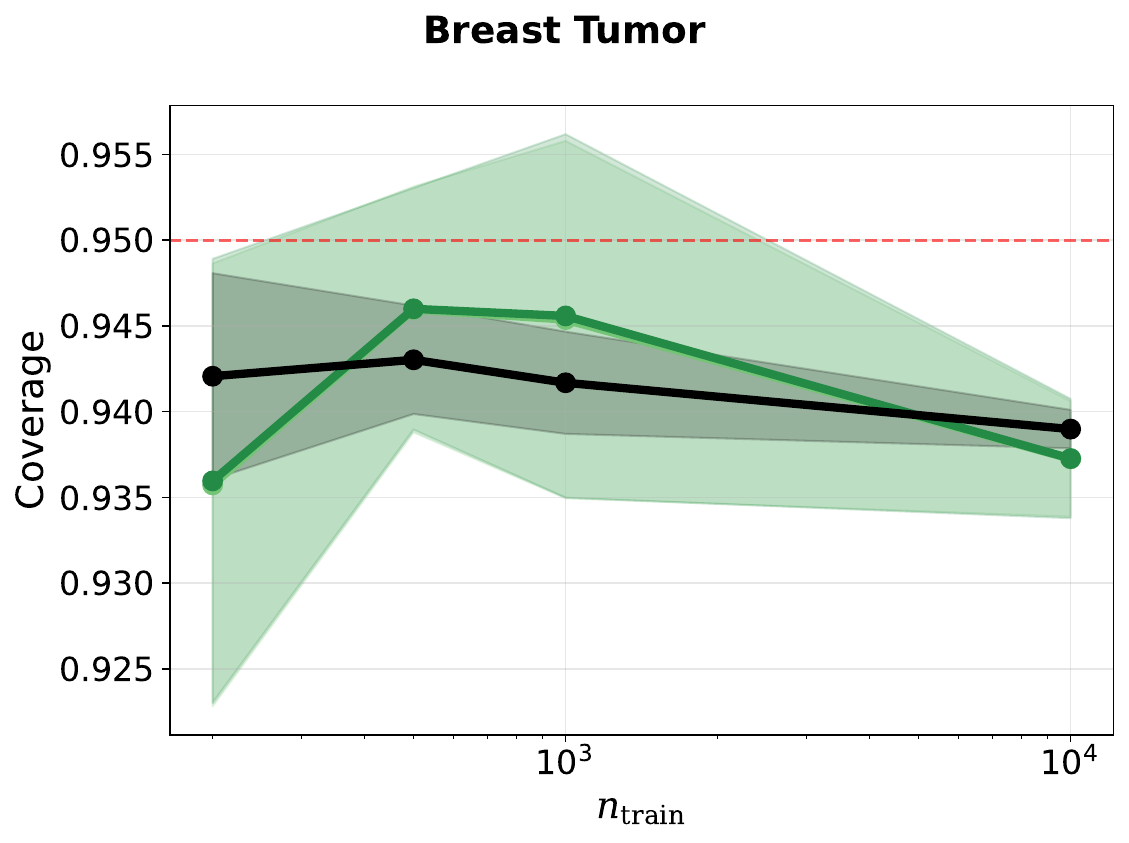}%
        \hfill
        \includegraphics[width=0.32\linewidth]{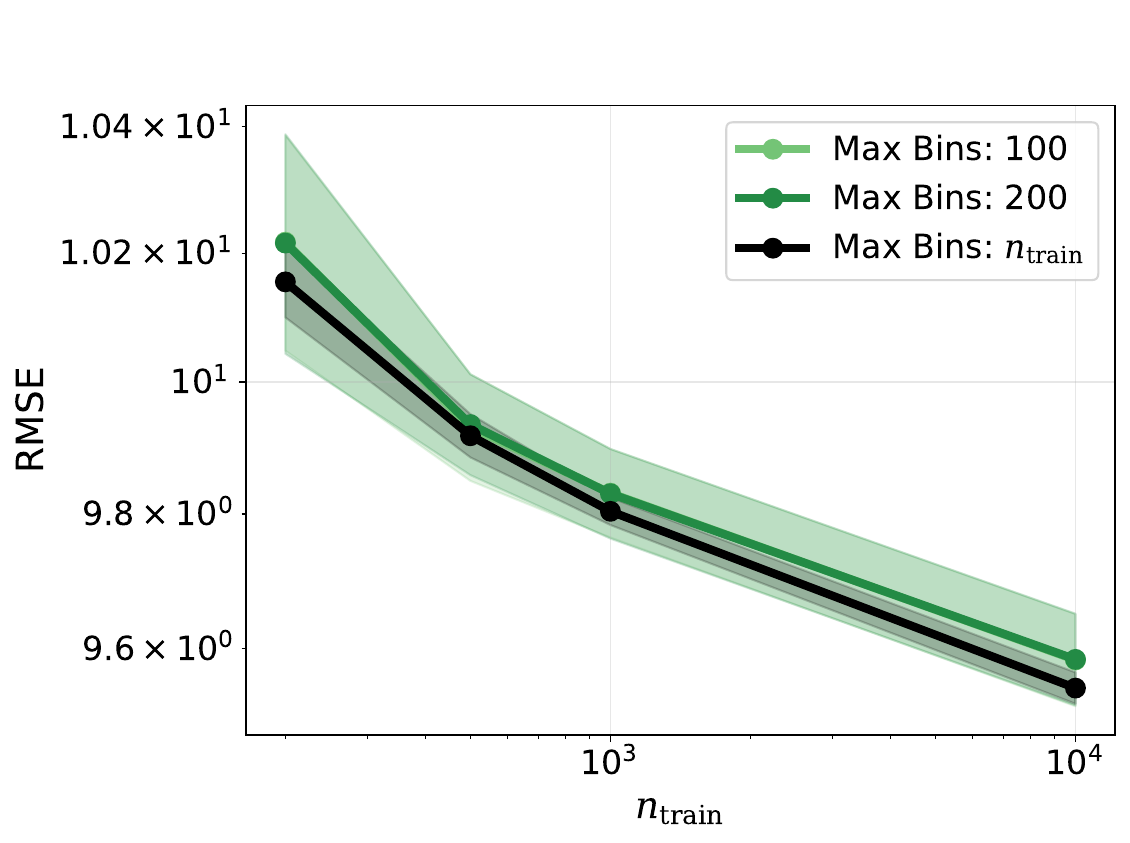}
    \end{minipage} \\
    \begin{minipage}{0.9\textwidth}
        \centering
        \includegraphics[width=0.32\linewidth]{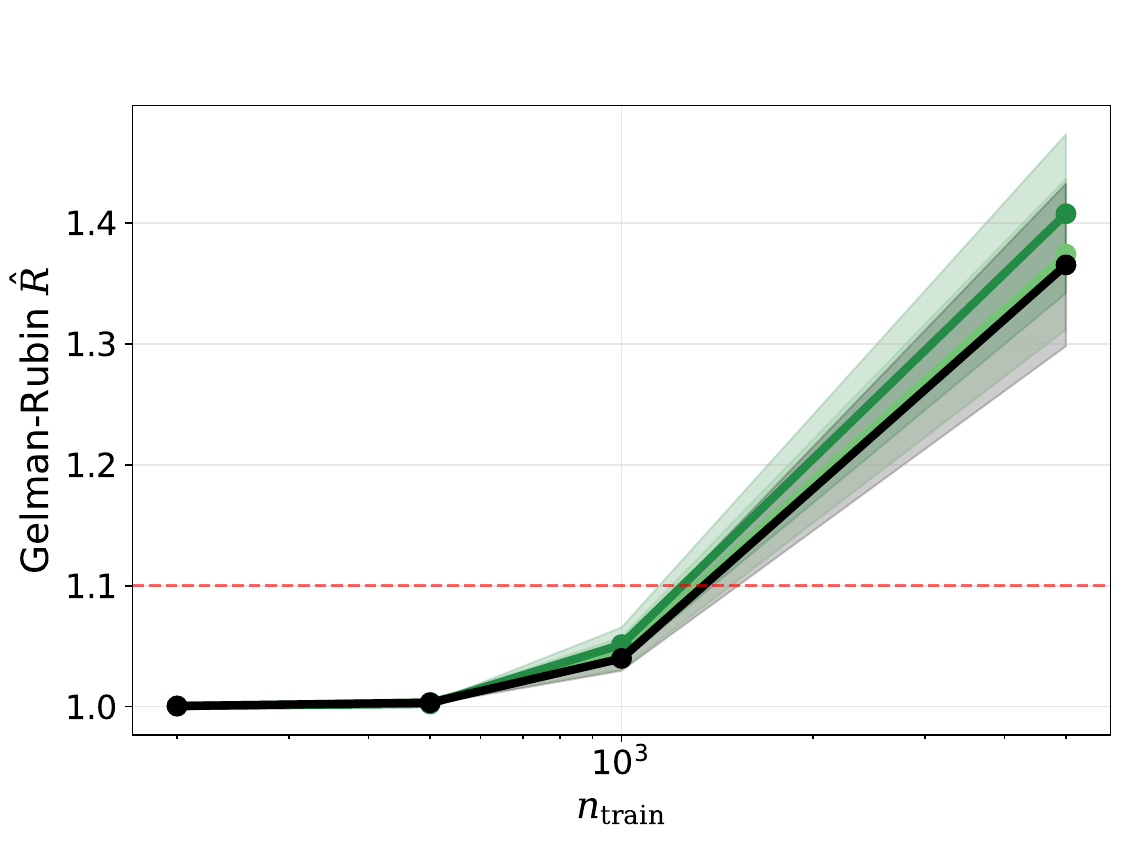}%
        \hfill
        \includegraphics[width=0.32\linewidth]{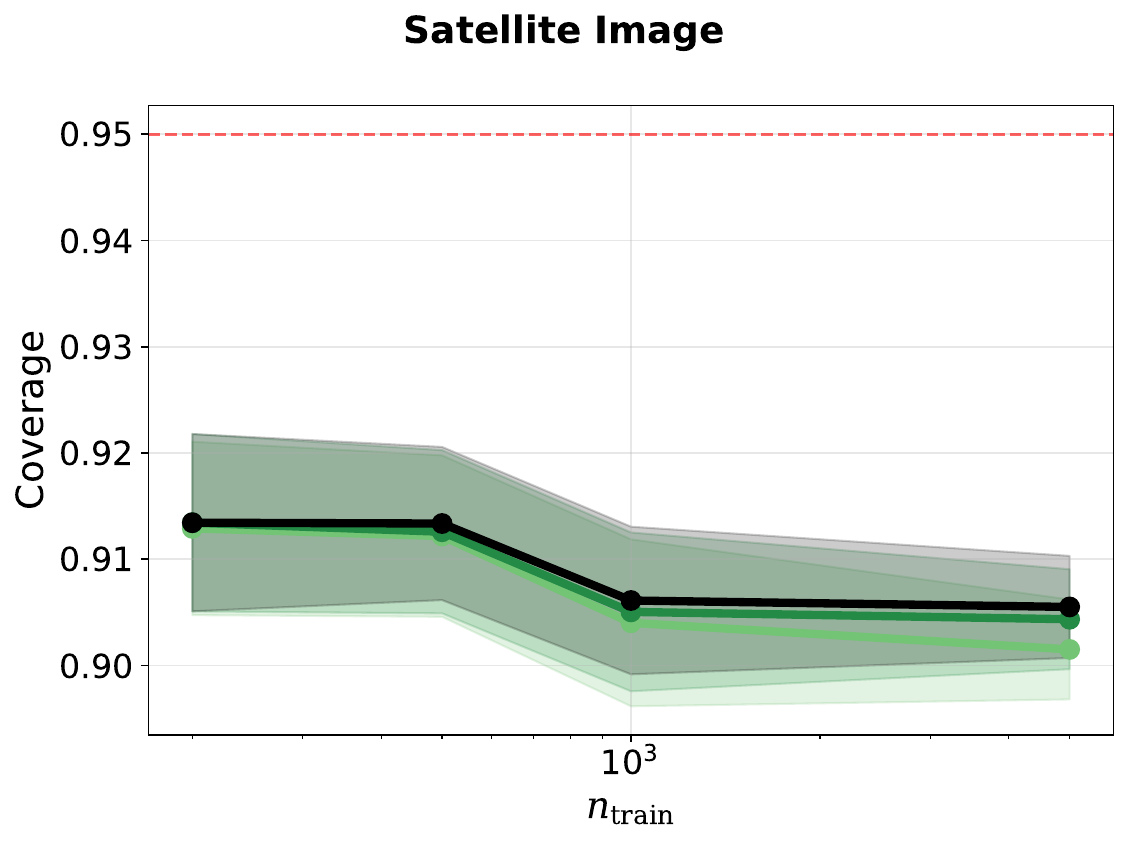}%
        \hfill
        \includegraphics[width=0.32\linewidth]{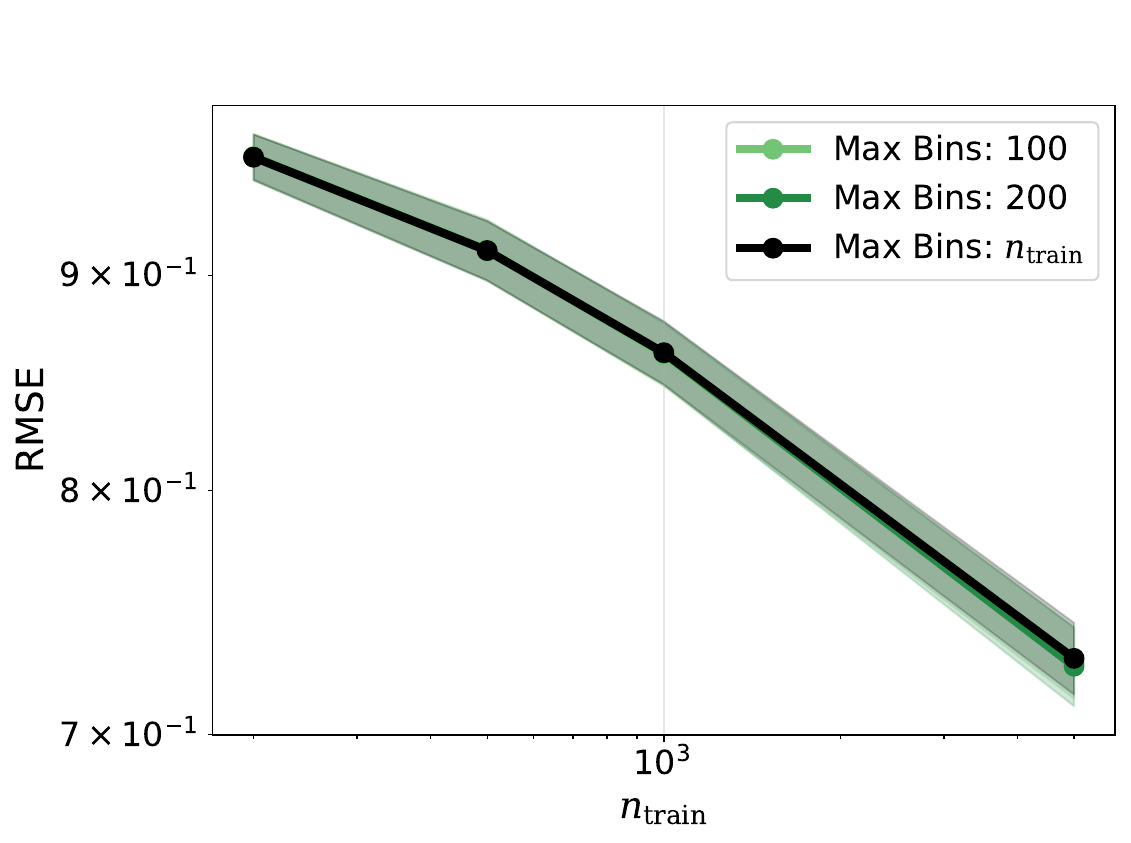}
    \end{minipage} \\
    \begin{minipage}{0.9\textwidth}
        \centering
        \includegraphics[width=0.32\linewidth]{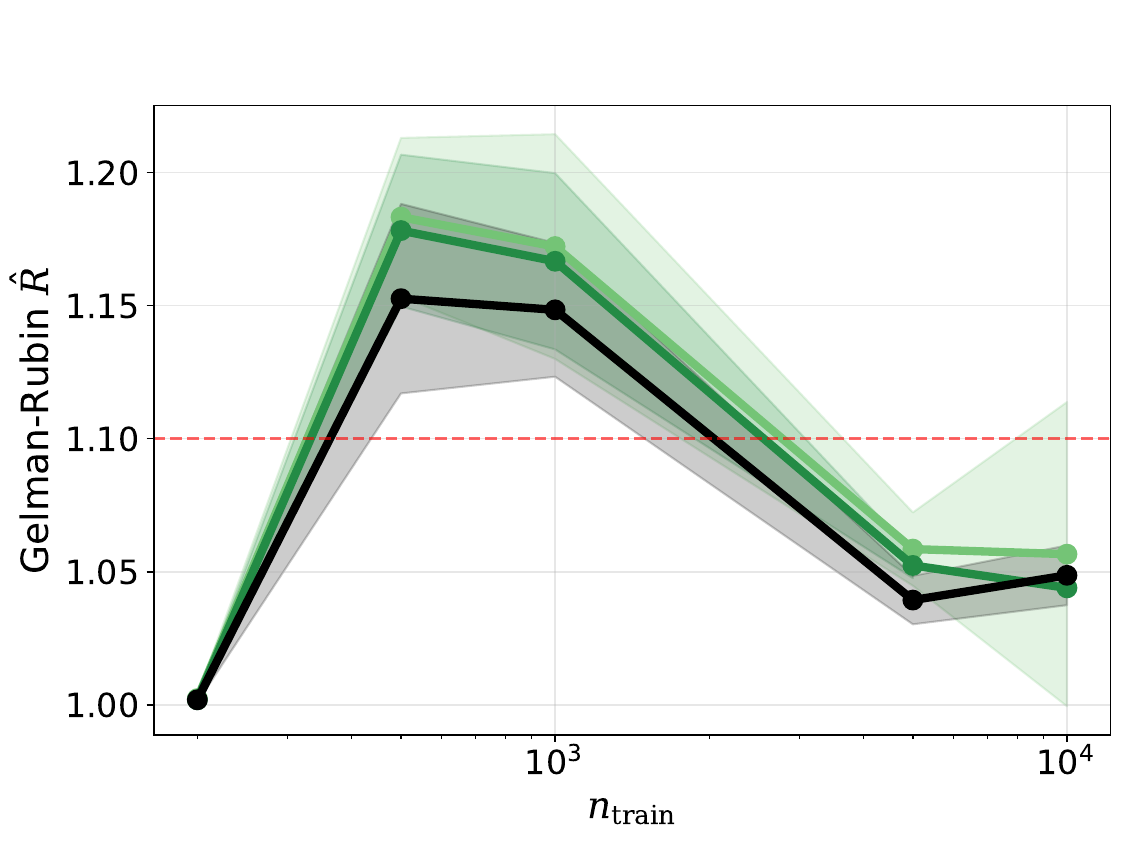}%
        \hfill
        \includegraphics[width=0.32\linewidth]{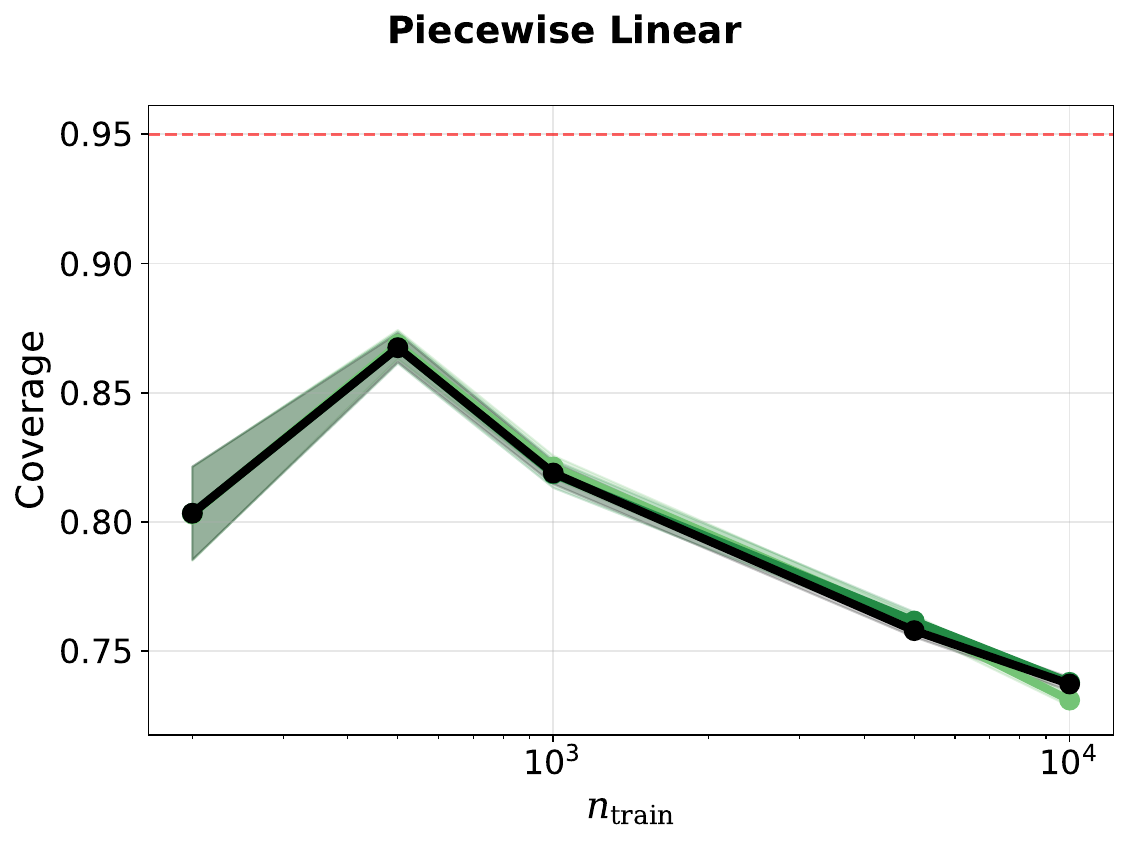}%
        \hfill
        \includegraphics[width=0.32\linewidth]{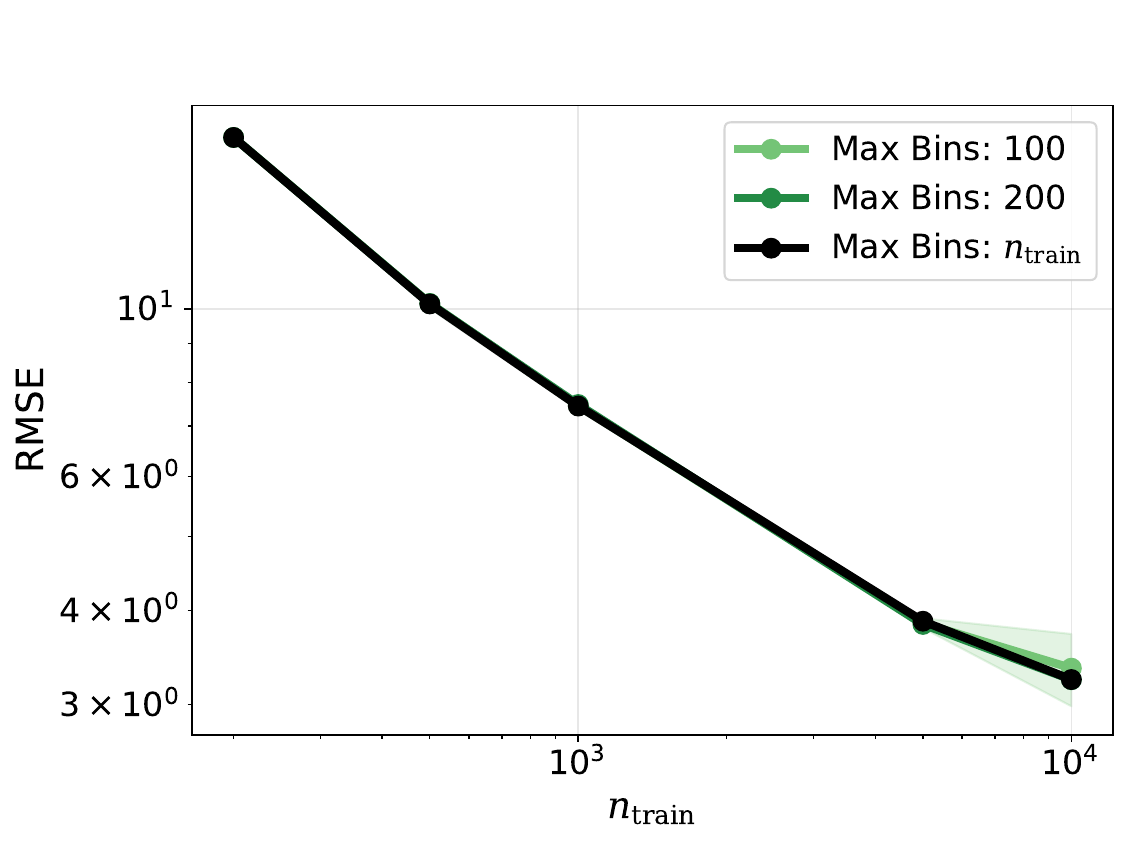}
    \end{minipage} \\
    \end{tabular}
    \caption{
    Values for Gelman-Rubin $\hat R$ (left), coverage (center), and RMSE (right) for the BART sampler under different number of candidate splitting thresholds.
    Results are plotted for California Housing, Low Dimensional Smooth, Echo Months, Breast Tumor, Satellite Image, and Piecewise Linear datasets (from top to bottom).
    Error bars represent $\pm 1.96$ standard errors from 25 replicates.
    }
    \label{fig:experiment_thresholds}
\end{figure}

\newpage

\begin{figure}[H]
    \centering
    \begin{tabular}{c}
    \begin{minipage}{0.9\textwidth}
        \centering
        \includegraphics[width=0.32\linewidth]{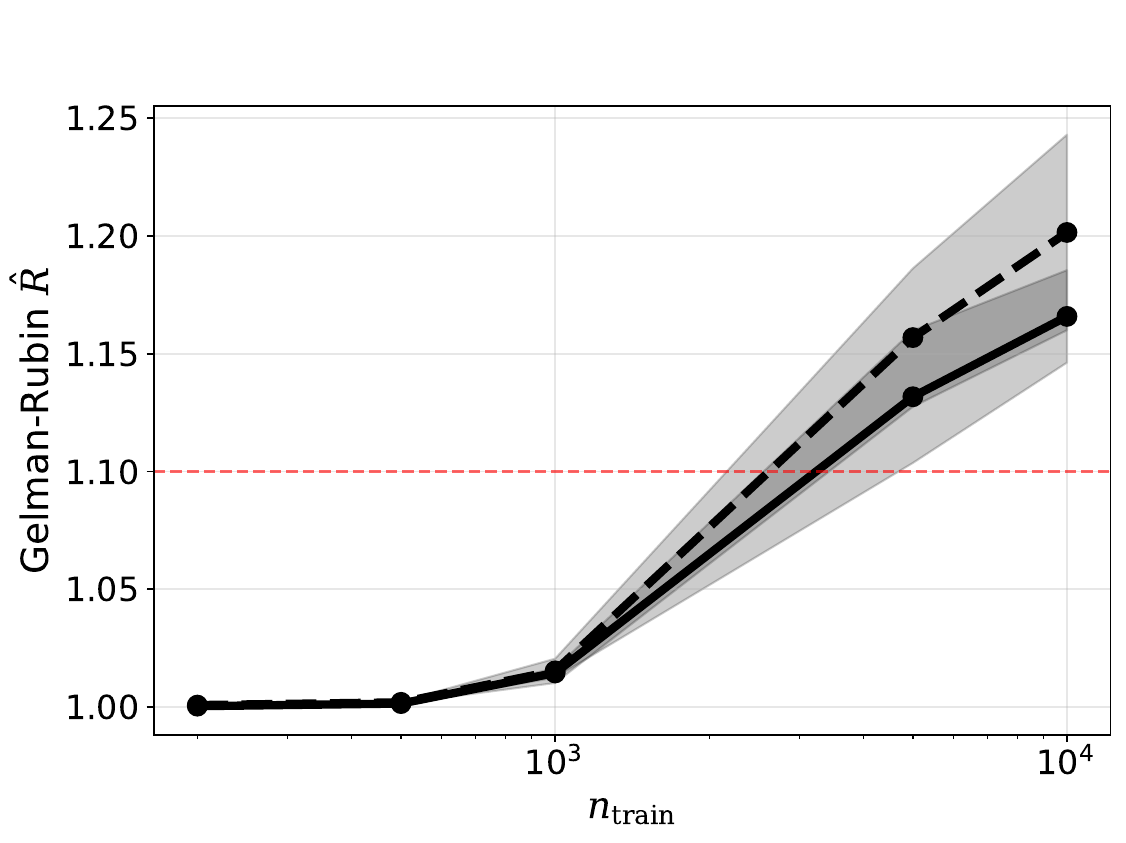}%
        \hfill
        \includegraphics[width=0.32\linewidth]{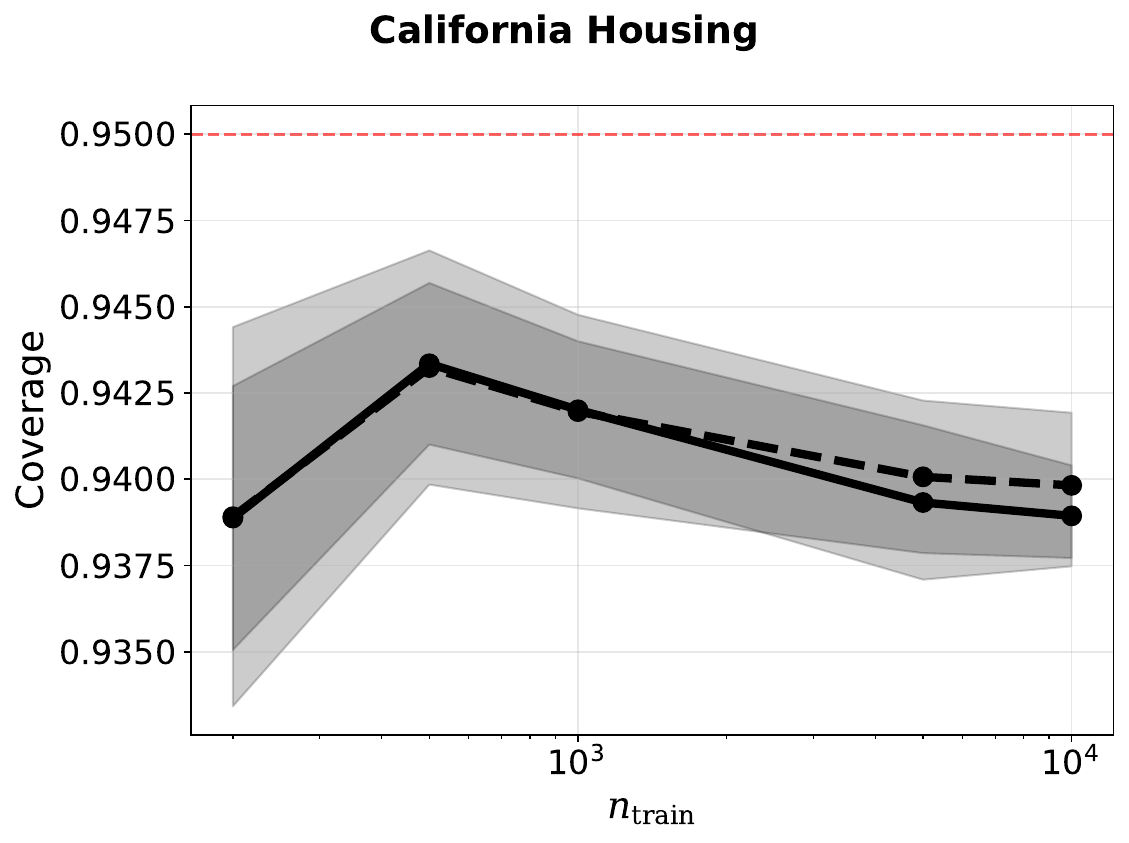}%
        \hfill
        \includegraphics[width=0.32\linewidth]{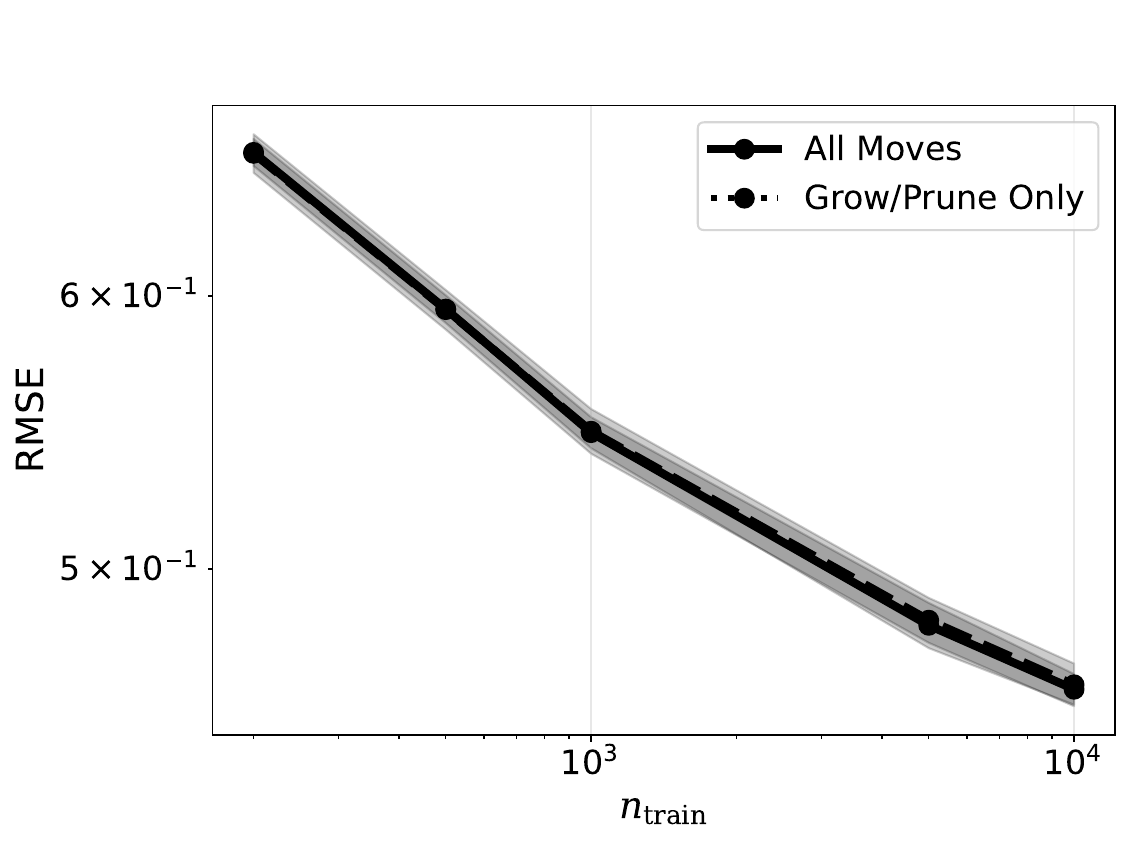}
    \end{minipage} \\
    \begin{minipage}{0.9\textwidth}
        \centering
        \includegraphics[width=0.32\linewidth]{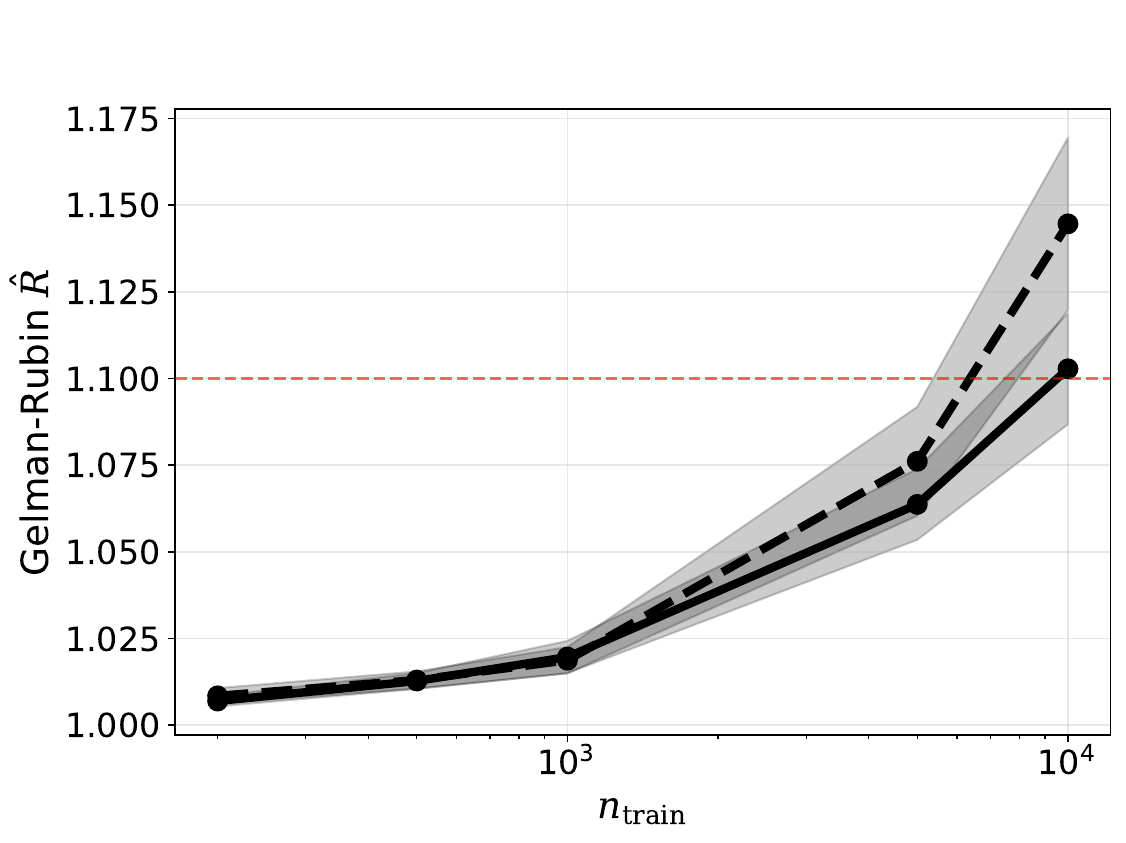}%
        \hfill
        \includegraphics[width=0.32\linewidth]{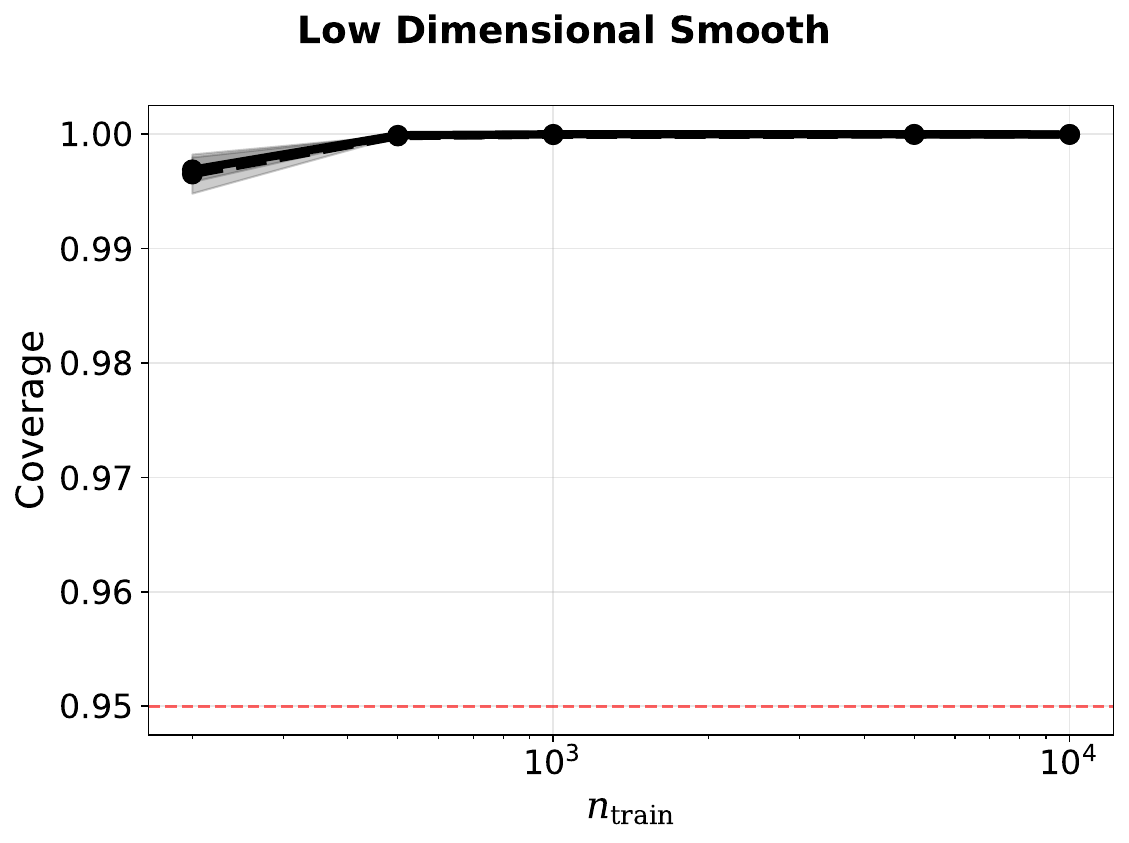}%
        \hfill
        \includegraphics[width=0.32\linewidth]{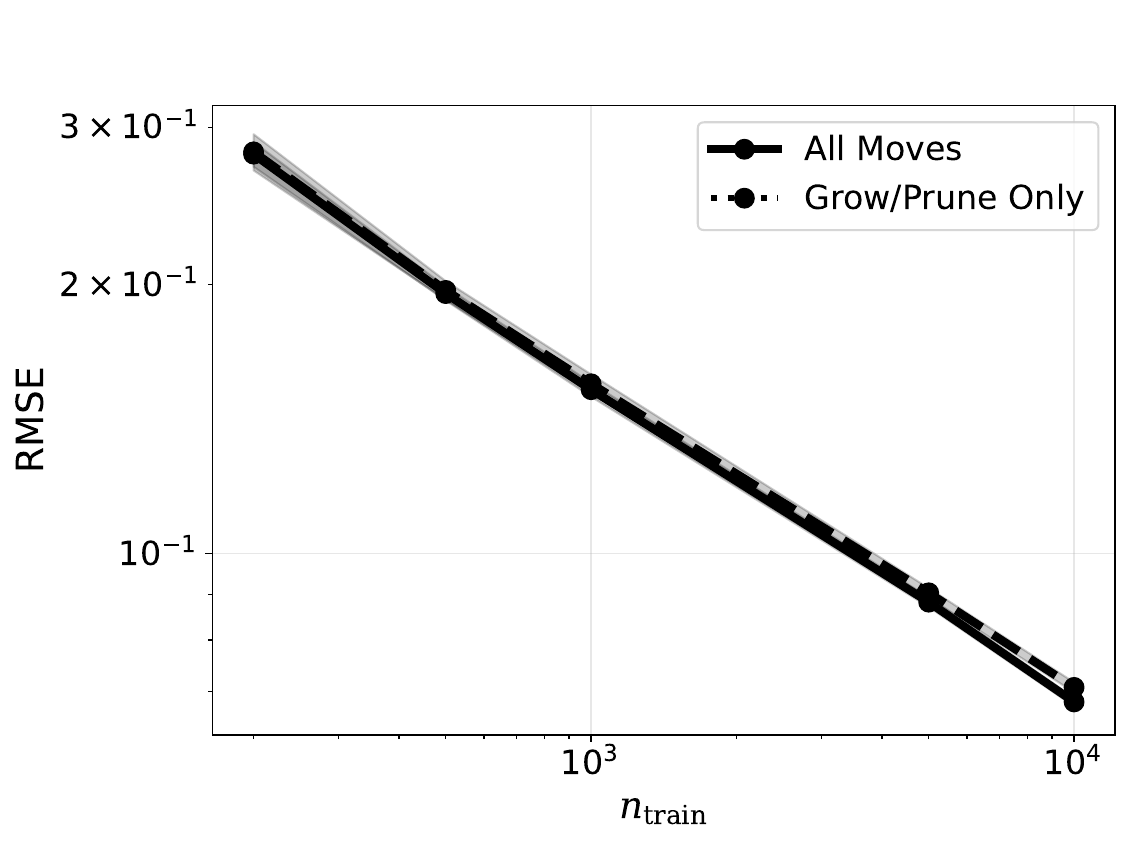}
    \end{minipage} \\
    \begin{minipage}{0.9\textwidth}
        \centering
        \includegraphics[width=0.32\linewidth]{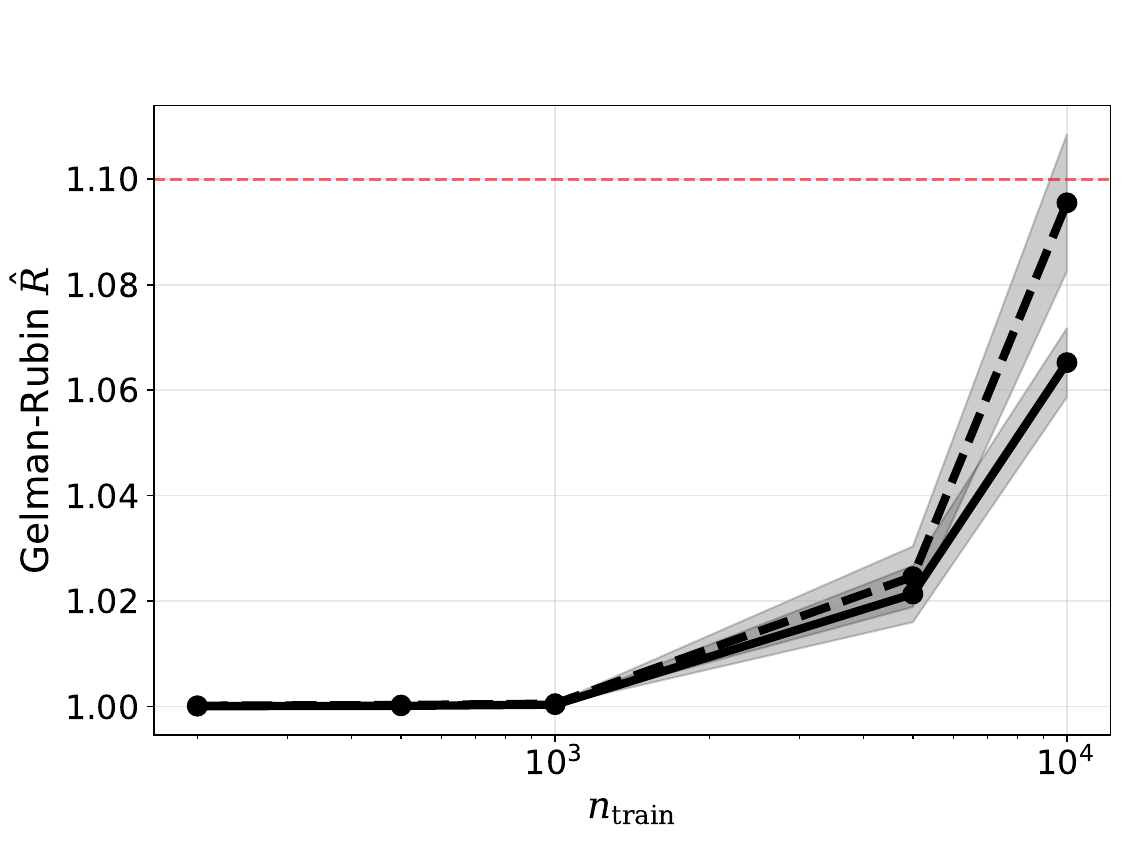}%
        \hfill
        \includegraphics[width=0.32\linewidth]{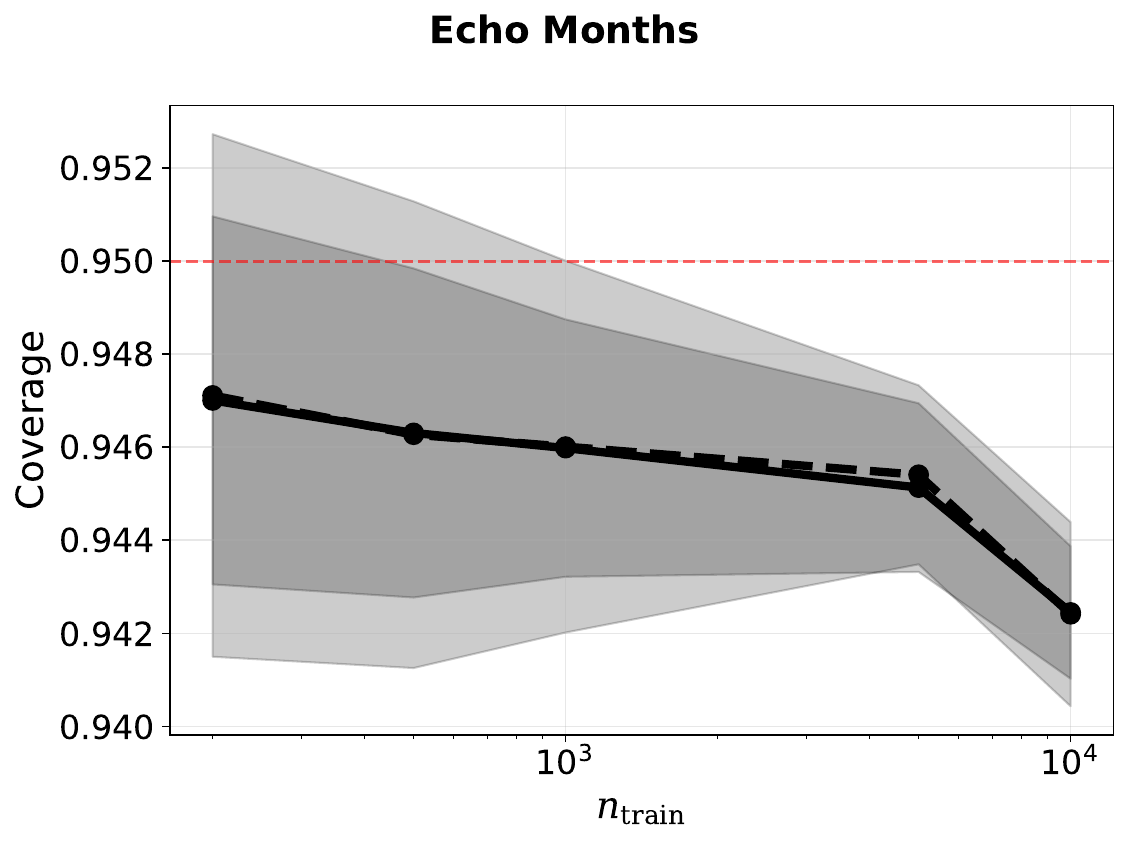}%
        \hfill
        \includegraphics[width=0.32\linewidth]{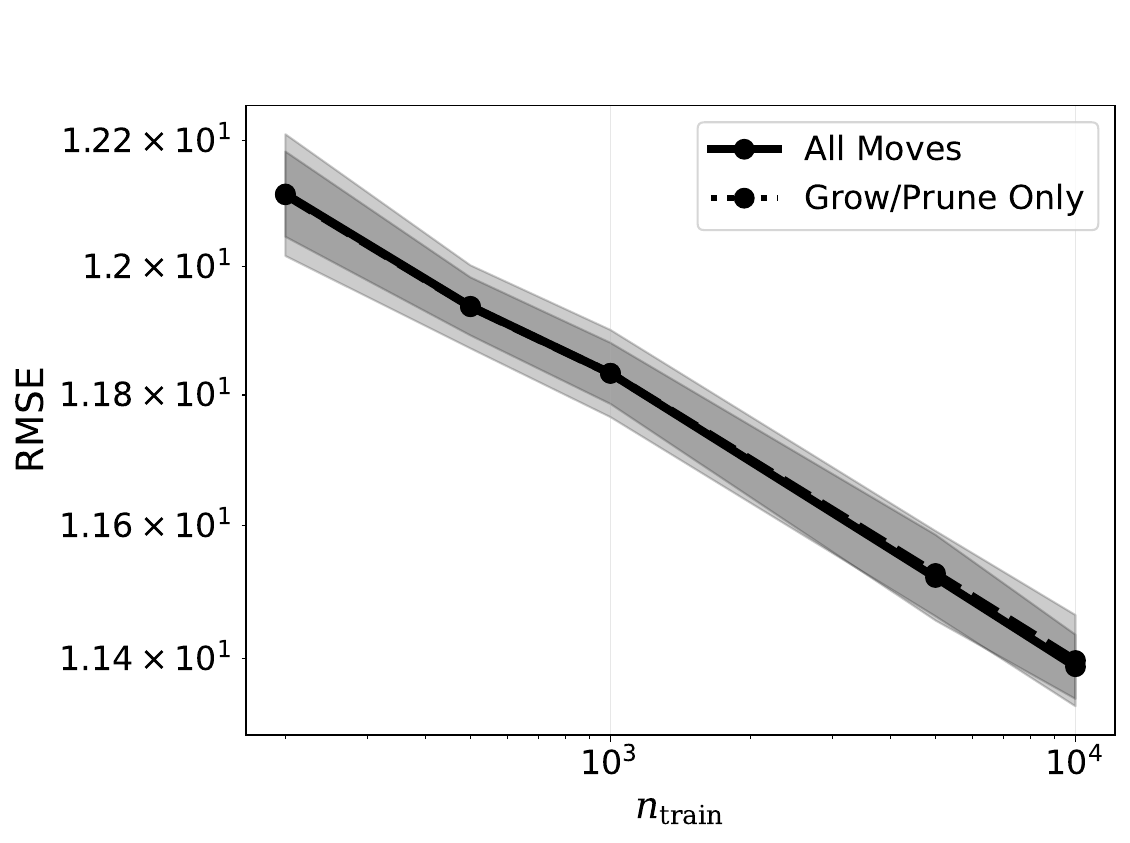}
    \end{minipage} \\
        \begin{minipage}{0.9\textwidth}
        \centering
        \includegraphics[width=0.32\linewidth]{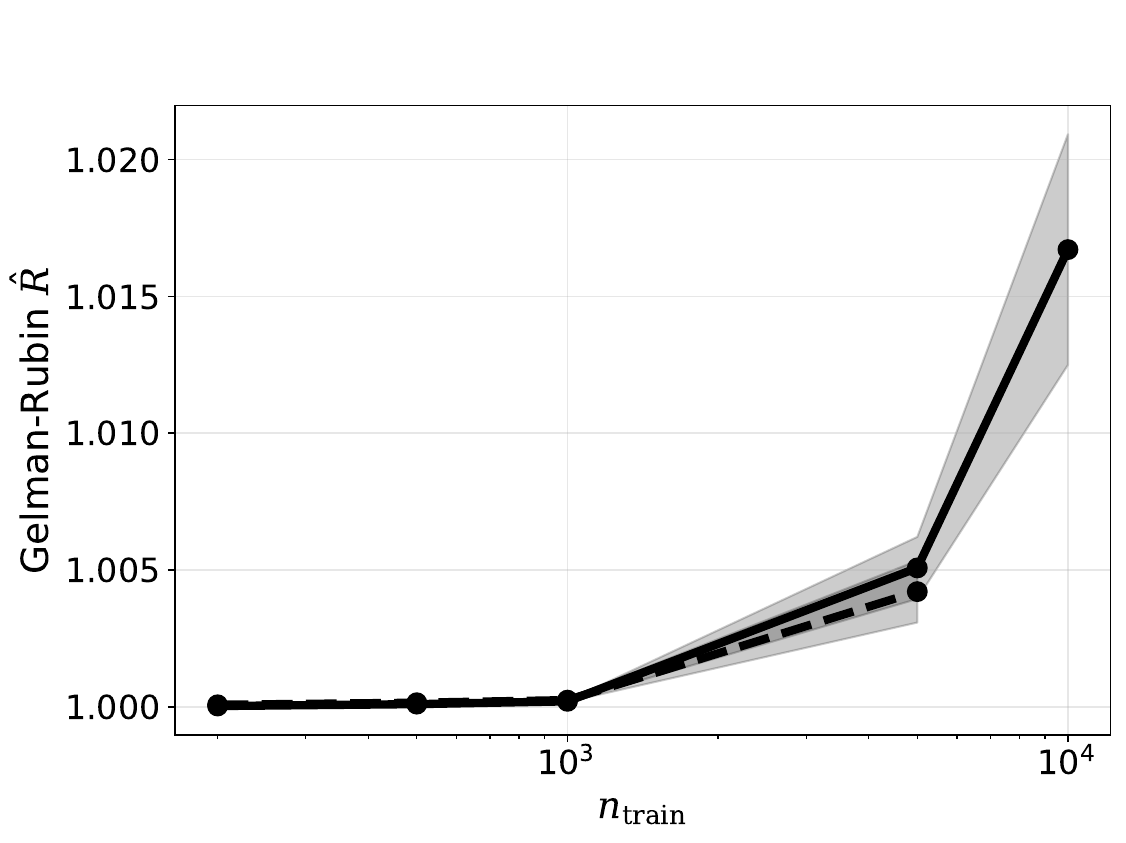}%
        \hfill
        \includegraphics[width=0.32\linewidth]{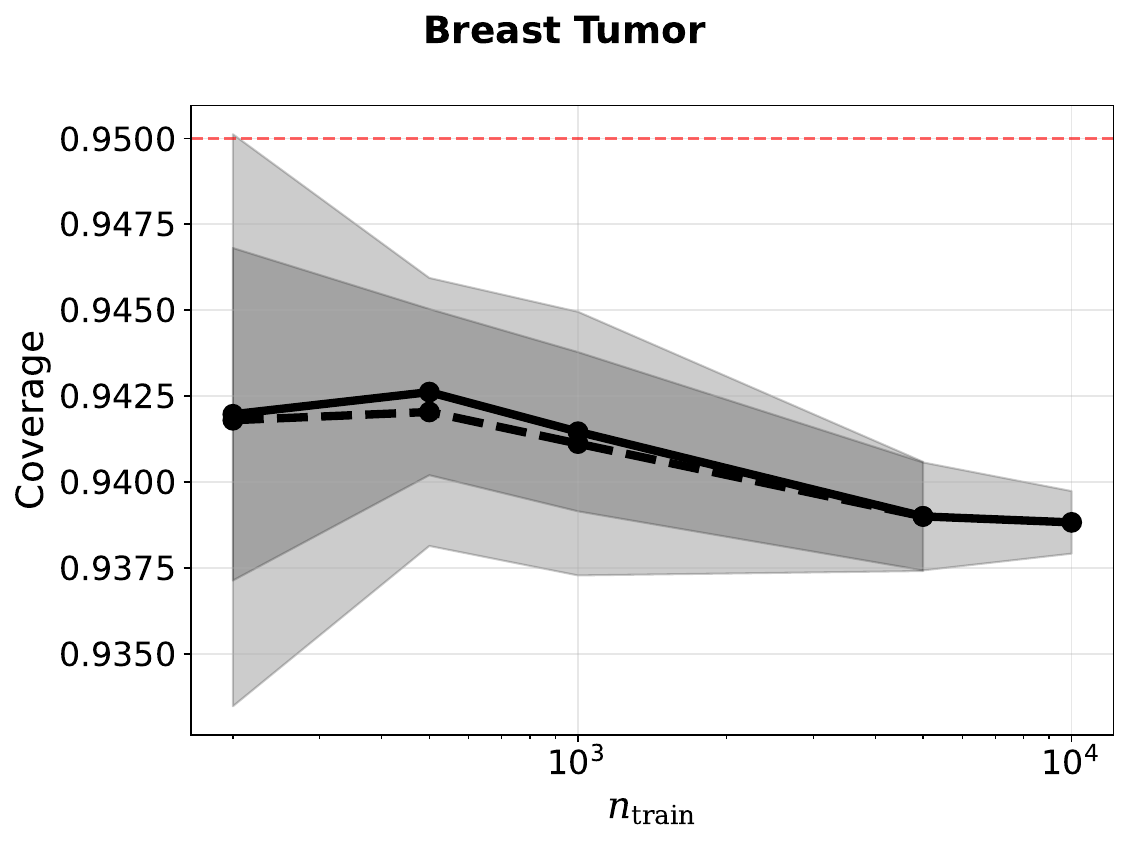}%
        \hfill
        \includegraphics[width=0.32\linewidth]{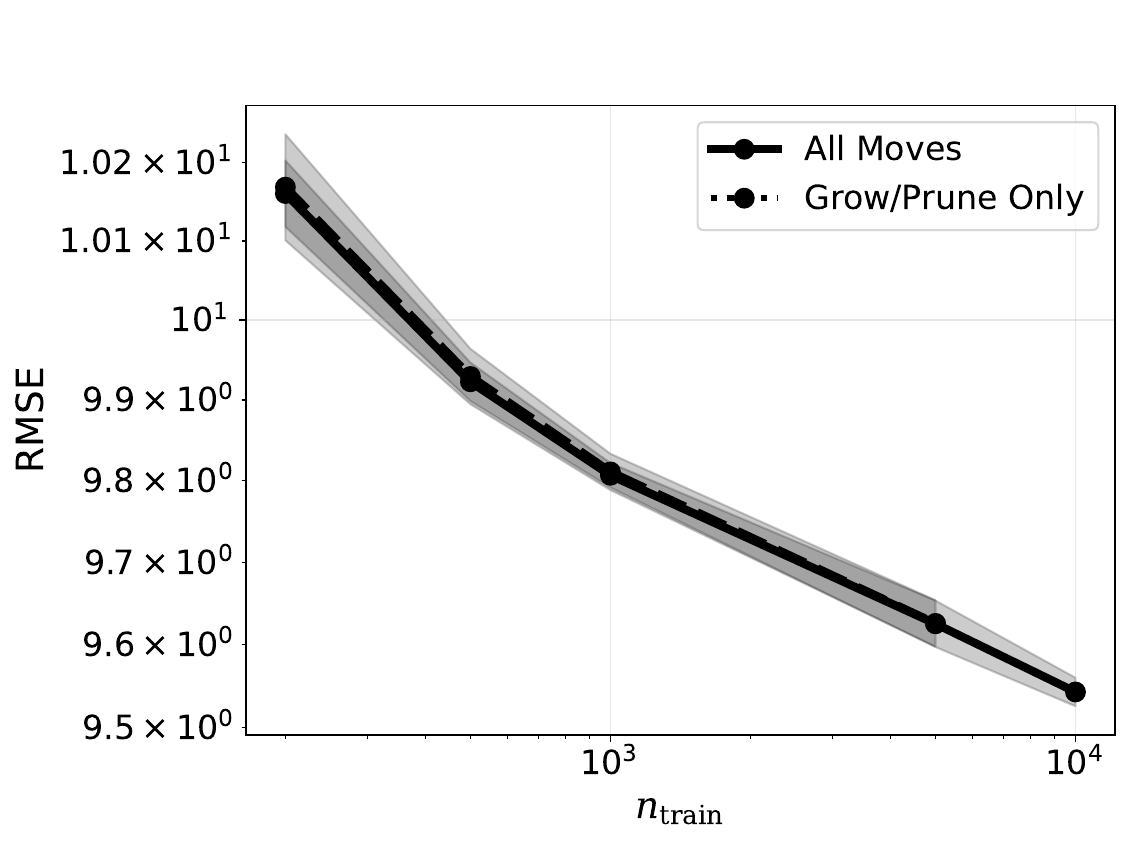}
    \end{minipage} \\
    \begin{minipage}{0.9\textwidth}
        \centering
        \includegraphics[width=0.32\linewidth]{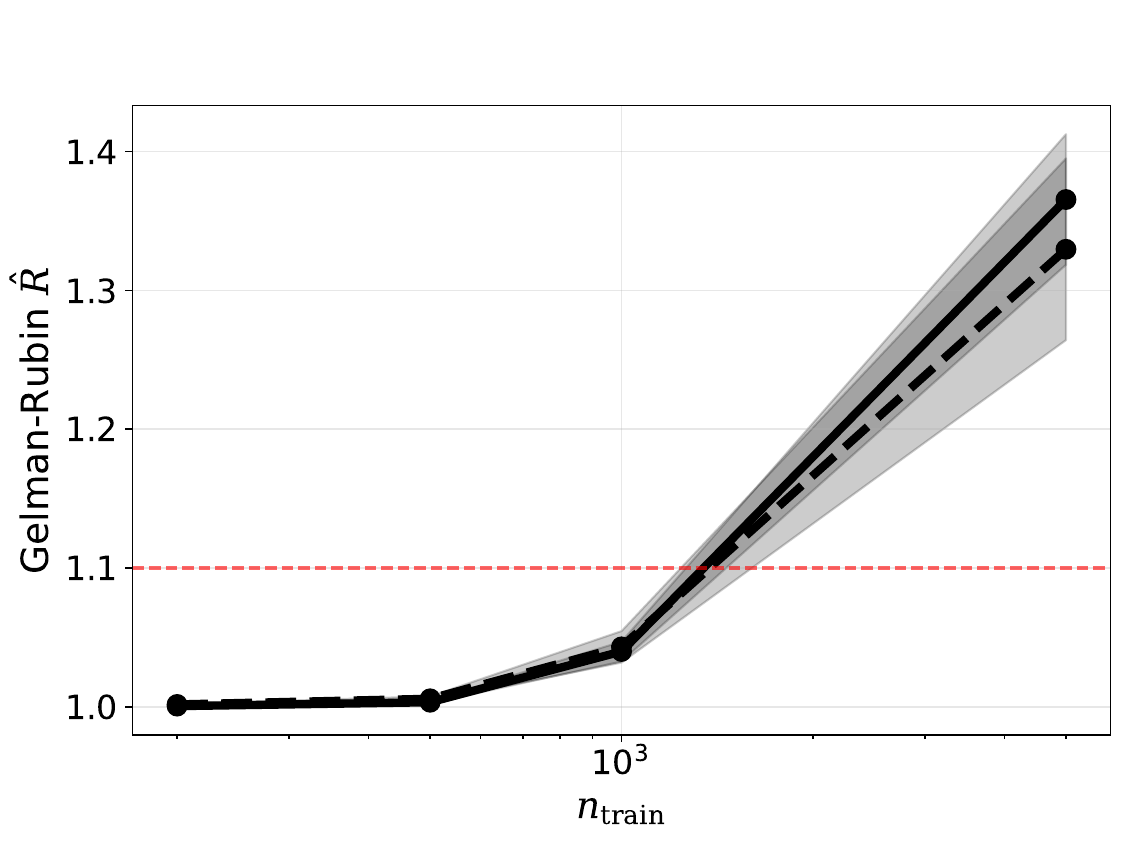}%
        \hfill
        \includegraphics[width=0.32\linewidth]{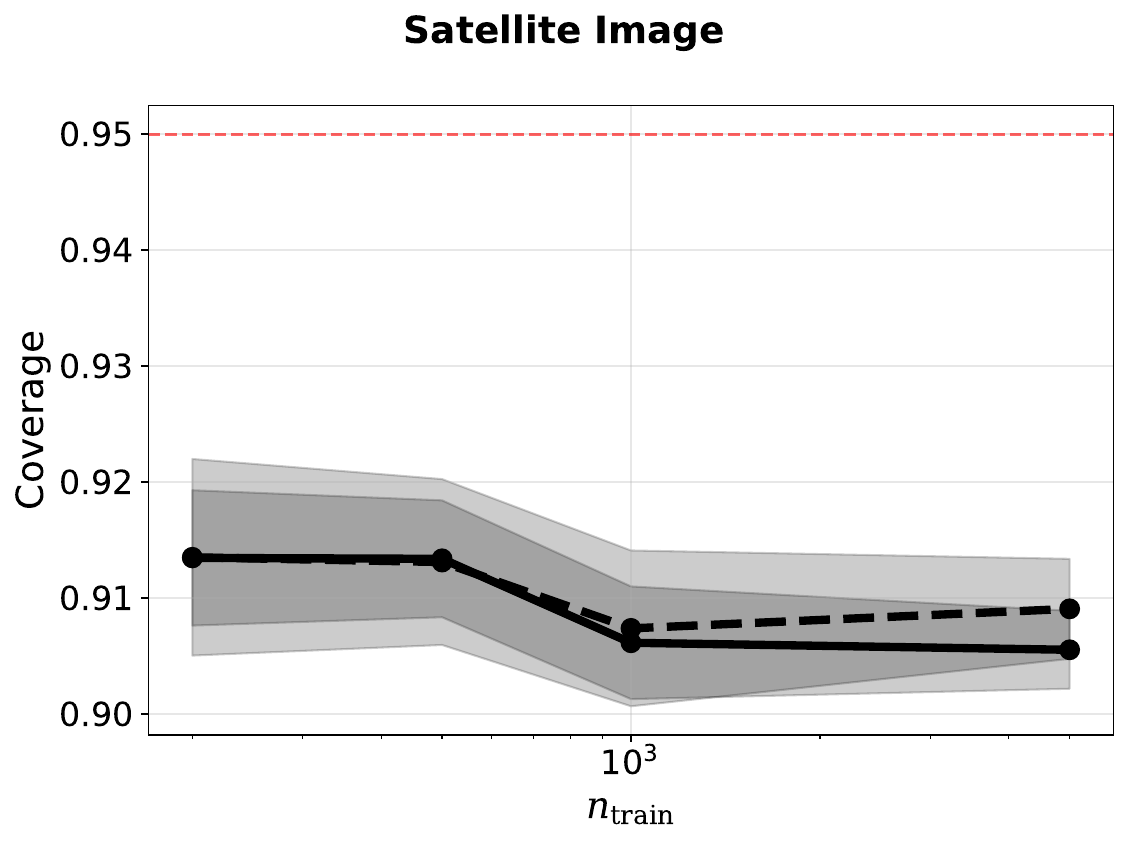}%
        \hfill
        \includegraphics[width=0.32\linewidth]{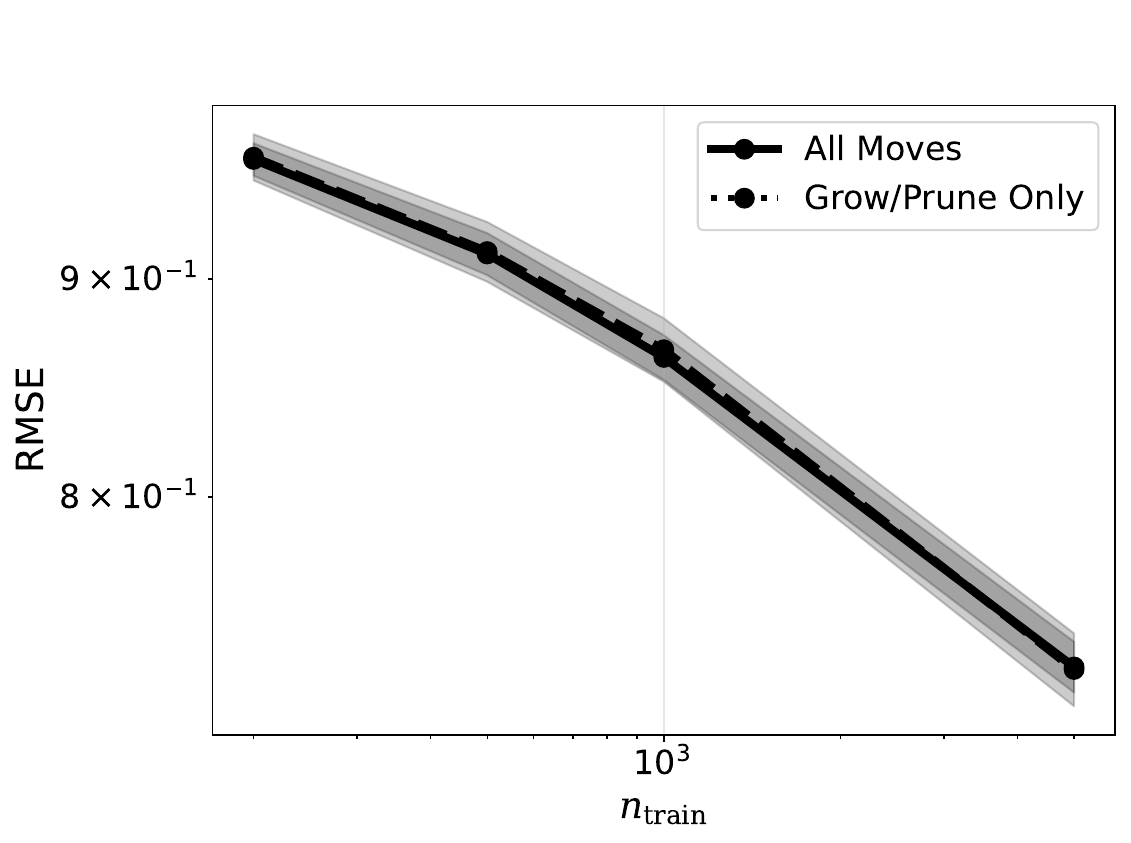}
    \end{minipage} \\
    \begin{minipage}{0.9\textwidth}
        \centering
        \includegraphics[width=0.32\linewidth]{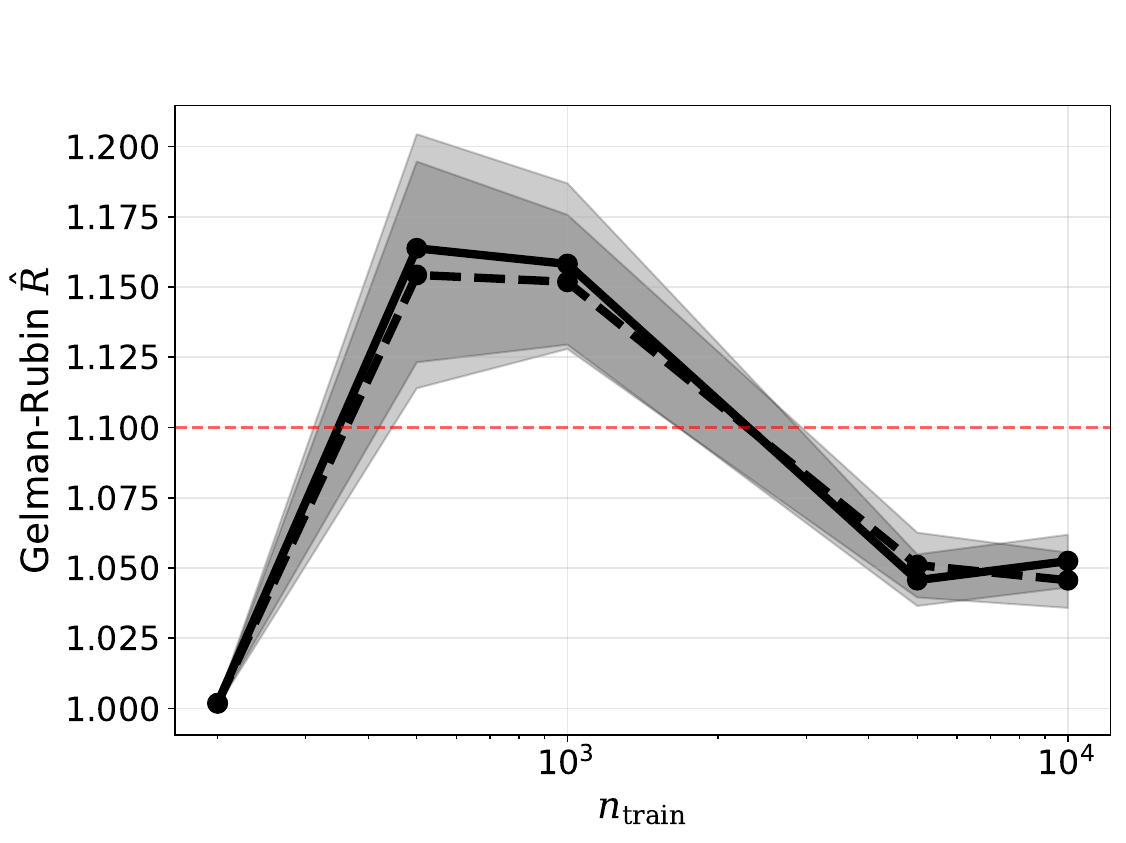}%
        \hfill
        \includegraphics[width=0.32\linewidth]{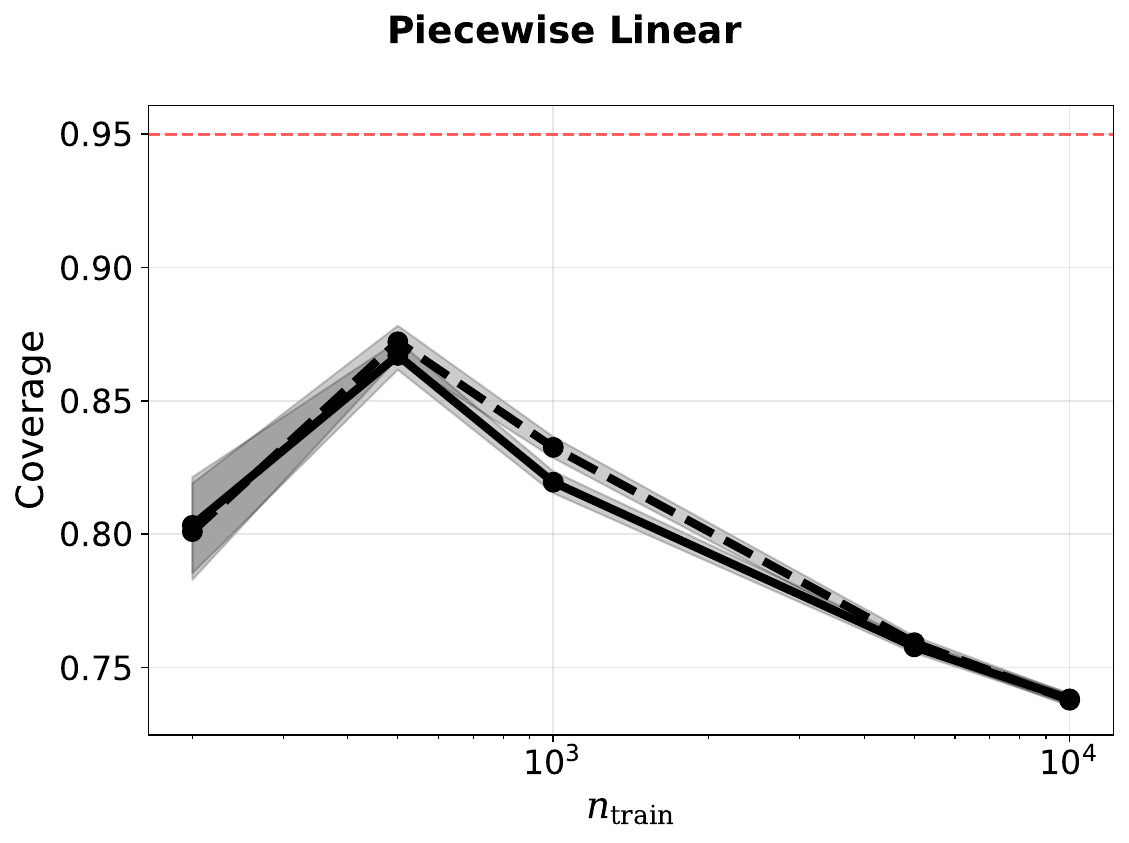}%
        \hfill
        \includegraphics[width=0.32\linewidth]{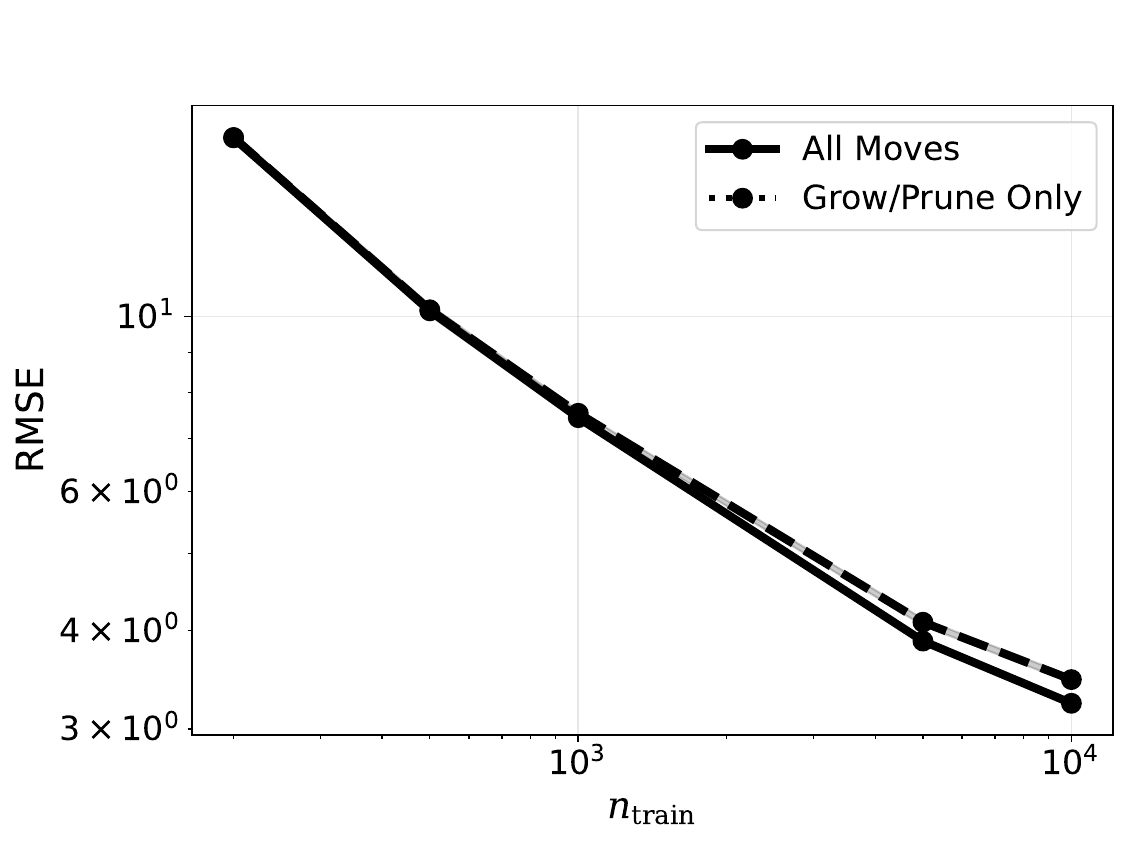}
    \end{minipage} \\
    \end{tabular}
    \caption{
    Values for Gelman-Rubin $\hat R$ (left), coverage (center), and RMSE (right) for the BART sampler under both the full move set, and the move set restricted to just ``grow'' and ``prune''.
    Results are plotted for California Housing, Low Dimensional Smooth, Echo Months, Breast Tumor, Satellite Image, and Piecewise Linear datasets (from top to bottom).
    Error bars represent $\pm 1.96$ standard errors from 25 replicates.
    }
    \label{fig:experiment_moveset}
\end{figure}

\newpage

\begin{figure}[H]
    \centering
    \begin{tabular}{c}
    \begin{minipage}{0.9\textwidth}
        \centering
        \includegraphics[width=0.32\linewidth]{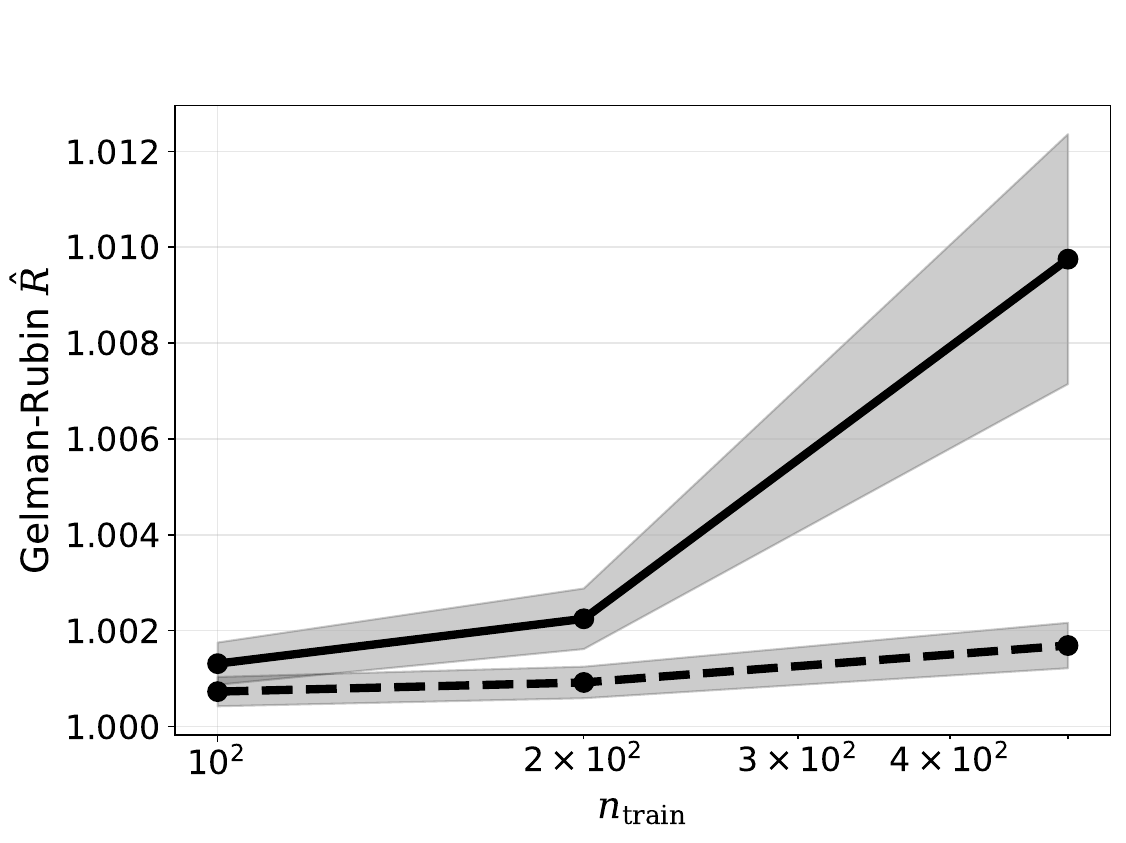}%
        \hfill
        \includegraphics[width=0.32\linewidth]{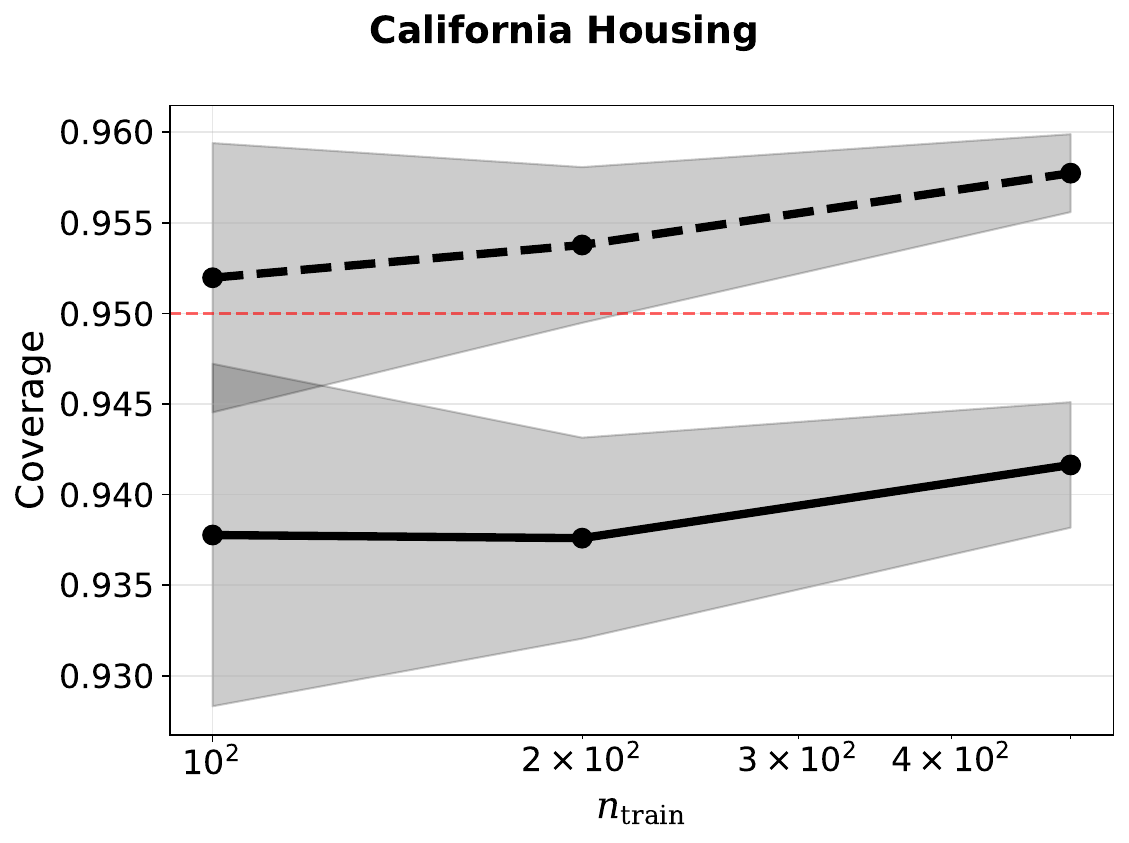}%
        \hfill
        \includegraphics[width=0.32\linewidth]{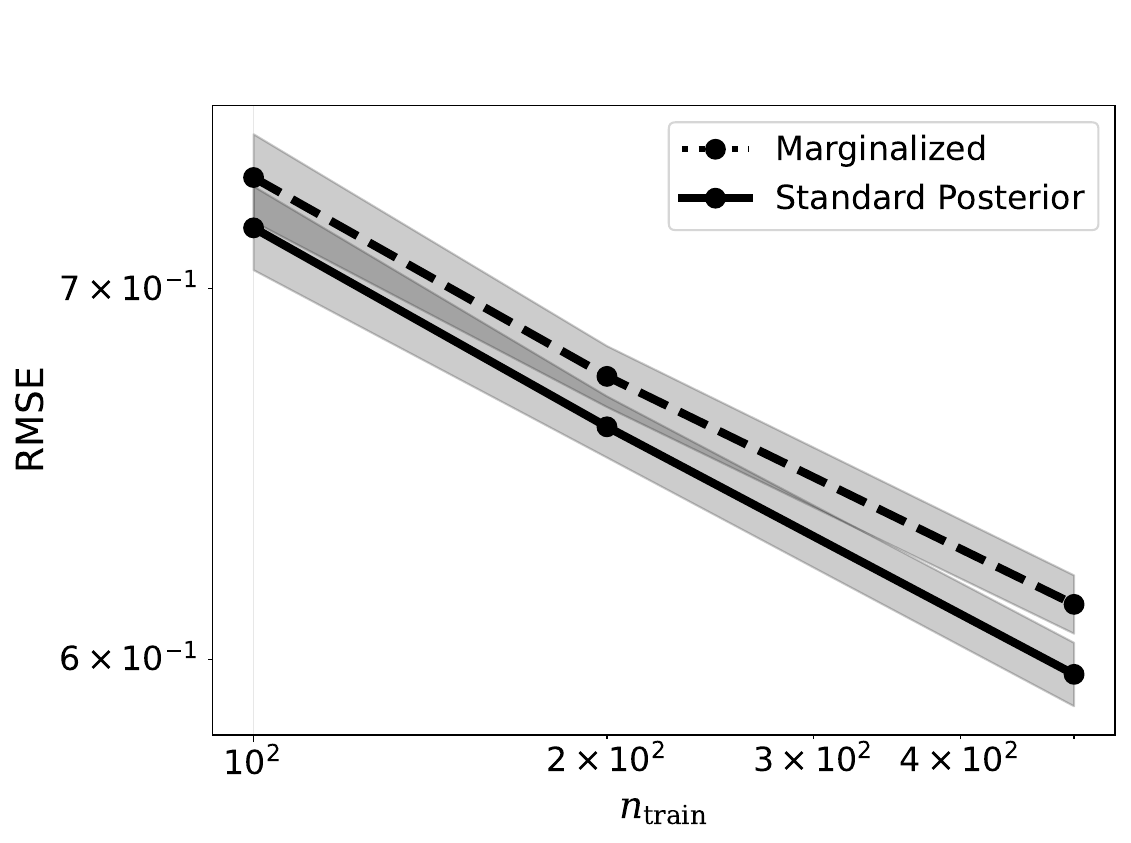}
    \end{minipage} \\
    \begin{minipage}{0.9\textwidth}
        \centering
        \includegraphics[width=0.32\linewidth]{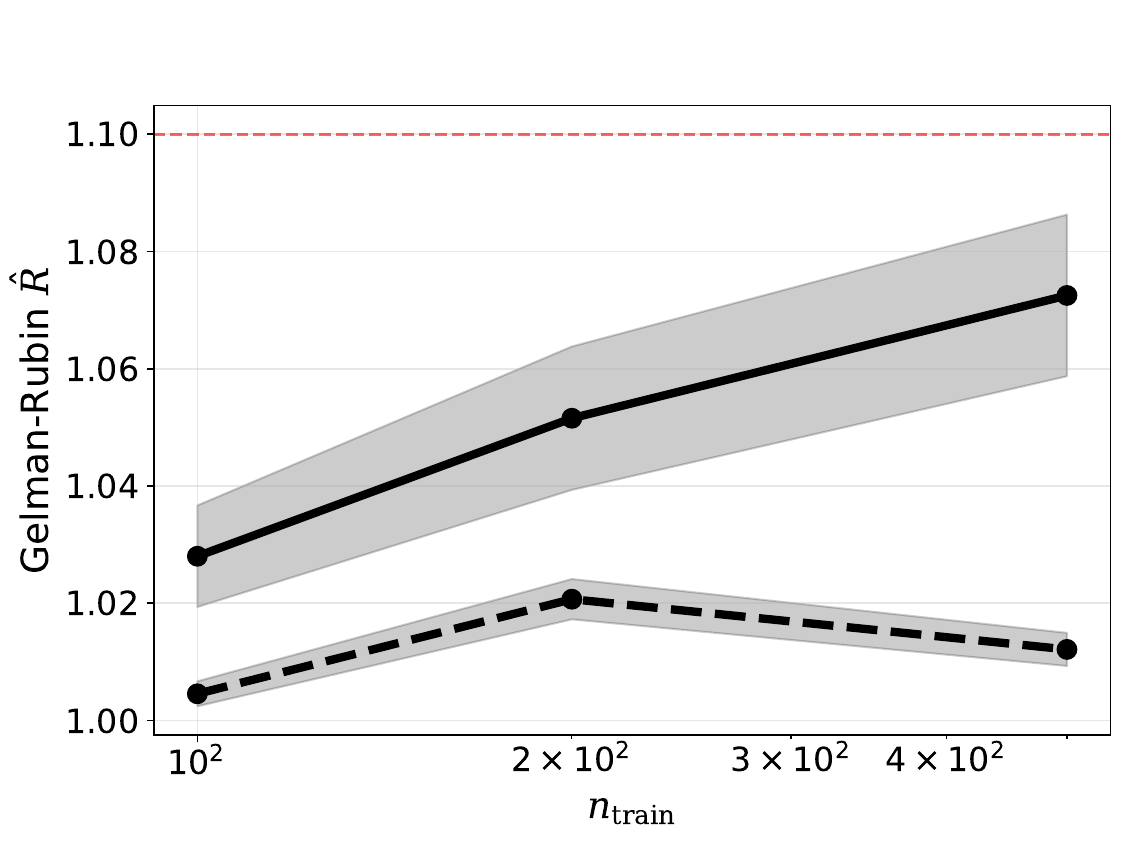}%
        \hfill
        \includegraphics[width=0.32\linewidth]{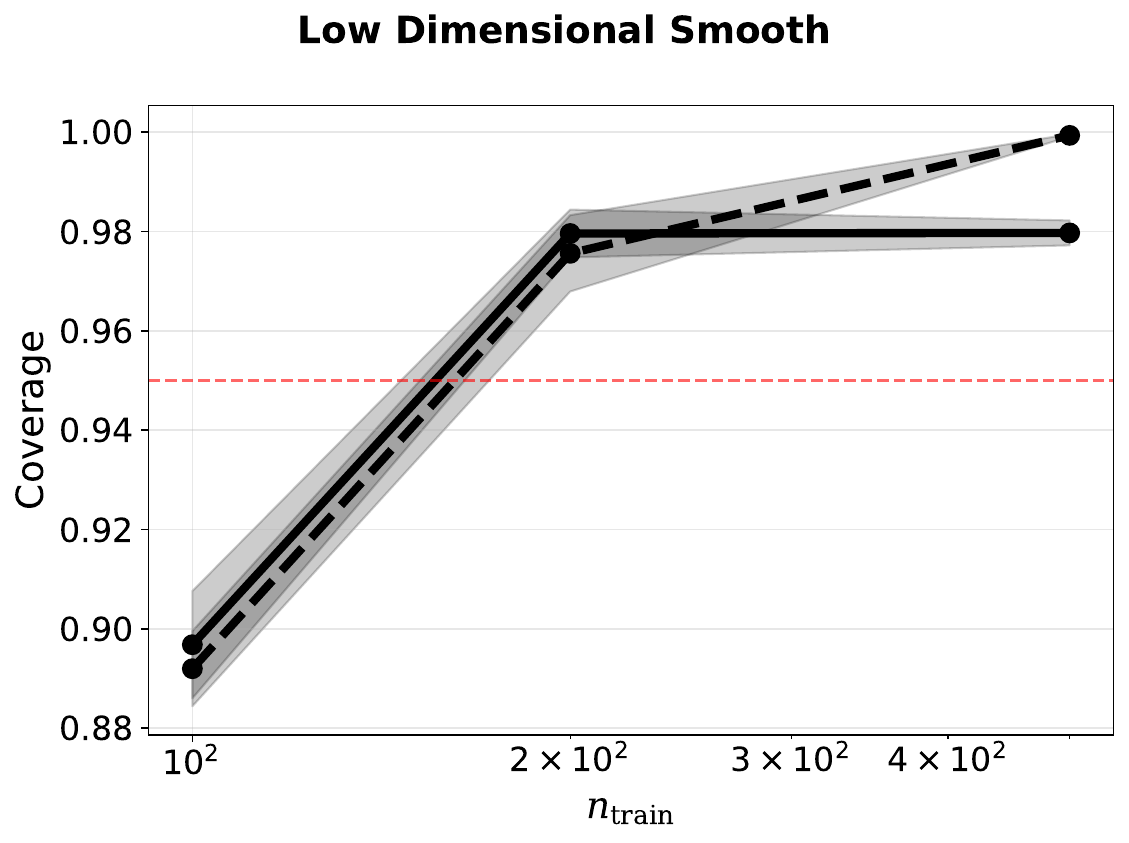}%
        \hfill
        \includegraphics[width=0.32\linewidth]{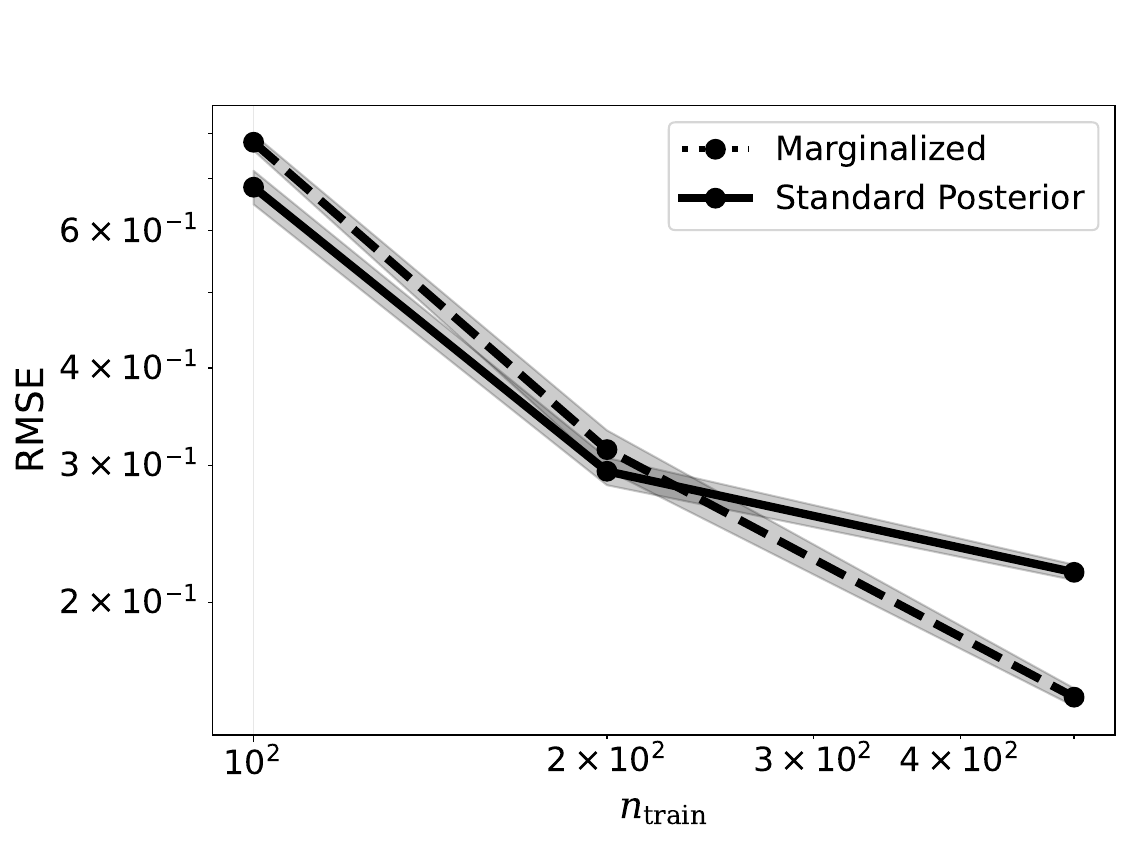}
    \end{minipage} \\
    \begin{minipage}{0.9\textwidth}
        \centering
        \includegraphics[width=0.32\linewidth]{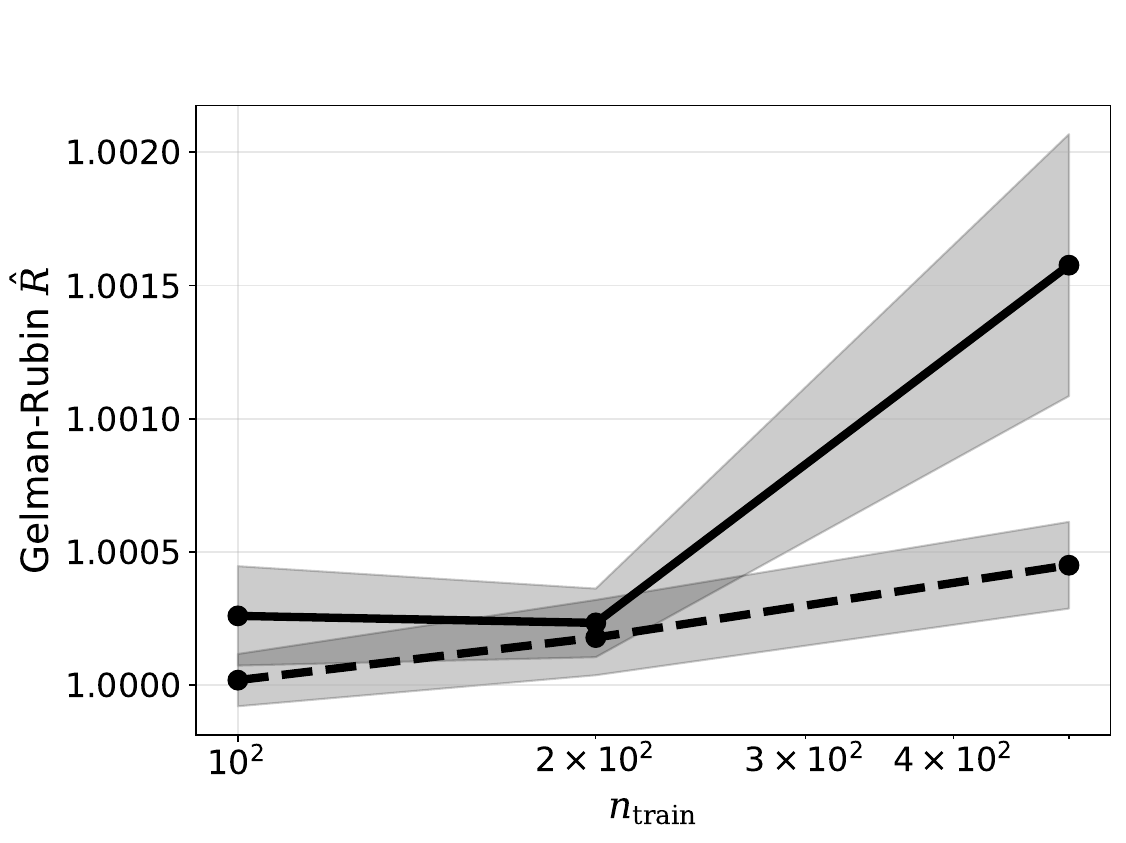}%
        \hfill
        \includegraphics[width=0.32\linewidth]{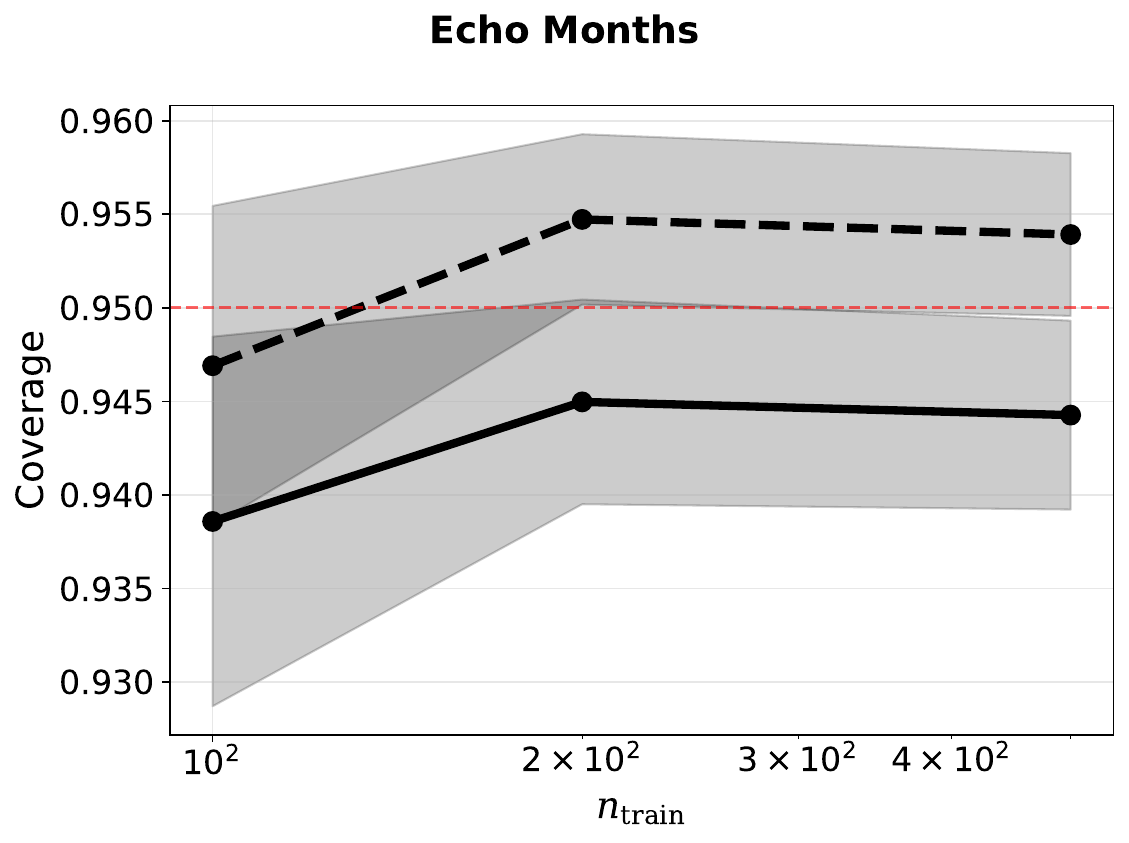}%
        \hfill
        \includegraphics[width=0.32\linewidth]{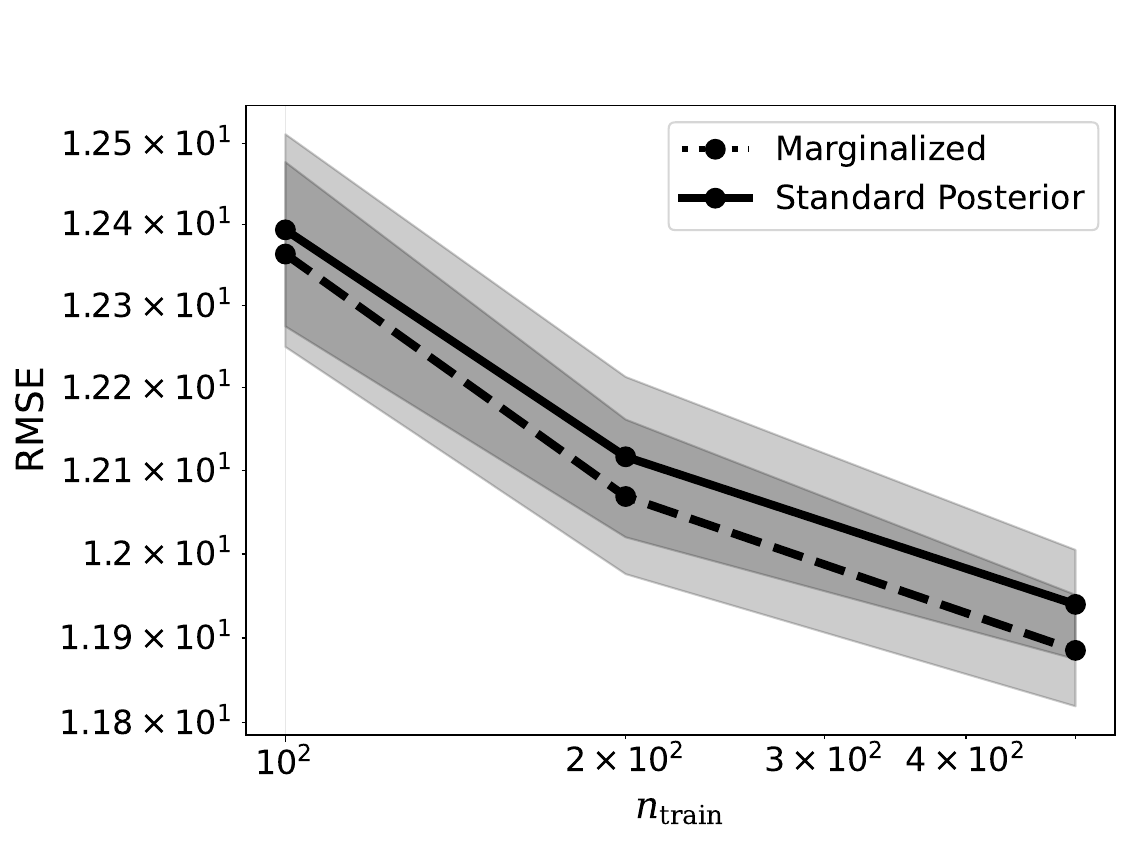}
    \end{minipage} \\
        \begin{minipage}{0.9\textwidth}
        \centering
        \includegraphics[width=0.32\linewidth]{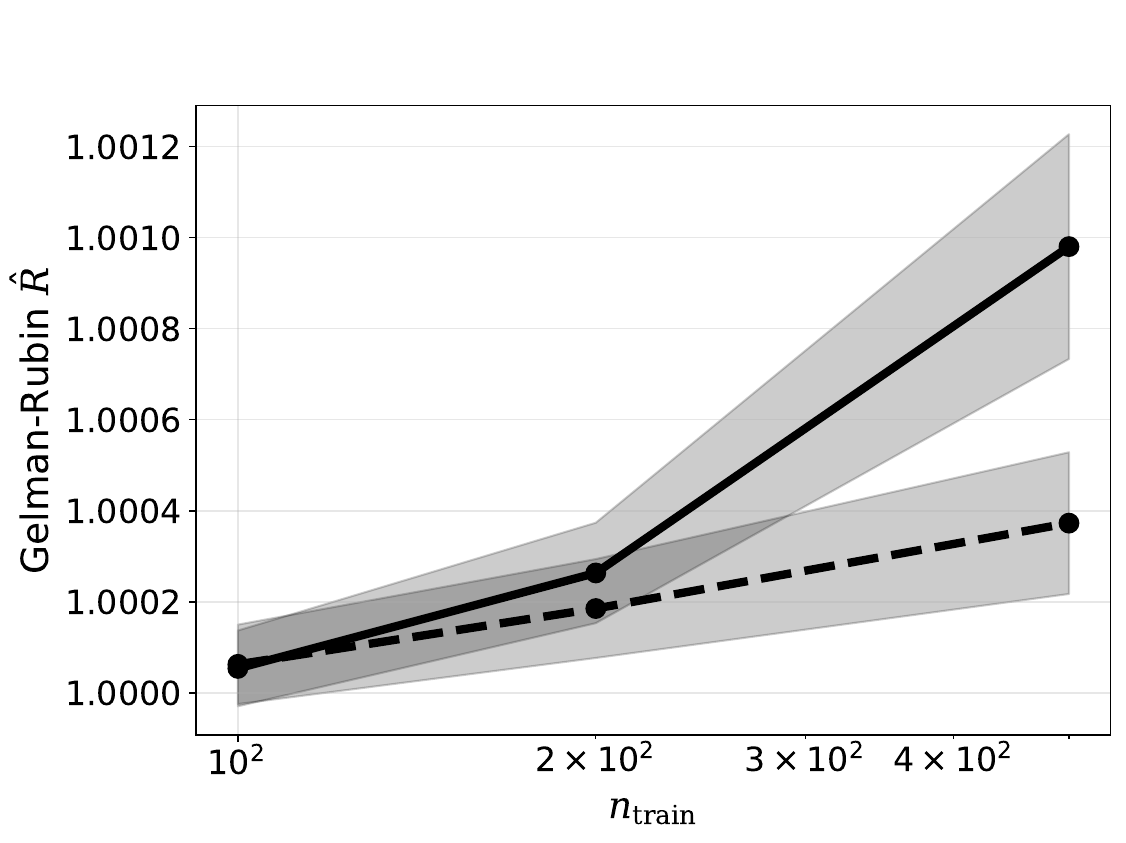}%
        \hfill
        \includegraphics[width=0.32\linewidth]{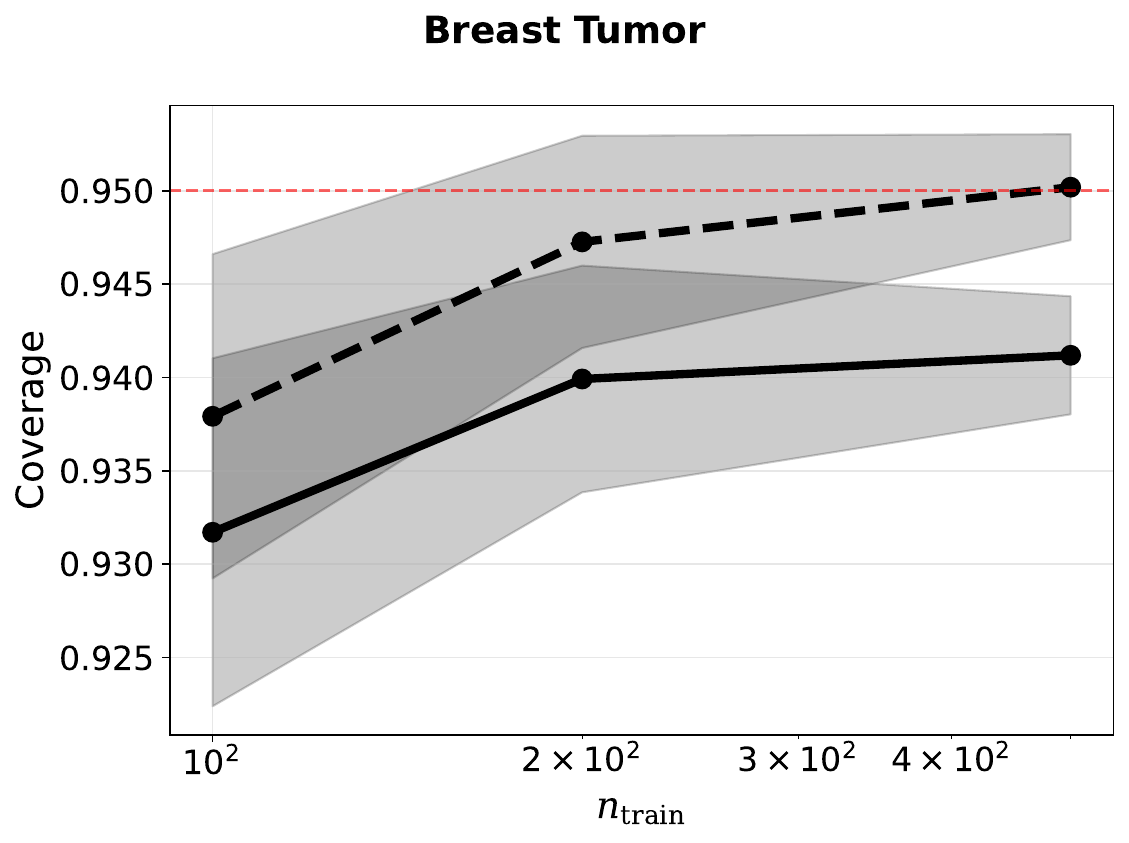}%
        \hfill
        \includegraphics[width=0.32\linewidth]{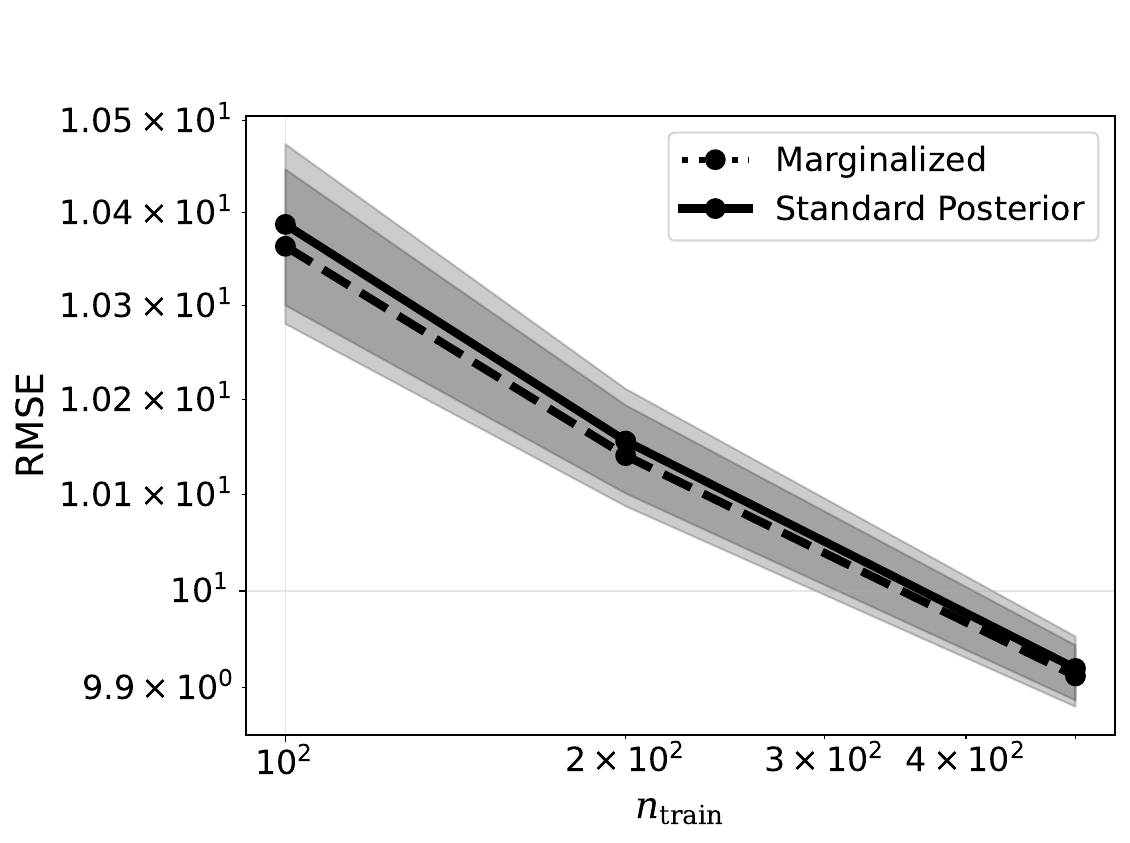}
    \end{minipage} \\
    \begin{minipage}{0.9\textwidth}
        \centering
        \includegraphics[width=0.32\linewidth]{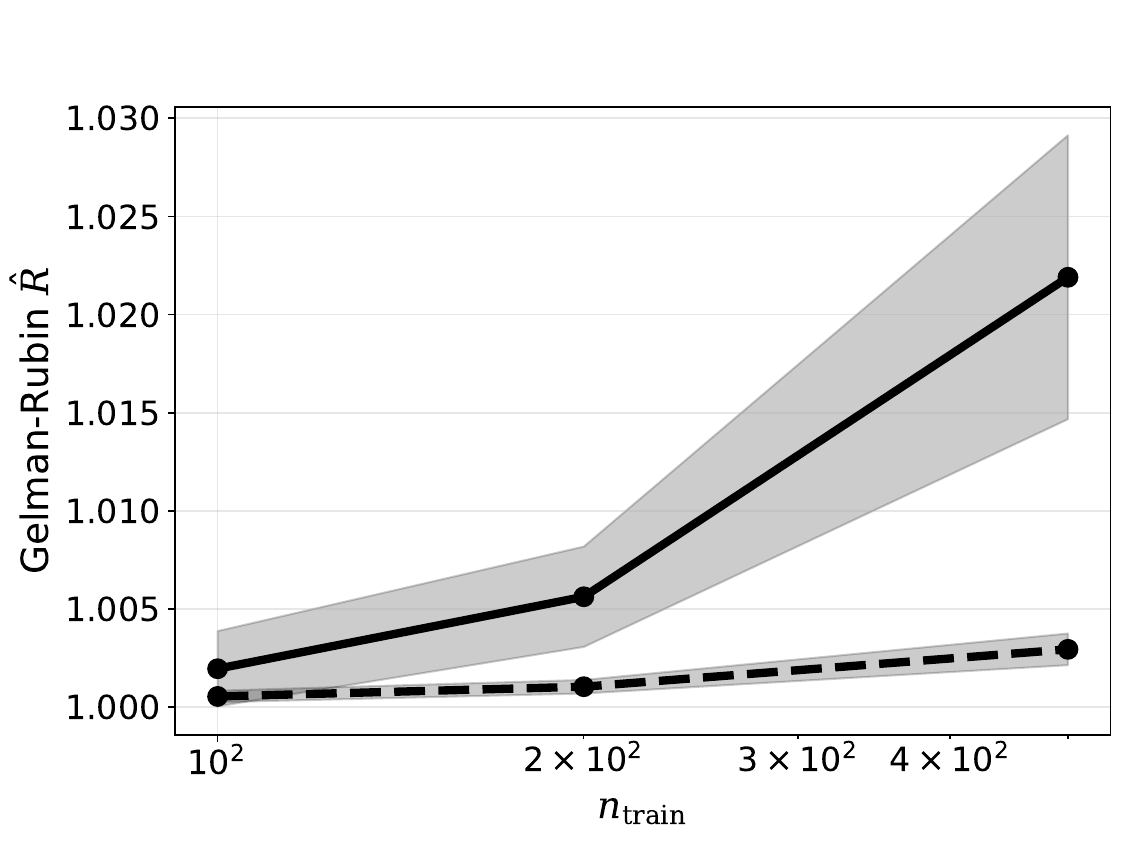}%
        \hfill
        \includegraphics[width=0.32\linewidth]{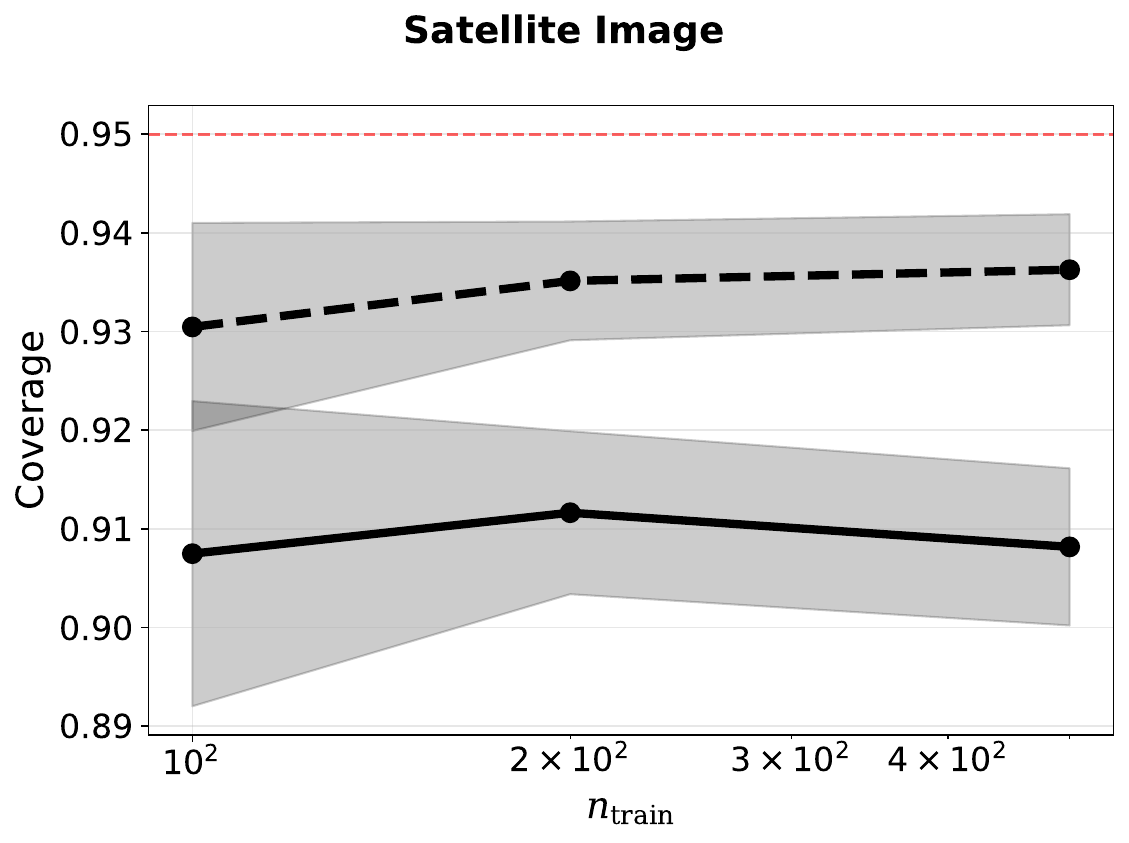}%
        \hfill
        \includegraphics[width=0.32\linewidth]{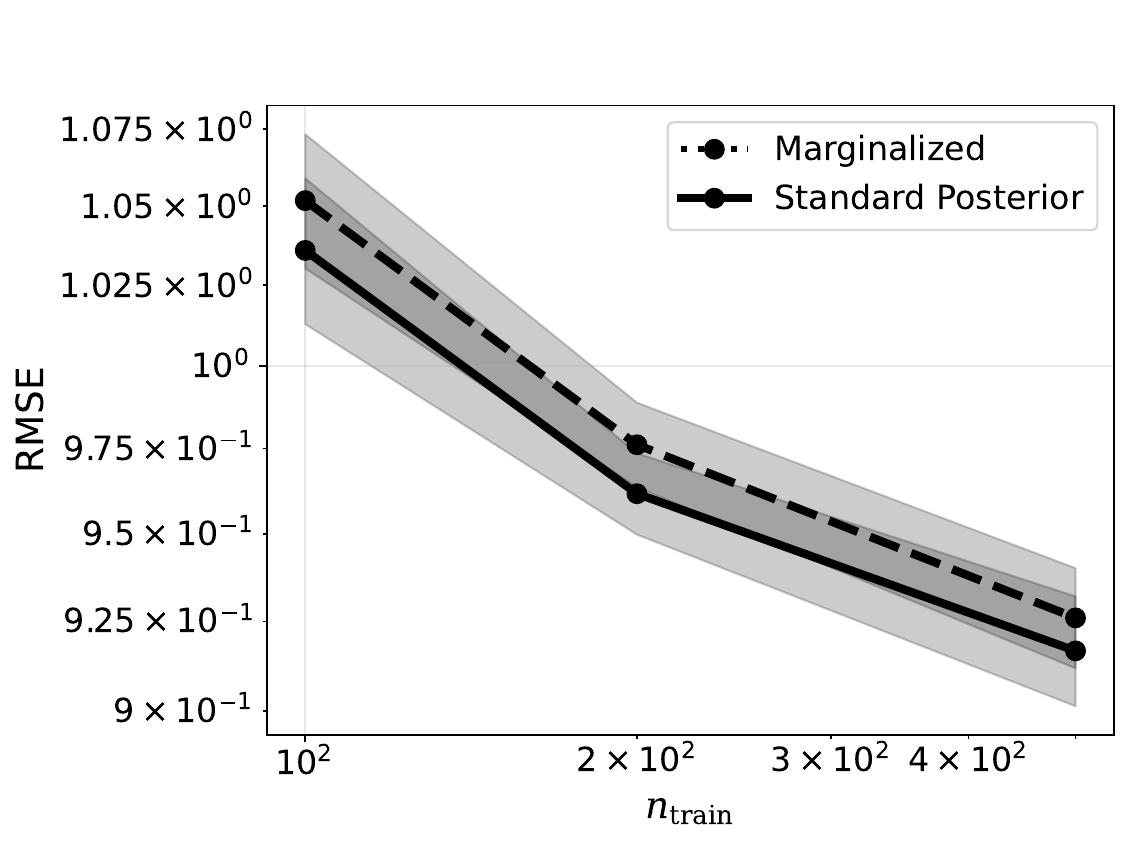}
    \end{minipage} \\
    \begin{minipage}{0.9\textwidth}
        \centering
        \includegraphics[width=0.32\linewidth]{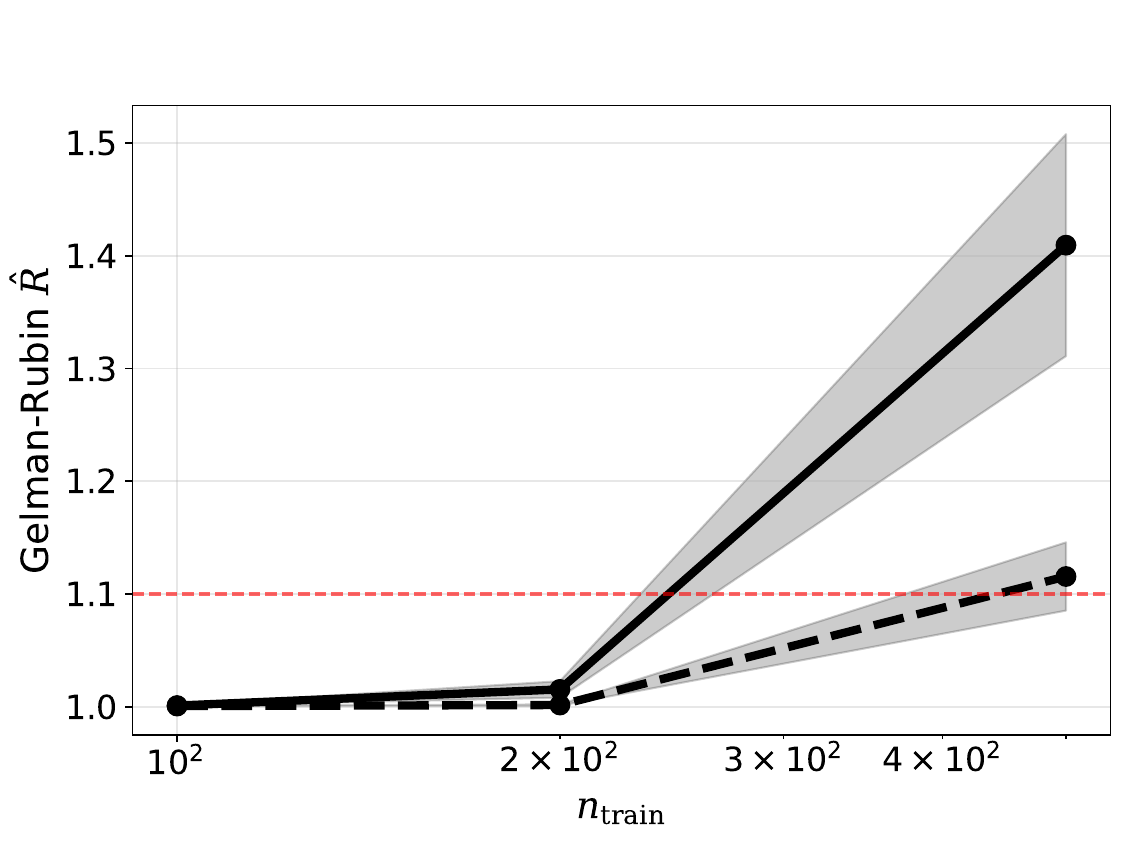}%
        \hfill
        \includegraphics[width=0.32\linewidth]{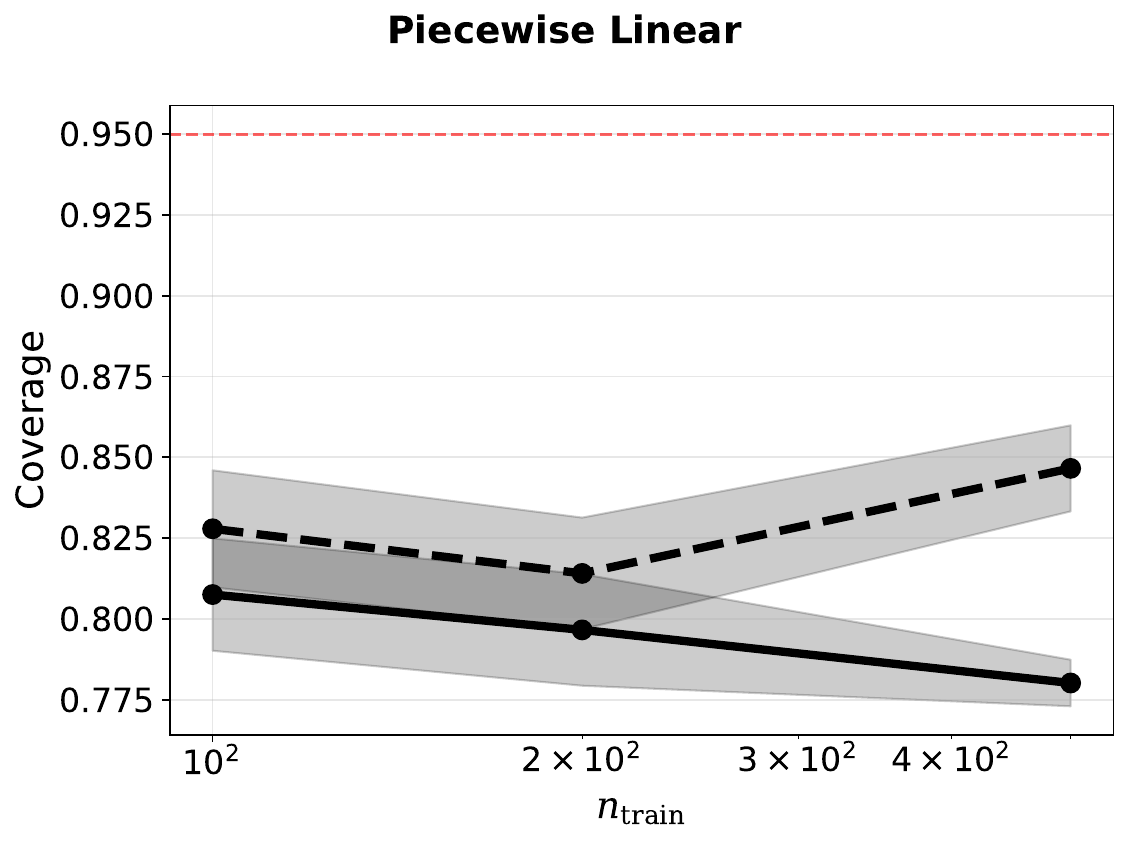}%
        \hfill
        \includegraphics[width=0.32\linewidth]{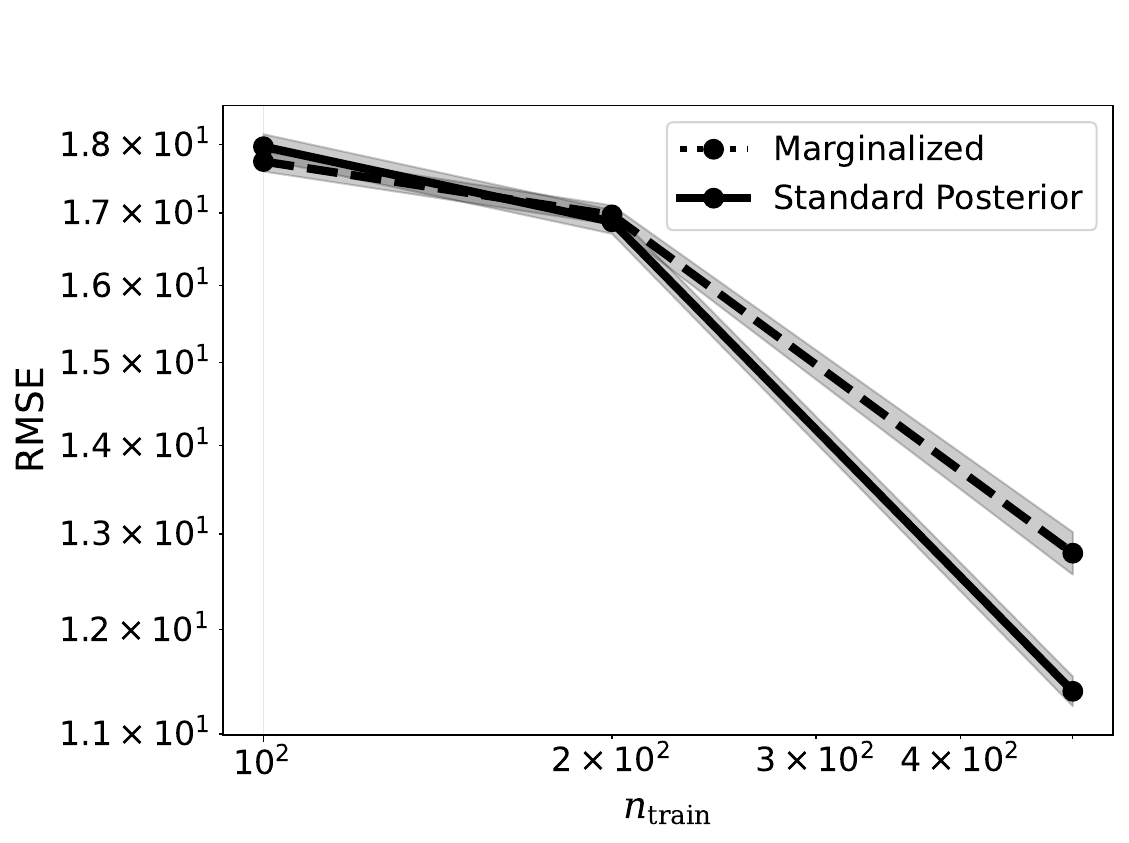}
    \end{minipage} \\
    \end{tabular}
    \caption{
    Values for Gelman-Rubin $\hat R$ (left), coverage (center), and RMSE (right) for the BART sampler under both the marginalized posterior calculation and the standard calculation for the Metropolis-Hastings acceptance probability.
    Results are plotted for California Housing, Low Dimensional Smooth, Echo Months, Breast Tumor, Satellite Image, and Piecewise Linear datasets (from top to bottom).
    Error bars represent $\pm 1.96$ standard errors from 25 replicates.
    }
    \label{fig:experiment_marginal}
\end{figure}

\newpage